%% file: ppgcc-thesis.tex
\documentclass[english]{pucrs-ppgcc}

\usepackage{graphicx}
\usepackage{hyperref}
\usepackage{epsfig}
\usepackage{boxedminipage}
\usepackage{fancyvrb}
\usepackage{epstopdf}
\usepackage{amssymb}
\usepackage{amsthm}
\usepackage{amsfonts}
\usepackage{amstext}
\usepackage{amsmath}
\usepackage{pdflscape}
\usepackage{afterpage}
\usepackage{capt-of}
\usepackage{booktabs}
\usepackage{subcaption}

\usepackage{mathtools}

\usepackage[mathscr]{euscript}
 \let\mathscr\relax

\newcommand{\matr}[1]{\ensuremath{\mathbf{#1}}}
\newcommand{\thereals}{\ensuremath{\mathbb{R}}}

\usepackage[
	]{todonotes}

\usepackage{multirow}
\usepackage{nicefrac}
\usepackage{algorithm}
\usepackage{algorithmicx}
\usepackage{algpseudocode}

\newtheorem{proposition}{Proposition}
\newtheorem{definition}{Definition}
\newtheorem{corollary}{Corollary}
\newtheorem{theorem}{Theorem}
\newtheorem{example}{Example}[chapter]

\algtext*{EndWhile}
\algtext*{EndIf}
\algtext*{EndFor}
\algtext*{EndProcedure}
\algtext*{EndFunction}

\usepackage{listings}
\newcommand\pred[1]{\texttt{\selectfont{#1}}}

\def\gets{:=}

\usepackage[normalem]{ulem}
\usepackage{todonotes}

\newcommand{\idest}{{i.e.}}
\newcommand{\exemp}{{e.g.}}

\newcommand{\etal}{{et al.}}

\newcommand{\incp}[1]{\widetilde{\mathcal{#1}}}
\newcommand{\inc}[1]{\widetilde{#1}}

\newcommand{\nominalmirroring}{$\eta$\textsc{Mirroring}}

\renewcommand*{\thepage}{\small\arabic{page}}
\usepackage{footmisc}

\author{Ramon Fraga Pereira}

\title{Reconhecimento de Objetivos em Dom{\'{\i}}nios Imperfeitos}
      {Goal Recognition over Imperfect Domain Models}

\tipotrabalho{\tese}         % Tese

\grau{\doutor}

\orientador{Prof. Dr. Felipe Meneguzzi (PUCRS)}
\coorientador{Dr. Miquel Ramírez (University of Melbourne)}

\begin{document}
	
\fontfamily{cmr}\selectfont	

%----------------------------------------------------------------
% Dedicatory
%----------------------------------------------------------------
% \dedicatoria{I would like to dedicate this work to my mother, Maria Elena Fraga, who has always supported me along my life. Thank you mom, love you.}
%

%----------------------------------------------------------------
% Epigrafe
%----------------------------------------------------------------
%\epigrafe{}

%----------------------------------------------------------------
% Acknowledgments
%----------------------------------------------------------------
% \begin{agradecimentos}
% Foremost, I would like to thank my advisor Prof. Dr. Felipe Meneguzzi, for his patience and motivation along the years we have been working together. His guidance and knowledge were essential to me to achieve everything I have achieved during my M.Sc and Ph.D.
%
% % I thankful to so many people, who have helped and guided me to develop this work.
% \end{agradecimentos}

%----------------------------------------------------------------
% Abstract
%----------------------------------------------------------------

\begin{abstract}{Goal Recognition, Plan Recognition, Automated Planning, Incomplete Discrete Domain Models, Approximate Continuous Domain Models, Landmarks}

Goal recognition is the problem of recognizing the intended goal of autonomous agents or humans by observing their behavior in an environment. 
Over the past years, most existing approaches to goal and plan recognition have been ignoring the need to deal with imperfections regarding the domain model that formalizes the environment where autonomous agents behave. In this thesis, we introduce the problem of goal recognition over \textit{imperfect domain models}, and develop solution approaches that explicitly deal with two distinct types of \textit{imperfect domains models}: (1) \textit{incomplete discrete domain models} that have \textit{possible}, rather than \textit{known}, preconditions and effects in action descriptions; and (2) \textit{approximate continuous domain models}, where the transition function is \textit{approximated} from past observations and not well-defined. 
We develop novel goal recognition approaches over \textit{imperfect domains models} by leveraging and adapting existing recognition approaches from the literature. Experiments and evaluation over these two types of \textit{imperfect domains models} show that our novel goal recognition approaches are accurate in comparison to baseline approaches from the literature, at several levels of observability and imperfections.

\end{abstract}

\tableofcontents     % Sumário               (OBRIGATÓRIO)

%----------------------------------------------------------------
% Chapters
%----------------------------------------------------------------

\include{cap1_Introduction}
\include{cap2_Background}
\include{cap3_GR_IncompleteDomains}
\include{cap4_GR_NominalModels}
\include{cap5_Related_Work}
\include{cap6_Conclusions}

%----------------------------------------------------------------
% Bibliography
%----------------------------------------------------------------

% \bibliographystyle{ppgcc-apalike}
%\bibliographystyle{ppgcc-alpha}
\bibliographystyle{ppgcc-num}

\bibliography{crossref,ppgcc-thesis}

%----------------------------------------------------------------
% Appendix
%----------------------------------------------------------------

\appendix
\include{apA_GR_Heuristics}
\include{apB_IncompleteDomains_Results}
\include{apC_NominalModels_Results}

\end{document}

%% file: cap1_Introduction.tex
%!TEX root = ppgcc-thesis.tex
%----------------------------------------------------------------------------------
% Chapter: Introduction
%
% + Context
% + Problem
% + Solution
% + Other details
%
%----------------------------------------------------------------------------------
\chapter{Introduction}\label{chapter:Introduction}

% \frm[inline]{I know I don't often advocate this in papers and articles, but in your thesis, you need to start more softly. Here I accept some (non-bullshit) motivation for the laymen. That is describing why is goal recognition is useful, where would it be useful. Try to relate this to the common CS person, rather than classical AI person. }

\emph{Goal Recognition} is the problem of discerning the intentions of autonomous agents or humans, given a sequence of observations as evidence of their behavior in an environment, and a domain model describing how the observed agents generate such behavior to achieve their goals~\cite{SchmidtSG_78}. Recognizing goals is important in several applications, especially for monitoring and anticipating agent behavior in an environment, including 
crime detection and prevention~\cite{GeibPlanRecognitionIntrusionDect_DARPA2001}, monitoring activities in elder-care~\cite{ProblemsWithElderCare_AAAI2002}, recognizing plans in educational environments~\cite{UzanDSG_PR_2015} and exploratory domains~\cite{PR_EXP_Mirsky2017}, and traffic monitoring~\cite{WellmanTraffic_2013}, among others~\cite{GeibPlanRecognitionIntrusionDect_DARPA2001,Granada2017,Mirsky_UISP17,Amado2018}.

Existing approaches to solving goal and plan recognition problems vary on the type of domain model used to describe the behavior of the observed agents, or plan generation, as well as the level of observability and noise in the observations used as evidence for recognizing goals and plans~\cite[Chapter 1]{ActivityIntentPlanRecogition_Book2014}. 
According to the literature of goal and plan recognition~\cite[Chapter 1]{ActivityIntentPlanRecogition_Book2014}, the most used types of domain models for recognizing goals and plans are \textit{plan--libraries} and \textit{planning domain theories}. 
Plan--library based approaches to goal and plan recognition have shown to be very fast and accurate in several levels of observability \cite{AvrahamiZilberbrandK_IJCAI2005,DoritGalAAAI07,PR_Mirsky_2016,MIRSKY_2018_AIJ}. However, formalizing plan--libraries is usually laborious and time-consuming, and requires a substantial amount of domain knowledge to represent the set of possible plans for achieving the goals. 
In contrast, recognition approaches that employ the use of planning domain theory and planning techniques have gradually relaxed such requirements~\cite{RamirezG_IJCAI2009,RamirezG_AAAI2010}, resulting in approaches that are very efficient and require much less domain knowledge~\cite{GoalRecognitionDesign_Keren2014,NASA_GoalRecognition_IJCAI2015,Sohrabi_IJCAI2016,PereiraNirMeneguzzi_AAAI2017}. 

A key limitation of most existing approaches to goal and plan recognition is that they fail to deal with incomplete and/or inaccurate domain information available in the domain model. 
Regardless of the type of domain model formalism used to describe the observed agent's behavior, most recognition approaches assume that the domain model is complete and correct, restricting their direct application to more realistic scenarios in which both the domain model and the observations have imperfect information. 
Specifically, realistic scenarios have two potential sources of imperfect information. 
The first stems from imperfections in domain models, especially when such models come from learning processes, approximated from data, or have unknown properties in the action descriptions, resulting in domain models that are not fully accurate due to such imperfections in their description, \idest, \textit{imperfect domain models}. 
The second stems from ambiguity with respect to the observations, namely, on how imperfect sensor data report features and properties of the world (\exemp, state properties), and on how actions performed by observed agents are realized in the environment.

While there are few research efforts on goal and plan recognition for dealing with \textit{imperfect domain models}, over the past few years the planning community has been addressing this problem in at least two fronts. 
First, with respect to planning over \textit{incomplete discrete domain models}, in which the planning domain model has annotations to specify what is \textit{unknown} in the model. 
Weber and Bryce~\cite{WeberBryce_ICAPS_2011}, and Nguyen~\etal~\cite{PlanningIncomplete_NguyenK_2014,Nguyen_AIJ_2017} have addressed this problem by developing efficient heuristic approaches for use with well-known and new automated planners. 
As for the second, we single out the task of planning over \textit{approximate continuous domain models}. 
Most recently, Say~\etal~\cite{SayWZS:ijcai17} developed an automated planner that can cope with \textit{approximate hybrid mixed discrete-continuous domain models} that are learned from data. 
Subsequently, Wu, Say, and Sanner~\cite{WuSS:nips17} developed a planner that relies on modern learning techniques over the same settings proposed by Say~\etal~\cite{SayWZS:ijcai17}. 

Notwithstanding these developments in \textit{Automated Planning} algorithms, comparatively little effort has been made in the goal and plan recognition community to address this particular problem. However, these developments motivate the key research questions of this thesis, specifically: \emph{Is it possible to recognize goals both quickly and accurately over imperfect domain models?} \emph{Is it possible to recognize goals over domain models that are approximated from data?} \emph{Are the current recognition approaches able to remain accurate without any modification/adaption for recognizing goals over imperfect domains?} 
In this thesis, we aim to address these questions by bringing the problem of goal recognition closer to more realistic scenarios. To do so, we introduce the problem of goal recognition over \textit{imperfect domain models}. 
To solve this problem, we develop novel goal recognition approaches that can cope with \textit{imperfect domain models}. 
More specifically, the approaches we develop in this thesis deal explicitly with two types of \textit{imperfect domain models}: \textit{incomplete discrete domain models} and \textit{approximate continuous domain models}. 
Thus, the main contributions of this thesis are twofold.

\noindent \textbf{Recognizing goals over incomplete discrete domain models}: The first contribution of this thesis is regarding the task of goal recognition over \textit{incomplete discrete domain models}. 
We formalize the problem of goal recognition in \textit{incomplete discrete domain models} by combining the standard formalization of Ramírez and Geffner~\cite{RamirezG_IJCAI2009,RamirezG_AAAI2010} for goal recognition as planning, and that of Nguyen~\etal~\cite{PlanningIncomplete_NguyenK_2014,Nguyen_AIJ_2017} for planning in incomplete domain models. 
The formalization of incomplete domains introduced in~\cite{PlanningIncomplete_NguyenK_2014,Nguyen_AIJ_2017} allows the use of annotations in the domain model to specify what is \textit{incomplete} and \textit{unknown} in the model. 
Such formalization specifies the incomplete part of the domain model by using \textit{possible} preconditions and effects in the description actions to specify what is unknown in the domain model. 
For recognizing goals over incomplete domain models, we develop recognition heuristics that refrain from the use of automated planners during the recognition process. Specifically, we develop novel approaches by enhancing well-known recognition heuristics from the literature~\cite{PereiraNirMeneguzzi_AAAI2017} that rely on the concept of \textit{landmarks}, and deal explicitly with incomplete domain models. 
In \textit{Automated Planning}, landmarks are states (or actions) that must be achieved (or executed) to achieve a goal from an initial state~\cite{Hoffmann2004_OrderedLandmarks}. 
To extract landmarks in incomplete domains and enhance such heuristics, we introduce new notions of landmarks, and we develop a landmark extraction algorithm, adapted from~\cite{Hoffmann2004_OrderedLandmarks}. 
We evaluate our enhanced recognition heuristics using new datasets constructed by modifying an existing dataset~\cite{Pereira_Meneguzzi_PRDatasets_2017} of planning--based goal recognition problems. 
We have built these new datasets by removing information from the complete domain model and annotating them with \textit{possible} preconditions and effects, which comprise the incomplete part of the domain model. 
We perform an ablation study to show and understand the effect of the new notions of landmarks on recognition performance. 
Experiments and evaluation show that our enhanced recognition approaches are fast and accurate for recognizing goals in large and non-trivial incomplete domain models at most levels of domain incompleteness when compared to the non-enhanced (original) recognition approaches in~\cite{PereiraNirMeneguzzi_AAAI2017}.

\noindent \textbf{Recognizing goals over approximate continuous domain models}: As for the second contribution of this thesis, we introduce the problem of goal recognition over \textit{approximate continuous domain models}. 
We develop novel recognition approaches over this setting by leveraging existing work on domain model acquisition via learning techniques for \textit{Hybrid Planning}~\cite{SayWZS:ijcai17}, and adapt well-known probabilistic approaches to goal recognition~\cite{RamirezG_AAAI2010,Kaminka_18_AAAI} in order to analyze how prediction errors from the acquired model impact on recognition. 
Specifically, we use the learning technique proposed by Say~\etal~\cite{SayWZS:ijcai17} to approximate the transition function of continuous domain models, and obtain \textit{approximate continuous domain models}, also named as \emph{nominal models}~\cite{ljung1998system} by the literature of \textit{Control}~\cite{borrelli:17:predictive}. 
We evaluate the recognition approaches over \textit{nominal models} empirically in complex recognition datasets that we built by using three benchmark domains based on the \emph{constrained} Linear–Quadratic Regulator (LQR) problem~\cite{bemporad:2002:lqr}, and two non-linear navigation domains proposed by Say~\etal~in~\cite{SayWZS:ijcai17}, with increasing dimensions of state and action spaces. 

\sigla{AUV}{Autonomous Underwater Vehicle}
The contributions above can be used to build different types applications in realistic scenarios, such as: online and offline goal recognition applied video streams~\cite{Granada2017}, in which the domain model can be generated automatically from the video frames~\cite{Amado2018}; learn the behavior of Autonomous Underwater Vehicles (AUVs) from data~\cite{FernandezGonzalez18_Scotty}, and recognize their intended scientific missions based on their interactions in the environment; among others. 
Thus, the work developed in this thesis brings the task of goal recognition closer to more realistic scenarios, not only by recognizing goals even when the available discrete models are imperfect, but also by performing the recognition task over continuous domains with approximate transition functions.

%$$$$$$$$$$$$$$$$$$$$$$$$$$$$$$$$$$$$$$$$$$$$$$$$$$$$$$$$$$$$$$$$$$$$$$$$$$$$$$$$$$
\section{Overview of Research Contribution}\label{section:Publications}

% \frm[inline]{Just reorganize the citations so that they are in reverse chronological order (i.e. most recent work first),}

Throughout the course of our Ph.D. research over the last four years, we have published our contributions as we developed and evaluated them. Particularly, the main contributions that underpin this thesis have been published in five conferences and one journal, as follows.

\sigla{AAAI}{Association for the Advancement of Artificial Intelligence}
\sigla{IJCAI}{International Joint Conference on Artificial Intelligence}
\sigla{ECAI}{European Conference on Artificial Intelligence}
\sigla{ICAPS}{International Conference on Automated Planning and Scheduling}
\begin{itemize}
	\item Ramon Fraga Pereira, Nir Oren, and Felipe Meneguzzi. \emph{Landmark-Based Approaches for Goal Recognition as Planning}~\cite{PereiraOM_AIJ_2020}. In Artificial Intelligence, Volume 279, 2020;
	\item Ramon Fraga Pereira, Mor Vered, Felipe Meneguzzi, and Miquel Ramírez. \textit{Online Probabilistic Goal Recognition over Nominal Models}~\cite{PereiraVMR_IJCAI19}. In Proceedings of the 28th International Joint Conferences on Artificial Intelligence (IJCAI), 2019;
	\item Ramon Fraga Pereira, André Grahl Pereira, and Felipe Meneguzzi. \textit{Landmark-Enhanced Heuristics for Goal Recognition in Incomplete Domain Models}~\cite{PereiraPM_ICAPS_19}. In Proceedings of the 29th International Conference on Automated Planning and Scheduling (ICAPS), 2019;
	\item Ramon Fraga Pereira and Felipe Meneguzzi. \textit{Goal Recognition in Incomplete Domain Models}~\cite{AAAI2018_PereiraMeneguzzi}. In Proceedings of the 32nd Association for the Advancement of Artificial Intelligence (AAAI)\footnote{This paper has been published as a student abstract at AAAI in 2018, and was  among the top-ten best student papers and selected as finalist for the 3-minute presentation contest.}, 2018;
	\item Ramon Fraga Pereira, Nir Oren, and Felipe Meneguzzi. \textit{Landmark-Based Heuristics for Goal Recognition}~\cite{PereiraNirMeneguzzi_AAAI2017}. In Proceedings of the 31st Association for the Advancement of Artificial Intelligence (AAAI), 2017; and
	\item Ramon Fraga Pereira and Felipe Meneguzzi. \textit{Landmark-Based Plan Recognition}~\cite{PereiraMeneguzzi_ECAI2016}. In Proceedings of the 22nd European Conference on Artificial Intelligence (ECAI), 2016.		
\end{itemize}

We have also published contributions which, while not directly claimed as part of this thesis, are nevertheless closely related to our contributions.

\sigla{AAMAS}{International Conference on Autonomous Agents and Multi-Agent Systems}
\sigla{IJCNN}{International Joint Conference on Neural Networks}
\sigla{PAIR}{Workshop on Plan, Activity, and Intent Recognition}
\begin{itemize}
	\item Ramon Fraga Pereira, Nir Oren, and Felipe Meneguzzi. \textit{Using Sub-Optimal Plan Detection to Identify Commitment Abandonment in Discrete Environments}~\cite{PereiraOM_TIST_2020}. In ACM Transactions on Intelligent Systems and Technology, Volume 11, 2020;
	\item Leonardo Amado, Ramon Fraga Pereira, Joao Paulo Aires, Mauricio Cecílio Magnaguagno, Roger Granada, Gabriel Paludo Licks, and Felipe Meneguzzi. \textit{LatRec: Recognizing Goals in Latent Space}~\cite{Amado_Demo_ICAPS_2019}. Demonstration at the 29th International Conference on Automated Planning and Scheduling (ICAPS), 2019; 	
	\item Leonardo Amado, Ramon Fraga Pereira, Joao Paulo Aires, Mauricio Cecílio Magnaguagno, Roger Granada, and Felipe Meneguzzi. \textit{Goal Recognition in Latent Space}~\cite{Amado2018}. In Proceedings of the International Joint Conference on Neural Networks (IJCNN), 2018;	
	\item Mor Vered, Ramon Fraga Pereira, Mauricio Cecílio Magnaguagno, Gal A. Kaminka, and Felipe Meneguzzi. \emph{Towards Online Goal Recognition Combining Goal Mirroring and Landmarks}~\cite{MorEtAl_AAMAS18}. In Proceedings of the 17th International Conference on Autonomous Agents and Multi-Agent Systems (AAMAS), 2018;
	\item Ramon Fraga Pereira and Felipe Meneguzzi. \textit{Goal Recognition in Incomplete STRIPS Domain Models}~\cite{PAIR18_PereiraMeneguzzi}. In the AAAI 2018 workshop on Plan, Activity, and Intent Recognition (PAIR), 2017;
	\item Roger Granada, Ramon Fraga Pereira, Juarez Monteiro, Rodrigo Barros, Duncan Ruiz, and Felipe Meneguzzi. Hybrid Activity and Plan Recognition for Video Streams~\cite{Granada2017}. In the AAAI 2017 workshop on Plan, Activity, and Intent Recognition (PAIR), 2017;	
	\item Ramon Fraga Pereira, Nir Oren, and Felipe Meneguzzi. \textit{Monitoring Plan Optimality using Landmarks and Domain-Independent Heuristics}~\cite{PAIR17_PereiraOrenMeneguzzi}. In the AAAI 2017 workshop on Plan, Activity, and Intent Recognition (PAIR), 2017; and
	\item Ramon Fraga Pereira, Nir Oren, and Felipe Meneguzzi. \textit{Detecting Commitment Abandonment by Monitoring Sub-Optimal Steps during Plan Execution}~\cite{PereiraOrenMeneguzzi_AAMAS2017}. In Proceedings of the 16th International Conference on Autonomous Agents and Multi-Agent Systems (AAMAS), 2017.
\end{itemize}

%$$$$$$$$$$$$$$$$$$$$$$$$$$$$$$$$$$$$$$$$$$$$$$$$$$$$$$$$$$$$$$$$$$$$$$$$$$$$$$$$$$
\section{Thesis Outline}\label{section:Outline}

We organized this thesis as follows. In Chapter~\ref{chapter:Background}, we provide the relevant background to this thesis\footnote{Thus, the reader familiar with \textit{Planning} and \textit{Goal Recognition} may safely skip Chapter~\ref{chapter:Background}.}, revisiting key concepts and terminologies that are essential to understand our contributions. 
After that, in Chapter~\ref{chapter:GR_IncompleteDomains}, we describe a new problem formulation for recognizing goals over incomplete discrete domain models, a landmark extraction algorithm over incomplete domains, along with a set of new notions of landmarks, and develop two landmark--enhanced heuristics that can cope with incomplete domain models. In Chapter~\ref{chapter:GR_NominalModels}, we introduce a novel problem formulation for goal recognition over \textit{nominal models}, and describe our solution approaches over this recognition setting. In Chapter~\ref{chapter:RelatedWork}, we survey the literature and present the related work on recognition under incomplete information, goal and plan recognition as planning, and recent work on planning over \textit{imperfect domain models}. Finally, in Chapter~\ref{chapter:Conclusions}, we conclude this thesis by addressing our main contributions, open issues and limitations of our proposed recognition approaches, as well as future avenues regarding the proposed approaches in this thesis.

%% file: cap2_Background.tex
%!TEX root = ppgcc-thesis.tex
%----------------------------------------------------------------------------------
% Chapter: Background and Representation
%
% + Planning
% 	- STRIPS
% 	- Landmarks
% + Planning in Incomplete Information
%	- Incomplete STRIPS
% + Goal and Plan Recognition as Planning
% + Optimal Control
%	- Finite Horizon Optimal Control Problems
%	- Actual and Nominal Models
%
%----------------------------------------------------------------------------------

\chapter{Background and Representation}\label{chapter:Background}

This thesis stands at the intersection of two fields of \textit{Artificial Intelligence} (AI), specifically, \textit{Automated Planning} and \textit{Goal Recognition}, and addresses some of the basic notations and concepts of \textit{Control Theory}.
Thus, in this chapter, we present the essential background for understanding the contributions of this thesis. 
In Section~\ref{section:Planning}, we review the background on \textit{Classical Planning} terminology. 
In Section~\ref{section:IncompletePlanning}, we present the terminology of planning over incomplete domain models. 
After, in Section~\ref{section:Landmarks}, we describe the concept of \textit{landmarks} in \textit{Automated Planning}, and how we exploit and build some of our recognition approaches using landmarks. 
Then, in Section~\ref{section:GoalRecognition}, we describe the task of \textit{Goal and Plan Recognition as Planning} (PRAP). Finally, we conclude this chapter, in Section~\ref{section:OptimalControl}, by presenting the terminology of \textit{Control Theory} we use to formalize \textit{Optimal Control} problems.

\sigla{AI}{Artificial Intelligence}
\sigla{PRAP}{Plan Recognition as Planning}

%----------------------------------------------------------------------------------
\section{Classical Planning}\label{section:Planning}

\textit{Planning} is the problem of finding a sequence of actions (\idest, plan) that achieves a particular goal state from an initial state~\cite{AIModernApproachRussell_2009}. 
Such problem can be seen as a directed graph, whose nodes represent states, edges represent the transition between states (caused by applying actions), and the solution is a path between two particular nodes (i.e., initial state and goal state) in this directed graph.
In this thesis, we adopt the terminology from Ghallab~\etal~\cite{AutomatedPlanning_Book2016} to represent states and actions in planning domains and problems. 
First, we define a \textit{state and its predicates} in the environment as Definition~\ref{def:State}. 

\begin{definition} [\textbf{Predicates and State}]\label{def:State}
A predicate is denoted by an n-ary predicate symbol $p$ applied to a sequence of zero or more terms ($\tau_1$, $\tau_2$, ..., $\tau_n$) -- terms are either constants or variables.
We refer to grounded predicates that represent logical values according to some interpretation as facts, which are divided into two types: positive and negated facts, as well as constants for truth ($\top$) and falsehood ($\bot$).
A state $S$ is a finite set of positive facts $f$ that follows the closed world assumption so that if $f \in S$, then $f$ is true in $S$. 
We assume a simple inference relation $\models$ such that $S \models f$ iff $f \in S$, $S \not\models f$ iff $f \not\in S$, and $S \models f_1 \land ... \land f_n$ iff $\{f_1, ..., f_n\} \subseteq S$.
\end{definition}
\simbolo{$S$}{State}

A \textit{planning domain model} aims to describe the environment dynamics through the specification of operators, using a limited first-order logic representation (as we defined above in Definition~\ref{def:State}) to define schemata for state-modification actions, as follows in Definition~\ref{def:Operator}.

\begin{definition} [\textbf{Operator and Action}]\label{def:Operator}
An operator $op$ is represented by a triple $\langle name(op), \\ \mathit{pre}(op), \mathit{eff}(op)\rangle$: $name(op)$ represents the description or signature of $op$; $\mathit{pre}(op)$ describes the preconditions of $op$, a set of predicates that must exist in the current state for $op$ to be executed; $\mathit{eff}(op)$ represents the effects of $op$. 
These effects are divided into $\mathit{eff}^{+}(op)$ (\idest, an add-list of positive predicates) and $\mathit{eff}^{-}(op)$ (\idest, a delete-list of negated predicates).
An action $a$ is a ground operator instantiated over its free variables. 
\end{definition}

We say that an action $a$ is applicable to a state $S$ if and only if $S \models \mathit{pre}(a)$, and generates a new state $S'$ such that:
\begin{align}
\label{eq:state_transition}
S' \gets (S \cup \mathit{eff}^{+}(a))/\mathit{eff}^{-}(a)
\end{align}
 
Thus, by following the notation of \textit{predicates}, \textit{states}, \textit{operator}, and \textit{actions} in Definitions~\ref{def:State}~and~\ref{def:Operator}, we formally define a \textit{planning domain model} in Definition~\ref{def:PlanningDomain}.

\simbolo{$\Xi$}{Planning Domain}
\simbolo{$\mathcal{O}$}{Operator definition}
\simbolo{$\mathcal{R}$}{Predicate definition}
\simbolo{$\mathcal{F}$}{Finite set of Facts of a Planning Domain}
\simbolo{$\mathcal{A}$}{Finite set of Actions of a Planning Domain}
\simbolo{$\mathit{pre}$}{Precondition}
\simbolo{$\mathit{eff}^{+}$}{Add Effect}
\simbolo{$\mathit{eff}^{-}$}{Delete Effect}
\begin{definition}[\textbf{Planning Domain}]\label{def:PlanningDomain}
A planning domain definition $\Xi$ is represented by a pair $\langle \mathcal{F}, \mathcal{A} \rangle$, which specifies the knowledge of the domain, and consists of:
\begin{itemize}
	\item A finite set of facts $\mathcal{F}$, \idest, a set of ground instantiated predicates, defining the environment state properties; and
	\item A finite set of actions $\mathcal{A}$, which is technically a set of ground instantiated operators, representing the actions that can be performed in the environment.
\end{itemize}	
\end{definition}

A \textit{planning instance}, comprises both a \textit{planning domain} and the elements of a \textit{planning problem}, describing a finite set of \textit{objects} of the environment, the \textit{initial state}, and the \textit{goal state} which an agent wishes to achieve, as formalized in Definition~\ref{def:PlanningInstance}.

\simbolo{$\Pi$}{Planning Instance}
\simbolo{$\mathcal{I}$}{Initial State}
\simbolo{$G$}{Goal State}
\begin{definition} [\textbf{Planning Instance}]\label{def:PlanningInstance}
A planning instance $\Pi$ is represented by a triple $\langle \Xi, \mathcal{I}, G\rangle$, and consists of:
\begin{itemize}
	\item $\Xi =  \langle \mathcal{F}, \mathcal{A}\rangle$ is the domain definition, where $\mathcal{F}$ is the set of facts, and $\mathcal{A}$ is the set of actions;
	\item $\mathcal{I} \subseteq \mathcal{F}$ is the initial state specification, which is defined by specifying the value for all facts in the initial state; and 
	\item $G \subseteq \mathcal{F}$ is the goal state specification, which represents a desired state to be achieved.
\end{itemize}
\end{definition}

A \textit{plan} is the solution of a \textit{planning instance}, as formalized in Definition~\ref{def:Plan}.

% \frm[inline]{You need to also define a state transition function, so that the definition of a plan makes sense (you can do this functionally, or patch your definition of the operator in Def.~\ref{def:Operator})}
% I have already done that above.

\simbolo{$\pi$}{Plan}
\begin{definition} [\textbf{Plan}]\label{def:Plan}
A plan $\pi$ for a planning instance $\Pi = \langle \Xi, \mathcal{I}, G\rangle$ is a sequence of actions $\langle$$a_1$, $a_2$, ..., $a_n$$\rangle$ that modifies the initial state $\mathcal{I}$ into a state $S\models G$ in which the goal state $G$ holds by the successive execution of actions in a plan $\pi$. 
A plan $\pi^{*}$ with length $|\pi^{*}|$ is optimal if there exists no other plan $\pi'$ for $\Pi$ such that $\pi' < \pi^{*}$.
\end{definition}

While instantiated actions have an associated cost, we take the assumption from \textit{Classical Planning} that this cost is 1 for all instantiated actions. 
Therefore, a plan $\pi$ is considered \textit{optimal} if its cost, and thus length, is \textit{minimal}.

Modern classical planners use a variety of heuristics to efficiently explore the search space of planning domains by estimating the cost to achieve a specific goal~\cite{AutomatedPlanning_Book2016}.  
In \textit{Classical Planning}, this estimate is often the number of actions to achieve the goal state from a particular state. 
For ease of explanation, we describe our planning--based techniques assuming a uniform action cost $c(a) = 1$ for all $a \in A$, but this is easily generalizable.  
Thus, the cost for a plan $\pi =[ a_1, a_2, ..., a_n ]$ is $c(\pi) = \Sigma c(a_{i})$.

Heuristics provide no guarantees about the accuracy of their estimations, however, when a heuristic never overestimates the cost to achieve a goal, it is called \textit{admissible} and guarantees optimal plans for certain search algorithms. 
A heuristic $h(s)$ is admissible if $h(s)$ $\leq$ $h^{*}(s)$ for all states, where $h^{*}(s)$ is the optimal cost to the goal state from state $s$. Heuristics that overestimate the cost to achieve a goal are called \textit{inadmissible}. 

%######################################################################
\subsection{STRIPS Domain Models}\label{subsection:STRIPS}

% \frm[inline]{Reading this through, I'm more confident that this subsection must go. Unless there is a very specific reason for you to want to refer to STRIPS as a section later on (which at the moment you do not), then get rid of this section and reuse the content from here as part of the other explanations.}
% I think that this subsection is quite important, since the formal definition of incomplete domain models is based the incomplete STRIPS, that's why I added this content.

\sigla{STRIPS}{Stanford Research Institute Problem Solver}
\sigla{PDDL}{Planning Domain Definition Language}
\simbolo{$\mathcal{D}$}{STRIPS Domain Model}
Classical planning representations often separate the definition of the initial state ($\mathcal{I}$) and goals state ($G$) as part of a planning problem to be used together with a planning domain model $\Xi$ (Definition~\ref{def:PlanningDomain}), such as STRIPS~\cite{STRIPSFikes1971} and PDDL~\cite{PDDLMcdermott1998}.
We define a STRIPS domain model over typed variables as $\mathcal{D} = \langle \mathcal{R}, \mathcal{O} \rangle$, where: $\mathcal{R}$ is a set of predicates with typed variables.
Grounded predicates represent logical values according to some interpretation as facts, which are divided into two types: positive and negated facts, as well as constants for truth ($\top$) and falsehood ($\bot$); $\mathcal{O}$ is a set of operators $op = \langle \mathit{pre}(op), \mathit{eff}(op) \rangle$, where $\mathit{eff}(op)$ can be divided into positive effects $\mathit{eff}^{+}(op)$ (the add list) and negative effects $\mathit{eff}^{-}(op)$ (the delete list). 
An operator $op$ with all variables bound is called an action and allows state change. 
An action $a$ instantiated from an operator $op$ is applicable to a state $S$ iff $S \models \mathit{pre}(a)$ and results in a new state $S'$ such that $S' \gets (S / \mathit{eff}^{-}(a)) \cup \mathit{eff}^{+}(a)$.

\simbolo{$\mathcal{P}$}{STRIPS Planning Problem}
\simbolo{$Z$}{Set of Typed Objects}
A STRIPS planning problem within $\mathcal{D}$ over a set of typed objects $Z$ is defined as $\mathcal{P} = \langle \mathcal{F}, \mathcal{A}, \mathcal{I}, G \rangle$, where: $\mathcal{F}$ is a set of facts (instantiated predicates from $\mathcal{R}$ and $Z$); $\mathcal{A}$ is a set of instantiated actions from $\mathcal{O}$ and $Z$; $\mathcal{I}$ is the initial state ($\mathcal{I} \subseteq \mathcal{F}$); and $G$ is the goal state, which represents a desired state to be achieved. 
A plan $\pi$ for a planning problem $\mathcal{P}$ is a sequence of actions $\langle a_1, a_2, ..., a_n \rangle$ that modifies the initial state $\mathcal{I}$ into a state $S\models G$ in which the goal state $G$ holds by the successive execution of actions in a plan $\pi$.
% \frm[inline]{This feels a bit repetitious from before, perhaps incorporate these descriptions alongside the definitions and state that they are STRIPS (which they are). The typing is additional.}

%----------------------------------------------------------------------------------
\section{Planning in Incomplete Domain Models}
\label{section:IncompletePlanning}

% \frm[inline]{In order for your background to be more self-contained, perhaps you may want to actually talk about planning algorithms. This is so that you can refer back to heuristics, their function, and the cost difference between doing PRAP (like Ramirez and Geffner), and the heuristic-based methods you have here. This then would make the link to Sections~\ref{section:Landmarks} and~\ref{section:GoalRecognition}}

Most planning algorithms and heuristics assume that the domain model is complete and correct, relaxing the need to deal with incomplete domain information. 
We argue that this assumption may be too strong for dealing with more realistic domains. 
Moreover, the effort of domain knowledge engineering that is required to model a complete and correct planning domain model can be laborious and substantial because of human error and/or lack of domain knowledge from the modeler~\cite{Kambhampati_AAAI07}. 

%\frm[inline]{The next couple of sentences are a compressed version of the related work, which I think is out of place in your proposal here (since you do have an actual longer related work section later). Just keep the sentence about the formalism you use.}

\textit{Planning in incomplete domain models} is similar to the concept of \textit{Classical Planning}, except that some actions are not completely specified and there are annotations specifying \textit{possible preconditions} and \textit{effects} of some actions in the domain definition~\cite{GarlandLesh_AAAI2002,WeberBryce_ICAPS_2011,PlanningIncomplete_NguyenK_2014,Nguyen_AIJ_2017}. 
To deal with incomplete information in domain models, Garland and Lesh~\cite{GarlandLesh_AAAI2002} developed the first planning approach in the literature that allows annotations about incompleteness in the domain definition. After, in~\cite{WeberBryce_ICAPS_2011}, Weber and Bryce developed a set of approaches to planning and acting in incomplete domain models, allowing annotations with regard to the incompleteness of actions in the domain model. Most recently, Nguyen~\etal~\cite{PlanningIncomplete_NguyenK_2014,Nguyen_AIJ_2017} use the same formalism from~\cite{WeberBryce_ICAPS_2011} to develop more modern approaches to planning in incomplete domain models. 
% Bonet and Geffner~\cite{Bonet_PlanningIncompleteState_AIPS00} develop a planning approach that deals with incomplete state information.
In this thesis, we follow the formalism of incomplete domain models from~\cite{PlanningIncomplete_NguyenK_2014,Nguyen_AIJ_2017} and use incomplete STRIPS domain models for modeling incomplete domains.

%######################################################################
% \newpage % FRM - Avoid putting artificial page breaks (which I understand for aesthetic reasons) until you have settled on the final content.
\subsection{Incomplete STRIPS Domain Models}

\simbolo{$\widetilde{\mathcal{D}}$}{Incomplete STRIPS Domain Model}
\simbolo{$\incp{P}$}{Incomplete STRIPS Planning Problem}
\simbolo{$\widetilde{\mathcal{O}}$}{Incomplete Operator definition}
\simbolo{$\widetilde{\mathit{pre}}$}{Possible Precondition}
\simbolo{$\widetilde{\mathit{eff}}^{+}$}{Possible Add Effect}
\simbolo{$\widetilde{\mathit{eff}}^{-}$}{Possible Delete Effect}
To represent incomplete domain models, we define an incomplete STRIPS domain model by following the formalism from~\cite{PlanningIncomplete_NguyenK_2014,Nguyen_AIJ_2017}, defined as $\widetilde{\mathcal{D}} = \langle \mathcal{R}, \widetilde{\mathcal{O}} \rangle$. 
Here, $\widetilde{\mathcal{O}}$ contains the definition of incomplete operators comprised of a six-tuple $\inc{op} = \langle \mathit{pre}(\inc{op}), \widetilde{\mathit{pre}}(\inc{op}), \mathit{eff}^{+}(\inc{op}), \mathit{eff}^{-}(\inc{op}),\\ \widetilde{\mathit{eff}}^{+}(\inc{op}), \widetilde{\mathit{eff}}^{-}(\inc{op}) \rangle$, where: $\mathit{pre}(\inc{op})$ and $\mathit{eff}(\inc{op})$ have the same semantics as in the STRIPS domain models; and \textit{possible} preconditions $\widetilde{\mathit{pre}}(\inc{op}) \subseteq \mathcal{R}$ that \textit{might} be required as preconditions, as well as $\widetilde{\mathit{eff}}^{+}(\inc{op}) \subseteq \mathcal{R}$ and $\widetilde{\mathit{eff}}^{-}(\inc{op}) \subseteq \mathcal{R}$ that \textit{might} be generated as \textit{possible} effects either as add or delete effects. An incomplete domain $\widetilde{\mathcal{D}}$ has a \textit{completion set} $\langle\langle \widetilde{\mathcal{D}} \rangle\rangle$ comprising all possible domain models derivable from an incomplete one. 
Namely, the number of all possible domain models is $\langle\langle \widetilde{\mathcal{D}} \rangle\rangle = 2^K$, where $K = \sum_{\inc{op} \in \widetilde{\mathcal{O}}}(|\widetilde{\mathit{pre}}(\inc{op})| + |\widetilde{\mathit{eff}}^{+}(\inc{op})| + |\widetilde{\mathit{eff}}^{-}(\inc{op})|)$.
An incomplete STRIPS planning problem derived from an incomplete STRIPS domain $\incp{D}$ and a set of typed objects $Z$ is defined as $\incp{P} = \langle \mathcal{F}, \incp{\mathcal{A}}, \mathcal{I}, G \rangle$, where: $\mathcal{F}$ is the set of facts (instantiated predicates from $Z$), $\incp{\mathcal{A}}$ is the set of incomplete instantiated actions from $\inc{\mathcal{O}}$ with objects from $Z$, $\mathcal{I} \subseteq \mathcal{F}$ is the initial state, and $G \subseteq \mathcal{F}$ is the goal state.

Most approaches for planning in incomplete domains~\cite{WeberBryce_ICAPS_2011,PlanningIncomplete_NguyenK_2014,Nguyen_AIJ_2017} assume that plans succeed under the \textit{most optimistic} conditions, which are: 

\begin{itemize}
	\item Possible preconditions $\widetilde{\mathit{pre}}$ do not need to be satisfied in a state $S$; 
	\item Possible add effects $\widetilde{\mathit{eff}}^{+}$ are always assumed to occur in the resulting state $S'$;
	\item Delete effects $\widetilde{\mathit{eff}}^{-}$ are ignored in the resulting state $S'$. 
\end{itemize}

Therefore, formally, an incomplete action $\inc{a}$ instantiated from an incomplete operator $\inc{op}$ is applicable to a state $S$ iff $S \models \mathit{pre}(\inc{a})$ and results in a new state $S'$ such that $S' \gets (S / \mathit{eff}^{-}(a)) \cup (\widetilde{\mathit{eff}}^{+}(\inc{a}) \cup \mathit{eff}^{+}(a))$. 
Thus, a valid plan $\pi$ that achieves a goal $G$ from $\mathcal{I}$ in an incomplete planning problem $\incp{P}$ is a sequence of actions that induces an \textit{optimistic} sequence of states. 
Example~\ref{example:Abstract} from Weber and Bryce~\cite{WeberBryce_ICAPS_2011} illustrates an abstract incomplete planning problem and a valid plan for it.

\begin{example}\label{example:Abstract}
Consider the following incomplete planning problem $\incp{P}$, where: 

\begin{itemize}
	\item $\mathcal{F} = \lbrace p,q,r,g \rbrace$; 
	\item $\incp{\mathcal{A}} = \lbrace \inc{a},\inc{b},\inc{c} \rbrace$, where: 
	\begin{itemize}
		\item $\mathit{pre}(\inc{a}) = \lbrace p,q \rbrace, \widetilde{\mathit{pre}}(\inc{a}) = \lbrace r \rbrace, \widetilde{\mathit{eff}}^{+}(\inc{a}) = \lbrace r \rbrace, \widetilde{\mathit{eff}}^{-}(\inc{a}) = \lbrace p \rbrace$
		\item $\mathit{pre}(\inc{b}) = \lbrace p \rbrace, \mathit{eff}^{+}(\inc{b}) = \lbrace r \rbrace, \mathit{eff}^{-}(\inc{b}) = \lbrace p \rbrace, \widetilde{\mathit{eff}}^{-}(\inc{b}) = \lbrace q \rbrace$
		\item $\mathit{pre}(\inc{c}) = \lbrace r \rbrace, \widetilde{\mathit{pre}}(\inc{c}) = \lbrace q \rbrace, \mathit{eff}^{+}(\inc{c}) = \lbrace g \rbrace$
	\end{itemize}
	\item $\mathcal{I} = \lbrace p,q \rbrace$; and
	\item $G = \lbrace g \rbrace$.
\end{itemize}

The $[\inc{a},\inc{b},\inc{c}]$ sequence of actions is a valid plan to achieve goal state $\lbrace g \rbrace$ from the initial state $\lbrace p,q \rbrace$.
It corresponds to the \textit{optimistic} state sequence: $s_{0} =  \lbrace p, q \rbrace, s_{1} = \lbrace p,q,r \rbrace, s_{2} = \lbrace q,r \rbrace, s_{3} = \lbrace q, r, g \rbrace$. 
The number of completions for this example is $|\langle\langle \widetilde{\mathcal{D}} \rangle\rangle| = 2^{5}$ (2 possible preconditions and 3 possible effects, \idest, 1 possible add effect and 2 possible delete effects).
\end{example}

%######################################################################
\section{Landmarks}\label{section:Landmarks}

In the \textit{Planning} literature~\cite{Hoffmann2004_OrderedLandmarks,Vicent_ActionLandmarks_2005,LandmarksRichter_2008}, \textit{landmarks} are defined as \textit{necessary} properties (alternatively, actions) that must be true (alternatively, executed) at some point in every valid plan (see Definition~\ref{def:Plan}) to achieve a particular goal from an initial state, being often partially ordered following the sequence in which they must be achieved. 
Hoffman~\etal~\cite{Hoffmann2004_OrderedLandmarks} define \emph{fact landmarks} (Definitions~\ref{def:FactLandmark}) as follows:

\begin{definition}[\textbf{Fact Landmark}]\label{def:FactLandmark}
Given a planning instance $\Pi = \langle \Xi, \mathcal{I}, G\rangle$, a formula $L$ is a landmark in $\Pi$ iff $L$ is true at some point along all valid plans that achieve $G$ from $\mathcal{I}$. 
In other words, a landmark is a type of formula (\exemp, conjunctive formula or disjunctive formula) over a set of facts that must be satisfied (or achieved) at some point along all valid plan executions.
\end{definition}

Vidal and Geffner~\cite{Vicent_ActionLandmarks_2005} define \textit{action landmarks} (Definition~\ref{def:ActionLandmark}) as necessary actions that must be executed at some point along all valid plans that achieve a goal state $G$ from an initial state $\mathcal{I}$. In this thesis, we do not explicitly use the concept of action landmarks, but rather use fact landmarks to build our planning--based approaches to goal recognition.

\begin{definition}[\textbf{Action Landmark}]\label{def:ActionLandmark}
Given a planning instance $\Pi = \langle \Xi, \mathcal{I}, G\rangle$, an action $A$ is a landmark in $\Pi$ iff $A$ is a necessary action that must be executed at some point along all valid plans that achieve $G$ from $\mathcal{I}$.
\end{definition}

\begin{figure}[h!]
  \centering
  \includegraphics[width=0.68\linewidth]{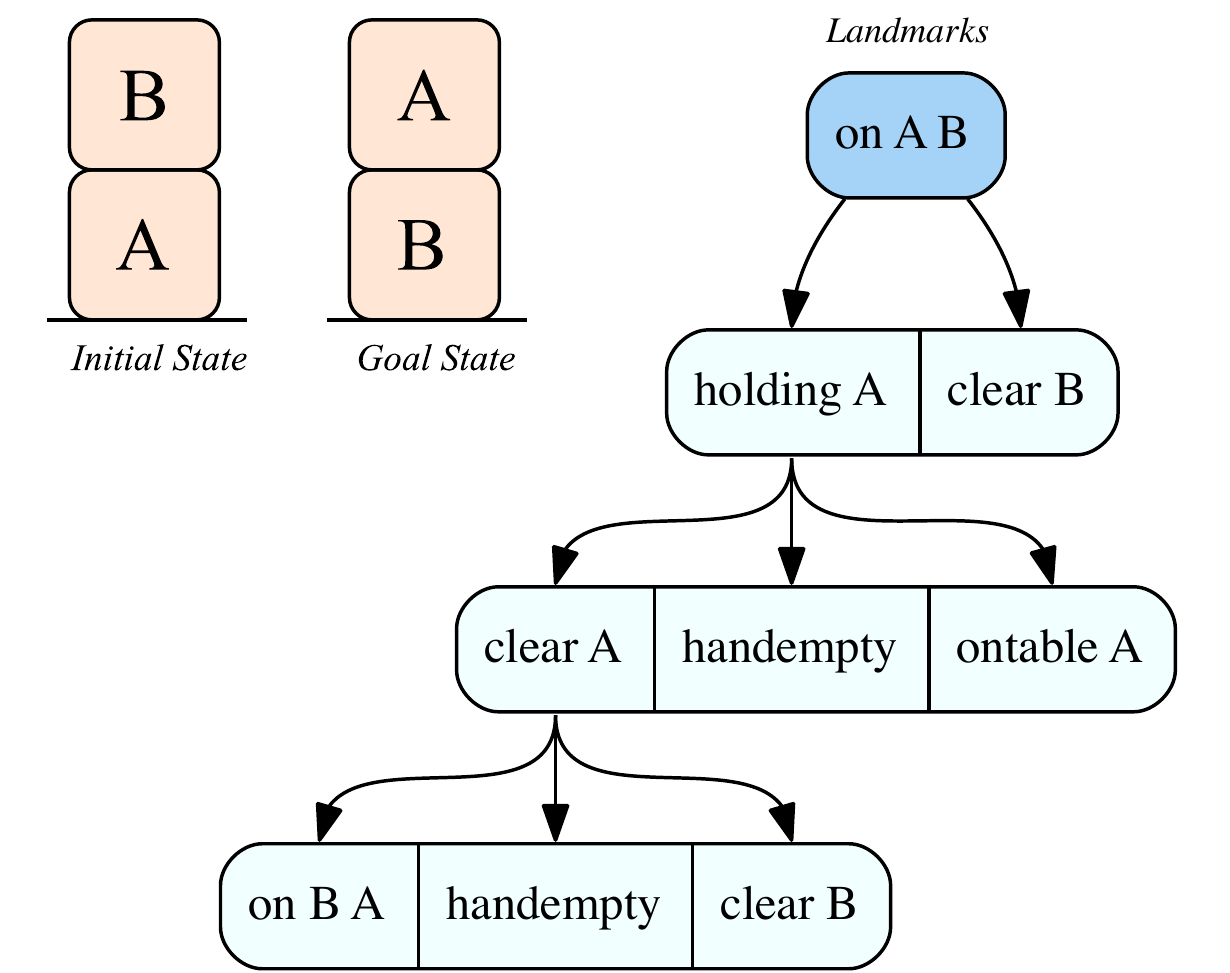}
  \caption{Ordered Landmarks for a \textsc{Blocks-World} problem instance. Connected boxes represent conjunctive landmarks.}
  \label{fig:Landmarks-BlocksWorld}
\end{figure}

From the concept of fact landmarks, Hoffmann~\etal~\cite{Hoffmann2004_OrderedLandmarks} introduce two types of landmarks as formulas: \textit{conjunctive} and \textit{disjunctive landmarks}. 
A \textit{conjunctive landmark} is a set of facts that must be true \textit{together} at some point in every valid plan to achieve a goal. 
A \textit{disjunctive landmark} is a set of facts such that \emph{at least one} of the facts must be true at some point in every valid plan to achieve a goal. Figure~\ref{fig:Landmarks-BlocksWorld} shows an example that illustrates a set of landmarks for a \textsc{Block-World}\footnote{\textsc{Blocks-World} is a \textit{Classical Planning} domain where a set of stackable blocks must be re-assembled on a table~\cite[Chapter~2, Page~50]{AutomatedPlanning_Book2011}.} problem instance. 
This example shows a set of conjunctive ordered landmarks (connected boxes) that must be true to achieve the goal state \pred{(on A B)}. 
For instance, to achieve the fact landmark \pred{(on A B)} which is also the goal state, the conjunctive landmark \pred{(and (holding A) (clear B))} must be true immediately before, and so on, as shown in Figure~\ref{fig:Landmarks-BlocksWorld}.

Whereas in \textit{Planning} the concept of landmarks is used to build heuristics~\cite{LandmarksRichter_2008} and planning algorithms~\cite{RichterLPG_2010}, in this thesis, much like our previous work~\cite{PereiraMeneguzzi_ECAI2016,PereiraNirMeneguzzi_AAAI2017}, we exploit the concept of landmarks to reason about agents' plan execution, and attempt to recognize the goals that such agent aims to achieve. 
Intuitively, we use landmarks as waypoints (or stepping stones) in order to monitor what an observed agent cannot avoid to achieve its goals. 
In Chapter~\ref{chapter:GR_IncompleteDomains}, we introduce new notions of landmarks and develop an algorithm for extracting landmarks over incomplete domain models.

%----------------------------------------------------------------------------------
\section{Goal and Plan Recognition as Planning}
\label{section:GoalRecognition}

\emph{Goal Recognition} is the task of recognizing which goal an observed agent aims to achieve by observing its interactions in an environment~\cite[Chapter~1, Page~3]{ActivityIntentPlanRecogition_Book2014}. \emph{Plan Recognition} can be seen as a superset of goal recognition (Figure~\ref{fig:Goal_Plan_Recognition}), namely, it is the task of recognizing which plan is being executed by an observed agent by observing its interactions in an environment~\cite[Chapter~3, Page~57]{ActivityIntentPlanRecogition_Book2014}.
In \textit{Goal and Plan Recognition}, such observed interactions (\idest, observations) are used as available evidence to recognize goals and plans. Observed interactions can be observed events performed by an agent in an environment, as well as actions (\exemp, a simple movement, cook, drive), and changing properties in an environment (\exemp, at home, at work, resting).
Approaches to goal and plan recognition are characterized according to the role that the observed agent performs during the recognition process in an environment~\cite{Armentano_AIJ_2007,ActivityIntentPlanRecogition_Book2014}, as follows.

\begin{itemize}
	\item \textbf{Intended Recognition} is the recognition process in which the observed agent is aware of the process of recognition. Therefore, in this kind of recognition process the observed agent usually cooperates with the process by notifying the recognizer about its interactions in the environment;
	\item \textbf{Keyhole Recognition\footnote{We note that the recognition approaches proposed in this thesis are limited to keyhole recognition, in which we make the assumption that the observed agents are either not aware that they are being observed or do not care about the recognition process.}} is defined as the recognition process in which the observed agent is unaware of the process of recognition, namely, the interactions performed by the observed agent are partially observable inputs to the recognition process; and
	\item \textbf{Obstructed Recognition} is the recognition in which the observed agent is aware of the process of recognition and obstructs purposely the process. In other words, the agent intentionally does not cooperate with the recognition process.
\end{itemize}	

\begin{figure}[h!]
  \centering
  \includegraphics[width=0.45\linewidth]{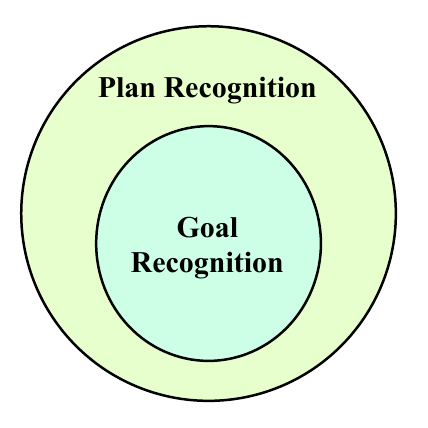}
  \caption{Plan Recognition as a superset of Goal Recognition~\cite{PattisonGoalRecognition_2010}.}
  \label{fig:Goal_Plan_Recognition}
\end{figure} 

To recognize goals and plans from agents' observations, we need a model that describes the agents' behavior and the environment, more specifically, a model that describes how observed agents act to achieve their goals in an environment. 
In the literature, the most commonly used types of models for goal and plan recognition are: \textit{plan--libraries} and \textit{planning domain definition}. 
A \textit{plan--library} can be seen as a domain-specific model that defines a pre-defined and static set of plans to achieve a set of goals, \idest, a know-how for achieving goals in a particular domain. 
While the use of plan--libraries has been proved to be computationally efficient~\cite{AvrahamiZilberbrandK_IJCAI2005,DoritGalAAAI07,Geib_PPR_AIJ2009,PR_Mirsky_2016,PR_EXP_Mirsky2017}, the process of modeling a plan--library requires substantial domain knowledge because the space of all possible plans for achieving goals must be encoded beforehand for every agent. 
Alternatively, goal and plan recognition approaches that use a model based on a \textit{planning domain definition} are called as \textit{Plan Recognition as Planning}~\cite{RamirezG_IJCAI2009,RamirezG_AAAI2010}.
The use of \textit{planning domain definition} for goal and plan recognition relaxes the need of modeling all plans to achieve a set of goals, by using only a planning domain-theory to describe facts and actions of the environment as domain knowledge. 
\textit{Plan Recognition as Planning} brings the process of goal and plan recognition closer to planning algorithms, allowing the use of automated planners~\cite{RamirezG_IJCAI2009,RamirezG_AAAI2010,Sohrabi_IJCAI2016} and planning concepts~\cite{PattisonGoalRecognition_2010,NASA_GoalRecognition_IJCAI2015,PereiraMeneguzzi_ECAI2016,PereiraNirMeneguzzi_AAAI2017} for recognizing goals and plans. 
Ram{\'{\i}}rez and Geffner~\cite{RamirezG_IJCAI2009} claim that plan recognition can be seen as planning in reverse, namely, in plan recognition we search for goals and plans that better explain the observations, in planning we search for a sequence of actions (\idest, a plan) to achieve a particular goal.

In this thesis, we define the problem of goal recognition as planning by following the formalism proposed by Ram{\'{\i}}rez and Geffner~in~\cite{RamirezG_IJCAI2009,RamirezG_AAAI2010}, as formalized Definition~\ref{def:GoalRecognition}.

\simbolo{$\mathcal{T}$}{Goal Recognition Problem as Planning}
\simbolo{$\mathcal{G}$}{Set of Hypothetical Candidate Goals}
\simbolo{$G^{*}$}{Correct Hidden Goal}
\simbolo{$Obs$}{Sequence of Observations}
\begin{definition}[\textbf{Goal Recognition Problem as Planning}]\label{def:GoalRecognition}
A goal recognition problem over a planning domain definition is a four-tuple $\mathcal{T} = \langle\Xi,\mathcal{I} ,\mathcal{G}, Obs\rangle$, where: 
\begin{itemize}
	\item $\Xi = \langle\mathcal{F}, \mathcal{A}\rangle$ is a planning domain definition; 
	\item $\mathcal{I}$ is the initial state; 
	\item $\mathcal{G}$ is the set of hypothetical candidate goals, which include a correct hidden goal $G^{*}$ (\idest, $G^{*} \in \mathcal{G}$); and 
	\item $Obs = \langle o_1, o_2, ..., o_n\rangle$ is an observation sequence of executed actions, with each observation $o_i \in \mathcal{A}$.
\end{itemize}
\end{definition}

We now formally define an observation sequence (Definition~\ref{def:ObservationSequence}) by following the formalism proposed by Ram{\'{\i}}rez and Geffner in~\cite{RamirezG_IJCAI2009,RamirezG_AAAI2010}, which defines an observation sequence as an action sequence.

\begin{definition}[\textbf{Observation Sequence}]\label{def:ObservationSequence}
An observation sequence $Obs = \langle o_1, o_2, ..., o_n\rangle$ is said to be satisfied by a plan $\pi = \langle a_1, a_2, ..., a_m\rangle$, if there is a monotonic function $f$ that maps the observation indices $j = 1, ..., n$ into action indices $i = 1, ..., n$, such that $a_{f(j)} = o_{j}$.
\end{definition}

Thus, the ideal solution for a goal recognition problem as planning is finding the single correct hidden goal $G^{*} \in \mathcal{G}$ that the observation sequence $Obs$ of a plan execution achieves. 
Most approaches to goal and plan recognition return either a probability distribution over the goals~\cite{RamirezG_IJCAI2009,RamirezG_AAAI2010,Sohrabi_IJCAI2016,NASA_GoalRecognition_IJCAI2015}, or a score associated to the set of hypothetical candidate goals~\cite{PereiraMeneguzzi_ECAI2016,PereiraNirMeneguzzi_AAAI2017}. In Chapters~\ref{chapter:GR_IncompleteDomains}~and~\ref{chapter:GR_NominalModels}, we explain how our approaches to goal recognition compute scores and probability distributions for estimating the correct hidden goal from the observation sequence.
As an example of how the goal recognition process works, consider the Example~\ref{exemp:GoalRecognition}, as follows.

\begin{example}\label{exemp:GoalRecognition}
To exemplify the goal recognition process, let us consider the \textsc{Blocks-World} example in {\normalfont Figure~\ref{fig:GoalRecognition-BlocksWorld}}. 
The initial state represents an initial configuration of stackable blocks, while the set of candidate goals is composed by the following stacked ``words'': {\normalfont \pred{RED}, \pred{BED}, \textit{and} \pred{SAD}}. 
Consider an observation sequence for a hidden goal {\normalfont \pred{RED}} consisting of the following action sequence: 
{\normalfont $[$\pred{(unstack D B)}, \pred{(putdown D)}, \pred{(unstack E A)}, \pred{(stack E D)}, \pred{(pickup R)}, \pred{(stack R E)}$]$}. 
By following the full observation sequence, we can easily infer that the hidden goal is indeed {\normalfont \pred{RED}}. 
However, if the we cannot observe the action {\normalfont \pred{(stack R E)}}, it is not trivial to infer that {\normalfont \pred{RED}} is indeed the goal the observation sequence aims to achieve. 
Thus, we could infer that more than one candidate goal could be pursuit by the observations.
\end{example}

\begin{figure}[h!]
  \centering
  \includegraphics[width=0.7\linewidth]{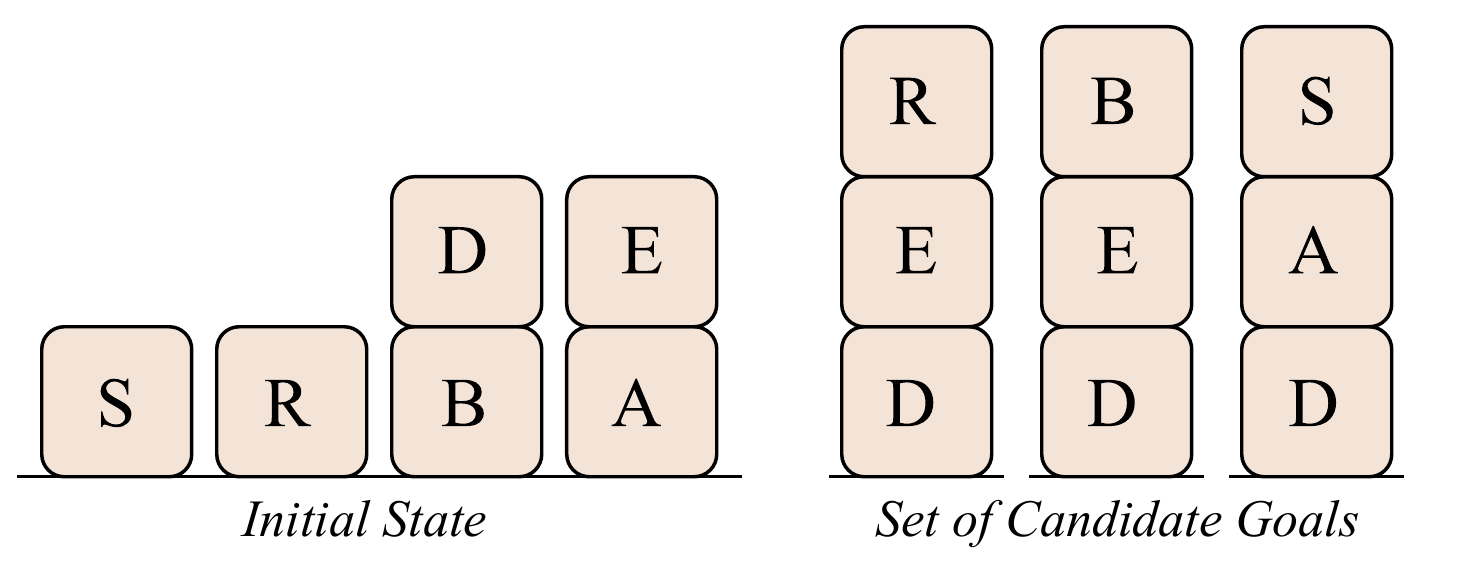}
  \caption{Goal Recognition example for \textsc{Blocks-World} domain.}
  \label{fig:GoalRecognition-BlocksWorld}
\end{figure}

The observation sequence in Example~\ref{exemp:GoalRecognition} represents a full observation sequence, \idest, all actions of an agent's plan are observed. 
In a partial observation sequence, we observe only a sub-sequence of actions of a plan that achieves a particular goal because some actions are missing or obfuscated. 
A noisy observation sequence contains one or more actions (or a set of facts) that might not be part of a plan that achieves a particular goal, \exemp, when a sensor fails and generates abnormal or spurious readings~\cite{Sohrabi_IJCAI2016}. 
In this thesis, we deal with partial and full observation sequences during the goal recognition process. We formally define a missing observation in Definition~\ref{def:MissingObservation}. Example~\ref{exemp:MissingObservationSequence} shows an example of a partial observation sequence containing misses observed actions.

%\frm[inline]{Contains a missing observation is a logical contradiction. I'm not sure I understand this definition..}
\begin{definition}[\textbf{Missing Observation Sequence}]\label{def:MissingObservation}
Let $\Pi = \langle \Xi, \mathcal{I}, G\rangle$ be a planning instance, $\pi$ a valid plan that achieves $G$ from $\mathcal{I}$, and $Obs$ an observation sequence that satisfies $\pi$. 
An observation sequence $Obs$ misses observations (is a partial observation sequence) with respect to the plan $\pi$ that achieves the goal $G$ from $\mathcal{I}$, if at least one of the observations $o \in Obs$ is part of a strict subset of ordered actions $\pi$.
\end{definition}

\begin{example}\label{exemp:MissingObservationSequence}
Let us consider that a valid plan to achieve a goal $G$ is $\pi = [a,b,c,d,e]$. Consider the following observation sequences $O_{m1}$, $O_{m2}$, and $O_{m3}$:
\begin{itemize}
	\item $O_{m1} = [a,d]$;
	\item $O_{m2} = [b,e]$; and
	\item $O_{m3} = [d,a,c]$
\end{itemize} 
Observation sequences $O_{m1}$ and $O_{m2}$ satisfy {\normalfont Definition~\ref{def:MissingObservation}}, and therefore, they are partial observation sequences and contain missing observed actions. 
$O_{m3}$ is not a partial observation sequence because it does not satisfy {\normalfont Definition~\ref{def:MissingObservation}} as the observation sequence $[d,a,c]$ is not a strict subset of ordered actions of the plan $\pi$. 
\end{example}

%----------------------------------------------------------------------------------
% \newpage
\section{Optimal Control}\label{section:OptimalControl}

In the fields of \textit{Engineering} and \textit{Mathematics}, \textit{Control Theory} is a subfield of study that deals with the behavior of dynamical systems~\cite{Bertsekas_DP_17}.
A \textit{dynamical system} can be seen as a mathematical model that aims to describe the behavior of a system. In a sense, a dynamical system is a system that describes mathematically the changes that occur over time in a physical system with geometrical space. This geometrical space is a set of possible states of the system, whereas the dynamics can be formalized as a function that transforms a state into another state. Thus, the aiming of the \textit{Control Theory} is to control a dynamical system such that its output follows a desired value, which may be a fixed or changing value.

To model the range of possible agent behavior in \textit{Control}, we use one of the branches of \textit{Control Theory}, namely, \textit{Optimal Control}. The objective of \textit{Optimal Control} is either to minimize or maximize a particular measure regarding the behavior of a dynamical system over time~\cite{Bertsekas_DP_17,Sutton_2018_RLI}. In this thesis, we model agent behavior using \textit{Finite-Horizon Optimal Control} problems, denoted as FHOC.

%######################################################################
\subsection{Finite Horizon Optimal Control Problems}\label{section:Background:FHOC}

\sigla{FHOC}{Finite-Horizon Optimal Control}
We formalize FHOC problems by following the terminology proposed by Bertsekas in~\cite{Bertsekas_DP_17}, incorporating and combining some elements typically used by the literature
on \textit{Control}~\cite{borrelli:17:predictive} and \textit{Planning}~\cite{Bonet_Planning_13} to account for constraints and goals\footnote{Referred to as \emph{target regions} in \textit{Control Theory}.}.
Transitions between states are described by a \textit{stationary}, \textit{discrete--time} dynamical system

% \frm[inline]{Refer to the equation at some point, otherwise this will look like it's just filler. You use it (albeit in reverse) in Definition~\ref{def:ActualModel}}
\begin{align}
\label{eq:dynamics}
x_{k+1} = f(x_k, u_k, w_k)
\end{align}

\noindent where for each time point $k \in [0,N]$, $x_k$ is the \textit{state}, $u_k$ is the \textit{control input} and
$w_k$ is a \textit{random variable} with a probability distribution that does not depend on past $w_j$, $j<k$.
For now, we make no further assumptions on the specific way states, inputs and perturbations interact. States $x_k$, controls
$u_k$, and disturbances $w_k$ are required to be part of spaces $S \subset \mathbb{R}^d$, $C
\subset \mathbb{R}^p$, and $D \subseteq \mathbb{R}^{d+p}$.
Controls $u_k$ are further required to belong to the set $U(x_k) \subset C$, for each state $x_k$ and time step $k$.
We note that the latter accounts for both the notion of \textit{preconditions} and \textit{bounds} on inputs.
Observed agents seek to transform initial states $x_0$ into states $x_N$ with specific properties. 
These properties are given as logical formulas over the components of states $x_k$, and the set of states $S_{G} \subseteq S$ are those where the desired property $G$, or goal, holds. The preferences of observed agents to pursue specific trajectories are accounted for with cost functions of the form

\begin{align}
\label{eq:objective_fhoc}
J(x_0) =  \mathbb{E} \{ g(x_{N}) + \sum_{k=0}^{N-1} g(x_k,u_k,w_k) \}
\end{align}

\noindent $g(x_{N})$ is the \textit{terminal cost}, $g(x_k,u_k,w_k)$ is the \textit{stage cost}, and $\mathbb{E}$ is the expectation operator with respect to the random variable $w_k$.
Thus, we define FHOC problems as an optimization problem whose solutions describe the range of possible optimal behaviors of observed agents

\begin{align}
\label{eq:rhc_objective}
& \min_{\pi \in \Pi}  \bigl\{ J_{\pi} \bigl( x_{0} \bigr) \bigr\}& & \\
\nonumber
subject\  to & \\
& u_k = \mu_k(x_k) & \\
\label{eq:rhc_transition}
& x_{k+1} =\ f(x_{k},u_{k},w_{k}) & \\
\label{eq:constraints}
& u_{k} \in\ U(x_k),\ x_{k} \in\ S,\ x_{N} \in\ S_G &
\end{align}

% \frm[inline]{You defined $\pi$ above as a sequential plan, and you are now defining it as something else. You need to have a definition for a policy. It's more or less fine to use the same symbol for both plans and policies, but you need to be upfront about the abuses of notation with something like that as a footnote: ``We use the $\pi$ symbol as both plans and policies throughout the text and clarify them when this is not obvious from context.'' \\
% Remember that $\pi$ being a policy implies that $\pi$ is either a total function of the states into the actions (i.e. $\pi: S \mapsto A$ such that you refer to $\pi(s)$), or a probability distribution of actions conditioned on states (i.e. $\pi: S \times A \mapsto \mathbb{R}$, or $\pi: S \times A \mapsto [0,1]$, and choices are $\arg\max_{a}\pi(s \mid a)$)}

\noindent where ${\cal I}$, the \textit{initial state}, is an arbitrary element of the set of states $S$. Solutions to
Equations~\ref{eq:rhc_objective}--\ref{eq:constraints} are \textit{policies} $\pi$
\begin{align}
\label{eq:policy}
\pi = \{ \mu_0, \mu_1, \ldots, \mu_k, \ldots, \mu_{N-1} \}
\end{align}
\noindent and $\mu_k$ is a function mapping states $x_k$ into controls $u_k$ $\in$ $C$. When
$\pi$ is such that $\mu_i$ $=$ $\mu_j$ for every $i$, $j$ $\in$ $[0,N]$, we say $\pi$ is
\textit{stationary}. We note the abuse of notation, and use the $\pi$ symbol to represent both \textit{plans} and \textit{policies} throughout the text and clarify them when this is not obvious from context. We also note that terminal constraints $x_N \in S_G$ can be dropped,
replacing them by terms in $g(x_k,u_k,w_k)$ that encode some measure of distance to $S_G$.
Costs $g(x_{N})$ are typically set to $0$ when terminal constraints are enforced, yet
this is a convention, and establishing preferences for specific states in $S_G$ over others is perfectly possible.

%######################################################################
\subsection{Actual and Nominal Control Models}

System Identification is the task of building and approximating mathematical models of dynamical systems from past collected data~\cite{ljung1998system} (\exemp, observed state transitions) and prior system knowledge. 
Existing work on the \textit{Control Theory} literature~\cite{Sirdi_2007_CSC,Kong_Tomi_NominalM_2013} usually refers to approximate models as \textit{nominal models}, whereas correct true (idealized) models are referred as \textit{actual models}. 
We formally define \textit{actual} and \textit{nominal models} as follows, in Definitions~\ref{def:ActualModel}~and~\ref{def:NominalModel}, respectively.

% \frm[inline]{Consider moving this subsection to closer in the thesis where you use it (since it's only one chapter, it might be better there). I don't feel strongly either way now, but might change my mind when I get to the right chapters.}
% Lets talk about it during our meeting.

\begin{definition}[\textbf{Actual Model}]\label{def:ActualModel}
An actual model is a model in which the transition function ${x}_{k+1} = {f}(x_k, u_k, w_k)$ is known and well defined.
\end{definition}

\begin{definition}[\textbf{Nominal Model}]\label{def:NominalModel}
A nominal model is a model in which the transition function $\hat{x}_{k+1} = \hat{f}(x_k, u_k, w_k)$ is approximate and acquired based on observed data. 
\end{definition}

In this thesis, we study the implications of using \textit{nominal models} for goal recognition over continuous actions and state spaces. In Chapter~\ref{chapter:GR_NominalModels}, we show how we acquire \textit{nominal models} from past observed state transitions by leveraging existing work that combines \textit{Deep Learning} and \textit{Planning}~\cite{SayWZS:ijcai17,WuSS:nips17}.

% %----------------------------------------------------------------------------------

%% file: cap3_GR_IncompleteDomains.tex
%!TEX root = ppgcc-thesis.tex
%----------------------------------------------------------------------------------
% Chapter: Heuristic Goal Recognition over Incomplete Domain Models
%
% + ?
%
%----------------------------------------------------------------------------------

\chapter{Goal Recognition over Incomplete Domain Models}\label{chapter:GR_IncompleteDomains}

Traditional work on \textit{Goal and Plan Recognition as Planning} relies on complete and correct domain models, relaxing the need to deal with incomplete domain information in relation to methods based on plan--libraries.
In this chapter, we present novel heuristic approaches to goal recognition that explicitly deal with \textit{incomplete domain information over discrete models}. 
In Section~\ref{subsection:GR_IncompleteDomains_Formalism}, we introduce the problem of goal recognition over \textit{incomplete domain models}. 
After, in Section~\ref{section:ExtractingLandmarksInIncompleteDomains}, we present new notions of \textit{landmarks} for \textit{incomplete domain models}, and how we modify an algorithm from the literature to extract such landmarks from incomplete domain information. 
Then, in Section~\ref{section:HeuristicGoalRecognitionIncompleteDomains}, we build goal recognition heuristics based on these new notions of landmarks over \textit{incomplete domain models}. 
Finally, in Section~\ref{section:GR_IncompleteDomains_ExperimentsEvaluation}, we show empirically that our heuristic approaches effectively deal with the incomplete part of the domain model to substantially improve recognition accuracy against a baseline.

%------------------------------------------------------
\section{Problem Formulation}\label{subsection:GR_IncompleteDomains_Formalism}

Most planning--based approaches to goal and plan recognition assume that a complete domain model is available to perform the recognition process~\cite{RamirezG_IJCAI2009,RamirezG_AAAI2010,GoalRecognitionDesign_Keren2014,NASA_GoalRecognition_IJCAI2015,Sohrabi_IJCAI2016,PereiraMeneguzzi_ECAI2016,PereiraNirMeneguzzi_AAAI2017}, relying on conventional models in which the domain information is considered to be \textit{known}, \textit{correct}, and \textit{well-defined}. 
In this thesis, we propose the problem of goal recognition over \textit{incomplete domain models}. 
Here, we assume that the \textit{observer} (goal recognizer) uses an \textit{incomplete domain model} to perform the recognition task, while the \textit{observed agent} is planning and acting in the environment with a \textit{complete} domain model. Besides that, like the planning approaches over incomplete domain models~\cite{WeberBryce_ICAPS_2011,Nguyen_AIJ_2017}~(recall from Section~\ref{section:IncompletePlanning}), we consider that the goal recognizer reasons about possible plans with incomplete actions (observations) by assuming that they succeed under the \textit{most optimistic} conditions, \idest, ignoring all \textit{possible} preconditions and delete effects (known and possible), and all effects are assumed to occur (known and possible).

To account for such incompleteness in the domain model, the incomplete domain model available to the observer contains annotations specifying \textit{possible preconditions} and \textit{effects} of some actions in the incomplete domain definition, much like the incomplete domain models from previous planning approaches~\cite{WeberBryce_ICAPS_2011,Nguyen_AIJ_2017}.  Figure~\ref{fig:GoalRecognition_IncompleteDomains} shows an overview of how we define the problem of goal recognition over \textit{incomplete domain models} from the perspective of both the observer and the observed agent.

\begin{figure}[h!]
  \centering
  \includegraphics[width=0.85\linewidth]{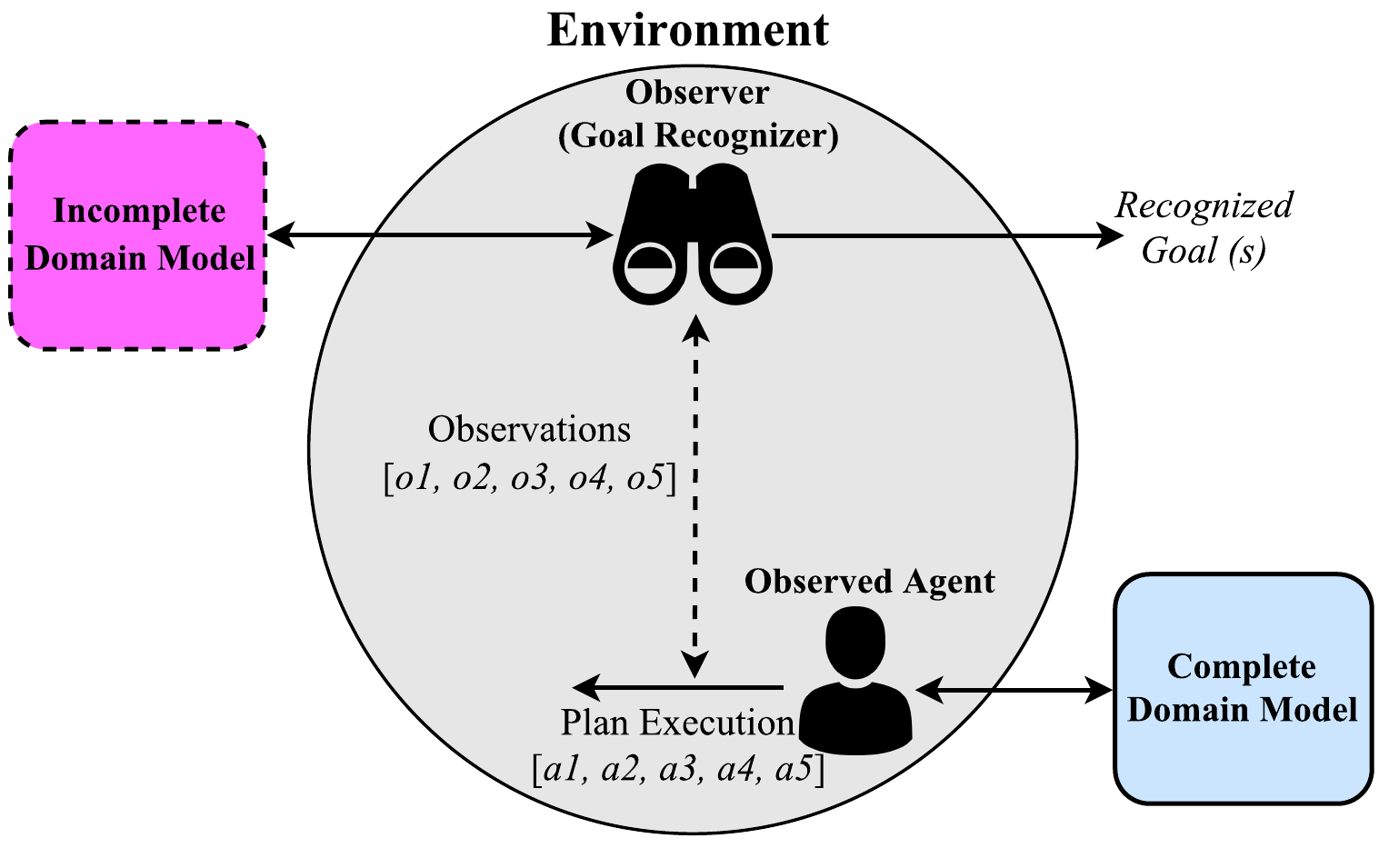}
  \caption{Problem overview to goal recognition in incomplete domain models.}
  \label{fig:GoalRecognition_IncompleteDomains}
\end{figure}

By combining the various notions of planning in incomplete domains~\cite{WeberBryce_ICAPS_2011,Nguyen_AIJ_2017} and observation sequence (Definition~\ref{def:ObservationSequence}), we now formally define the \textit{goal recognition problem} over an \textit{incomplete planning domain} in Definition~\ref{def:GoalRecognitionIncompleteDomains}, following the formalism of Ram{\'{\i}}rez and Geffner~\cite{RamirezG_IJCAI2009,RamirezG_AAAI2010} (Definition~\ref{def:GoalRecognition}).

\simbolo{$\incp{T}$}{Goal Recognition Problem over an Incomplete Domain}
\begin{definition}[\textbf{Goal Recognition Problem over an Incomplete Domain Model}]\label{def:GoalRecognitionIncompleteDomains}
A goal recognition problem over an incomplete domain model is a quintuple $\incp{T} = \langle \incp{D}, Z, \mathcal{I}, \mathcal{G}, Obs \rangle$, 
where: 
\begin{itemize}
	\item $\incp{D} = \langle \mathcal{R}, \widetilde{\mathcal{O}} \rangle$ is an incomplete domain model (with possible preconditions and effects). 
	$Z$ is the set of typed objects in the environment, in which $\mathcal{F}$ is the set instantiated predicates (\idest, a finite set of facts) from $\mathcal{R}$ with objects from $Z$, and $\incp{\mathcal{A}}$ is the set of incomplete instantiated actions from $\inc{\mathcal{O}}$ with objects from $Z$;
	\item $\mathcal{I} \in \mathcal{F}$ is the initial state;
	\item $\mathcal{G}$ is the set of hypothetical candidate goals, which include a correct hidden goal $G^{*}$ (\idest, $G^{*}$ $\in$ $\mathcal{G}$); and
	\item $Obs = \langle o_1, o_2, ..., o_n\rangle$ is an observation sequence of executed actions $o_i \in \incp{\mathcal{A}}$. $Obs$ corresponds to a sequence of actions (\idest, a plan) that achieves the correct hidden goal $G^{*}$ from the initial state $\mathcal{I}$ over a single complete model $\mathcal{D}^{*} \in \langle\langle \widetilde{\mathcal{D}} \rangle\rangle$ that is known to the observed agent, but not to the recognizer.
	% \frm[inline]{I'm not sure which formal structure you refer to here: $\widetilde{\mathcal{D}}$ is an incomplete domain. Perhaps you should spell out in this definition that there exists a complete model $\mathcal{D}$ that generates the observation, but which the observer is not aware.}
	% \rfp[inline]{I modified the problem formulation, I think that now it is consistent with your comment above, right?!}
\end{itemize}
\end{definition}

Ideally, a solution for a goal recognition problem in incomplete domain models $\incp{T}$ is to find the single correct hidden goal $G^{*} \in \mathcal{G}$ that the observation sequence $Obs$ of a plan execution achieves. 
As most goal recognition approaches, observations consist of the action signatures of the underlying plan\footnote{Our approaches are not limited to using just actions as observations, and it can also deal with logical facts as observations, much like the approach from~\cite{Sohrabi_IJCAI2016}.}, more specifically, we observe incomplete actions with possible precondition and effects, in which some of the preconditions might be required and some effects might change the environment.
A full (or complete) observation sequence contains all of the action signatures of the plan executed by the observed agent, whereas a partial observation sequence contains only a sub-sequence of actions of a plan, and thus misses some of the actions actually executed in the environment (Definition~\ref{def:MissingObservation}). In summary, the problem we address in this thesis is \textit{keyhole goal recognition} under partial observability~\cite[Chapter 1]{ActivityIntentPlanRecogition_Book2014}.

%------------------------------------------------------
\section{Extracting Landmarks in Incomplete Domain Models}
\label{section:ExtractingLandmarksInIncompleteDomains}

% Landmarks are facts (or actions) that must be achieved (or executed) at some point in all valid plans to achieve a goal from an initial state~\cite{Hoffmann2004_OrderedLandmarks}. The concept of landmarks is often used to build heuristics~\cite{LandmarksRichter_2008,Pommerening_LMCUT_2012} for planning algorithms~\cite{RichterLPG_2010,HelmertFastDownward_2011}. However, in the planning literature, most landmark--based heuristics extract landmarks from complete and correct domain models.
Recall from Section~\ref{section:Landmarks} that landmarks are necessary conditions to achieve a goal from an initial state in a given planning problem (Definitions~\ref{def:FactLandmark}~and~\ref{def:ActionLandmark}). In the \textit{Planning} literature, most landmark--based heuristics extract landmarks from complete and correct domain models.
In this thesis, we introduce new notions of fact landmarks in \textit{incomplete domain models}, and extend the landmark extraction algorithm proposed by Hoffmann~\etal~in~\cite{Hoffmann2004_OrderedLandmarks} to extract landmarks from incomplete STRIPS domain models.

To represent landmarks and their ordering, the algorithm of Hoffmann~\etal~\cite{Hoffmann2004_OrderedLandmarks} uses a tree in which nodes represent landmarks and edges represent necessary prerequisites between landmarks. 
Each node in the tree represents a conjunction of facts that must be true simultaneously at some point during plan execution, and the root node is a landmark representing the goal state. 
For extracting landmarks, this landmark extraction algorithm~\cite{Hoffmann2004_OrderedLandmarks} uses a Relaxed Planning Graph (\textsc{RPG})~\cite{ReachabilityBryceK_2007}, which is a leveled graph that ignores the delete-list effects of all actions, thus containing no mutex relations~\cite{FFHoffmann_2001}, as formalized in Definition~\ref{definition:RPG}. Algorithm~\ref{alg:BuildRPG} shows a pseudo-code to build a full RPG structure from a STRIPS planning problem $\mathcal{P}$.

\simbolo{\textsc{RPG}}{Relaxed Planning Graph.}
\begin{definition} [\textbf{Relaxed Planning Graph}]\label{definition:RPG}
An RPG is a leveled graph structure that ignores the delete-list effects of all actions, and in this way, there are no mutex relation in this graph structure. Thus, as a leveled graph, the graph levels are structured as follows: $\mathcal{F}_{0}$, $\mathcal{A}_{0}$, $\mathcal{F}_{1}$, $\mathcal{A}_{1}$, ..., $\mathcal{F}_{m-1}$, $\mathcal{A}_{m-1}$, $\mathcal{F}_{m}$ of fact sets $\mathcal{F}_{i}$ (fact levels) and action sets $\mathcal{A}_{i}$ (action levels). 
Fact level $\mathcal{F}_{0}$ contains the facts that are true in the initial state, the action level $\mathcal{A}_{0}$ contains those actions whose preconditions are reached from $\mathcal{F}_{0}$, subsequently, $\mathcal{F}_{1}$ contains $\mathcal{F}_{0}$ plus the add effects of the actions in $\mathcal{A}_{0}$. The graph is built until reaching the goal state in the last fact level, or when the building process fails, if at some point before reaching the goals no new facts are inserted in the graph. Thus, the RPG is composed of $\mathcal{F}_{i} \subseteq \mathcal{F}_{i+1}$ and $\mathcal{A}_{i} \subseteq \mathcal{A}_{i+1}$ for all $i$.
\end{definition}

\floatname{algorithm}{Algorithm}
\begin{algorithm}
    \caption{Build a Relaxed Planning Graph (\textsc{RPG}).}
    \textbf{Input:} $\mathcal{P} = \langle \mathcal{F}, \mathcal{A}, \mathcal{I}, G \rangle$ STRIPS \textit{planning problem}, \textit{in which} $\mathcal{F}$ \textit{is a finite set of facts and} $\mathcal{A}$ \textit{is a finite set of actions}, $\mathcal{I}$ \textit{is the initial state}, \textit{and} $G$ \textit{is the goal state}.
    \\\textbf{Output:} \textsc{RPG} \textit{relaxed planning graph}.
	\label{alg:BuildRPG}
    \begin{algorithmic}[1]
        \Function{BuildFullRPG}{$\mathcal{F}, \mathcal{A}, \mathcal{I}, G$}
			\State $i \gets 0$ 
			\State \textsc{RPG.$\mathcal{F}_{0}$} $\gets \mathcal{I}$ \Comment{\textit{Initialize the} \textsc{RPG} \textit{with facts of the initial state.}}
			\While{$G$ $\nsubseteq$ \textsc{RPG.$\mathcal{F}_{i}$}}
				\State \textsc{RPG.$\mathcal{A}_{i}$} $\gets$ all action $\mathit{a} \in \mathcal{A}$ such that $\mathit{pre}(\mathit{a}) \in$ \textsc{RPG.$\mathcal{F}_{i}$}
				\State \textsc{RPG.$\mathcal{F}_{i+1}$} $\gets$ \textsc{RPG.$\mathcal{F}_{i}$} $\cup$ $(\mathit{eff}^{+}(a), \forall a \in$ \textsc{RPG.$\mathcal{A}_{i})$}
				\If{\textsc{RPG.$\mathcal{F}_{i+1}$} $\equiv$ \textsc{RPG.$\mathcal{F}_{i}$}}
					\State \textbf{return} $\textsc{Failure To Build RPG}$ \Comment{\textit{The algorithm fails whether at some point before reaching the facts of the goal no new fact level is added in the graph.}}
				\EndIf
        			\State $i \gets i + 1$
        		\EndWhile
            \State \textbf{return} \textsc{RPG}
        \EndFunction
    \end{algorithmic}
\end{algorithm}

After building the \textsc{RPG}, the algorithm of Hoffmann~\etal~\cite{Hoffmann2004_OrderedLandmarks} extracts a set of \textit{landmark candidates} by back-chaining from the \textsc{RPG} level in which all facts of the goal state $G$ are possible, and, for each sub-goal (fact) $g$ in $G$, it checks which facts must be true until the first level of the \textsc{RPG}. 
For example, if a fact $LB$ is a landmark and all actions that achieve $LB$ share $LA$ as a precondition, then $LA$ is a landmark candidate. 
To confirm that a landmark candidate is indeed a necessary fact to achieve the goal state $G$, and thus an actual landmark, the algorithm builds a new \textsc{RPG} removing all actions that achieve the landmark candidate and checks the solvability over this modified problem. 
If the modified problem is \textit{unsolvable}, then such landmark candidate is a necessary fact to achieve $G$, and therefore, a fact landmark. 
This means that the actions that achieve the landmark candidate are necessary to solve the original planning problem. According to Blum and Furst~\cite{BlumFastPlanning_95}, deciding the solvability of a relaxed planning problem using an \textsc{RPG} structure can be done in polynomial time.

%###############################################################
\subsection{Optimistic Relaxed Planning Graph}

\simbolo{\textsc{ORPG}}{Optimistic Relaxed Planning Graph.}
In order to extract landmarks in \textit{incomplete domain models}, we develop and formalize an Optimistic Relaxed Planning Graph (\textsc{ORPG}). 
An \textsc{ORPG} is leveled graph that deals with incomplete domain models by assuming the most \textit{optimistic} conditions.
Thus, besides ignoring the delete-effects of all actions, this graph also ignores possible preconditions and possible delete-effects, whereas we use all possible and known add effects to build the optimistic graph. We formally define an ORPG in Definition~\ref{definition:ORPG}.

\begin{definition} [\textbf{Optimistic Relaxed Planning Graph}]\label{definition:ORPG}
An ORPG is a leveled graph and similar to an RPG structure, however, it considers the optimistic assumption, in which possible preconditions do not need to be satisfied in a state, possible add effects are always assumed to occur, and delete effects are always ignored. As a leveled graph, the graph levels in the ORPG are structured as follows: $\incp{\mathcal{F}}_{0}$, $\incp{\mathcal{A}}_{0}$, $\incp{\mathcal{F}}_{1}$, $\incp{\mathcal{A}}_{1}$, ..., $\incp{\mathcal{F}}_{m-1}$, $\incp{\mathcal{A}}_{m-1}$, $\incp{\mathcal{F}}_{m}$ of fact sets $\incp{\mathcal{F}}_{i}$ (fact levels) and action sets $\incp{\mathcal{A}}_{i}$ (incomplete action levels). 
Fact level $\incp{\mathcal{F}}_{0}$ contains the facts that are true in the initial state, while the incomplete action level $\incp{\mathcal{A}}_{0}$ contains those actions whose preconditions are reached from $\incp{\mathcal{F}}_{0}$ (possible preconditions are ignored). $\incp{\mathcal{F}}_{1}$ contains $\incp{\mathcal{F}}_{0}$, plus the facts in the add effects and possible add effects of the incomplete actions in $\incp{\mathcal{A}}_{0}$, and so on until reaching the goal state in the last fact level, or when the building process fails, if at some point before reaching the goals no new facts are inserted in the graph. 
Thus, the ORPG structure is composed of $\incp{\mathcal{F}}_{i} \subseteq \incp{\mathcal{F}}_{i+1}$
and $\incp{\mathcal{A}}_{i} \subseteq \incp{\mathcal{A}}_{i+1}$ for all $i$.
\end{definition}

In practice, to build an \textsc{ORPG}, we must modify four steps in the original \textsc{RPG} algorithm (formally described in Algorithm~\ref{alg:BuildRPG}) resulting in a function called \textsc{BuildFullORPG}. Below we show the necessary modifications that must be done in Algorithm~\ref{alg:BuildRPG} to build a full \textsc{ORPG} for a given incomplete STRIPS planning problem. 

\begin{enumerate}
	\item The input must be an incomplete STRIPS planning problem $\incp{P} = \langle \mathcal{F}, \incp{\mathcal{A}}, \mathcal{I}, G \rangle$;
	\item Build (Line 3) and output an \textsc{ORPG} instead of an \textsc{RPG};
	\item In Line 5: create a new action level by selecting all incomplete instantiated action $\inc{a} \in \incp{\mathcal{A}}$ such that $\mathit{pre}(\inc{a}) \in$ \textsc{ORPG.$\incp{\mathcal{F}}_{i}$}; and
	\item In Line 6: create a new fact level with \textsc{ORPG.$\incp{\mathcal{F}}_{i}$} $\cup$ $(\mathit{eff}^{+}(\inc{a}) \cup \widetilde{\mathit{eff}}^{+}(\inc{a}), \forall \inc{a} \in$ \textsc{ORPG.$\incp{\mathcal{A}}_{i})$};
\end{enumerate}

Replacing an \textsc{RPG} for an \textsc{ORPG} allows us to detect new types of landmarks over incomplete domain models, and thus, extract more landmarks than the original algorithm of Hoffmann~\etal~\cite{Hoffmann2004_OrderedLandmarks}. Next, we introduce new notions of landmarks and a novel landmark extraction algorithm that accounts for the incompleteness in incomplete STRIPS domain models.

%###############################################################
\subsection{Landmark Extraction Algorithm for Incomplete Domains}

For extracting landmarks in incomplete domains models, we adapt the extraction algorithm developed by Hoffman~\etal~\cite{Hoffmann2004_OrderedLandmarks} to extract \textit{definite} and \textit{possible} landmarks from incomplete planning problems by building an \textsc{ORPG} instead of the original \textsc{RPG} and making other modifications. We formally define the notions of \textit{definite} and \textit{possible} landmarks in Definitions~\ref{def:DefiniteLandmark} and~\ref{def:PossibleLandmark}, respectively.

\begin{definition}[\textbf{Definite Landmark}]\label{def:DefiniteLandmark}
A definite landmark $L_{D}$ is a fact (landmark) that is extracted from a known add effect $\mathit{eff}^{+}(a)$ of an achiever\footnote{An achiever is an action at the level before a candidate landmark in the RPG that can be used to achieve this candidate landmark.} $a$ (action) in the ORPG.
\end{definition}

\begin{definition}[\textbf{Possible Landmark}]\label{def:PossibleLandmark}
A possible landmark $L_{P}$ is a fact (landmark) that is extracted from a possible add effect $\widetilde{\mathit{eff}}^{+}(a)$ of an achiever $a$ (action) in the ORPG and is such that $L_{P} \cap L_{D} = \emptyset$.
\end{definition} 

Unlike the algorithm proposed by Hoffman~\etal~\cite{Hoffmann2004_OrderedLandmarks} that besides extracting the set of landmarks also approximates the order that they must be achieved, our extraction algorithm only extracts the set of landmarks (\idest, a set of individual facts, not a set of formulas, as defined in Definition~\ref{def:FactLandmark}) for a given incomplete planning problem, avoiding the need to compute the order relation between the extracted landmarks.
The landmark extraction algorithm for incomplete domain models we develop is formally described in Algorithm~\ref{alg:LandmarksExtractionIncompleteDomains}. 
Our landmark extraction algorithm takes as input an incomplete domain definition $\incp{D} = \langle \mathcal{R}, \widetilde{\mathcal{O}} \rangle$, a set of typed objects $Z$, in which  $\mathcal{F}$ is the set of facts instantiated from $\mathcal{R}$ and $Z$, $\incp{\mathcal{A}}$ is the set of incomplete actions instantiated from $\widetilde{\mathcal{O}}$ and $Z$, $\mathcal{I}$ is the initial state, and $G$ is the goal state. Initially, the algorithm builds a full \textsc{ORPG} structure from an initial state to goal state (Line~\ref{alg:LandmarksExtraction:BuildORPG}). 
If the goal state is not reachable in the \textsc{ORPG}, then landmarks are not extracted, whereas if the goal state is reachable, a set of landmark candidates $C$ is initialized with the facts from the goal state (Line~\ref{alg:LandmarksExtraction:Candidates}). 
Subsequently, the algorithm iterates over the set of landmark candidates in $C$ (Line~\ref{alg:LandmarksExtraction:While}), and, for each landmark $l$, the algorithm identifies fact landmarks that must be true immediately before $l$. This iteration ends when the set of landmark candidates $C$ is empty. 
Then, in Line~\ref{alg:LandmarksExtraction:ActionsORPG}, from landmark candidate $l$, the set of actions $A'$ is extracted from the \textsc{ORPG}, comprising all actions in the \textsc{ORPG} at the action level immediately before $l$. 
In other words, $A'$ represents those actions that achieve the facts in the level in which $l$ is in the \textsc{ORPG}, these actions are called \textit{achievers}. 
In Line~\ref{alg:LandmarksExtraction:ForActions}, a for iteration filters the set $A'$ for those actions such that can achieve $l$, \idest, actions that contain $l$ in their (known or possible) effects, formally, $\forall \inc{a} \in A'$ such that $l \in (\mathit{eff}^{+}(\inc{a}) \cup \widetilde{\mathit{eff}}^{+}(\inc{a}))$. 
From the filtered actions $A'$, our algorithm takes as new landmark candidates those facts that are in the preconditions of every action $\inc{a}$ in $A'$ (Line~\ref{alg:LandmarksExtraction:ForPreconds}). 
In Line~\ref{alg:LandmarksExtraction:isLandmark}, the algorithm checks if a landmark candidate $f$ is indeed a landmark using a function called \textsc{IsLandmark}. 
This function evaluates whether a landmark candidate is a necessary condition to achieve a goal~\cite{Hoffmann2004_OrderedLandmarks}. 
For example, consider an \textsc{ORPG}$'$ structure built from a planning problem $\mathcal{I}$ and $G$ in which every action level in this \textsc{ORPG}$'$ does not contain actions that achieve the landmark candidate $l'$. 
Given this modified \textsc{ORPG}$'$ structure, we test the solvability of the planning problem $\mathcal{I}$ and $G$. So, if this problem is unsolvable, then the landmark candidate $l'$ is indeed a landmark. 
More specifically, it means that the actions that achieve the fact $l'$ must be part of the \textsc{ORPG}$'$ to solve this planning problem. Namely, the actions that contain the fact $l'$ in their effects are necessary to solve the problem.
Thus, in Line~\ref{alg:LandmarksExtraction:SelectLandmarks}, the algorithm selects the facts which are indeed landmarks, and will be used extract other landmarks. 
The algorithm stores facts that must be true together (Line~\ref{alg:LandmarksExtraction:LandmarksTogether}). 
In Lines~\ref{alg:LandmarksExtraction:AddingDefiniteLandmarks} and~\ref{alg:LandmarksExtraction:AddingPossibleLandmarks}, the algorithm stores fact landmarks according to their types, as formalized in Definitions~\ref{def:DefiniteLandmark} and~\ref{def:PossibleLandmark}. 
Finally, in Line~\ref{alg:LandmarksExtraction:Return}, our algorithm returns the set of extracted \textit{definite} and \textit{possible} landmarks, respectively, $\mathcal{L}_{G}$ and $\mathcal{\widetilde{L}}_{G}$.

Based on the incomplete planning problem formalized in Example~\ref{example:Abstract}, we now exemplify how our adapted extraction algorithm extracts landmarks, as follows in Example~\ref{example:LandmarkExtraction}.

\begin{example}\label{example:LandmarkExtraction}
Recalling that for Example~\ref{example:Abstract}, the set of facts is $\mathcal{F} = \lbrace p,q,r,g \rbrace$, the set of incomplete actions is $\incp{\mathcal{A}} = \lbrace \inc{a},\inc{b},\inc{c} \rbrace$, and the initial and goal states are $\mathcal{I} = \lbrace p,q \rbrace$ and $G = \lbrace g \rbrace$.
Therefore, for this example, the set of \textit{definite} and \textit{possible} landmarks when using Algorithm~\ref{alg:LandmarksExtractionIncompleteDomains} is $\lbrace p,q,r,g \rbrace$. 
The set of \textit{definite} landmarks is $\lbrace p,r,g \rbrace$ (Light-Blue in Figure~\ref{fig:ORPG_Example}), and the set of \textit{possible} landmarks is $\lbrace q \rbrace$ (Light-Yellow in Figure~\ref{fig:ORPG_Example}). 
The original landmark extraction algorithm proposed by~Hoffmann~\etal~in~\cite{Hoffmann2004_OrderedLandmarks} (without the most \textit{optimistic} conditions), returns $\lbrace p,r,g \rbrace$ as landmarks. 
The tradditional landmark extraction algorithm does not extract $q$ as a fact landmark because it does not assume the most \textit{optimistic} condition that possible add effects always occur, therefore, the action $a$ was not considered as a possible achiever (action). 
Thus, by using our new extraction algorithm along with an \textsc{ORPG} instead an RPG we can extract not only \textit{definite} landmarks but also \textit{possible} landmarks, obtaining more landmarks than the original algorithm of~Hoffmann~\etal~\cite{Hoffmann2004_OrderedLandmarks}. Figure~\ref{fig:ORPG_Example} shows an \textsc{ORPG} for the incomplete planning problem illustrated in Example~\ref{example:Abstract}.
\end{example}

\floatname{algorithm}{Algorithm}
\begin{algorithm}[h!]
    \caption{Landmark Extraction Algorithm for Incomplete Domain Models.}
    \textbf{Input:} $\incp{D} = \langle \mathcal{R}, \widetilde{\mathcal{O}} \rangle$ \textit{is the incomplete domain definition}, $Z$ \textit{is a set of typed objects}, where $\mathcal{F}$ \textit{is the set of facts instantiated from} $\mathcal{R}$ \textit{and} $Z$, $\incp{\mathcal{A}}$ \textit{is the set of incomplete actions instantiated from} $\widetilde{\mathcal{O}}$ \textit{and} $Z$, $\mathcal{I}$ \textit{is the initial state, and} $G$ \textit{is the goal state.}
    \\\textbf{Output:} 
	$\mathcal{L}_{G}$ \textit{set of definite landmarks} \textit{and} $\mathcal{\widetilde{L}}_{G}$ \textit{set of possible landmarks}.
	\label{alg:LandmarksExtractionIncompleteDomains}
    \begin{algorithmic}[1]
        \Function{ExtractLandmarks}{$\incp{D}, Z, \mathcal{I}, G$}
		\State \textsc{ORPG} $\gets \textsc{BuildFullORPG}(\mathcal{F}, \incp{\mathcal{A}}, \mathcal{I}, G)$ \label{alg:LandmarksExtraction:BuildORPG}
		\If{$G$ $\nsubseteq$ \textsc{ORPG}.\textsc{LastFactLevel}}
			\State \textbf{return} $\textsc{Failure To Extract Landmarks}$
		\EndIf
		\State $\mathcal{L}_{G}, \mathcal{\widetilde{L}}_{G} \gets \langle$ $\rangle$ \Comment{\textit{Set of extracted (definite and possible) landmarks.}}
		\State $\mathcal{L}_{G} \gets G$ \Comment{\textit{Trivially, the facts in the goal state are (definite) landmarks.}} 
		\State $C \gets G$ \Comment{\textit{Set of landmark candidates.}} \label{alg:LandmarksExtraction:Candidates}
		\While{$C \neq \emptyset$} \label{alg:LandmarksExtraction:While}
			\State $l \gets C$.\textsc{pop} \Comment{\textit{Candidate landmark.}}
			\State $A' \gets$ \textsc{ORPG.$\incp{\mathcal{A}}_{(\textsc{ORPG}.\textsc{level}(l)-1)}$\label{alg:LandmarksExtraction:ActionsORPG} \Comment{\textit{All actions in the ORPG at the action level immediately before the candidate fact landmark} $l$.}}
			\For{each incomplete action $\inc{a}$ in $A'$ such that $l \in (\mathit{eff}^{+}(\inc{a}) \cup \widetilde{\mathit{eff}}^{+}(\inc{a}))$} \label{alg:LandmarksExtraction:ForActions}
				\State $L_{Together} \gets \langle$ $\rangle$ \Comment{\textit{Set of extracted fact landmarks that must be true together.}}
				\For{each fact $f$ in $\mathit{pre}(\inc{a})$} \label{alg:LandmarksExtraction:ForPreconds}
					\If{\textsc{IsLandmark($f$, $\mathcal{I}$, $G$)}} \label{alg:LandmarksExtraction:isLandmark}
						\State $C \gets C$ $\cup$ $f$ \label{alg:LandmarksExtraction:SelectLandmarks}
						\State $L_{Together} \gets L_{Together}$ $\cup$ $f$ \label{alg:LandmarksExtraction:LandmarksTogether}
					\EndIf
				\EndFor
				\If{$L_{Together} \neq \emptyset$ and $l \in \mathit{eff}^{+}(\inc{a})$} \Comment{\textit{Check if} $l$ \textit{is a definite landmark.}}
					\State $\mathcal{L}_{G} \gets \mathcal{L}_{G}$ $\cup$ $L_{Together}$ \label{alg:LandmarksExtraction:AddingDefiniteLandmarks} \Comment{\textit{Add definite landmarks into} $\mathcal{L}_{G}$.}
				\ElsIf{$L_{Together} \neq \emptyset$ and $l \in \widetilde{\mathit{eff}}^{+}(\inc{a})$} \Comment{\textit{Check if} $l$ \textit{is a possible landmark.}}
					\State $\mathcal{\widetilde{L}}_{G}$ := $\mathcal{\widetilde{L}}_{G}$ $\cup$ $L_{Together}$ \label{alg:LandmarksExtraction:AddingPossibleLandmarks} \Comment{\textit{Add possible landmarks into} $\mathcal{\widetilde{L}}_{G}$.}
				\EndIf
			\EndFor
		\EndWhile \label{alg:LandmarksExtraction:Endwhile}
		\State \textbf{return} $\mathcal{L}_{G}, \mathcal{\widetilde{L}}_{G}$ \label{alg:LandmarksExtraction:Return}
        \EndFunction
    \end{algorithmic}
\end{algorithm}

\begin{figure}[h!]
  \centering
  \includegraphics[width=0.6\linewidth]{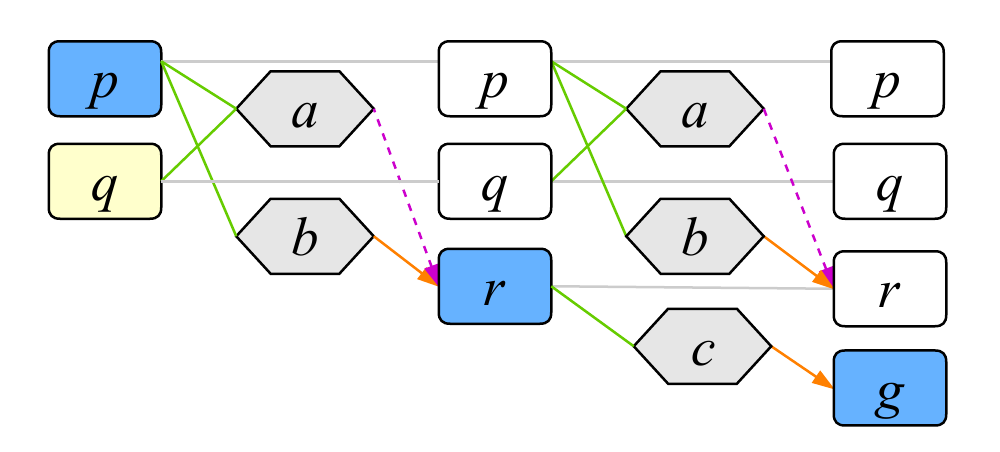}
  \caption{\textsc{ORPG} for the abstract incomplete planning problem presented in Example~\ref{example:Abstract}. Green arrows represent preconditions, Orange arrows represent add effects, and Purple dashed arrows represent possible add effects. Light-Blue boxes represent the set of \textit{definite} landmarks and Light yellow boxes represent the set of \textit{possible} landmarks. Grey hexagons represent actions.}
  \label{fig:ORPG_Example}
\end{figure}

We note that in this thesis, we use this new landmark extraction algorithm to build heuristics to goal recognition over \textit{incomplete domain models}. Although we use these new notations of landmarks for goal recognition, we argue that these new notions could be easily used to build heuristics for planning in incomplete domain models.

%--------------------------------------------------------------------------------------------
\newpage
\section{Goal Recognition Heuristics over Incomplete Domain Models}\label{section:HeuristicGoalRecognitionIncompleteDomains}

We now develop novel goal recognition heuristics that rely on landmarks over \textit{incomplete domain models}\footnote{The heuristics we present in this section are enhanced heuristics based on recognition heuristics over \textit{complete and correct domain models}. In~\ref{appendixA:goalrecognition_heuristics}, we develop the original landmark--based heuristics for goal recognition over complete domains models, presenting practical examples and theoretical properties of such heuristics.}. More specifically, we enhance recognition heuristics by exploiting the new notions of landmarks in \textit{incomplete domain models}.
Key to our enhanced heuristic approaches to goal recognition over \textit{incomplete domain models} is collecting evidence of achieved landmarks during observations to recognize which goal is more consistent with the observations in a plan execution. 
Figure~\ref{fig:GR_IncompleteDomains_Landmarks} illustrates the basic idea of our heuristic approaches to goal recognition over \textit{incomplete domain models}. 
Basically, we start extracting landmarks for every goal in the set of candidate goals from the initial state. 
After that, we check which landmarks have been achieved in the observations for every candidate goal. 
Then, we compute a score for the goals based on the ratio of the achieved landmarks and the total amount of landmarks, and finally, we consider as the most likely goal the goals with the highest computed score. 
Because landmarks are necessary conditions on plans that achieve a goal, such heuristics can underpin the task of goal recognition when observing either optimal or sub-optimal behaviors. 
That is, since to achieve a goal from an initial state, an observed agent has to achieve all landmarks, regardless of whether the observed agent is executing an optimal or a sub-optimal plan.

\begin{figure}[h!]
  \centering
  \includegraphics[width=0.8\linewidth]{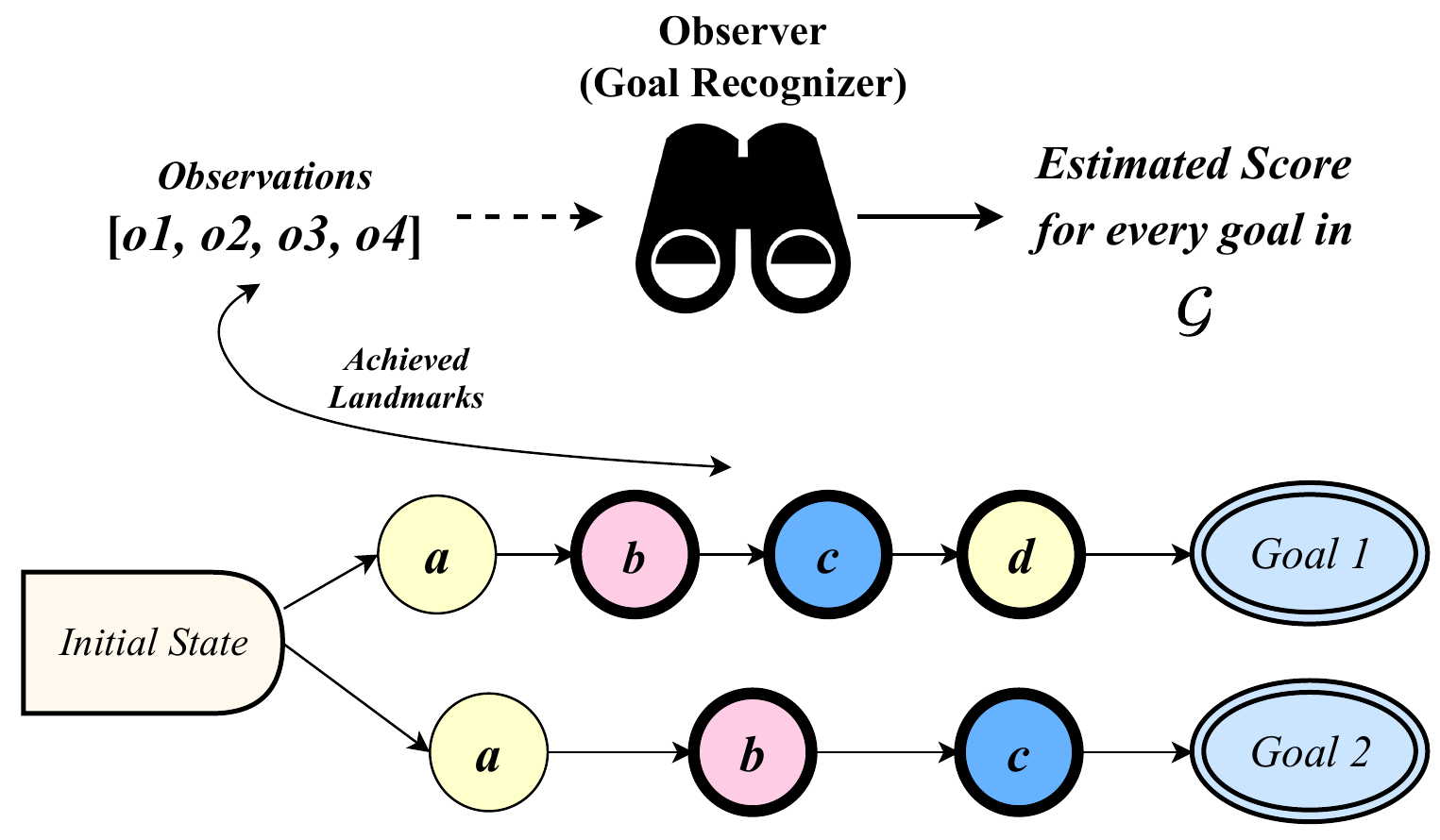}
  \caption{Heuristic Goal Recognition over Incomplete Domain Models.}
  \label{fig:GR_IncompleteDomains_Landmarks}
\end{figure}

To build our heuristic approaches to goal recognition over \textit{incomplete domain models}, we exploit and combine the concepts of \textit{definite} and \textit{possible} landmarks, with that of \textit{overlooked} landmarks (Definition~\ref{def:OverlookedLandmark}). 

\begin{definition}[\textbf{Overlooked Landmark}]\label{def:OverlookedLandmark}
An overlooked landmark $L_{O}$ is an actual landmark, a necessary fact for all valid plans towards a goal from an initial state, that was not detected by approximate landmark extraction algorithms.
\end{definition}

Most landmark extraction algorithms extract only a sub-set of landmarks for a given planning problem~\cite{Landmarks_Zhugivan_2003,Hoffmann2004_OrderedLandmarks,LandmarksRichter_2008,HM_Landmarks_2010}, and to overcome this issue and obtain more information to build heuristic approaches, we aim to extract \textit{overlooked} landmarks by analyzing preconditions and effects in the observed actions of an observation sequence. 
Since we are dealing with incomplete domain models, and it is possible that some incomplete planning problems have few (or no) \textit{definite} and/or \textit{possible} landmarks, we extract \textit{overlooked} landmarks from the evidence in the observations as we process them in order to enhance the set of landmarks useable by our goal recognition heuristics. 
% \frm[inline]{This would be a nice opportunity to plug a comment about the results from our work with Kin (that extracting all the landmarks does not make a whole lot of difference)}

We extract \textit{overlooked} landmarks \textit{on the fly} (\idest, during the goal recognition process), and this \textit{on the fly} extraction checks the facts in the preconditions and effects of the observations to extract landmarks that have not been extracted by our extraction algorithm developed in Section~\ref{section:ExtractingLandmarksInIncompleteDomains}. 
Specifically, our \textit{on the fly} extraction checks if the facts in the known preconditions and known possible add effects are not in the set of extracted \textit{definite} and \textit{possible} landmarks, and if they are not, we check if these facts are \textit{overlooked} landmarks. 
In order to do that, we use the \textsc{isLandmark} function that builds a new \textsc{ORPG} by removing actions that achieve a fact (\idest, a potentially \textit{overlooked} landmark) and checks the solvability of this modified problem. 
If the modified problem is indeed unsolvable, then this fact is an \textit{overlooked} landmark. 
We check every candidate goal $G$ in $\mathcal{G}$ using this function to extract additional (\textit{overlooked}) landmarks.
Example~\ref{example:Overlooked} illustrates how we extract \textit{overlooked} landmarks \textit{on the fly} for a given candidate goal and an observed action.

\begin{example}\label{example:Overlooked}
To exemplify how we extract overlooked landmarks during the recognition process, let us consider the goal state $\lbrace g \rbrace$ defined in Example~\ref{example:Abstract} as a candidate goal $G$, for a recognition problem with initial state $\mathcal{I} = \lbrace p,q \rbrace$, and a sequence of observations $Obs = a$. 
Assume that a landmark extraction algorithm extracts $\mathcal{L} = \lbrace p,q,g \rbrace$ for initial state $\mathcal{I}$ and candidate goal $G$.  
Given the observed action $a$, we check if the facts in the preconditions and known/possible effects of $a$ are in $\mathcal{L}$, and these facts are: $\lbrace p,q,r \rbrace$. 
Since $r$ is not in $\mathcal{L}$, $r$ can be an overlooked landmark. 
To check if $r$ is an overlooked landmark, we build ORPG removing all actions that achieve $r$ (\idest, actions $a$ and $b$) and check whether $G$ remains solvable in this ORPG. 
In this case, the goal atom $g$ is unreachable (and $G$ is unsolvable) because $r$ is a necessary fact to achieve $g$. 
Figure~\ref{fig:ORPG_Example} illustrates that $g$ is unachievable without actions $a$ and $b$, and consequently without $r$, because there is no action that adds this fact. 
Thus, $r$ is an overlooked landmark that was not extracted by the extraction algorithm, but it was extracted on the fly from an observed action during plan execution. 
\end{example}

After presenting the basic idea of how our recognition heuristics work, next, we develop two goal recognition heuristics over incomplete domain models that rely on the notions of landmarks we introduced before. 
In Section~\ref{subsec:goal_completion_heuristic}, we develop a recognition heuristic that estimates \textit{goal completion} by calculating the ratio between achieved landmarks and total amount of landmarks for a goal, and this heuristic is called \textit{Enhanced Goal Completion Heuristic}. 
We then develop in Section~\ref{subsec:uniqueness_heuristic} a recognition heuristic that uses the concept of \textit{landmark uniqueness value}, representing the information value of the landmark for some specific candidate goal when compared to landmarks for all candidate goals, and this heuristic is called \textit{Enhanced Uniqueness Heuristic}.

%###############################################################
\subsection{Enhanced Goal Completion Heuristic for Incomplete Domain Models}\label{subsec:goal_completion_heuristic}

With the new notions of landmarks for incomplete domains in place, we now present a novel recognition heuristic for recognizing goals in incomplete domain models, called \textit{Enhanced Goal Completion Heuristic}, denoted as $h_{\widetilde{GC}}$. 
We enhance the \emph{Goal Completion Heuristic} ($h_{gc}$) presented in \ref{appendixA:goalrecognition_heuristics} to cope with incomplete domain models.
Our enhanced heuristic combines the notions of \textit{definite}, \textit{possible}, and \textit{overlooked} landmarks, and estimates the correct intended goal in the set of candidate goals by calculating the ratio between achieved \textit{definite} ($\mathcal{AL}_{G}$), \textit{possible} ($\mathcal{\widetilde{AL}}_{G}$), and \textit{overlooked} ($\mathcal{ANL}_{G}$) landmarks and the amount of \textit{definite} ($\mathcal{L}_{G}$), \textit{possible} ($\mathcal{\widetilde{L}}_{G}$), and \textit{overlooked} ($\mathcal{NL}_{G}$) landmarks. 
This estimate, formalized in Equation~\ref{eq:GoalCompletionHeuristic}, represents the percentage of achieved landmarks for a candidate goal from observations, more specifically, it represents the percentage of completion of a goal based on the achieved landmarks and the total number of extracted landmarks. 

\simbolo{$h_{\widetilde{GC}}$}{Goal Completion Heuristic.}
\begin{equation}
\label{eq:GoalCompletionHeuristic}
h_{\widetilde{GC}}(G) = \left(\frac{\mathcal{AL}_{G} + \mathcal{\widetilde{AL}}_{G} + \mathcal{ANL}_{G}}{\mathcal{L}_{G} + \mathcal{\widetilde{L}}_{G} + \mathcal{NL}_{G}}\right)
\end{equation}

Even when a candidate goal has achieved all its landmarks during the recognition process for this heuristic, it is possible that such candidate goal is not the correct intended goal based on the observations. 
Such errors were possible in our approach for complete domains when recognizing goals from suboptimal plans.
This limitation is exacerbated when we consider the set of \textit{possible} landmarks ($\mathcal{\widetilde{L}}_{G}$) in the heuristic computation of $h_{\widetilde{GC}}$. 
This type of landmark is extracted from the incomplete part of the domain model, so they can be considered as ``uncertain'' landmarks, \idest, landmarks that might not be necessary conditions to achieve a goal from an initial state. 
In Section~\ref{section:GR_AblationStudy}, we perform an ablation study and show the real impact of the new notations of landmarks on the recognition performance over our enhanced heuristics.

Example~\ref{example:GoalCompletion} illustrates how our enhanced goal completion heuristic works for recognizing goals over incomplete domains models.

\begin{example}\label{example:GoalCompletion}
To exemplify the recognition process using $h_{\widetilde{GC}}$, let us consider the goal recognition problem and the sets of landmarks presented in Figure~\ref{fig:GR_IncompleteDomains_Landmarks}. In this example, we have two candidate goals, Goal 1 (correct intended goal) and Goal 2, and the sets of landmarks for these goals are: Goal 1 $ = \lbrace a,b,c,d \rbrace$ and Goal 2 $ = \lbrace a,b,c \rbrace$. Blue circles represent definite landmarks (c), yellow circles represent possible landmarks (a,d), and pink circles represent overlooked landmarks (b). The achieved landmarks for the goals Goal 1 and Goal 2 are respectively: $\lbrace b,c,d \rbrace$ and $\lbrace b,c \rbrace$. By using our proposed enhanced goal completion heuristic $h_{\widetilde{GC}}$ to estimate which goal is the intended one in this example, we have the following heuristic estimation for Goal 1 and Goal 2:

\begin{itemize}
	\item $h_{\widetilde{GC}}$(Goal 1) = $\left(\frac{\mathcal{AL}_{G} + \mathcal{\widetilde{AL}}_{G} + \mathcal{ANL}_{G}}{\mathcal{L}_{G} + \mathcal{\widetilde{L}}_{G} + \mathcal{NL}_{G}}\right)$ = $\left(\frac{1 + 1 + 1}{1 +2 +1}\right)$ = $\left(\frac{3}{4}\right)$ = 0.75
	\item $h_{\widetilde{GC}}$(Goal 2) = $\left(\frac{\mathcal{AL}_{G} + \mathcal{\widetilde{AL}}_{G} + \mathcal{ANL}_{G}}{\mathcal{L}_{G} + \mathcal{\widetilde{L}}_{G} + \mathcal{NL}_{G}}\right)$ = $\left(\frac{1 + 0 + 1}{1 + 1 +1}\right)$ = $\left(\frac{2}{3}\right)$ = 0.66
\end{itemize}
Thus, according to the scores computed by $h_{\widetilde{GC}}$, the goal with the highest score is Goal 1.
\end{example}

% \frm[inline]{Maybe move this to the next subsection, or remove to avoid repetition.}
We now present the computational complexity of the \textit{Enhanced Goal Completion Heuristic} $h_{\widetilde{GC}}$.
Besides extracting \textit{definite} and \textit{possible} landmarks for every candidate goal ($EL$), this heuristic iterates over the set of candidate goals $\mathcal{G}$, and the observations sequence $Obs$. During the iteration over the observation sequence $Obs$, we extract \textit{overlooked} landmarks from the observations $Obs$ ($EOL$). The heuristic also iterates over the extracted landmarks $(\mathcal{L}_{G} + \mathcal{\widetilde{L}}_{G} + \mathcal{NL}_{G})$ to compute the achieved landmarks.
The heuristic computation of $h_{\widetilde{GC}}$ ($HC$) is linear on the number of landmarks. Thus, the complexity of this heuristic approach is: $O(EL + |\mathcal{G}|\cdot|Obs|\cdot|EOL|\cdot|\mathcal{L}_{G} + \mathcal{\widetilde{L}}_{G} + \mathcal{NL}_{G}| + HC)$. 

%###############################################################
\subsection{Enhanced Uniqueness Heuristic for Incomplete Domain Models}\label{subsec:uniqueness_heuristic}

Most goal recognition problems contain multiple candidate goals that share common fact landmarks, generating ambiguity to recognize correctly the intended goal. 
Evidently, landmarks that are common to multiple candidate goals are less useful for recognizing a goal than landmarks that exist for only a single goal. 
Thus, computing how unique (and thus informative) each landmark is can help disambiguate similar goals for a set of candidate goals. 
We now develop a second goal recognition heuristic based on this intuition. 
To develop this heuristic, we leverage the concept of \textit{landmark uniqueness} (introduced in~\ref{appendixA:goalrecognition_heuristics}), which is the inverse frequency of a landmark among the landmarks found in a set of candidate goals. 
For example, consider a landmark $L$ that occurs only for a single goal within a set of candidate goals; the uniqueness value for such a landmark is intuitively the maximum value of 1. 
Equation~\ref{eq:LandmarksUniqueness_Incomplete} formalizes this intuition, describing how the \textit{landmark uniqueness value} is computed for a landmark $L$ and a set of landmarks for all candidate goals $K_{\mathcal{G}}$.

\begin{equation}
\label{eq:LandmarksUniqueness_Incomplete}
L_{\mathit{Uniq}}(L, K_{\mathcal{G}}) = \left(\frac{1}{\displaystyle\sum_{\mathcal{L} \in K_{\mathcal{G}}} |\{L |L \in \mathcal{L}\}|}\right)
\end{equation}

We use the concept of \textit{landmark uniqueness value} to estimate which candidate goal is the intended one by summing the uniqueness values of the landmarks achieved in the observations. 
Unlike our previous heuristic (Equation~\ref{eq:GoalCompletionHeuristic}), which estimates progress towards goal completion by analyzing just the set of achieved landmarks, the landmark--based uniqueness heuristic estimates the goal completion of a candidate goal $G$ by calculating the ratio between the sum of the uniqueness value of the achieved landmarks of $G$ and the sum of the uniqueness value of all landmarks of a goal $G$. To build the uniqueness heuristic over \textit{incomplete domain models}, we enhance the original uniqueness heuristics ($h_{uniq}$) presented in~\ref{appendixA:goalrecognition_heuristics} by using the concepts of \textit{definite}, \textit{possible}, and \textit{overlooked} landmarks. We store the set of \textit{definite} and \textit{possible} landmarks of a goal $G$ separately into $\mathcal{L}_{G}$ and $\mathcal{\widetilde{L}}_{G}$, and the set of \textit{overlooked} landmarks into $\mathcal{NL}_{G}$. 
Thus, the uniqueness heuristic effectively weighs the completion value of a goal by the informational value of a landmark so that unique landmarks have the highest weight. 
To estimate goal completion using the \textit{landmark uniqueness value}, we calculate the uniqueness value for every extracted (\textit{definite}, \textit{possible}, and \textit{overlooked}) landmark in the set of landmarks of the candidate goals using Equation~\ref{eq:LandmarksUniqueness_Incomplete}.
Since we use three types of landmarks and they are stored in three different sets, we compute the landmark uniqueness value separately for them, storing the landmark uniqueness value of \textit{definite} landmarks $\mathcal{L}_{G}$ into $\Upsilon_{\mathcal{L}}$, the landmark uniqueness value of \textit{possible} landmarks $\mathcal{\widetilde{L}}_{G}$ into $\Upsilon_{\mathcal{\widetilde{L}}}$, and the landmark uniqueness value of \textit{overlooked} landmarks $\mathcal{NL}_{G}$ into $\Upsilon_{\mathcal{NL}_{G}}$.
Our \textit{Enhanced Uniqueness Heuristic} is denoted as $h_{\widetilde{UNIQ}}$ and formally defined in Equation~\ref{eq:HeuristicLandmarksUniqueness_Incomplete}.

\simbolo{$h_{\widetilde{UNIQ}}$}{Uniqueness Heuristic.}
\begin{equation}
\label{eq:HeuristicLandmarksUniqueness_Incomplete}
h_{\widetilde{UNIQ}}(G) = 
	\left(\frac{
		{\displaystyle\sum_{\mathcal{A}_{L} \in \mathcal{AL}_{G}}\Upsilon_{\mathcal{L}}(\mathcal{A}_{L})} + 
		{\displaystyle\sum_{\widetilde{\mathcal{A}_{L}} \in \mathcal{\widetilde{AL}}_{G}}\Upsilon_{\mathcal{\widetilde{L}}}(\widetilde{\mathcal{A}_{L}})} + 
				{\displaystyle\sum_{\mathcal{ANL} \in \mathcal{ANL}_{G}}\Upsilon_{\mathcal{NL}_{G}}(\mathcal{ANL})}
		}
		{
		{\displaystyle\sum_{L \in \mathcal{L}_{G}}\Upsilon_{\mathcal{L}}(L)} + 
		{\displaystyle\sum_{\tilde{L} \in \mathcal{\widetilde{L}}_{G}}\Upsilon_{\mathcal{\widetilde{L}}}(\tilde{L})} + 
		{\displaystyle\sum_{\mathcal{NL} \in \mathcal{NL}_{G}}\Upsilon_{\mathcal{NL}_{G}}(\mathcal{NL})}
	}\right)
\end{equation}

To exemplify how our enhanced uniqueness heuristic $h_{\widetilde{UNIQ}}$ works for recognizing goals, we use and follow the same example we used before for our previous heuristics, as shown in Example~\ref{example:Uniqueness}.

\begin{example}\label{example:Uniqueness}
Considering the goal recognition problem and the sets of landmarks presented in Figure~\ref{fig:GR_IncompleteDomains_Landmarks}, we have the following landmark uniqueness values for the landmarks:

\begin{itemize}
	\item Based on the set of landmarks for all goals $K_{\mathcal{G}} = \lbrace a,b,c,d \rbrace$, the landmark uniqueness values are:
	\begin{itemize}
		\item $L_{\mathit{Uniq}}(a, K_{\mathcal{G}})$ = 1/2 = 0.5
		\item $L_{\mathit{Uniq}}(b, K_{\mathcal{G}})$ = 1/2 = 0.5
		\item $L_{\mathit{Uniq}}(c, K_{\mathcal{G}})$ = 1/2 = 0.5
		\item $L_{\mathit{Uniq}}(d, K_{\mathcal{G}})$ = 1/1 = 1
	\end{itemize}
	\item After computing the landmark uniqueness value for the set of landmarks, we now can compute the scores for Goal 1 and Goal 2 using $h_{\widetilde{UNIQ}}$, as follows:
	\begin{itemize}
		\item $h_{\widetilde{UNIQ}}$(Goal 1) = $\left(\frac{0.5 + 1 + 0.5}{0.5 + 1.5 + 0.5}\right)$ = $\left(\frac{2}{2.5}\right)$ = 0.80
		\item $h_{\widetilde{UNIQ}}$(Goal 2) = $\left(\frac{0.5 + 0 + 0.5}{0.5 + 0.5 + 0.5}\right)$ = $\left(\frac{1}{1.5}\right)$ = 0.66
	\end{itemize}
\end{itemize}
Thus, according to the scores computed by $h_{\widetilde{UNIQ}}$, the goal with the highest score is Goal 1.
\end{example}

We now formalize the computational complexity of the \textit{Enhanced Uniqueness Heuristic} $h_{\widetilde{UNIQ}}$. Similar to the previous enhanced heuristic approach, this heuristic first extracts \textit{definite} and \textit{possible} landmarks for every candidate goal ($EL$), and then iterates over the set of candidate goals $\mathcal{G}$, and the observations sequence $Obs$. \textit{Overlooked} landmarks are extracted ($EOL$) during the iteration over the observation sequence $Obs$. This heuristic also iterates over the extracted landmarks $(\mathcal{L}_{G} + \mathcal{\widetilde{L}}_{G} + \mathcal{NL}_{G})$ to compute the achieved landmarks.
Different than the other enhanced heuristic, in this heuristic, we weight each landmark by how common this landmark is across all goal hypotheses. 
We call this weight the \textit{uniqueness value} ($CLUniq$), and its computation is linear on the number of fact landmarks. 
The heuristic computation of $h_{\widetilde{UNIQ}}$ ($HC$) is also linear on the number of landmarks. 
Thus, the complexity of this heuristic approach is: $O(EL + |\mathcal{G}|\cdot|Obs|\cdot|EOL|\cdot|\mathcal{L}_{G} + \mathcal{\widetilde{L}}_{G} + \mathcal{NL}_{G}| + CLUniq + HC)$.

%--------------------------------------------------------------------------------------------
% \newpage
\section{Experiments and Evaluation}\label{section:GR_IncompleteDomains_ExperimentsEvaluation}

In this section, we describe the experiments carried out to evaluate our goal recognition heuristics over incomplete domain models, describing how we have built and modified datasets from literature (Sections~\ref{subsec:GR_IncompleteDomains_Domains} and~\ref{subsec:GR_IncompleteDomains_Datasets_Setup}), as well as describing the metrics we used for evaluation (Section~\ref{subsec:GR_IncompleteDomains_Metrics}). In Section~\ref{subsec:GR_IncompleteDomains_NumberCompleteDomains}, we present the average number of possible complete domains models over the datasets we built with incomplete domains models, showing the complexity of recognizing goals over incomplete domain models.
We compare the recognition performance of our enhanced heuristics (Sections~\ref{section:GR_AblationStudy}~and~\ref{subsec:GR_IncompleteDomains_ROCSpace}) against the two original landmark--based heuristic approaches~\ref{appendixA:goalrecognition_heuristics}, which we use as baselines. We also perform an ablation study that evaluates the effect and impact of the various types of landmarks on the recognition performance of our heuristic approaches. 

%######################################################################
\subsection{Domains}\label{subsec:GR_IncompleteDomains_Domains}

% \frm[inline]{TLDR (I've read these descriptions a million times by now :-D.). Having said that, could you comment on the quirks of each domain (e.g. the problems with Kitchen, etc.)}
We empirically evaluated our goal recognition heuristics by using fifth-ten domains from the planning literature\footnote{\texttt{http://ipc.icaps-conference.org}}.
Six of these planning domains have been also used in the evaluation of other goal and plan recognition approaches~\cite{RamirezG_IJCAI2009,RamirezG_AAAI2010,NASA_GoalRecognition_IJCAI2015,Sohrabi_IJCAI2016,PereiraNirMeneguzzi_AAAI2017}. Specifically, we evaluate our heuristics over a variety of different types of planning domains, \idest, navigation domains, logistic-based domains, plan-library adapted domains, among others. We summarize these domains as follows. 

\begin{itemize}
	\item \textsc{Blocks-World} (\textsc{Blocks}) is a domain that consists of a set of blocks, a table, and a robot hand. Blocks can be stacked on top of other blocks or on the table. A block that has nothing on it is clear. The robot hand can hold one block or be empty. The goal is to find a sequence of actions that achieves a final configuration of blocks;
	
	\item \textsc{Campus} is a domain that consists of finding what activity is being performed by a student from his observations on a campus environment;

	\item \textsc{Depots} is a domain that combines transportation and stacking. For transportation, packages can be moved between depots by loading them on trucks. For stacking, hoists can stack packages on palettes or other packages. The goal is to move and stack packages by using trucks and hoists between depots;	
	
	\item \textsc{Driver-Log} (\textsc{Driver}) is a domain that consists of drivers that can walk between locations and trucks that can drive between locations. Walking from locations requires traversal of different paths. Trucks can be loaded with or unloaded of packages. Goals in this domain consists of transporting packages between locations;
	
	\item \textsc{Dock-Worker-Robots (DWR)} is a domain that involves a number of cranes, locations, robots, containers, and piles, in which goals involve transporting containers to a final destination according to a desired order;
	
	\item \textsc{IPC-Grid} domain is a domain consists of an agent that moves in a grid from connected cells to others by transporting keys in order to open locked locations;
	
	\item \textsc{Ferry} is a domain that consists of set of cars that must be moved to desired locations using a ferry that can carry only one car at a time;
	
	\item \textsc{Intrusion-Detection} (\textsc{Intrusion}) represents a domain where a hacker tries to access, vandalize, steal information, or  perform a combination of these attacks on a set of servers;
	
	\item \textsc{Kitchen} is a domain that consists of home-activities, in which the goals can be preparing dinner, breakfast, among others;
	
	\item \textsc{Logistics} is a domain which models cities, and each city contains locations. These locations are airports. For transporting packages between locations, there are trucks and airplanes. Trucks can drive between cities. Airplanes can fly between airports. The goal is to get and transport packages from locations to other locations;
	
	\item \textsc{Miconic} is a domain that involves transporting a number of passengers using an elevator to reach destination floors;
	
	\item \textsc{Rovers} is a domain that consists of a set of rovers that navigate on a planet surface in order to find samples and communicate experiments;
	
	\item \textsc{Satellite} is a domain that involves using one or more satellites to make observations, by collecting data and down-linking the data to a desired ground station;

	\item \textsc{Sokoban} is a domain that involves an agent whose goal is to push a set of boxes into specified goal locations in a grid with walls; and

	\item \textsc{Zeno-Travel} (\textsc{Zeno}) is a domain where passengers can embark and disembark onto aircraft that can fly at two alternative speeds between locations.
\end{itemize}

%######################################################################
\subsection{Datasets and Setup}\label{subsec:GR_IncompleteDomains_Datasets_Setup}

For experiments and evaluation, we used and modified openly available goal and plan recognition datasets \cite{Pereira_Meneguzzi_PRDatasets_2017}\footnote{\scriptsize\url{https://doi.org/10.5281/zenodo.825878}}, which contain thousands of recognition problems. 
These datasets contain large and non-trivial planning problems (with optimal and sub-optimal plans as observations, \idest, optimal and sub-optimal behaviors) for the fifth-ten planning domains described in the previous section, including domains and problems from datasets that were developed by Ram{\'{\i}}rez and Geffner~\cite{RamirezG_IJCAI2009,RamirezG_AAAI2010}\footnote{\scriptsize\url{https://sites.google.com/site/prasplanning}}. 
All planning domains in these datasets are encoded using the STRIPS fragment of PDDL \cite{PDDLMcdermott1998}. 
Each goal and plan recognition problem in these datasets contains a (complete) domain definition, an initial state, a set of candidate goals, a correct hidden goal in the set of candidate goals, and an observation sequence. 
An observation sequence contains actions that represent an optimal plan or sub-optimal plan that achieves a correct hidden goal, and this observation sequence can be full or partial. 
A full observation sequence represents the whole plan that achieves the hidden goal, \idest, 100\% of the actions having been observed. 
A partial observation sequence represents a plan for the hidden goal, varying in 10\%, 30\%, 50\%, or 70\% of its actions having been observed. 
To evaluate our goal recognition approaches over incomplete domain models, we modify the (complete) domain models of these datasets by following the formalism of incomplete STRIPS, adding annotated possible preconditions and effects (add and delete lists). 
Thus, the only modification to the original datasets is the generation of new, incomplete, domain models for each recognition problem, varying the percentage of incompleteness (possible preconditions and effects) in these domains. 

To build goal recognition datasets with incomplete domain models, we vary the percentage of incompleteness of a domain from 20 to 80 percent (\idest, 20\%, 40\%, 60\%, and 80\%). 
For example, consider that a complete domain has, for all its actions, a total of 10 preconditions, 10 add effects, and 10 delete effects. 
A derived model with 20\% of incompleteness needs to have 2 possible preconditions (8 known preconditions), 2 possible add effects (8 known add effects), and 2 possible delete effects (8 known delete effects), and so on for other percentages of incompleteness.  Like~\cite{PlanningIncomplete_NguyenK_2014,Nguyen_AIJ_2017}, we used the following conditions to generate incomplete domain models with possible preconditions, possible add effects, and possible delete effects: 

\begin{enumerate}
	\item We randomly move a percentage of known preconditions and effects into possible lists of preconditions and effects;
	\item We randomly add possible preconditions from delete effects that are not preconditions of a corresponding operator; and 
	\item We randomly add into possible lists (of preconditions, add effects, or delete effects) predicates whose parameters fit into the operator signatures and are not precondition or effects of the operator.
\end{enumerate}

By following all these three conditions, we generated three different incomplete STRIPS domain models from a complete STRIPS domain model, since the lists of preconditions and effects are generated randomly. 
Thus, each percentage of domain incompleteness has three domain models with different possible lists of preconditions and effects.

We ran all sets of experiments using a single core of a 12 core Intel(R) Xeon(R) CPU E5-2620 v3 @ 2.40GHz with 16GB of RAM in a Linux environment using Java. The JavaVM ran experiments with a 2GB memory limit and a 2-minute time limit. 

%######################################################################
\subsection{Evaluation Metrics}\label{subsec:GR_IncompleteDomains_Metrics}

We evaluate our heuristic approaches using the standard metrics of \textit{Precision} (ratio of correct positive predictions among all predictions) and \textit{Recall} (ratio between true positive results and total true positive and false negative results). 
In order to present a unified metric, we report the \textit{F1-score} (harmonic mean) of \textit{Precision} and \textit{Recall}. To perform our ablation study, we use the \textit{Correlation} ($C$) between the averages of \textit{F1-score} ($F_{1}$) and the absolute number of the various types of landmarks (\textit{definite} $D$, \textit{possible} $P$, and \textit{overlooked} $O$ landmarks) over all domains and problems, and \textit{Spread in} $\mathcal{G}$ as $S$, representing the average number of returned (recognized) goals. 
We decided to use the \textit{Correlation} in order to show the impact of each type of landmark in the \textit{F1-score} over the evaluated goal recognition problems. 
More specifically, we aim to show the association (or relationship) between the \textit{F1-score} and the number of extracted landmarks. 
\textit{Correlation} is a real value in $[-1,1]$ such that $-1$ represents an \textit{anti-correlation} between the landmarks and the \textit{F1-score}, $0$ represents no correlation between the landmarks and the \textit{F1-score}, whereas a value of $1$ represents that more landmarks correlate to a higher \textit{F1-score}. 

Besides these metrics, we use a graphical plot to evaluate accuracy performance of our heuristic approaches over incomplete domain models. 
To do so, we adapt the \textit{Receiver Operating Characteristic} (ROC) curve metric to highlight the trade-off between true positive and false positive results.
A ROC curve is often used to compare not only true positive predictions, but also to compare the false positive predictions of the experimented approaches. 
Here, each prediction result of our goal recognition approaches represents one point in the space, and thus, instead of a curve, our graphs show the spread of our results over ROC space. 
In the ROC space, the diagonal line represents a random guess to recognize a goal from observations. 
This diagonal line divides the ROC space in such a way that points above the diagonal represent good classification results (better than random guess), whereas points below the line represent poor results (worse than random guess). 
The best possible (perfect) prediction for recognizing goals are points in the upper left corner (\idest, coordinate x = 0 and y = 100) in ROC space. 

%######################################################################
\subsection{Experimental Results: The Average Number of Possible Complete Domains}\label{subsec:GR_IncompleteDomains_NumberCompleteDomains}

Our heuristic approaches recognize goals at very low recognition time for most incomplete planning domains and problems, taking at most 2.7 seconds, including the process of extracting landmarks, among all goal recognition problems, apart from \textsc{IPC-Grid} and \textsc{Sokoban}, which took substantial recognition time (for more detail, please see \ref{appendixB:goalrecognition_incompletedomains}). 
More specifically, only 1092 (20\% of domain incompleteness) out of 4368 problems for \textsc{IPC-Grid} and \textsc{Sokoban} do not exceed the time limit of 2 minutes (for both our approaches and the baselines).
\textsc{Sokoban} exceeds the time limit of 2 minutes for most goal recognition problems because this dataset contains large problems with a huge number of objects, leading to an even larger number of instantiated predicates and actions. 
For example, as domain incompleteness increases (\idest, the ratio of possible and definite preconditions and effects), the number of possible actions (moving between cells and pushing boxes) increases substantially in a grid with 9x9 cells and 5 boxes as there are very few known preconditions for several possible preconditions. 
As a basis of comparison, state-of-the-art planners~\cite{PlanningIncomplete_NguyenK_2014,Nguyen_AIJ_2017} for incomplete domain models take substantially more time than 2-minute timeout to generate a single plan for domains that our heuristic approaches recognize goals in less than 2 seconds. 
For example, CPISA~\cite{Nguyen_AIJ_2017} takes $\approx$ 300 seconds to find a plan with 25 steps in domains (\exemp, \textsc{Satellite}) with 2 possible preconditions and 3 possible add effects, whereas our dataset contains much more complex incomplete domains and problems.
The average number of possible complete domain models $|\langle\langle \widetilde{\mathcal{D}} \rangle\rangle|$ is huge for several domains, showing that the task of goal recognition over incomplete domain models can be quite difficult and complex if we take into account the number of possible complete domain models. 
For instance, the average number of possible complete domains in this dataset varies between 9.18 (\textsc{Sokoban} with 20\% of domain incompleteness) and $7.84^{15}$ (\textsc{Rovers} with 80\% of domain incompleteness). 
The average number of possible complete domain models $|\langle\langle \widetilde{\mathcal{D}} \rangle\rangle|$ is huge for several domains (\textsc{Campus}, \textsc{DWR}, \textsc{Kitchen}, and \textsc{Rovers}), showing that the task of goal recognition in incomplete domains models is quite difficult and complex. 
Table~\ref{tab:NumberOfPossibleDomains} shows the average number of possible complete domain models for all domains we use in our experiments.

\begin{table}[h!]
\fontfamily{cmr}\selectfont
\small
\begin{tabular}{ccccc}
\toprule
\hline
{\textit{Incompleteness of} $\mathcal{\widetilde{D}}$ (\%)} & 20\% & 40\% & 60\% & 80\% \\ \hline
\#             & $|\langle\langle \widetilde{\mathcal{D}} \rangle\rangle|$    & $|\langle\langle \widetilde{\mathcal{D}} \rangle\rangle|$    & $|\langle\langle \widetilde{\mathcal{D}} \rangle\rangle|$    & $|\langle\langle \widetilde{\mathcal{D}} \rangle\rangle|$    \\ \hline
\textsc{Blocks}   &  42.22    &  1782.89    &  75281.09    &  3178688.03    \\
\textsc{Campus}   &  7131.55    &  50859008.46    &  
3.63E+11    &  2.59E+15    \\ 
\textsc{Depots}   &  168.89    &  28526.20    &   4817990.10   &   813744135.40   \\ 
\textsc{Driver}   &  48.50    &  2352.53   &  114104.80    &  5534417.30    \\ 
\textsc{DWR}   &   512.00   &   262144.00   &   134217728.00   &  6.88E+10    \\ 
\textsc{Ferry}   &   8.00   &  64.00    &   512.00   &   4096.00   \\ 
\textsc{Intrusion}   & 16.00     &  256.00    &  4096.00    &  65536.00    \\ 
\textsc{IPC-Grid}   & 10.55     &  111.43    &  1176.26    &  12416.75    \\ 
\textsc{Kitchen}   &  2767208.65    & 7.66E+12     &  2.11E+19    &  5.86E+25    \\ 
\textsc{Logistics}   &  27.85    &  776.04    &   21618.81   &  602248.76    \\ 
\textsc{Miconic}   & 9.18     &   84.44   &   776.05   &  7131.55    \\ 
\textsc{Rovers}   &  9410.14    & 88550676.93     &  8.34E+11    & 7.84E+15     \\ 
\textsc{Satellite}   &  27.85    &  776.04    &  21618.81    &  602248.76    \\ 
\textsc{Sokoban}   &  9.18    &   84.44   &   776.04   & 7131.55     \\ 
\textsc{Zeno}   &  48.50    &  2352.53    &  114104.80    &  5534417.30    \\ 
\bottomrule
\end{tabular}
\centering
\caption{The average number of possible complete domain models $|\langle\langle \widetilde{\mathcal{D}} \rangle\rangle|$ for all domains we use in our experiments.}
\label{tab:NumberOfPossibleDomains}
\end{table}

%----------------------------------------------------------------------------------
\newpage
\subsection{Experimental Results: An Ablation Study of The Impact of New Notions of Landmarks on The Recognition Performance}\label{section:GR_AblationStudy}

Since the key contribution of our heuristic approaches to goal recognition over incomplete domains are based on the new types of landmarks (\textit{definite}, \textit{possible}, and \textit{overlooked}), as opposed to the traditional landmarks from \textit{Classical Planning}, in this section, we want to objectively measure the effect of these new types of landmark on the recognition performance of our heuristics. 
Thus, we now present an ablation study that consists of measuring the performance of our heuristic approaches using some possible combinations of landmark types. 
%This ablation study aims to evaluate only linear combinations of landmark types, and we note that some combinations are not included in this thesis because they had poor performance and seem to be not relevant for the ablation study.\frm{This sentence opens you to criticism. Explain why you left some combinations out (other than they make your approach look bad)}

Table~\ref{tab:AblationStudy} summarizes the results of our ablation study by aggregating the average results over the datasets we generated for incomplete domains considering all levels of domain incompleteness (20\%, 40\%, 60\%, and 80\%) and observability (10\%, 30\%, 50\%, 70\%, and 100\%). In \ref{appendixB:goalrecognition_incompletedomains}, we show in detail the results for all fifth-ten domains by varying the level of domain incompleteness and observability.
We denote the original landmark--based heuristic approaches (\ref{appendixA:goalrecognition_heuristics}) as Baseline ($h_{gc}$) and Baseline ($h_{uniq}$). 
We run the experiments for the baselines ignoring the incomplete part of the domain model, \idest, \textit{all possible preconditions and effects}. This allows us to evaluate the effectiveness of the new types of landmarks over incomplete domains in our enhanced heuristics.
We denote our enhanced heuristics as $\widetilde{GC}$ and $\widetilde{UNIQ}$, and denote the combination over the various types of landmarks using \textit{D + P + O} as the combination of \textit{definite, possible}, and \textit{overlooked} landmarks, and other four combinations of these landmark types (\textit{D + O}, \textit{P + O}, \textit{P}, and \textit{O}). 
For this ablation study, we report the results by evaluating our enhanced heuristic approaches and the baselines using the \textit{F1-score} metric ($F_{1}$), the average number of recognized goals \textit{Spread in} $\mathcal{G}$ ($S$), the \textit{Correlation} between the averages of landmarks and \textit{F1-scores} (\textit{Correlation} of \textit{definite} landmarks $CD$, \textit{Correlation} of \textit{possible} landmarks $CP$, and \textit{Correlation} of \textit{overlooked} landmarks $CO$), as well as the average number of extracted landmarks for all types ($D$, $P$, and $O$).

We compute the \textit{Correlation} between the averages of landmarks and \textit{F1-scores} over all domains and degrees of incompleteness, columns $CD$, $CP$ and $CO$ in Table~\ref{tab:AblationStudy}, and plot these correlations as a function of the level of incompleteness in Figures~\ref{fig:gc-definite_possible_overlooked}--\ref{fig:uniq-possible_overlooked} as opposed to how traditional landmarks affect performance on the baseline approaches (\textit{Correlation} varying between 0.32 and 0.67 for a lower \textit{F1-score}, Table~\ref{tab:AblationStudy}, lines 1 and 2). 
Figures~\ref{fig:gc-definite_possible_overlooked} and~\ref{fig:uniq-definite_possible_overlooked} show how all types of landmark correlate to the performance of heuristic approaches over incomplete domain models (represented by $\widetilde{GC}$ (\textit{D+P+O}) and $\widetilde{UNIQ}$ (\textit{D+P+O})). 
At low levels of incompleteness (20\% and 40\%) we have larger \textit{F1-scores} and larger numbers of \textit{definite} landmarks, leaving a smaller number of \textit{overlooked} landmarks to be inferred \textit{on the fly}. 
Under these conditions, \textit{overlooked} landmarks start off at a slight anti-correlation with performance. 
As the level of incompleteness of the domain description increases, the number of \textit{definite} landmarks decreases, but their \textit{Correlation} to performance increases. 
This suggests that an increase in the number of inferred landmarks leads to better performance. The number of \textit{overlooked} landmarks remains broadly the same over time, as they are tied to the amount of information in the observations more than they are tied to the information in the domain description, and their \textit{Correlation} to performance monotonically increases as the incompleteness increases. 
This indicates that \textit{overlooked} landmarks play an increasingly important role in the recognition performance as incompleteness increases, giving more information to our enhanced heuristics. 
The number of \textit{possible} landmarks also varies with domain incompleteness, initially increasing as the number of possible effects increases, to subsequently decrease as the number of possible effects leads to less bottlenecks in the state-space to yield landmarks. 
As the number of possible effects increases, so does their unreliability as sources of landmarks, which is reflected in their decreasing \textit{Correlation} to performance. 

As we ablate landmarks, performance drops most substantially when we remove either \textit{definite} or \textit{possible} landmarks from the enhanced heuristics (Figures~\ref{fig:gc-possible_overlooked} and \ref{fig:uniq-possible_overlooked}), indicating their importance to recognition accuracy. 
We can also see that when using \textit{overlooked} landmarks exclusively, in $\widetilde{GC}$ (\textit{O}) and $\widetilde{UNIQ}$ (\textit{O}), it provides a close approximation of the performance of the technique using all landmark types. Therefore, this is strong evidence that \textit{overlooked} landmarks are one of the most important contributions of this thesis.

Figure~\ref{fig:F1-score_comparison} compares all evaluated approaches with respect to \textit{F1-score} averages, varying the domain incompleteness from 20\% to 80\%. 
Thicker lines represent the recognition approaches that have higher \textit{F1-scores}. 
Note that our enhanced heuristics ($\widetilde{GC}$ and $\widetilde{UNIQ}$) that combine the use of the new types of landmarks, \idest, \textit{D+O}, \textit{D+P+O}, and $O$ are the approaches that have the higher \textit{F1-scores} over the datasets, showing that using \textit{overlooked} landmarks substantially improves goal recognition accuracy. 

\afterpage{
\begin{landscape}
\begin{table*}[tb]
\centering
\setlength\tabcolsep{1.7pt}
\fontsize{9}{12}\selectfont
\fontfamily{cmr}\selectfont

\begin{tabular}{ccccccclcccccclcccccclcccccc}
\toprule
		 \cline{2-7} \cline{9-14} \cline{16-21} \cline{23-28}
         & \multicolumn{6}{c}{Domain Incompleteness 20\%} &
		 & \multicolumn{6}{c}{Domain Incompleteness 40\%} &
		 & \multicolumn{6}{c}{Domain Incompleteness 60\%} &
		 & \multicolumn{6}{c}{Domain Incompleteness 80\%}
		 \\ \cline{2-7} \cline{9-14} \cline{16-21} \cline{23-28}
         & $|D|$    & $|P|$    & $|O|$	&  $|S|$  & $|F_{1}|$    & $CD/CP/CO$    &  & $|D|$    & $|P|$    & $|O|$    &  $|S|$  &  $|F_{1}|$    & $CD/CP/CO$  &  & $|D|$    & $|P|$    & $|O|$    &  $|S|$  & $|F_{1}|$    & $CD/CP/CO$  &  & $|D|$    & $|P|$    & $|O|$    &  $|S|$  &  $|F_{1}|$    & $CD/CP/CO$  \\ \cline{1-7} \cline{9-14} \cline{16-21} \cline{23-28}
Baseline ($\mathit{h_{gc}}$) 	& 11.5 & 0 & 0 & 1.53 & 0.44 & 0.35/0/0  &
							& 4.9 & 0 & 0 & 2.27 & 0.36 & 0.42/0/0 &
							& 3.5 & 0 & 0 & 3.32 & 0.31 & 0.65/0/0 &
							& 2.9 & 0 & 0 & 4.85 & 0.34 & 0.67/0/0 \\

Baseline ($\mathit{h_{uniq}}$) 	& 11.5 & 0 & 0 & 1.48 & 0.44 & 0.32/0/0 &
							& 4.9 & 0 & 0 & 1.95 & 0.34 & 0.36/0/0 &
							& 3.5 & 0 & 0 & 2.83 & 0.32 & 0.33/0/0 &
							& 2.9 & 0 & 0 & 4.12 & 0.33 & 0.40/0/0 \\

$\widetilde{GC}$ (\textit{D+P+O})	& 11.6 & 1.4 & 10.1 & 1.34 & 0.75 & 0.22/0.13/-0.51 &
						 			& 9.6 & 2.1 & 9.6 & 1.46 & 0.74 & 0.56/0.23/-0.20 &
									& 7.1 & 2.3 & 9.4 & 1.75 & 0.68 & 0.73/0.41/0.16 &
									& 6.8 & 1.2 & 9.3 & 1.98 & 0.65 & 0.61/0.02/0.33 \\

$\widetilde{GC}$ (\textit{D+O}) & 11.6 & 0 & 12.3 & 1.37 & 0.74 & 0.23/0/-0.42 &
								& 9.6 & 0 & 11.9 & 1.54 & 0.73 & 0.57/0/-0.11 &
								& 7.1 & 0 & 12.8 & 1.88 & 0.68 & 0.76/0/0.37 &
								& 6.8 & 0 & 10.9 & 2.10 & 0.64 & 0.62/0/0.35 \\

$\widetilde{GC}$ (\textit{P+O}) & 0 & 1.4 & 20.7 & 5.22 & 0.31 & 0/-0.13/-0.44 &
								& 0 & 2.1 & 16.8 & 4.77 & 0.34 & 0/-0.06/-0.30 &
								& 0 & 2.3 & 14.2 & 4.13 & 0.33 & 0/0.25/-0.22 &
								& 0 & 1.2 & 13.3 & 4.87 & 0.27 & 0/-0.20/-0.03 \\

$\widetilde{GC}$ (\textit{P}) 	& 0 & 1.4 & 0 & 2.14 & 0.22 & 0/0.29/0 &
								& 0 & 2.1 & 0 & 2.12 & 0.28 & 0/0.07/0 &
								& 0 & 2.3 & 0 & 2.09 & 0.26 & 0/0.56/0 &
								& 0 & 1.2 & 0 & 1.85 & 0.18 & 0/0.06/0 \\

$\widetilde{GC}$ (\textit{O}) 	& 0 & 0 & 23.1 & 1.88 & 0.71 & 0/0/-0.21 &
								& 0 & 0 & 20.0 & 2.02 & 0.70 & 0/0/0.05 &
								& 0 & 0 & 17.4 & 2.30 & 0.65 & 0/0/0.42 &
								& 0 & 0 & 15.1 & 2.46 & 0.63 & 0/0/0.34 \\

$\widetilde{UNIQ}$ (\textit{D+P+O}) & 11.6 & 1.4 & 10.1 & 1.32 & 0.71 & 0.21/0.23/-0.37 &
									& 9.6 & 2.1 & 9.6 & 1.41 & 0.70 & 0.28/0.21/-0.10 &
									& 7.1 & 2.3 & 9.4 & 1.65 & 0.64 & 0.82/0.59/0.31 &
									& 6.8 & 1.2 & 9.3 & 1.90 & 0.62 & 0.61/0.034/0.41 \\

$\widetilde{UNIQ}$ (\textit{D+O}) 	& 11.6 & 0 & 12.3 & 1.34 & 0.68 & 0.35/0/-0.26 &
									& 9.6 & 0 & 11.9 & 1.52 & 0.65 & 0.38/0/0.11 &
									& 7.1 & 0 & 12.8 & 1.81 & 0.61 & 0.66/0/0.41 &
									& 6.8 & 0 & 10.9 & 2.10 & 0.62 & 0.63/0/0.34 \\

$\widetilde{UNIQ}$ (\textit{P+O}) 	& 0 & 1.4 & 20.7 & 5.23 & 0.32 & 0/-0.14/-0.52 &
									& 0 & 2.1 & 16.8 & 4.96 & 0.35 & 0/-0.19/-0.53 &
									& 0 & 2.3 & 14.2 & 4.80 & 0.33 & 0/0.10/-0.33 &
									& 0 & 1.2 & 13.3 & 5.31 & 0.29 & 0/-0.57/-0.45 \\

$\widetilde{UNIQ}$ (\textit{P})		& 0 & 1.4 & 0 & 2.05 & 0.28 & 0/-0.27/0 &
									& 0 & 2.1 & 0 & 2.00 & 0.27 & 0/-0.42/0 &
									& 0 & 2.3 & 0 & 1.99 & 0.27 & 0/-0.19/0 &
									& 0 & 1.2 & 0 & 1.81 & 0.27 & 0/-0.55/0 \\

$\widetilde{UNIQ}$ (\textit{O}) 	& 0 & 0 & 23.1 & 1.75 & 0.69 & 0/0/-0.11 &
									& 0 & 0 & 20.0 & 1.91 & 0.77 & 0/0/0.03 &
									& 0 & 0 & 17.4 & 2.16 & 0.63 & 0/0/0.39 &
									& 0 & 0 & 15.1 & 2.35 & 0.61 & 0/0/0.33 \\
\hline
\bottomrule
\end{tabular}
\caption{Experimental results for our ablation study, comparing the baseline approaches $h_{gc}$ and $h_{uniq}$ (\ref{appendixA:goalrecognition_heuristics}) against our enhanced heuristics using various combinations of landmark types. This table aggregates the average results for all metrics over all datasets for all levels of observability.}
\label{tab:AblationStudy}
\end{table*}
\end{landscape}
}

\newpage

In \ref{appendixB:goalrecognition_incompletedomains} we report detailed results for all evaluated approaches and domains, varying not only domain incompleteness but also the percentage of observability of the observation sequence, showing the averages for all types of landmarks, \textit{F1-score}, and \textit{Correlation}. Namely, each inner table in Tables~\ref{tab:AblationStudy_Observability}~and~\ref{tab:AblationStudy_Observability_70_100} summarize the results for each percentage of observability (10\%, 30\%, 50\%, 70\%, and 100\%) over the evaluated datasets. Thus, it is possible to see that the combination of all notions of landmarks (\textit{D+P+O}) when applied for both enhanced heuristics outperforms the other combinations, including the baseline approaches (\ref{appendixA:goalrecognition_heuristics}), in all variations of domain incompleteness and observability. 

\begin{figure}[h!]
	\centering
\vspace{-3mm}	
	\begin{subfigure}[b]{0.41\linewidth}
		\centering
 	   \includegraphics[width=1\linewidth]{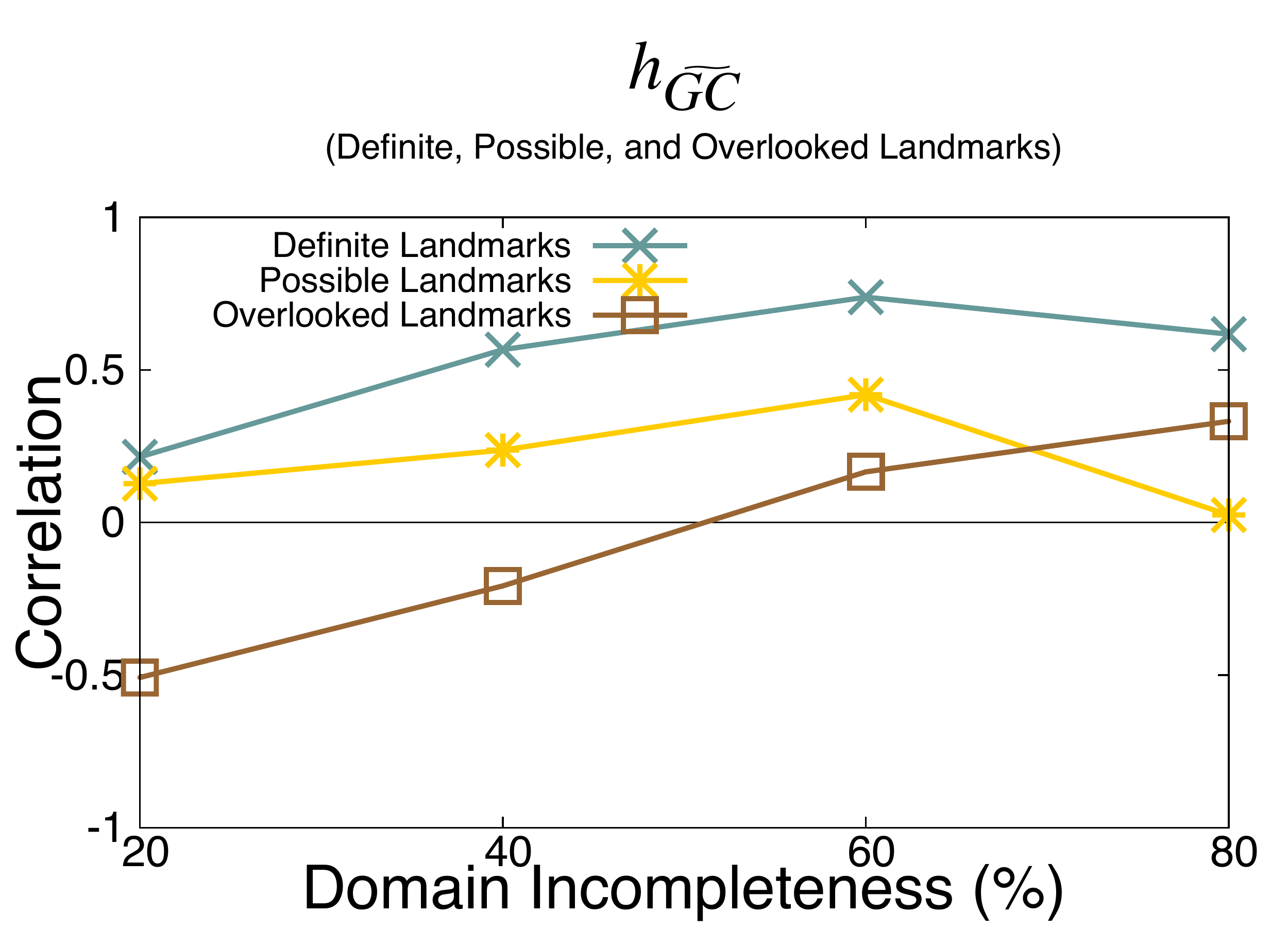}
	    \caption{$\mathit{h_{\widetilde{GC}}}$ (\textit{D+P+O}).}
	    \label{fig:gc-definite_possible_overlooked}
	\end{subfigure}
	\begin{subfigure}[b]{0.41\linewidth}
		\centering
	    \includegraphics[width=1\linewidth]{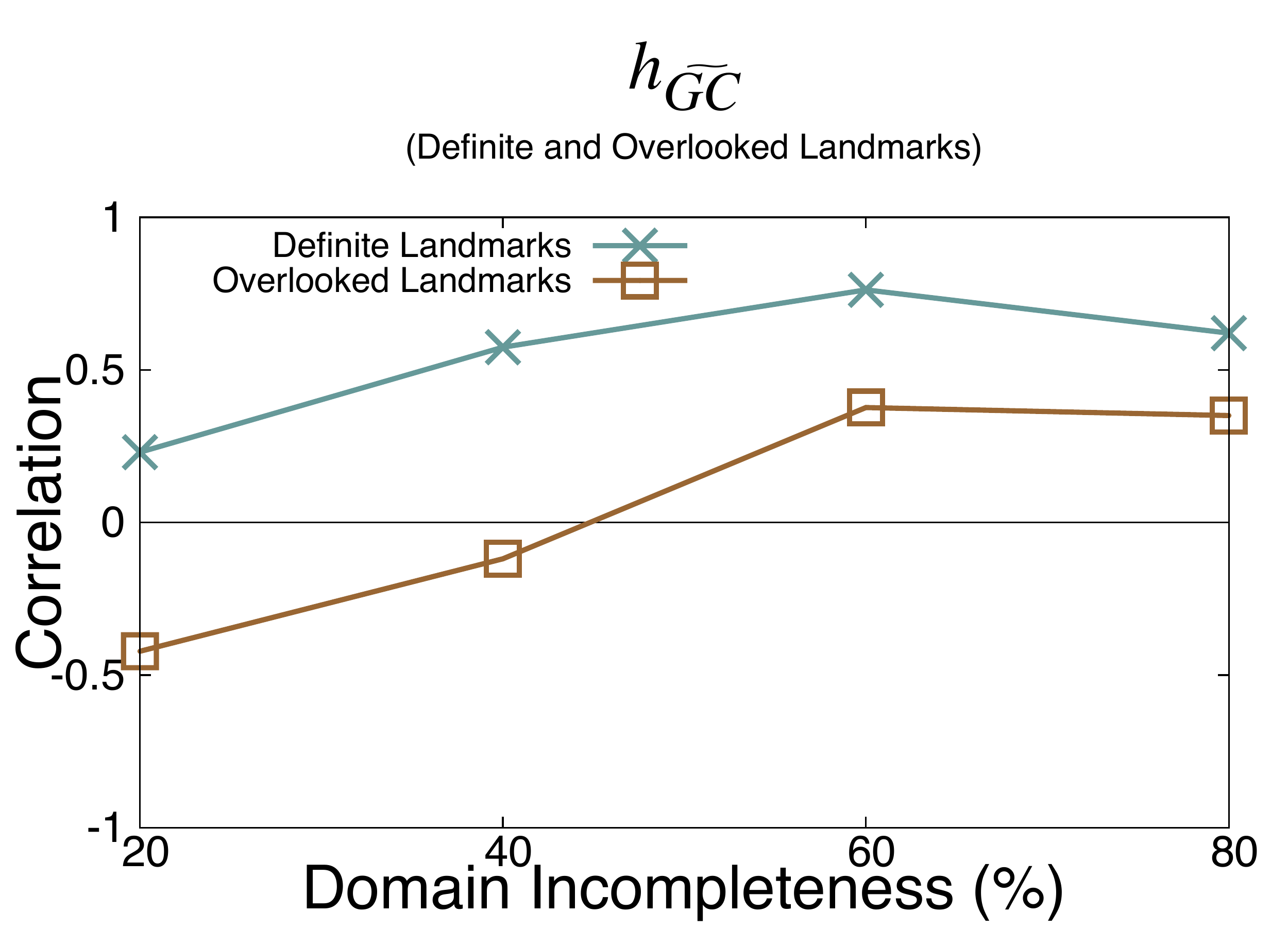}
	    \caption{$\mathit{h_{\widetilde{GC}}}$ (\textit{D+O}).}
	    \label{fig:gc-definite_overlooked}
	\end{subfigure}

	\begin{subfigure}[b]{0.41\linewidth}
		\centering
	    \includegraphics[width=1\linewidth]{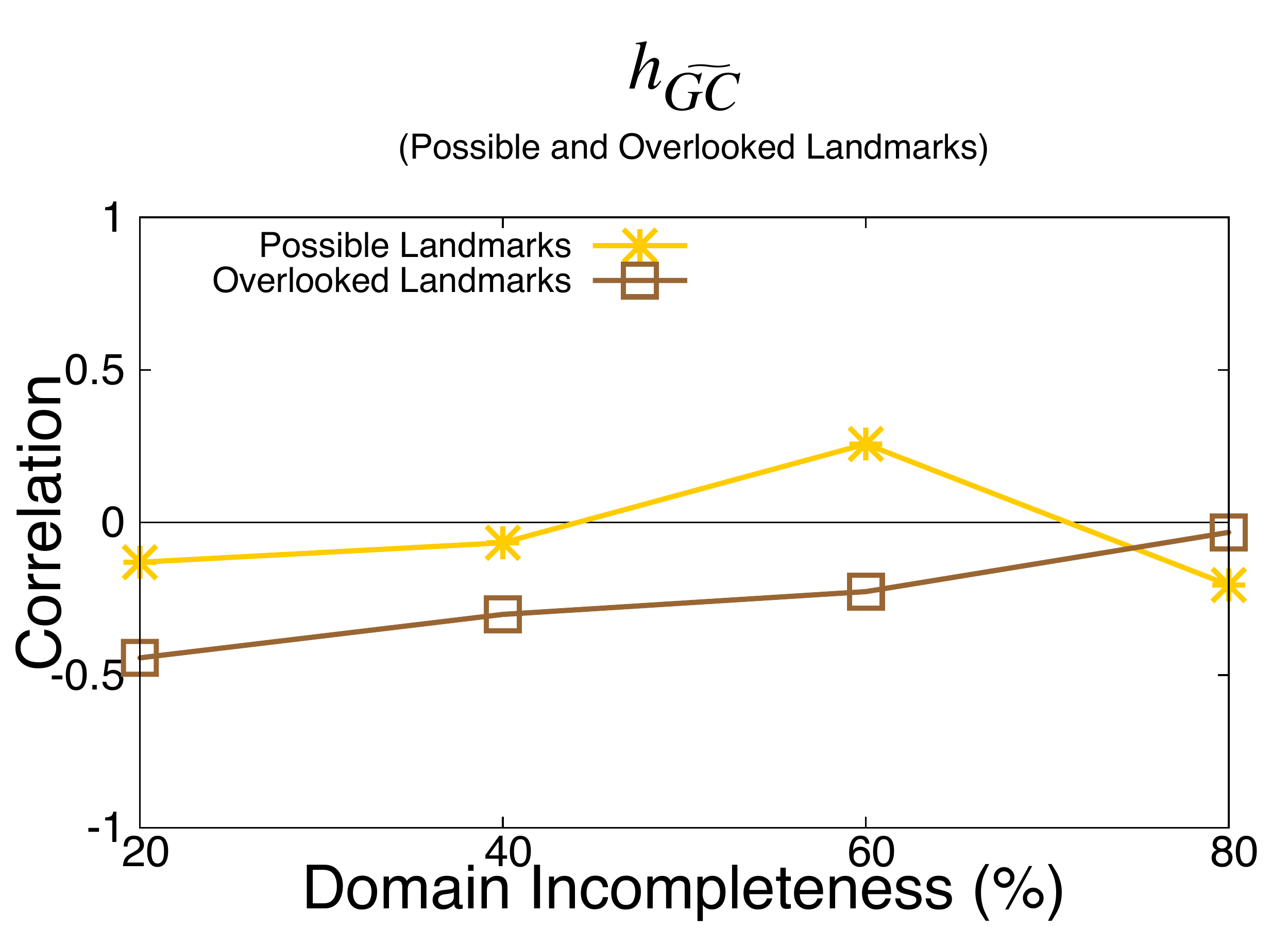}
	    \caption{$\mathit{h_{\widetilde{GC}}}$ (\textit{P+O}).}
	    \label{fig:gc-possible_overlooked}
	\end{subfigure}	
	\begin{subfigure}[b]{0.41\linewidth}
		\centering
	    \includegraphics[width=1\linewidth]{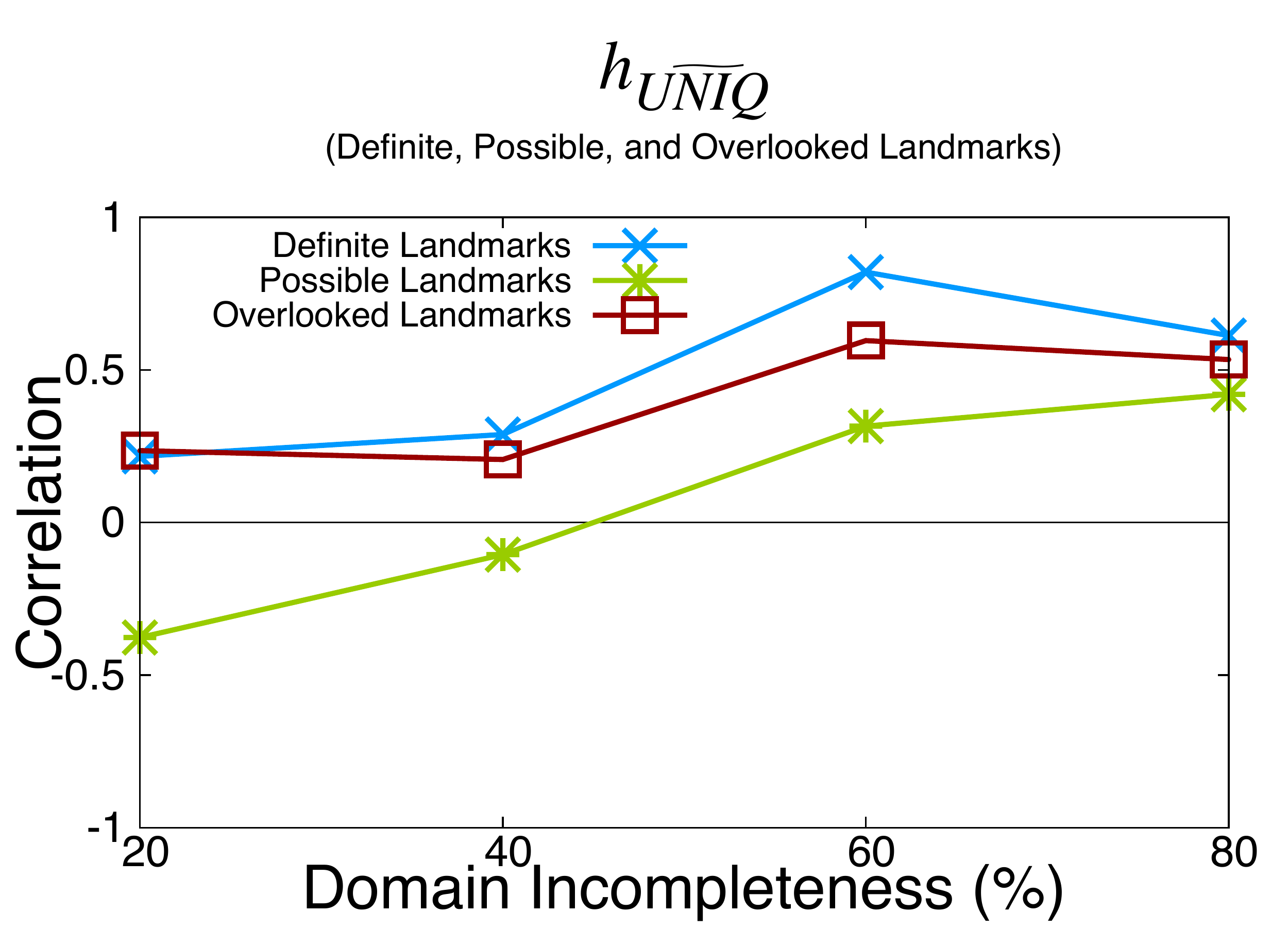}
	    \caption{$\mathit{h_{\widetilde{UNIQ}}}$ (\textit{D+P+O}).}
	    \label{fig:uniq-definite_possible_overlooked}
	\end{subfigure} 

	\begin{subfigure}[b]{0.4\linewidth}
		\centering
	    \includegraphics[width=1\linewidth]{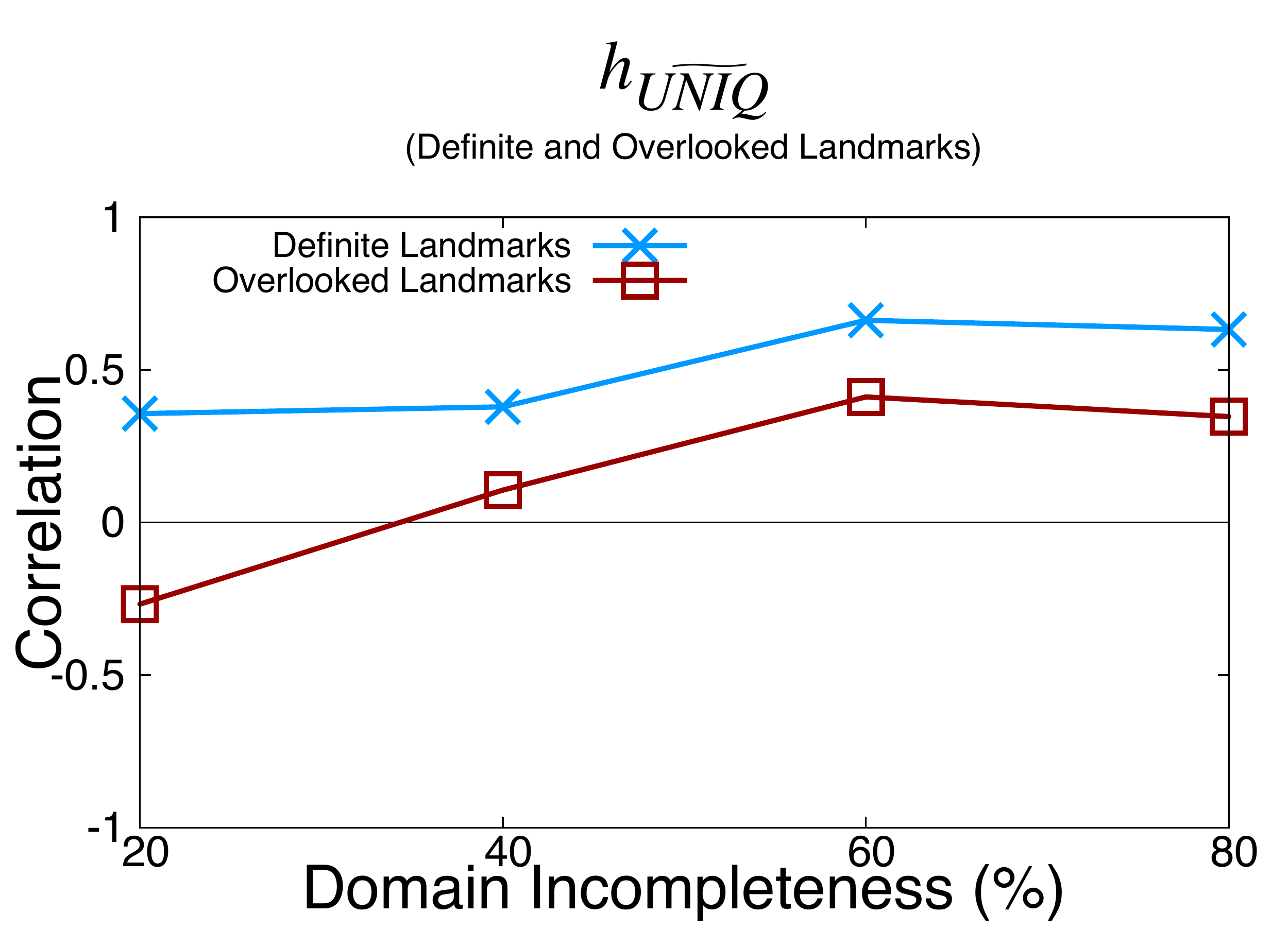}
	    \caption{$\mathit{h_{\widetilde{UNIQ}}}$ (\textit{D+O}).}
	    \label{fig:uniq-definite_overlooked}
	\end{subfigure}
	\begin{subfigure}[b]{0.4\linewidth}
		\centering
	    \includegraphics[width=1\linewidth]{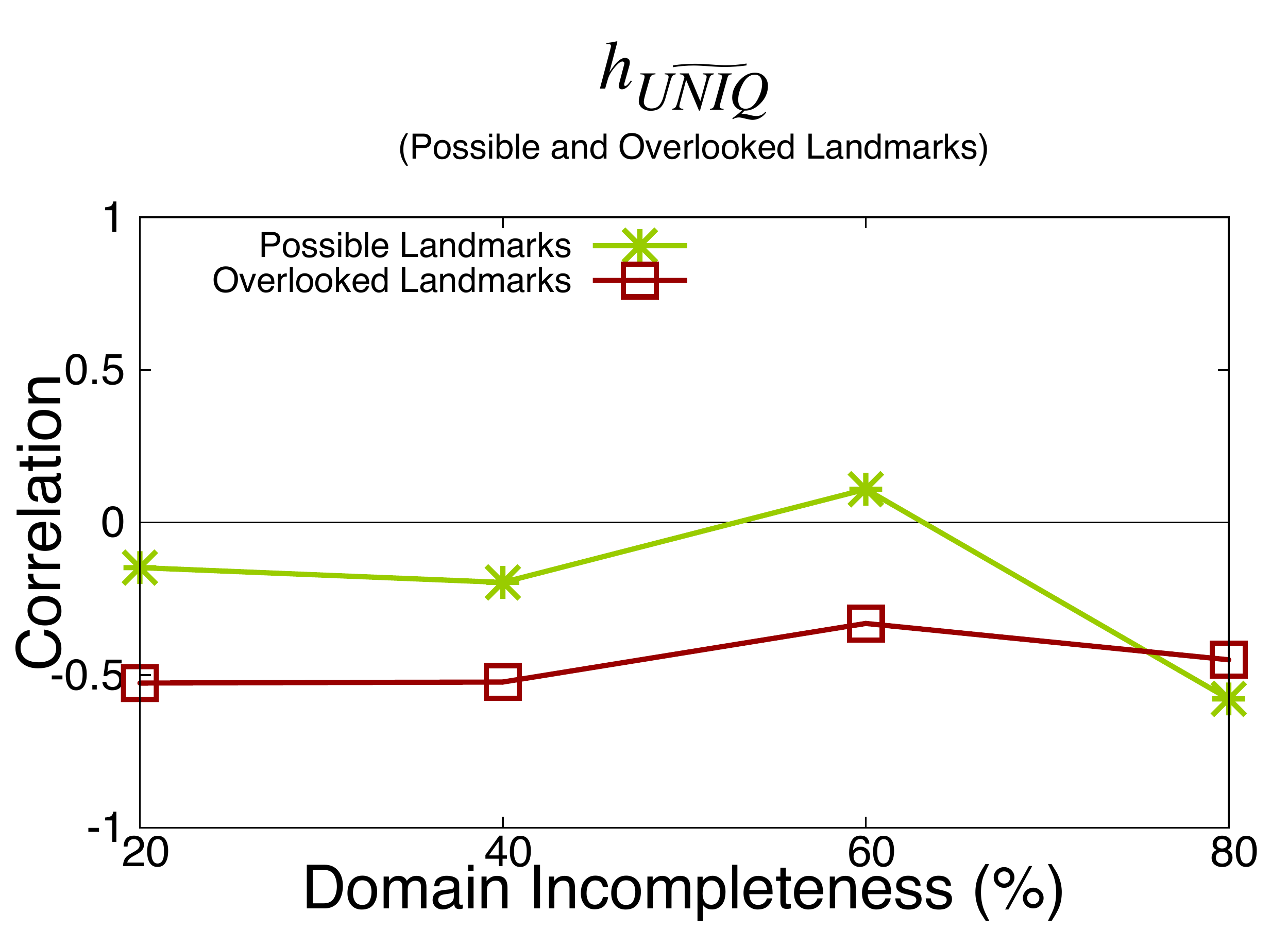}
	    \caption{$\mathit{h_{\widetilde{UNIQ}}}$ (\textit{P+O}).}
	    \label{fig:uniq-possible_overlooked}
	\end{subfigure} 
	\caption{\textit{Correlation} of landmarks to performance (\textit{F1-score}). .}
	\label{fig:correlation}
\end{figure}
% \vspace{-2mm}
% 
\begin{figure}[h!]
	\centering
	\includegraphics[width=0.65\linewidth]{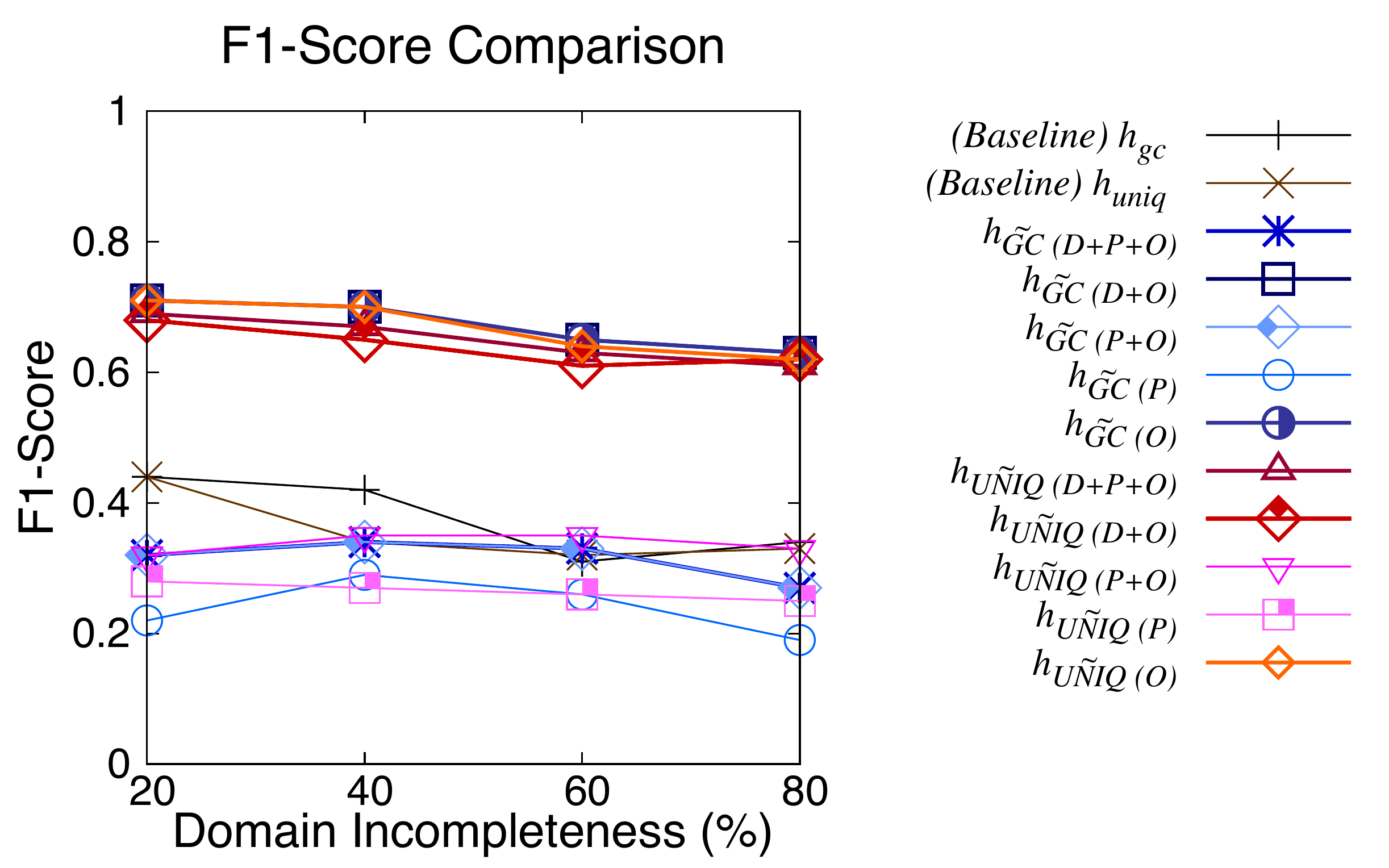}
	\caption{\textit{F1-score} average of all evaluated approaches over the datasets varying the domain incompleteness.}
	\label{fig:F1-score_comparison}
\end{figure}

%######################################################################
\subsection{Experimental Results: ROC Space Analysis}\label{subsec:GR_IncompleteDomains_ROCSpace}

We now present our second set of experiments, comparing the results of our enhanced heuristic approaches against the baselines($h_{gc}$ and $h_{uniq}$) (\ref{appendixA:goalrecognition_heuristics}) using ROC space, which shows the trade-off between true positive and false positive results. 
The use of ROC space allows us to see graphically what approach is more accurate for recognizing goals over the datasets we used. Therefore, the approach that has more points in the upper left corner (\idest, true positive rate equals to 100\%, and false positive rate equals to 0\%) is the most accurate approach over the used datasets.
Figure~\ref{fig:ROC_Curve_Uniq} shows ROC space graphs corresponding to recognition performance over the four percentages of domain incompleteness we used in our experiments. 
We aggregate multiple recognition problems for all domains and plot these results in ROC space varying the percentage of domain incompleteness. 

We report the results of our enhanced heuristics ($\mathit{h_{\widetilde{GC}}}$ and $\mathit{h_{\widetilde{UNIQ}}}$) in Figure~\ref{fig:ROC_Curve_Uniq} by using the combination of landmarks that has the best results when applied to our heuristics (\idest, using \textit{definite}, \textit{possible}, and \textit{overlooked} landmarks, as shown in the previous section), against the baselines $\mathit{h_{gc}}$ and $\mathit{h_{uniq}}$ (\ref{appendixA:goalrecognition_heuristics}), that uses just the landmarks extracted by a traditional landmark extraction algorithm, \idest, ignoring the incomplete part of the domain model (possible preconditions and effects).
Although the true positive rate is high for most recognition problems at most percentages of domain incompleteness, as the percentage of domain incompleteness increases, the false positive rate also increases, leading to several problems being recognized with a performance close to the random guess line. 
This happens because the number of extracted landmarks decreases significantly as the number of known preconditions and effects diminishes, and consequently, all candidate goals have few (if any) landmarks.
For example, in several cases in which domain incompleteness is 60\% and 80\%, the set of landmarks is quite similar, leading our enhanced heuristics to return more than one candidate goal as the correct one. 
Thus, there are more returned goals during the recognition process as incompleteness increases. 
These results show that our enhanced heuristics perform better and are more accurate than the baselines. It is possible to see that both our enhanced heuristics aggregate most points in the left corner, while the points for the baseline approaches are closer to (and sometimes below) the random guess line.

\begin{figure*}[t]
\centering
\begin{minipage}[t]{0.24\linewidth}
	\centering
    \includegraphics[width=1\linewidth]{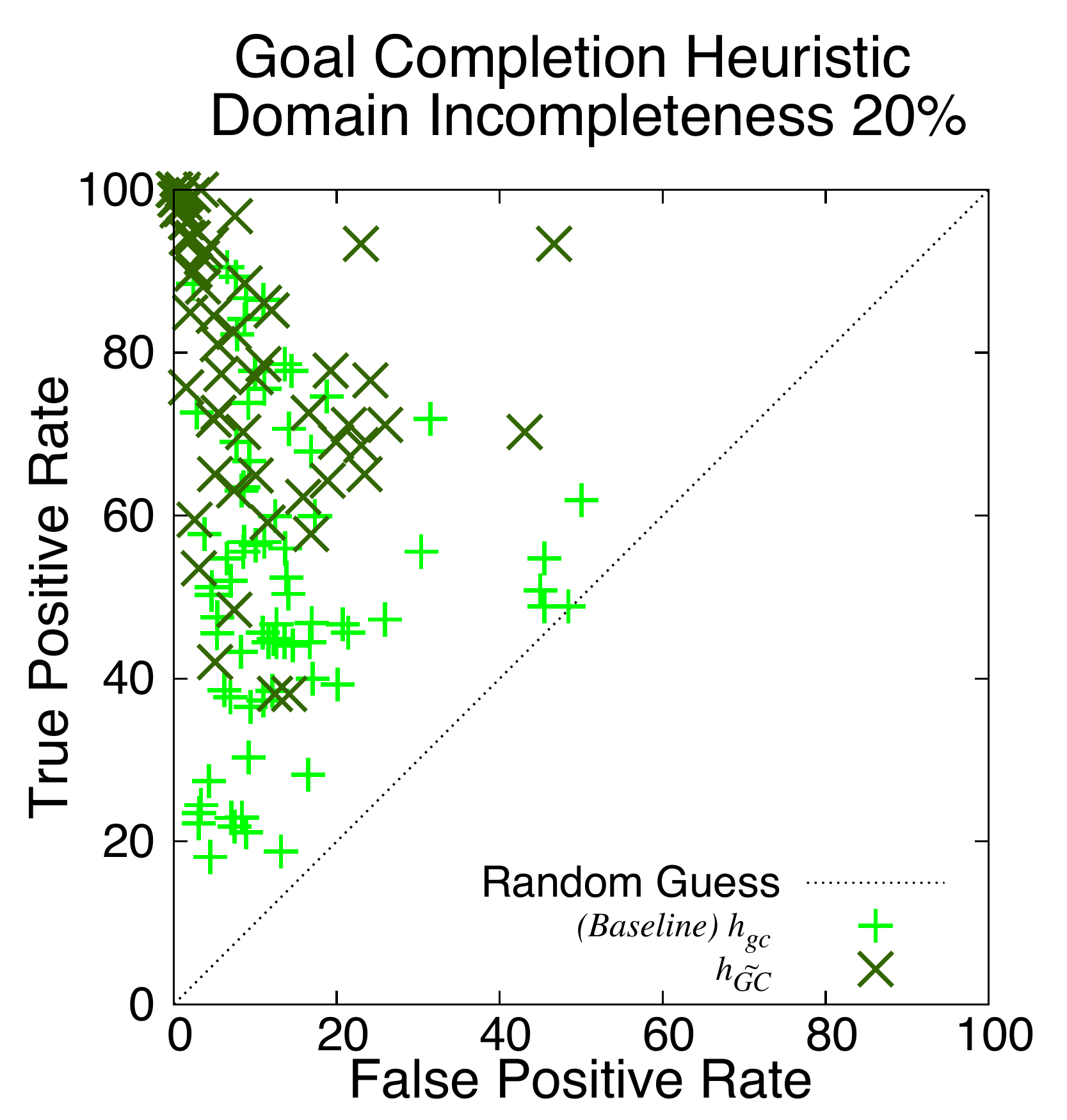}
	\label{fig:ROC_Curve_20_GC}
\end{minipage}
\begin{minipage}[t]{0.24\linewidth}
	\centering
    \includegraphics[width=1\linewidth]{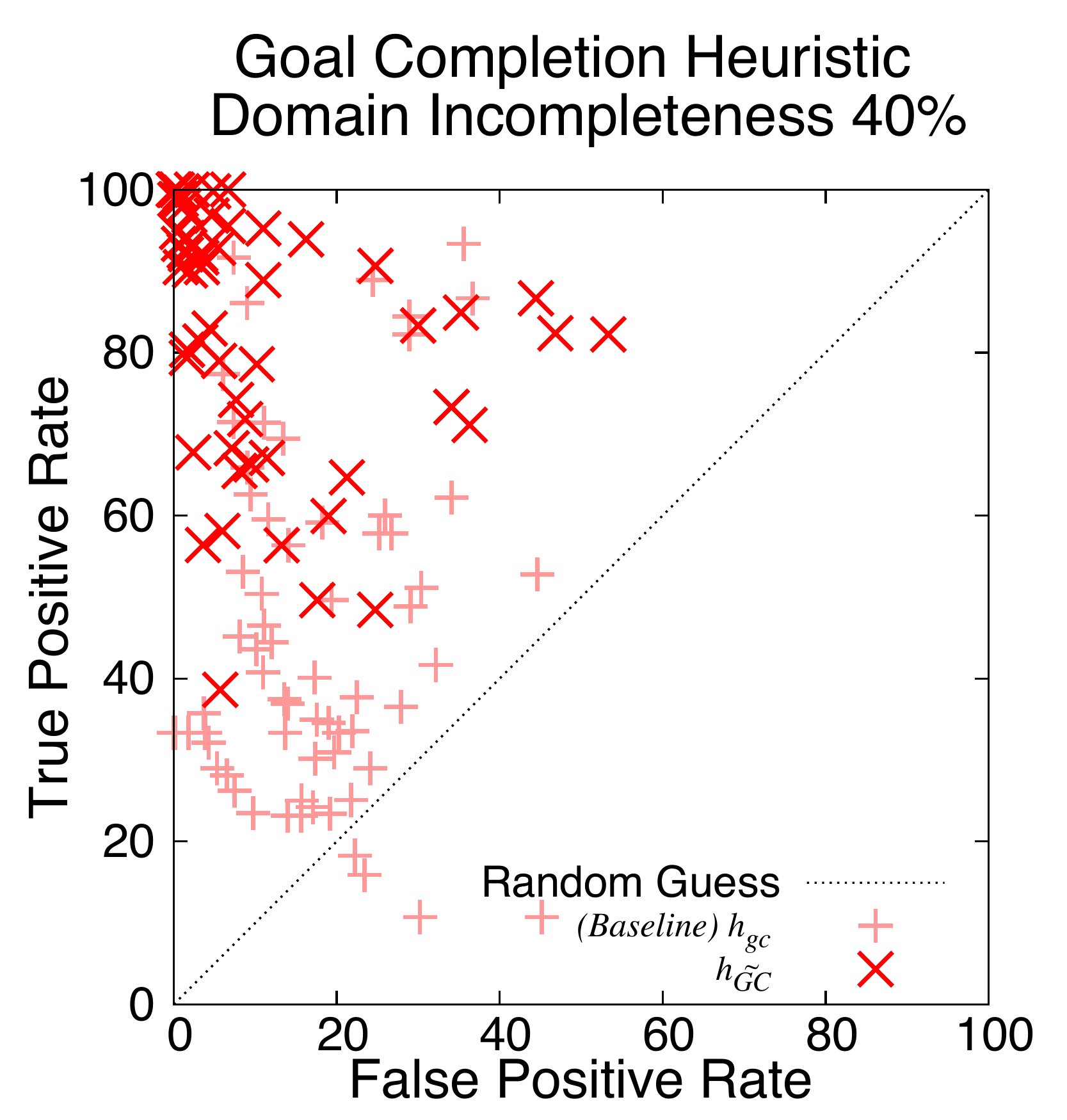}
	\label{fig:ROC_Curve_40_GC}
\end{minipage}
\begin{minipage}[t]{0.24\linewidth}
	\centering
    \includegraphics[width=1\linewidth]{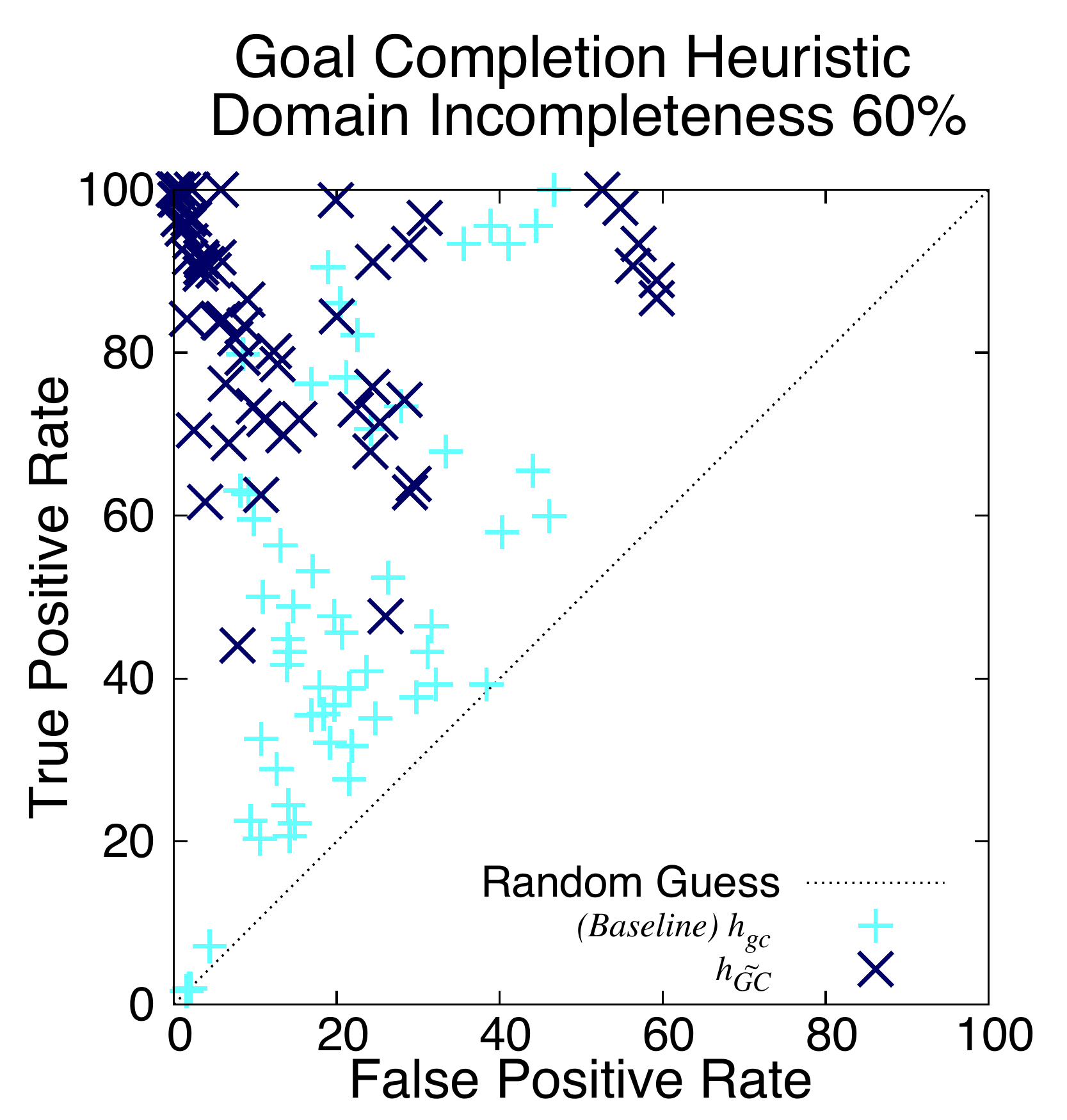}
	\label{fig:ROC_Curve_60_GC}
\end{minipage}
\begin{minipage}[t]{0.24\linewidth}
	\centering
    \includegraphics[width=1\linewidth]{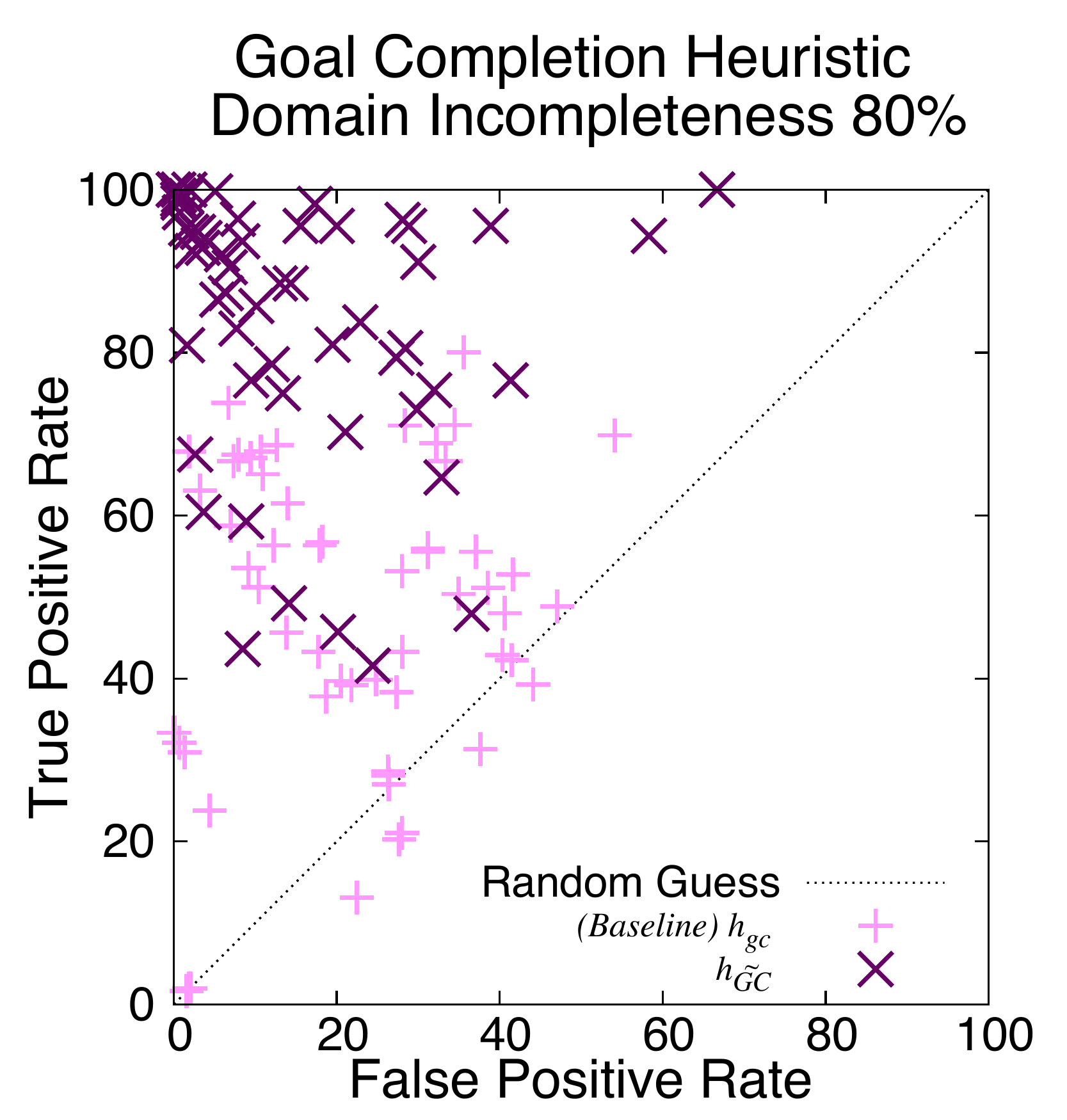}
	\label{fig:ROC_Curve_80_GC}
\end{minipage}

\centering
\begin{minipage}[h]{0.24\linewidth}
	\centering
    \includegraphics[width=1\linewidth]{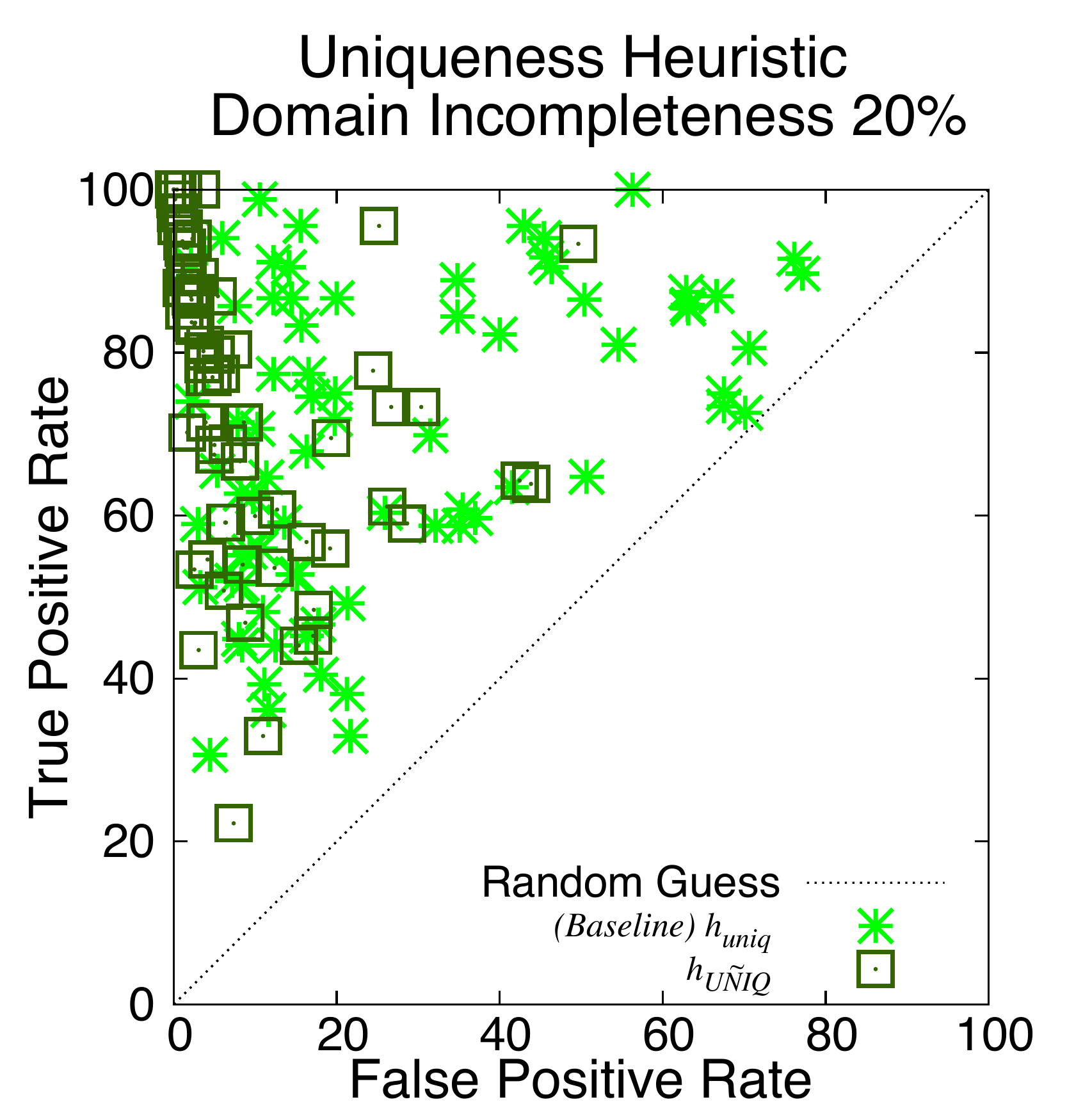}
	\label{fig:ROC_Curve_20_Uniq}
\end{minipage}
\begin{minipage}[h]{0.24\linewidth}
	\centering
    \includegraphics[width=1\linewidth]{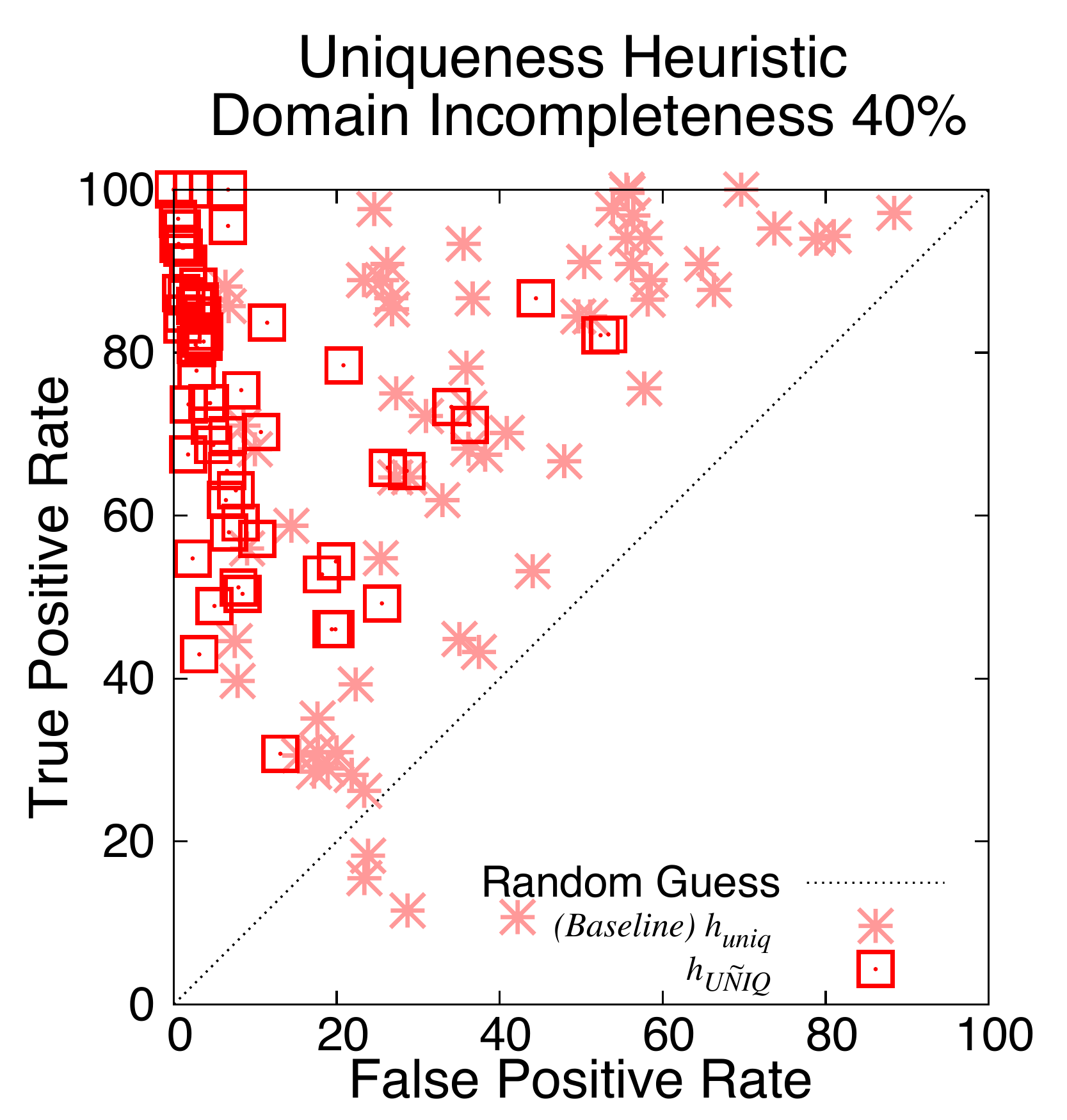}
	\label{fig:ROC_Curve_40_Uniq}
\end{minipage}
\begin{minipage}[h]{0.24\linewidth}
	\centering
    \includegraphics[width=1\linewidth]{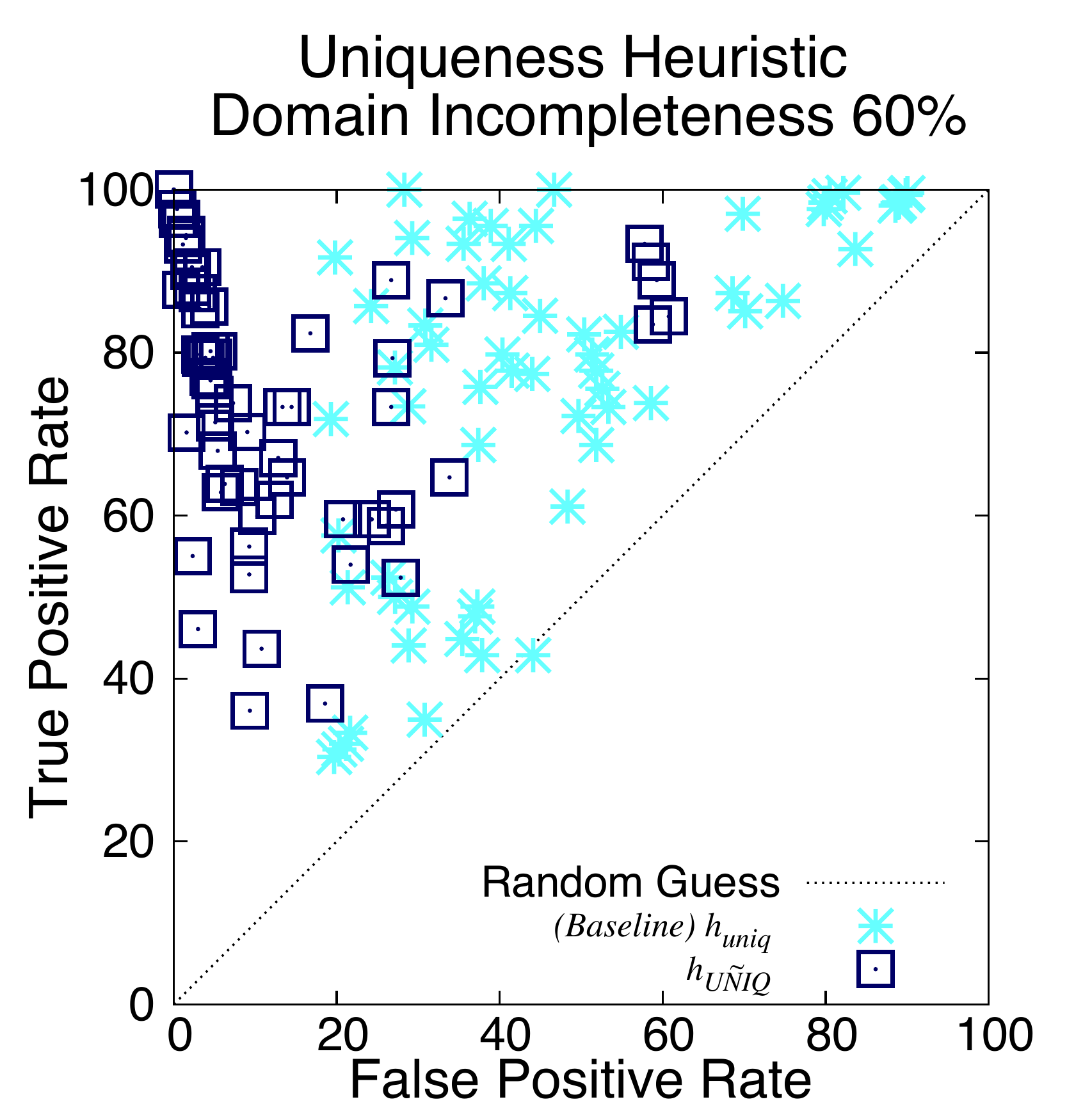}
	\label{fig:ROC_Curve_60_Uniq}
\end{minipage}
\begin{minipage}[h]{0.24\linewidth}
	\centering
    \includegraphics[width=1\linewidth]{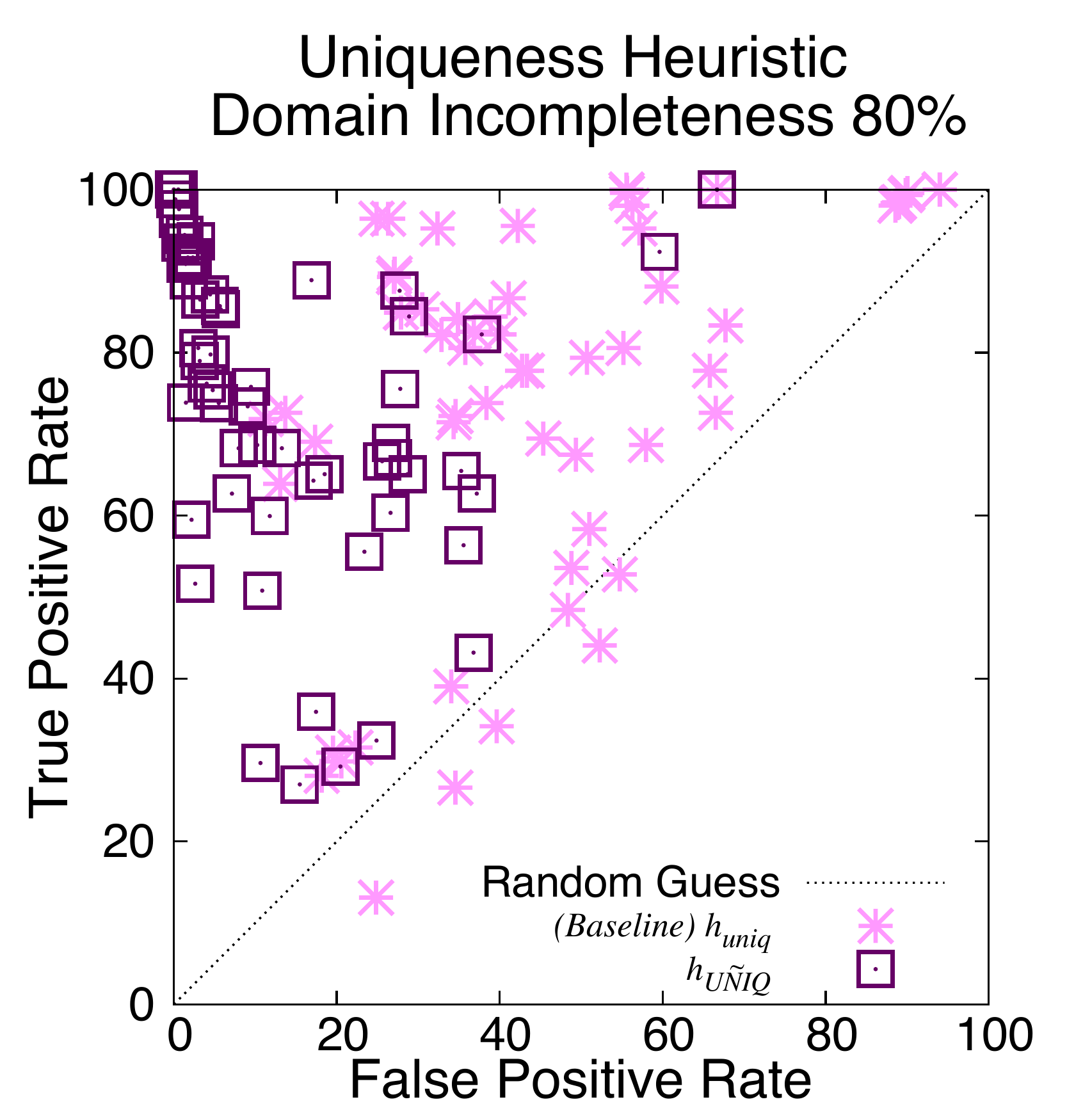}
	\label{fig:ROC_Curve_80_Uniq}
\end{minipage}
\caption{ROC space analysis for all four percentage of domain incompleteness, comparing our enhanced heuristic approaches against the baselines $h_{gc}$ and $h_{uniq}$ (\ref{appendixA:goalrecognition_heuristics}).}
\label{fig:ROC_Curve_Uniq}
\end{figure*}

%----------------------------------------------------------------------------------
\newpage
\section{Chapter Remarks}\label{section:GR_IncompleteDomains_Remarks}

In this chapter, we developed novel goal recognition heuristic approaches that cope with incomplete domain models that represent possible preconditions and effects besides traditional complete models where such information is assumed to be known. 
We developed novel recognition heuristics by exploiting the new notions of landmarks over incomplete domain models.
Our new notions of landmarks include that of \textit{possible} landmarks for incomplete domains as well as \textit{overlooked} landmarks that allow us to compensate fast but non-exhaustive landmark extraction algorithms, the latter of which can also be employed to improve existing goal and plan recognition approaches~\cite{PereiraMeneguzzi_ECAI2016,PereiraNirMeneguzzi_AAAI2017}. 
Experiments over thousands of goal recognition problems in fifth-ten planning domain models show two key results of our enhanced heuristics. 
First, these enhanced heuristics are fast and accurate when dealing with incomplete domains at all variations of observability and domain incompleteness. 
The use of novel heuristics frees us from using full-fledged incomplete-domain planners as part of the recognition process. 
Approaches that use planners for goal recognition are already very expensive for complete domains and are even more so in incomplete domains, since they often generate plans taking into consideration many of the possible models, and even then they often fail to generate robust plans for these domains. 
Second, our ablation study shows that our new notions of landmarks have a substantial impact on the accuracy of our heuristics over simply ignoring the uncertain information from the domain model, as we use in the baseline approaches. 
Importantly, the ablation study shows that overlooked landmarks contribute substantially to the accuracy of our approach. 
As future work, we envision such techniques to be instrumental in using learned planning models~\cite{AAAI2018_MasataroAsai_PlanningLatSpace} for goal recognition~\cite{Amado2018}. 

In summary, we have presented in this chapter the following contribution to the goal and plan recognition community.

\begin{enumerate}
    \item We extended the goal and plan recognition problem introduced by Ramírez and Geffner in~\cite{RamirezG_IJCAI2009,RamirezG_AAAI2010}, and defined a new problem for recognizing goals over incomplete domains;
    \item We introduced new notions of landmarks for incomplete domains models, \idest, \textit{definite}, \textit{possible}, and \textit{overlooked} landmarks;
    \item We developed a novel algorithm to extract these new notions of landmarks over incomplete domains; and
    \item We enhanced landmark--based heuristics from the literature~\cite{PereiraNirMeneguzzi_AAAI2017} to cope with incomplete domains and use our notions of landmarks, and showed that these notions of landmarks have a significant impact on the recognition performance over incomplete domain models.
\end{enumerate}    

%% file: cap4_GR_NominalModels.tex
%!TEX root = ppgcc-thesis.tex
%----------------------------------------------------------------------------------
% Chapter: Goal Recognition over Nominal Models
%
%----------------------------------------------------------------------------------
\chapter{Goal Recognition over Nominal Models}\label{chapter:GR_NominalModels}

Existing model--based approaches to goal and plan recognition rely on expert knowledge to produce 
symbolic descriptions of the dynamic constraints domain objects are subject to, and these are assumed to yield correct predictions. 
In this chapter, we develop goal recognition approaches (Sections~\ref{section:GR_NominalModels_Mirroring} and~\ref{section:GR_NominalModels_Counterfactual}) that drop this assumption, and consider the use of \textit{nominal models} that we can be \textit{learned} from observations on transitions from systems with unknown dynamics. 
Leveraging existing work on the acquisition of domain models via \textit{Deep Learning} for \textit{Hybrid Planning}~\cite{SayWZS:ijcai17} we adapt and evaluate existing goal recognition approaches~\cite{RamirezG_AAAI2010,Mor_ACS_16,Kaminka_18_AAAI} to analyze how prediction
error, inherent to system dynamics identification and model learning techniques, have an impact over recognition error rates. 
We evaluate the proposed recognition approaches over \textit{nominal models} empirically in Section~\ref{section:GR_NominalModels_ExperimentsEvaluation}, using three benchmark domains based on the \emph{constrained} Linear–Quadratic Regulator (LQR) problem~\cite{bemporad:2002:lqr}, with increasing dimensions of state and action spaces, and two variations of a non--linear navigation domain proposed by Say~\etal~in~\cite{SayWZS:ijcai17}.
\sigla{LQR}{Linear–Quadratic Regulator}

%#########################################################
\section{Problem Formulation}\label{section:GR_NominalModels_Formalism}

Assuming the availability of complete and correct models is considered a strong assumption by the literature of \textit{Control} and \textit{Robotics}~\cite{mitrovic:10:adaptive}, especially when dealing with real-world and practical applications~\cite{LQR_1998,LQR_2013,Kong_Tomi_NominalM_2013,borrelli:17:predictive}, where \textit{actual model} parameters are usually unknown, and sometimes these parameters may change over time due to wear and tear of the physical components of a robot or autonomous vehicles. The state transition of the underlying system dynamics can be obtained from observations on the behavior of other agents~\cite{borrelli:17:predictive}, random excitation, or the simulation of plans and control trajectories derived from \textit{actual models} (Definition~\ref{def:ActualModel}). 
In this thesis, we adopt this stance to define the task of goal recognition over \textit{nominal models}, which models that are estimated (or learned) from past observed state transitions (Definition~\ref{def:NominalModel}).

We formally define, in Definition~\ref{def:goal_recognition_nominalmodels}, the task of goal recognition over Finite-Horizon Optimal Control (FHOC) problems (Section~\ref{section:Background:FHOC}) and \textit{nominal models} (Definition~\ref{def:NominalModel}) by following the formalism of Ramírez and Geffner~\cite{RamirezG_IJCAI2009,RamirezG_AAAI2010}, as follows.

% \frm[inline]{Suggestion for your definition environments, why don't you just redefine them so that the identifier of the definition is in bold, rather than do this kind of hack?}
\begin{definition}[\textbf{Goal Recognition Problem over a Nominal Model}]\label{def:goal_recognition_nominalmodels}
A goal recognition problem over a nominal model is given by:
\begin{itemize}
	\item An estimated transition function $\hat{f}(x_k, u_k, w_k)$, such that, $\hat{f}(x_k, u_k, w_k)$ $=$ $\hat{x}_{k+1}$, where $x_{k}$ is a state, $u_{k}$ a control input, and $w_k$ is a random variable;
	\item A cost function $J$; 
	\item An initial state $\mathcal{I}$, \idest, an arbitrary element of the set of states $S$; 
	\item A set of hypothetical candidate goals $\mathcal{G}$, including a correct hidden goal $G^{*}$ (\idest, $G^{*}$ $\in$ $\mathcal{G}$); 
	\item A sequence of observations $Obs = \langle o_1, o_2, ..., o_m\rangle$; and 
	\item A horizon $H$ (\idest, a fixed number of steps).
\end{itemize}	
\end{definition}

For recognizing goals over \textit{nominal} models, we define the sequence of observations $Obs$ to be a partial trajectory of \textit{states} $x \in S$ induced by a policy
$\pi$ (Equation~\ref{eq:policy}) that minimizes the cost function $J$. 
In general, a finite but indeterminate number of
intermediate states may be missing between any two observations $o_i$,
$o_{i+1}$ $\in$ $Obs$. 

% \frm{I now realize you have defined $\pi$ to be a sequential plan (i.e. a sequence of actions), but here (and for control, which you also need to fix in the background), $\pi$ being a policy implies that $\pi$ is either a total function of the states into the actions (i.e. $\pi: S \mapsto A$ such that you refer to $\pi(s)$), or a probability distribution of actions conditioned on states (i.e. $\pi: S \times A \mapsto \mathbb{R}$, or $\pi: S \times A \mapsto [0,1]$, and choices are $\arg\max_{a}\pi(s \mid a)$)}

% \rfp[inline]{Please, have a look at Equation~\ref{eq:policy} in the Background, it seems that I have defined something similar you commented above, right?! I added a reference to the equation, I think that now it is clearer, no?!}
% \frm[inline]{It is, just make a footnote back in Equation~\ref{eq:policy} that you will be using these two definitions of $\pi$ to be consistent with the terminology of the area, and it should be clear from context which one you use. }

In this thesis, we perform the task of goal recognition over \textit{nominal models} in two forms: \textit{online} and \textit{offline} recognition. More specifically, we draw a distinction between
\emph{online}~\cite{baker:09:cognition,Mor_ACS_16} and
\emph{offline} goal recognition, in which the former is a sequence of $m$
goal recognition problems where the observation sequence $Obs$ is obtained incrementally, while in the later the observation sequence $Obs$ is available immediately. 
We borrow the term \textit{judgment point} from Baker et al.~\cite{baker:09:cognition} to refer to the act of solving each of the $m$ goal recognition problems that follow from the arrival of each new observation.

Informally, solving a goal recognition problem requires us to select a \emph{candidate} goal $\hat{G} \in {\cal G}$ such that $\hat{G}=G^*$, on the basis of how well $\hat{G}$ predicts or explains the observation sequence $Obs$~\cite{baker:09:cognition,RamirezG_AAAI2010}. 
Typically, this cannot be done exactly, but it is possible to produce a probability distribution~\cite{RamirezG_IJCAI2009,RamirezG_AAAI2010,NASA_GoalRecognition_IJCAI2015,Sohrabi_IJCAI2016} over the set of hypothetical candidate goals $G$ $\in$ ${\cal G}$ and $Obs$, where the goals that best explain $Obs$ are the most probable ones. 
We illustrate the goal recognition process over \textit{nominal models} in Figure~\ref{fig:GoalRecognition_NominalModels}, following Definition~\ref{def:goal_recognition_nominalmodels}. 
Note that in Figure~\ref{fig:GoalRecognition_NominalModels} we refer to the approximate transition function $\hat{f}$ as ``black box'', since it is not directly accessible in our setting, and represented with a ``black box'' neural network. 
Next, in Section~\ref{section:DNN_NominalModels}, we show how we use neural networks as \textit{nominal models}.

\begin{figure}[h!]
  \centering
  \includegraphics[width=0.65\linewidth]{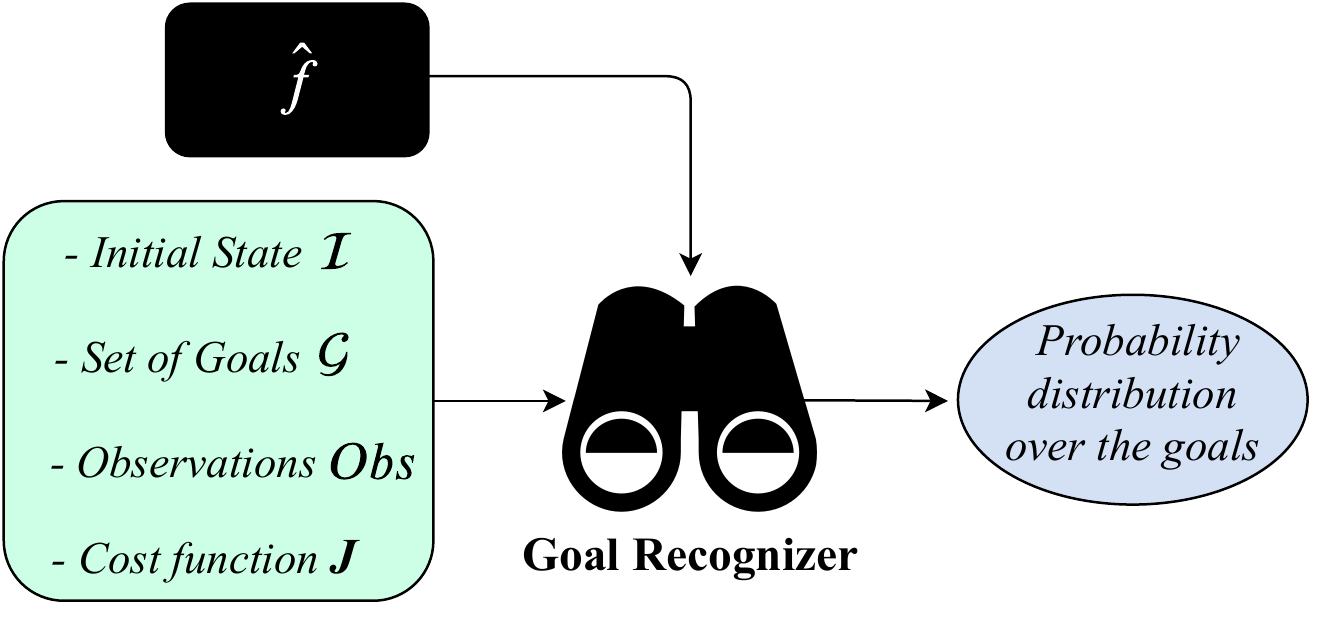}
  \caption{Problem overview to goal recognition over \textit{nominal models}.}
  \label{fig:GoalRecognition_NominalModels}
\end{figure}

%@@@@@@@@@@@@@@@@@@@@@@@@@@@@@@@@@@@@@@@@@@@@@@@@@@@@@@@@@@@@@
% \newpage
\section{DNNs as Nominal Models}\label{section:DNN_NominalModels}
\sigla{DNN}{Deep Neural Network}
\sigla{ReLU}{Rectified Linear Unit}

% \rfp[inline]{Please, have a look at this section.}

Artificial Neural Networks have shown to be very effective at learning and approximating linear and non--linear functions from data~\cite{DDNs_Functions_2000}. 
To learn and approximate the dynamics constraints in FHOC problems, and therefore, acquire \textit{nominal models}, we leverage existing work in \textit{Automated Planning} and \textit{Machine Learning} that uses Deep Neural Networks (DNNs) to approximate (linear and non--linear) functions~\cite{SayWZS:ijcai17,WuSS:nips17,SayS:ijcai18}.

Learning the dynamics or the transition between states of a domain model from data can be formalized as the problem of finding the parameters $\theta$ for a function
\[
	\hat{f}(x_k,u_k,w_k; \theta)
\]
\noindent that minimize a given \emph{loss function} ${\cal L}({\cal K}, \theta)$ over a dataset ${\cal K}$ $=$
$\{$ $(x$, $u$, $y)$ $|$ $y = f(x,u,w)$, $y \in S$ $\}$, where $x$ is the state, $u$ is the control input, $w$ represents a random variable, and $y$ is the resulting state after applying $f(x,u,w)$. 
In this thesis, we use the procedure and neural architecture reported by Say~\etal~\cite{SayWZS:ijcai17} to acquire $\hat{f}$, namely, a DNN using Rectified Linear Units (ReLUs)~\cite{nair:10:icml} as the activation function,
given by $h(x)= max(x,0)$. A DNN is densely connected, and consists of $L$ layers, $\theta$ $=$ $(\matr{W},$
$\matr{b})$, where $\matr{W} \in \thereals^{d \times d \times L}$ and
$\matr{b}$ $\in$ $\thereals^{1\times L}$.
We use the loss function proposed by Say et al.~\cite{SayWZS:ijcai17}, as follows:
\[
\sum_{j}^{|{\cal D}|} || \hat{y}_j - y_j|| + \lambda \sum_{l}^{L} ||\matr{W}_l||^2
\]
\noindent where $\hat{y}_j$ $=$ $\hat{f}(x_j,u_j;\theta)$ and $\lambda$ is
a Tikhonov $L^2$-regularization hyper-parameter~\cite{goodfellow:16:dl}. 
In the
context of optimization for Machine Learning, using this regularization technique
induces the optimization algorithm to overestimate the variance of the dataset,
so weights associated with unimportant directions of the gradient of ${\cal L}$ decay
away during training.
As noted by Goodfellow et al.~\cite{goodfellow:16:dl},
ReLU networks represent very succinctly a number of linear approximation surfaces that is exponential in the number of layers $L$~\cite{montufar:14:nips}. 
This strongly suggests that ReLU networks can displace Gaussian process estimation~\cite{rasmussen:02:gaussian} as a good initial choice to approximate complex non--linear stationary random processes, such as those in Equation~\ref{eq:dynamics}, with the further advantage that, as demonstrated in~\cite{yamaguchi:16:icra,SayWZS:ijcai17}, DNNs can be directly used in
Equation~\ref{eq:rhc_transition}, so existing optimization algorithms can be used off--the--shelf. 
Figure~\ref{fig:DNN_NominalModel} illustrates the neural architecture reported by Say~\etal~\cite{SayWZS:ijcai17}, in which we use to represent \textit{nominal models}. 

% \frm[inline]{The dashed lines look to me like skip connections (i.e. connections between layers that are not adjacent). They certainly do not look like dense connections. A dense connection is when every unit in a layer is connected to every other unit in the subsequent layer. Fix the definitions and the images.}
\begin{figure}[h!]
  \centering
  \includegraphics[width=0.7\linewidth]{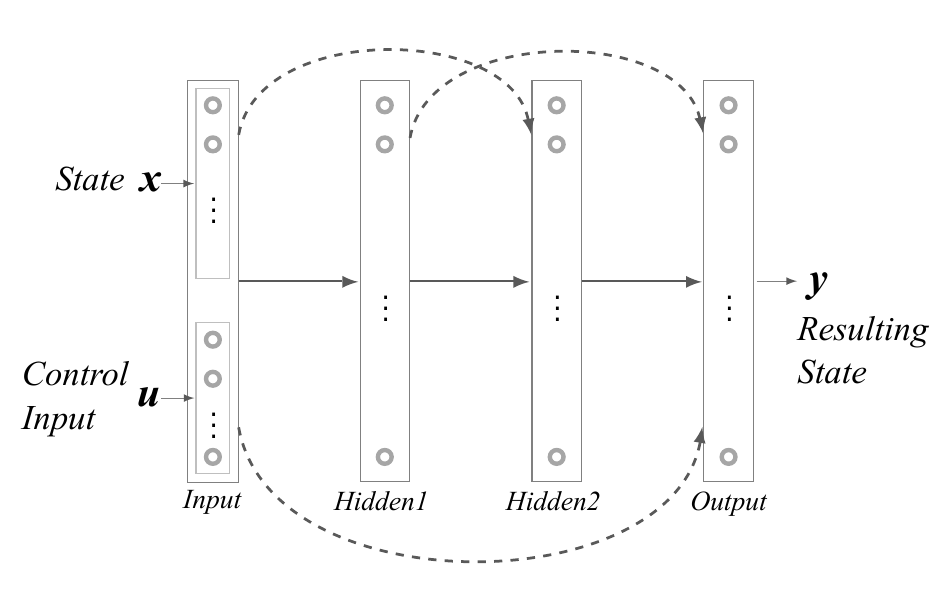}
  \caption{An example of a neural network with 4 layers (input layer, 2 hidden layers, and output layer), adapted from~\cite{SayWZS:ijcai17}. Arrows represent full connections between the input and output nodes. Every hidden layer output unit is passed through a ReLU unit. Dashed arrows are optional dense connections.}
  \label{fig:DNN_NominalModel}
\end{figure}

%#########################################################
\section{Probabilistic Goal Recognition}\label{subsection:GR_NominalModels_Probs}

We follow Ram{\'{\i}}rez and Geffner~\cite{RamirezG_AAAI2010} (R\&G10) and adopt the modern probabilistic
interpretation of Dennet's \textit{principle of rationality}~\cite{dennett:1983},
the so--called \textit{Bayesian Theory of Mind}, as introduced
by a series of ground-breaking cognitive science studies by Baker
et al.~\cite{baker:09:cognition}\cite[Chapter 7]{ActivityIntentPlanRecogition_Book2014}.
R\&G10
set the probability distribution over the set of hypothetical goals ${\cal G}$ and the observation sequence $Obs$ introduced above to be the Bayesian posterior conditional probability

% \frm[inline]{When writing about probabilities, use $P(X \mid Y)$ rather than $P(X | Y)$, the latter is symbolically incorrect.}

\begin{align}
\label{eq:posterior}
P(G \mid Obs) =  \alpha P(Obs \mid G) P(G)
\end{align}
\noindent where $P(G)$ is a \textit{prior} probability assigned to goal $G$, $\alpha$ is a normalization factor inversely proportional to the probability of $Obs$, and $P(Obs \mid G)$ is
\begin{align}
P(Obs \mid G) = \sum_{\pi} P(Obs \mid \pi) P(\pi \mid G)
\label{eq:likelihood}
\end{align}

\noindent $P(Obs \mid \pi)$ is the probability of observing $Obs$ by executing a policy (or a plan) $\pi$ and $P(\pi \mid G)$ is the probability of an agent choosing plan $\pi$ to achieve the goal $G$.
Crucially, this later probability is defined to be a function that
compares a measure of the efficiency of $\pi$ with some suitably defined baseline of rationality, ideally the optimal plan or policy for $G$.
In Sections~\ref{section:GR_NominalModels_Mirroring}~and~\ref{section:GR_NominalModels_Counterfactual}, we discuss two well--known existing approaches to approximate Equation~\ref{eq:likelihood}, both reasoning over \textit{counterfactuals}~\cite{pearl:09:causality} in different ways.
These approaches frame in a probabilistic setting the so--called \emph{but-for} test of causality~\cite{halpern:16:causality}.
That is, if a candidate goal $G$ is to be considered the \textit{cause} for observations $Obs$ to happen, evidence of $G$ being \textit{necessary} for $Obs$ to happen is required. 
The changes to existing approaches are motivated by us wanting to retain the ability to \textit{compute} counterfactual trajectories when transition functions cannot be directly manipulated.

%#########################################################
\section{Goal Recognition as Nominal Mirroring}\label{section:GR_NominalModels_Mirroring}

\textit{Mirroring}~\cite{Mor_ACS_16} is an online goal recognition approach that works on both continuous and discrete domain models. 
For each of the candidate goal $G$ in ${\cal G}$, Halpern's \emph{but-for} test is implemented by comparing two plans: an \textit{ideal} plan and the \emph{observation-matching} plan ($O$-plan). 
Ideal plans are optimal plans computed for every candidate goal $G$ in ${\cal G}$ from the initial state $\mathcal{I}$, which are \textit{pre-computed} before the recognition process starts. 
The $O$-plan is also computed for every pair $({\cal I}, G)$ and it is required to visit every state in the observation sequence $Obs$. 
$O$-plans are made of a \emph{prefix}, that results from concatenating the $O$-plans computed for previous judgment points~\cite{baker:09:cognition}, and a \emph{suffix}, a plan computed from the last observed state to each candidate goal $G$.
The \emph{but-for} test is implemented by making of use of Theorem 7 in~\cite{RamirezG_IJCAI2009}, that amounts to considering a candidate $G$ to be \emph{necessary} for $Obs$ to happen, if the cost of optimal plans and
those consistent with the observation $Obs$ are the same. 
Vered et al.~\cite{Mor_ACS_16} show that $O$-plans are indeed consistent with $Obs$ so Ram{\'{\i}}rez and Geffner's results apply. 
The test was later cast in a probabilistic framework by Kaminka et al.~\cite{Kaminka_18_AAAI}, with Equation~\ref{eq:likelihood} becoming

\begin{align}
\label{eq:o_pi_matching_error}
P(Obs \mid G) = [1 + \epsilon(\pi_{Obs,G},\pi_G)]^{-1}
\end{align}

\noindent $\epsilon(\pi_{Obs,G},\pi_G)$ above is the \textit{matching error} of $\pi_G$, the ideal plan for $G$
with regard to $\pi_{Obs,G}$, the $O$-plan for the observations.
Under the assumption that $w$ is a random variable $w$ $\sim$ ${\cal N}(0, \sigma)$ with values given by
a Gaussian distribution with mean $0$ and standard deviation $\sigma$,
$\epsilon$ can be used to account for the influence of $w$ as long as 
$\sigma$ remains an order
of magnitude smaller than the values given by $f(x,u,0)$. Kaminka et al.~\cite{Kaminka_18_AAAI} define the \textit{matching error} $\epsilon$ as the sum of the squared errors between states in the trajectory of $\pi_G$,
and those found along the trajectory of $\pi_{Obs,G}$. Under the second assumption that the selected
ideal plans for a goal $G$ are the most likely too, Kaminka et al. $\epsilon$ is an unbiased
estimator for the likelihood of $\pi_{Obs,G}$.

Having established the suitability of Kaminka et al.~\cite{Kaminka_18_AAAI} means to bring about the \emph{but--for}
test to FHOC problems, we now describe how we depart from their method to obtain $O$-plans
$\pi_{Obs,G}$. In this thesis, for online goal recognition we construct $\pi_{Obs,G}$ by calling a planner once for
each \emph{new} observation $o$ added to $Obs$ \emph{and} candidate goal $G$, rather than just once per candidate goal $G$ as proposed by Kaminka et al.~\cite{Kaminka_18_AAAI}.
In doing so, it allows us to enforce consistency with observations $Obs$, since the couplings between states, inputs and perturbation in $\hat{f}(x,u,w)$ are no longer available so we can influence them with additional constraints, but are rather ``hidden'' in the network parameters.
As the first observation $o_1$ is obtained, we call a planner to solve Equations~\ref{eq:rhc_objective}--\ref{eq:constraints} (from Chapter~\ref{chapter:Background}), setting the initial state $x_0$ to ${\cal I}$ and $x_N$ to $x_{o_1}$, the state embedded in $o_1$. 
The resulting trajectory $m^{1}_{Obs}$ is then used to initialize $\pi_{Obs,G}^{-}$ $=$ $\langle m^{1}_{Obs} \rangle$. 
We then invoke the planner again, this time setting $x_0=x_{o_1}$ and some suitably defined constraints such that $x_N \in S_G$ for every candidate goal $G$. 
The resulting trajectories $m^{1}_G$ are used to define the $O$-plans $\pi_{Obs,G}$ $=$ $\pi_{Obs,G}^{-}$ $\oplus$ $m^{1}_G$, which are compared with the pre-computed ideal plans $\pi_G$ to evaluate $P(Obs|G)$, according to Equation~\ref{eq:o_pi_matching_error}.
As further observations $o_i$, $i>1$, are received, we obtain trajectories $m^{i}_{Obs}$ as above but setting $x_0=x_{o_{i-1}}$ and $x_N$ to $x_{o_i}$, which are used to update $\pi_{Obs,G}^{-}$ setting $\pi_{Obs,G}^{-}$ $=$ $\pi_{Obs,G}^{-}$ $\oplus$ $m^{i}_{Obs}$. 
Trajectories $m^{i}_G$ are obtained by setting initial states to $x_{o_i}$, and concatenated to the updated $\pi_{Obs,G}^{-}$ to obtain the $O$-plan for the $i$-th judgment point. 
As an extension of the original \textit{Mirroring} approach developed by Vered~\etal~\cite{Mor_ACS_16}, and an adaptation of the probabilistic framework of Kaminka et al.~\cite{Kaminka_18_AAAI}, we call this approach as \textit{Nominal Mirroring}, and denote it as \nominalmirroring.

Thus, given the plans $\pi_{Obs,G}$ and $\pi_G$ for every goal $G \in \mathcal{G}$, we can then try to maximize $P(Obs \mid \pi)$, calculating the \textit{matching error} $\epsilon(\pi_{Obs,G},\pi_G)$ for these two plans. 
To calculate the \textit{matching error} $\epsilon$, we use a state-distance metric $E$ over the plans $\pi_{Obs,G}$ and $\pi_G$. 
Here, we use the Euclidean distance as the state-distance metric $E$, like Kaminka et al. have used for goal recognition over continuous domains~in~\cite{Kaminka_18_AAAI}. 
As a result, the best matching error for $\epsilon(\pi_{Obs,G},\pi_G)$ is 0 when these plans are identical (i.e., state-variables $x_{k}$ for both plans with exactly the same values) according to the state-distance metric $E$. 
Therefore, if $\epsilon(\pi_{Obs,G},\pi_G) = 0$, then $P(Obs \mid \pi)$ is equal to $1$ (Equation~\ref{eq:o_pi_matching_error}). 

Figure~\ref{fig:GR_NominalMirroring} illustrates graphically how we calculate the \textit{matching error} $\epsilon$ between ideal plans and $O$-plans for recognizing goals over \textit{nominal models}. 
For example, consider that the \textit{matching error} $\epsilon$ between the $O$-plan $\pi_{Obs,G}$ and the ideal $\pi_G$ in Figure~\ref{fig:GR_NominalMirroring} is 0.5, we can see that the plans are similar, and thus, we can use $\epsilon$ to compute $P(Obs|\pi)$ (Equation~\ref{eq:o_pi_matching_error}), as follows: $P(Obs \mid \pi) = [1 + 0.5]^{-1} = 0.66$.

\begin{figure}[h!]
  \centering
  \includegraphics[width=0.7\linewidth]{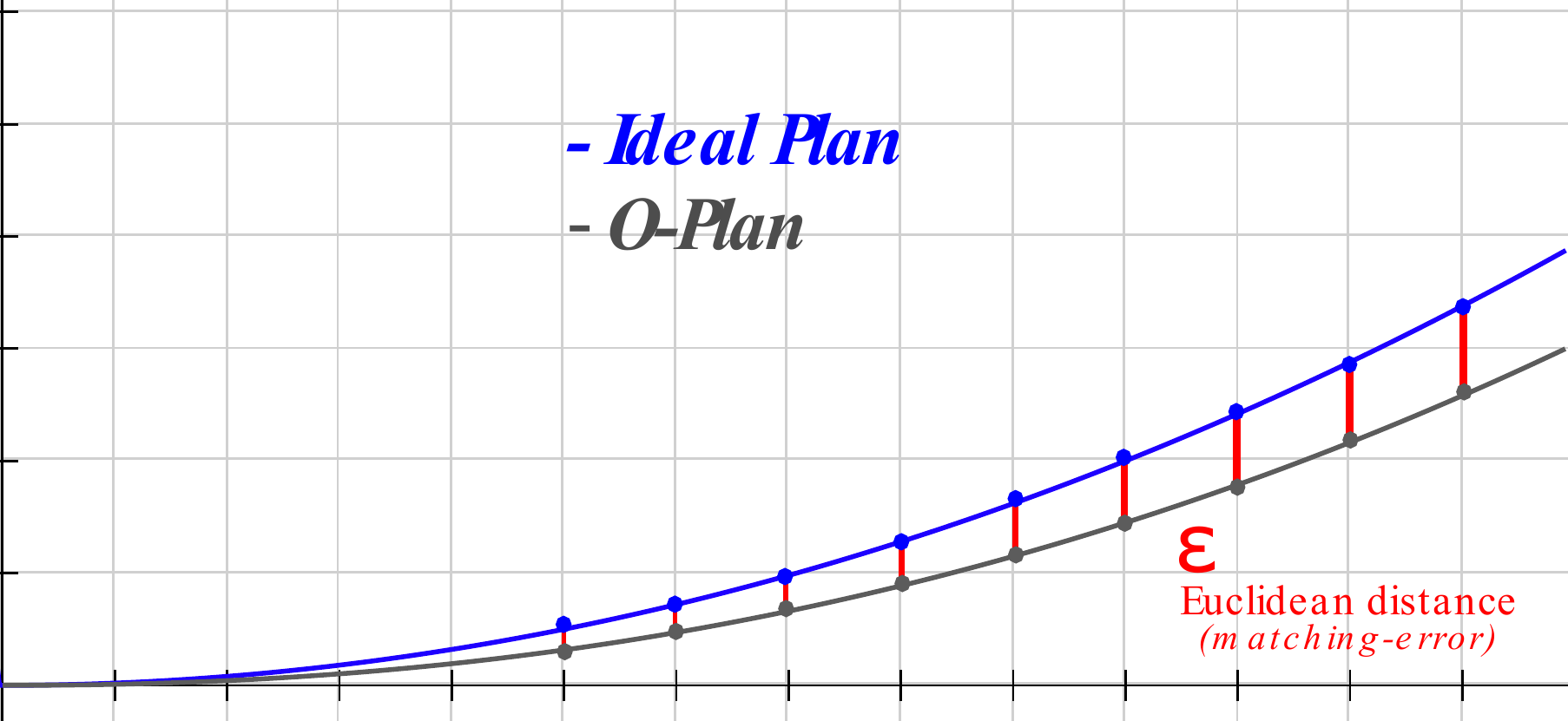}
  \caption{\textit{Nominal Mirroring} approach example.}
  \label{fig:GR_NominalMirroring}
\end{figure}

To analyze the computational complexity of \nominalmirroring, we use as a baseline Vered et al.'s~\cite{Mor_ACS_16}
original \textit{Mirroring} approach, which requires $|\mathcal{G}|$ calls to a planner per observation, and $|\mathcal{G}| + (|\mathcal{G}||Obs|)$ calls overall.
In contrast, our \nominalmirroring~requires  $|\mathcal{G}|+1$ calls to a planner
per observation, and $|\mathcal{G}| + |Obs| + (|\mathcal{G}||Obs|)$ calls overall.
This is a slight overhead which, on the basis of the results in Section~\ref{section:GR_NominalModels_ExperimentsEvaluation}, seems to be amortized enough by the accuracy and robustness of our new method.

%#########################################################
\section{Goal Recognition Based on Cost Differences}\label{section:GR_NominalModels_Counterfactual}

% \rfp[inline]{I modified this section. What do you think?!}

We now develop a novel goal recognition approach over \textit{nominal models} based on cost differences, inspired by the probabilistic framework of R\&G10. 
In~\cite{RamirezG_AAAI2010}, R\&G10 implement Halpern's \textit{but-for} test~\cite{halpern:16:causality} by determining whether plans exist that, while achieving $G$, either guarantee that $Obs$ happens, or prevent it from happening, the later being the \textit{counterfactual} plan~\cite{pearl:09:causality}. 
When no such plans exist, we obtain proof of $G$ not being \textit{sufficient} cause for $Obs$. 
Typically though, goals remain feasible, yet costs of plans change, making $G$ less likely to be the cause of $Obs$ whenever the cost of achieving $G$ is smaller when $Obs$ does not take place. 
We retain this definition of the test, yet we do not obtain such plans from the solution of a suitably modified version of $\hat{f}(x,u,w)$, as R\&G10 do, by manipulating preconditions and effects of actions. 
That is not possible in \textit{nominal models} setting, since couplings between state variables, actions and perturbation are not represented explicitly. 
Instead, to overcome such domain characteristic, we modify the cost function $J$ by introducing \emph{artificial potential fields}~\cite{warren:89:icra} centered on each observation $o$ in $Obs$ that increase or decrease costs for valid trajectories. 

Let $J_{\pi}(x_0; G)$ be the cost functions for each of the hypothetical candidate goals $G$ $\in$ ${\cal G}$. 
For a given observation sequence $Obs$ $=$ $(o_1,$ $\ldots$, $o_m)$, we define two different cost functions: $J^{+}$ that aims to \textit{increase} the costs along the trajectories by \textit{complying with the observation sequence} $Obs$; and $J^{-}$ that aims to \textit{decrease} the costs by \textit{avoiding the observation sequence} $Obs$. 
Formally, $J^{+}$ and $J^{-}$ are as follows:
\begin{align}
\label{eq:Jplus}
J_{\pi}^{+}(x_0; G, Obs) = 
 g(x_N)  + &   \sum_{k=0}^{N-1} \biggl(g(x_k, \pi(x_k)) + \sum_{j=1}^{m} h(x_k, o_j)\biggr)
\end{align}
\begin{align}
\label{eq:Jminus}
J_{\pi}^{-}(x_0; G, Obs) = 
g(x_N) + & \sum_{k=0}^{N-1} \biggl(g(x_k, \pi(x_k)) - \sum_{j=1}^{m} h(x_k, o_j)\biggr)
\end{align}
\noindent where $h(x,o)$ is a potential field function
\begin{align}
	\label{eq:potential}
	h(x,o) = 1 - exp\{-\gamma\ \ell(x-o)\}
\end{align}
\noindent where the exponent is given as some suitably defined function over the
difference of vectors $x$ and $o$. 
Recall that both $x$ and $o$ $\in$ $\mathbb{R}^d$.
For this thesis, we have chosen the sum \textit{smooth abs} functions
\[
\ell(u) = \sum_{i}^{d} \sqrt{u_{i}^{2} + p^2} + p
\]
where $u_i$ is the $i$-th component of the vector $x-o$ and $p$ is a parameter we set to $1$. 
These functions have been reported by Tassa et al.~\cite{tassa:12:iros} to avoid numeric issues in trajectory optimization over long horizons.
The potential field is used in Equation~\ref{eq:Jplus} to increase, with respect to $J_{\pi}(x_0;G)$, the cost of those trajectories that stay away from $Obs$. 
Conversely, in Equation~\ref{eq:Jminus} it reduces the cost for trajectories that avoid $Obs$. 
Let $T^+$ and $T^-$ be sets of $r$ best trajectories $t_{i}^{+}$, $t_{i}^{-}$ for either cost function, we introduce a \textit{cost difference function}, denoted as $\Delta$, as follows:
\begin{align}
	\label{eq:delta_O_G}
	\Delta(Obs,G) = \frac{1}{r} \sum_{i}^r J_{t^{-}_i}(x_0; G) - J_{t^{+}_i}(x_0; G)
\end{align}
\noindent where $J_{t^{-}_i}(x_0; G)$, and respectively $J_{t^{+}_i}(x_0; G)$, is the result
of evaluating the \emph{original} cost function $J_\pi(x_0;G)$ setting $\pi$ to
be the deterministic policy that follows from trajectories $t_{i}^{+}$ and $t_{i}^{-}$.
We define the likelihood of $Obs$ given $G$ as R\&G10~\cite{RamirezG_AAAI2010} do, as follows:
\begin{align}
	\label{eq:likelihood_delta}
	P(Obs \mid G) = [1 + \exp\{-\beta\ \Delta(Obs,G)\}]^{-1}
\end{align}
\noindent with the proviso that $\beta$ needs to be adjusted so as to be the inverse
of the order of magnitude of $\Delta(Obs,G)$.
Figure~\ref{fig:GR_CostDifference} graphically illustrates the \textit{cost difference} computation between $J^{+}$ and $J^{-}$. 
Note that $J^{+}$ computes trajectories that comply with the observation sequence $Obs$, while $J^{-}$ computes trajectories that aim to avoid achieving the observed states in $Obs$.

\begin{figure}[h!]
  \centering
  \includegraphics[width=0.45\linewidth]{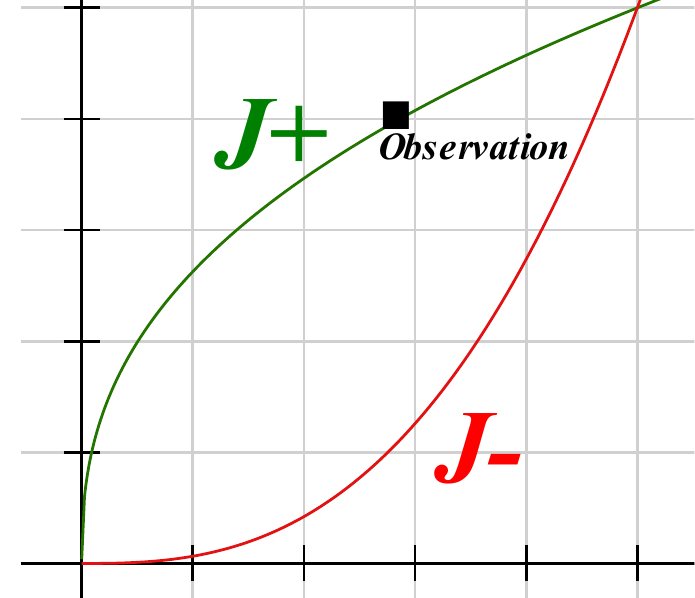}
  \caption{\textit{Cost Difference} approach example.}
  \label{fig:GR_CostDifference}
\end{figure}

In comparison to our previous approach (\nominalmirroring), the computational complexity of our cost difference approach is slightly different, it requires $2|\mathcal{G}|$ calls to a planner per observation, and $2 |\mathcal{G}||Obs|$ calls overall.
While all methods require a number of calls linear on $|\mathcal{G}||Obs|$, evaluating Equation~\ref{eq:likelihood_delta} tends to be more expensive, as $J^{+}$ and $J^{-}$ contain several non--linear terms and their derivatives are also costlier to compute. 
This is relevant as most if not all of the optimization algorithms that we can use to solve Equations~\ref{eq:rhc_objective}--\ref{eq:constraints} rely on gradient--based techniques~\cite{calafiore:14:optimization}.

%#########################################################
\section{Experiments and Evaluation}\label{section:GR_NominalModels_ExperimentsEvaluation}

We now present the experiments and empirical evaluations we carried out of the goal recognition approaches proposed in the previous sections.
Sections \ref{subsection:Domains_NM} and \ref{subsection:Datasets_NM} introduce the benchmark domains we used and describe how we generated the datasets for learning the transition function and the goal recognition tasks.
In Sections~\ref{subsection:Learning_Results} and \ref{subsection:GoalRecognition_Results}, we report the quality of \textit{nominal models} obtained and the performance of our goal recognition approaches over both \textit{actual} and \textit{nominal models} with \textit{linear} and \textit{non--linear} system dynamics.

%@@@@@@@@@@@@@@@@@@@@@@@@@@@@@@@@@@@@@@@@@@@@@@@@@@@@@@@@@
\subsection{Domains}\label{subsection:Domains_NM}

For experiments and evaluation, we use both \textit{linear} and \textit{non--linear} domain models. 
For \textit{linear} domains, we use three benchmark domains based on the constrained Linear–Quadratic Regulator (LQR) problem~\cite{bemporad:2002:lqr}, a general and well--understood class of \textit{Optimal Control} problems with countless practical applications, such as automotive active suspension control systems~\cite{LQR_1998}, optimal control systems for Unmanned Air Vehicle (UAV)~\cite{LQR_2013}, among others~\cite{LQR_BETETO2018422}.
As for \textit{non--linear} domains, we use two other benchmark domains based on a navigation domain that has a highly \textit{non--linear} transition function, also used in \cite{SayWZS:ijcai17,WuSS:nips17} to evaluate planning approaches over learned and approximate domain models.

To represent the transition between states for the LQR--based domains, we use a
discrete-time \emph{deterministic}, \emph{linear} dynamical system
\begin{align}
    \label{eq:lqr}
    x_{{k+1}} = Ax_{k} + Bu_{k}
\end{align}
\noindent and trajectories must minimize the \emph{quadratic} cost function
\begin{align}
    \label{eq:lqr_J}
    {\displaystyle
    J = x_{N}^T Qx_{N}^T + \sum \limits_{k=0}^{N-1}\left(x_{k} Q x_{k}^{T} + u_{k} R u_{k}^{T}\right)}
\end{align}
\noindent where $Q \in \mathbb{R}^{d\times d}$ and $R \in \mathbb{R}^{p\times p}$.
All matrices are set to $I$ of appropriate dimensions, but $R$ which is set to $10^{-2}I$.
Action inputs $u_k$ are subject to simple ``box'' constraints of the form
$lb(u)$ $\leq$ $u_k$ $\leq$ $ub(u)$. 
We note that he unconstrained LQR problem has an analytical solution~\cite{Bertsekas_DP_17} as the cost function is globally convex and dynamics are linear. 
This simplifies the analysis of the  behavior of optimization algorithms for training DNNs and computing trajectories. 
We consider two types of tasks. 
In the first LQR--based domain, which we call \textbf{1D--LQR--Navigation}, states $x_k \in \mathbb{R}^2$ represent the position and velocity of a particle, control inputs $u_k \in \mathbb{R}$ represent instant acceleration. 
Goal states require reaching a given position, yet leave terminal velocities unconstrained. 
The second LQR--based domain, \textbf{2D--LQR--Navigation}, has higher dimensionality as states $x_k \in \mathbb{R}^{4n}$ represent position and velocities of $n$ vehicles on a plane, and control inputs $u_k \in \mathbb{R}^{2n}$ represent instant accelerations along the $x$ and $y$ axis. 
As in the previous domain, goal states only require reaching specific positions. 
For the second domain, we have two variations of this domain: 2D LQR--based navigation domain with $n == 1$, a domain with a single vehicle, denoted as \textbf{2D--LQR--Navigation--SV}, and another variation with $n == 2$, a domain with multiple-vehicles, denoted as \textbf{2D--LQR--Navigation--MV}.

As for the navigation domains with \textit{non--linear} system dynamics, we use the same \textit{non--linear} navigation domain defined by Say~\etal~in~\cite{SayWZS:ijcai17}. 
This domain consists of a navigation domain in continuous space, in which the environment has higher slippage in the center that affects directly the way the agent moves in the environment. 
The transition function for this \textit{non--linear} navigation domain is defined as follows:
\begin{align}
    \label{eq:nav_2d_3d}
	x_{k+1} &= x_{k} + u_{k} \cdot 2 / (1 + exp(-2 \cdot \Delta d_{p}))-0.99 
\end{align}
\noindent where states $x_k \in \mathbb{R}^D$ represent the location of an agent with $D$ dimensions, control inputs $u_k \in \mathbb{R}^D$ represent actions to move an agent over the states axis with $D$ dimensions, and $\Delta d_{p}$ is the Euclidean distance between $x_k$ and the center of the navigation environment. 
Based on the \textit{non--linear} transition function defined in above in Equation~\ref{eq:nav_2d_3d}, the trajectories for this domain must minimize the following cost function
\begin{align}
    \label{eq:nav_J}
    {\displaystyle
    J = - \sum_{k-0}^{N-1} |x_{G} - x_{k}|}
\end{align}
\noindent in which the aiming of this cost function is to minimize the total Manhattan distance from the goal location $x_{G}$ and the current location $x_{k}$.
As the number of dimensions $D$ is arbitrary, we consider two types of domains, domains with 2 and 3 dimensions, and we denote these two types of navigation domains as \textbf{2D--NAV} and \textbf{3D--NAV}, respectively. 

%@@@@@@@@@@@@@@@@@@@@@@@@@@@@@@@@@@@@@@@@@@@@@@@@@@@@@@@@@
\subsection{Learning and Recognition Datasets}\label{subsection:Datasets_NM}

\sigla{RDDL}{Relational Dynamic Influence Diagram Language}
To build the datasets and learn the system dynamics for the domains discussed previously, we generated $500$ different tasks (\idest, pairs of states $x_0$ and $x_G$) for the LQR--based domains, and $1000$ different tasks for the non--linear navigation domains. 
We set the horizon $H = 100$ for the LQR--based domains, resulting in three different datasets with $50,000$ transitions each. 
As for the non--linear navigation domains (2D and 3D), we set the horizon $H = 20$, resulting in $20,000$ transitions for each domain. 
To generate the trajectories for each of the tasks, we first encoded the FHOC problems for each of the domains using the \textsc{RDDL} (Relational Dynamic Influence Diagram Language) domain description language~\cite{sanner:11:rddl}. 
This domain description language enabled us to use the suite of hybrid planners recently
proposed by Bueno \etal.~\cite{bueno:19:aaai}, which we refer to as $\textsc{Tf-Plan}$. These planners rely on state-of-the-art gradient-based optimization techniques that obtain the gradients directly from the symbolic \textsc{RDDL} encoding, and have been shown to offer good performance on linear and non--linear dynamics with complex non--linear cost functions. From the extracted trajectories, we build the training data for each of these domains using the states $x_{k}$, the control inputs $u_{k}$, and the resulting states $x_{k+1}$ from $x_{k}$ and $u_{k}$.

For goal recognition, we use a benchmark consisting of $30$ recognition problems for 
\textbf{1D--LQR--Navigation}, and benchmarks with $30$ problems for \textbf{2D--LQR--Navigation}
setting $n$ to $1$ and $2$. 
For the non--linear navigation domains, we use $10$ recognition problems for both \textbf{2D--NAV} and \textbf{3D--NAV}. 
We ensure that the initial states and hypothetical candidate goals are significantly different from the ones used to learn the \textit{nominal models}.
Each recognition problem considers at most $5$ goals hypotheses $\mathcal{G}$, and observations $Obs$ comprise either $5$ or $10$ states,
\idest, 5\% or 10\% of observability when $H = 100$, for the LQR--based domains, and $1$, $2$, $6$, or $10$ states (respectively, 5\%, 10\%, 30\%, or 50\%), when $H = 20$, for the non--linear navigation domains.
To generate $Obs$, we used \textsc{Tf-Plan} to find a best trajectory from each initial state ${\mathcal I}$ to the hidden intended goal $G^* \in {\cal G}$. 
All states in the trajectories found have the same probability of appearing in any given $Obs$.

Figure~\ref{fig:GR_LQR_2D} shows an example of a goal recognition problem for the \textbf{2D--LQR--Navigation} domain with a single vehicle. 
This example illustrates an example of a goal recognition problem with three candidate goals (represented by X's), an initial state (represented by a triangle), and three observed states (represented by hexagons). 
Consider that the intended goal is the candidate goal in the middle (Candidate Goal 1), and from this, we note that, based on the information provided by this goal recognition, it is possible to see that it is not trivial to say which goal is the intended one. 
This example shows the task of recognition goals over LQR--based domains in continuous space is not trivial, especially when having an approximate transition function to compute the trajectories for the possible goals considering the observations.

\begin{figure}[h!]
  \centering
  \includegraphics[width=0.7\linewidth]{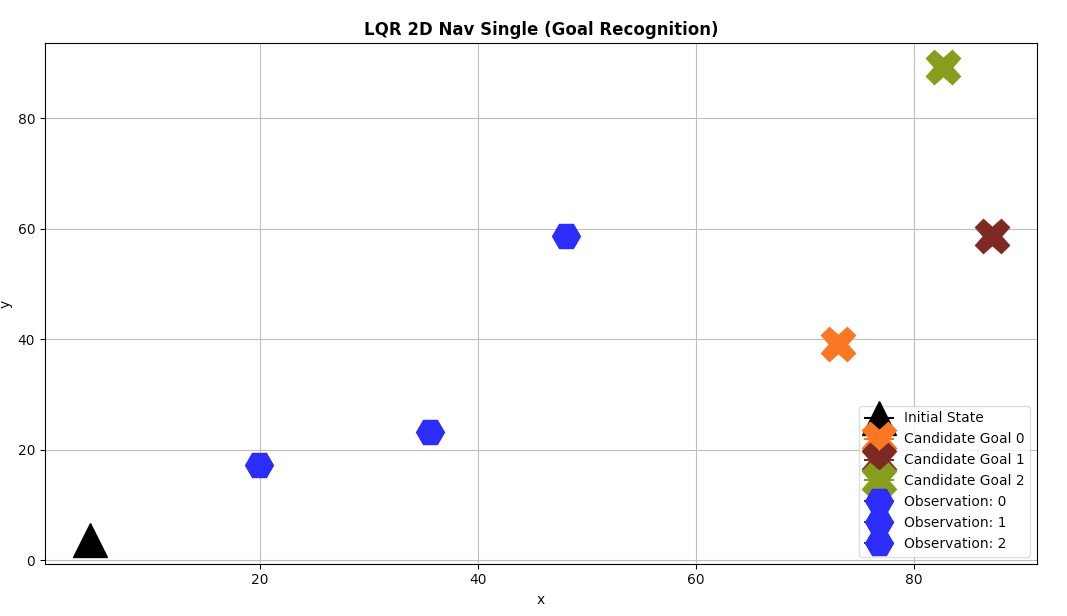}
  \caption{An example of a goal recognition problem for the 2D--LQR--Navigation domain with a single vehicle.}
  \label{fig:GR_LQR_2D}
\end{figure}

%@@@@@@@@@@@@@@@@@@@@@@@@@@@@@@@@@@@@@@@@@@@@@@@@@@@@@@@@@
\subsection{Learning Results}\label{subsection:Learning_Results}

Like in Section~\ref{section:DNN_NominalModels}, we learn the system dynamics of models (transition function) from data using the learning approach proposed by Say et al.~in~\cite{SayWZS:ijcai17}.
For the training stage, we configured the DNN proposed by~\cite{SayWZS:ijcai17} to use the same hyper-parameters to obtain \textit{nominal models} for all domains, namely, $1$ hidden layer, a batch size of $128$ transitions, and we set the learning rate to $0.01$ and dropout rate to $0.1$.
The training stage has stopped for all domains after $300$ epochs. 
We note that we used exactly the same DNN configuration to learn the system dynamics for all domains we described in the previous section.

Table~\ref{tab:results_learning} shows the Mean Squared Error (MSE)\footnote{The Mean Squared Error (MSE) aims to measure the average squared difference between the estimated value and the actual value.} of the best out of
$10$ trials after $300$ epochs of training for all domains.
From the errors reported in Table~\ref{tab:results_learning} we conclude that using off-the-shelf the learning approach of Say et al.~\cite{SayWZS:ijcai17} results in nominal models of very high quality, \emph{as judged by the loss function} they propose.
The MSE in Table \ref{tab:results_learning} indicates we obtain accurate \textit{nominal models} using the off-the-shelf learning approach of Say et al.~\cite{SayWZS:ijcai17}. 
Next, in Section~\ref{subsection:GoalRecognition_Results}, we show how our goal recognition approaches perform over the obtained \textit{nominal models}.

\begin{table}[h!]
\fontfamily{cmr}\selectfont	
\centering
\begin{tabular}{ll}
\toprule
\hline
\textit{Domain}    & \textit{MSE} \\ \hline
1D--LQR--Navigation    & $4.5$ $\cdot$ $10^{-5}$ \\ \hline
2D--LQR--Navigation-SV    & $1.7$ $\cdot$ $10^{-4}$ \\ \hline
1D--LQR--Navigation-MV & $9.6$ $\cdot$ $10^{-6}$   \\ \hline
2D--NAV & $6.8$ $\cdot$ $10^{-6}$   \\ \hline
3D--NAV & $5.1$ $\cdot$ $10^{-5}$   \\ \hline
\bottomrule
\end{tabular}
\caption{Mean Squad Error (MSE) for all domains we used.}
\label{tab:results_learning}
\end{table}

%@@@@@@@@@@@@@@@@@@@@@@@@@@@@@@@@@@@@@@@@@@@@@@@@@@@@@@@@@
\subsection{Goal Recognition Results}\label{subsection:GoalRecognition_Results}

% We now present the experimental results of our recognition approaches over both \textit{actual} and \textit{nominal models}.
For recognizing goals over \textit{actual models}, we use the implementation of the $\textsc{Tf-Plan}$ planner used in~\cite{bueno:19:aaai} that takes as input a domain model formalized in \textsc{RDDL}.
For \textit{nominal models}, we used the implementation of $\textsc{Tf-Plan}$ in~\cite{WuSS:nips17} that takes as input a domain model represented as a DNN.
For both planners we set the learning rate to $0.01$, batch size equals to $128$, and the number of epochs to $300$.

To evaluate our goal recognition approaches over both \textit{actual} and \textit{nominal models}, we use some of the metrics already used in the literature in goal recognition~\cite{RamirezG_IJCAI2009,PereiraNirMeneguzzi_AAAI2017,AAAI2018_PereiraMeneguzzi}.
These metrics are the measures of True Positive Rate (TPR) and False Positive Rate (FPR).
\textit{TPR} is given by the number of true positive results ($1$ when $G^{*}$ maximizes $P(G \mid Obs)$, $0$
otherwise) over the sum of true positive results and false positive results, \idest, the number of
candidate goals maximizing $P(G \mid Obs)$. 
A higher \textit{TPR} indicates better performance, as it measures how often the true intended goal is calculated reliably.
\textit{FPR} is the average number of candidate goals $G \neq G^{*}$ that maximize $P(G \mid Obs)$, measuring how often goals other than the true intended one are found to be as good as or better explanation for $Obs$ than $G^{*}$. 
We also use the \textit{Top-$k$} metric, typically used in machine learning to evaluate classifiers, setting $k$ to $2$, to measure the frequency in which $G^{*}$ was among the top $k$ candidate goals as ranked by $P(G \mid Obs)$, and complements the two previous measures.

We separate the goal recognition results in two different tables: in Table \ref{tab:recognition_results}, we show the results for the LQR--based domains, the domains in which the transition function is linear, whereas, in Table~\ref{tab:recognition_results-non--linear}, we show the results for the non--linear navigation domains.
In these tables, we analyze the performance of the goal recognition approaches presented in Sections~\ref{section:GR_NominalModels_Mirroring} (\nominalmirroring)
and~\ref{section:GR_NominalModels_Counterfactual} ($\Delta(Obs,G)$) in three different settings, from left to right: 
(1) online goal recognition (\textsc{Online}), considering the response of the goal recognition algorithm for \textit{each} judgment point corresponding to an observation in $Obs$; 
(2) offline goal recognition (\textsc{Offline}), when we consider only the  \textit{last} judgment point (i.e., all observed states in $Obs$); and 
(3) considering only the \textit{first} judgment point (\textsc{1st Observation}), \idest, $o_1 \in Obs$.
Note that, in these tables, we aggregate and summarize the average results for all evaluated domains, and in~\ref{appendixC:goalrecognition_nominalmodels}, we provide extensive experimental evaluation for all domain models separately, for both linear and non--linear domains. 

\begin{table}[t!]
\setlength\tabcolsep{1.7pt}
\fontfamily{cmr}\selectfont
\fontsize{12}{13}\selectfont
\centering
\begin{tabular}{llllllllllllllllll}
\toprule
		 \multicolumn{17}{c}{Linear LQR--Based Domains} \\
\toprule
\hline
& &	& \phantom{a} & &
\multicolumn{3}{c}{\sc Online} & \phantom{a} & &
\multicolumn{3}{c}{\sc Offline} & &
\multicolumn{3}{c}{\sc 1st Observation} \\
\cmidrule{6-8} \cmidrule{11-13} \cmidrule{15-17}
\textit{Approach} & \textit{M} &  \textit{Obs} (\%) &&  \textit{N} &  \textit{Top-2} &  \textit{TPR} &  \textit{FPR} &&  \textit{N} &  \textit{Top-2} &  \textit{TPR} &  \textit{FPR} && \textit{Top-2} &  \textit{TPR} &  \textit{FPR} \\
\midrule
\nominalmirroring 		&     A &          5 &&  450 &    0.87 &   0.77 &   0.05 &&   90 &    0.97 &   0.93 &   0.01 &&  0.67 & 0.44 & 0.12 \\
$\Delta(Obs,G)$ 		&     A &          5 &&  450 &    0.49 &   0.24 &   0.16 &&   90 &    0.49 &   0.30 &   0.15 &&  0.49 & 0.26 & 0.16 \\ 

\nominalmirroring 		&     A &         10 &&  900 &    0.90 &   0.78 &   0.05 &&   90 &    0.98 &   0.96 &   0.01 &&  0.64 & 0.32 & 0.15 \\
$\Delta(Obs,G)$ 		&     A &         10 &&  900 &    0.45 &   0.26 &   0.16 &&   90 &    0.44 &   0.28 &   0.15 &&  0.47 & 0.28 & 0.15 \\

\midrule

\nominalmirroring 		&    N &          5 &&  450 &    0.66 &   0.46 &   0.12 &&   90 &    0.83 &   0.67 &   0.07 &&  0.44 & 0.24 & 0.17 \\
$\Delta(Obs,G)$ 		&    N &          5 &&  450 &    0.45 &   0.22 &   0.17 &&   90 &    0.43 &   0.26 &   0.16 &&  0.51 & 0.18 & 0.18 \\

\nominalmirroring 		&    N &         10 &&  900 &    0.71 &   0.49 &   0.11 &&   90 &    0.87 &   0.72 &   0.06 &&  0.41 & 0.26 & 0.16 \\
$\Delta(Obs,G)$ 		&    N &         10 &&  900 &    0.45 &   0.26 &   0.16 &&   90 &    0.42 &   0.21 &   0.17 &&  0.52 & 0.32 & 0.15 \\
\hline
\bottomrule
\end{tabular}
\caption{Experimental results of our recognition approaches over both \textit{actual} and \textit{nominal} models for the LQR--based domains. $M$ represents the model type (A represents \textit{actual models}, and N represents \textit{nominal models}), \% $Obs$ is observation level, and $N$ is the total number of observed states. Note that the average number of goal hypothesis $|\mathcal{G}|$ in the datasets is 5, and the planning horizon $H$ for these datasets is $100$.}
\label{tab:recognition_results}
\end{table}

As for the results for the linear (LQR--based) domains in Table~\ref{tab:recognition_results}, we note that under the parameters used for the planners, \nominalmirroring~clearly dominates $\Delta(Obs,G)$ in all settings by a wide margin, except for the results when considering the first observation over \textit{nominal models}. 
Seeking an explanation for the poor performance of $\Delta(Obs,G)$, we dug deeper into the experimental data in order to find how often \nominalmirroring~outperformed $\Delta(Obs,G)$ and vice versa. 
Interestingly, we found that \nominalmirroring~outperforms $\Delta(Obs,G)$, according to the \textit{Top--2} measure, in $37$\% of the judgment points considered, $\Delta(Obs,G)$ is superior in $8.9$\% of the cases and both approaches are in agreement and correct in $42.1$\% of cases. 
This suggests that $\Delta(Obs,G)$ could be sensitive to one of the parameters used to calculate the trajectories. 
Over this ablation study, we detected that the number of epochs is the key parameter, as it directly affects how far from optimal are the trajectories found. 
We also observed that varying the number of epochs had counter-intuitive results, as the approximations to the optimal values of $J^{+}$ and $J^{-}$ do not get better or worse in a linear fashion. 
Instead, we often observed costs improve (or worsen) for either cost functions at different rates, sometimes changing the sign of $\Delta(Obs,G)$. 
To analyze the impact of the number of epochs over the $\Delta(Obs,G)$, we ran and tested the $\Delta(Obs,G)$ over a limited number of instances, setting the number of epochs to $3,000$, and we observed a significant improvement which brought it to be in agreement with the performance of \nominalmirroring~if not sometimes superior. 
Of course, this entailed an increase of run times by roughly an order of magnitude. 
Thus, this leads us to conclude that the relatively good results of $\Delta(Obs,G)$ in Table~\ref{tab:recognition_results} are due to the fact that $J^{+}$ and $J^{-}$ are closer to the convex ideal in Equation~\ref{eq:lqr_J}, as they include less non--linear terms $h(x,o)$, so \textsc{Tf-Plan} is less likely to get trapped in a local minima with adverse results for recognition accuracy early on.

\begin{table}[t!]
\setlength\tabcolsep{1.7pt}
\fontfamily{cmr}\selectfont
\fontsize{12}{13}\selectfont
\centering
\begin{tabular}{llllllllllllllllll}
\toprule
		 \multicolumn{17}{c}{Non--Linear Navigation Domains } \\
\toprule
\hline
& &	& \phantom{a} & &
\multicolumn{3}{c}{\sc Online} & \phantom{a} & &
\multicolumn{3}{c}{\sc Offline} & &
\multicolumn{3}{c}{\sc 1st Observation} \\
\cmidrule{6-8} \cmidrule{11-13} \cmidrule{15-17}
\textit{Approach} & \textit{M} &  \textit{Obs} (\%) &&  \textit{N} &  \textit{Top-2} &  \textit{TPR} &  \textit{FPR} &&  \textit{N} &  \textit{Top-2} &  \textit{TPR} &  \textit{FPR} && \textit{Top-2} &  \textit{TPR} &  \textit{FPR} \\
\midrule
\nominalmirroring 		&     A &          5 &&   20 &    0.85 &   0.85 &   0.04 &&   20 &    0.85 &   0.85 &   0.04 &&  0.85 & 0.85 & 0.04 \\
$\Delta(Obs,G)$ 		&     A &          5 &&   20 &    0.80 &   0.65 &   0.09 &&   20 &    0.80 &   0.65 &   0.09 &&  0.80 & 0.65 & 0.09 \\

\nominalmirroring 		&     A &         10 &&   40 &    0.90 &   0.78 &   0.06 &&   20 &    1.00 &   0.90 &   0.03 &&  0.80 & 0.65 & 0.09 \\
$\Delta(Obs,G)$ 		&     A &         10 &&   40 &    0.55 &   0.38 &   0.16 &&   20 &    0.55 &   0.30 &   0.17 &&  0.55 & 0.45 & 0.14 \\

\nominalmirroring 		&     A &         30 &&  120 &    0.89 &   0.75 &   0.06 &&   20 &    1.00 &   1.00 &   0.00 &&  0.75 & 0.55 & 0.11 \\
$\Delta(Obs,G)$ 		&     A &         30 &&  120 &    0.55 &   0.33 &   0.17 &&   20 &    0.55 &   0.35 &   0.16 &&  0.55 & 0.35 & 0.16 \\

\nominalmirroring 		&     A &         50 &&  200 &    0.88 &   0.73 &   0.07 &&   20 &    1.00 &   1.00 &   0.00 &&  0.75 & 0.35 & 0.16 \\
$\Delta(Obs,G)$ 		&     A &         50 &&  200 &    0.57 &   0.32 &   0.17 &&   20 &    0.85 &   0.40 &   0.15 &&  0.50 & 0.25 & 0.19 \\

\midrule

\nominalmirroring 		&    N &          5 &&   20 &    0.85 &   0.70 &   0.07 &&   20 &    0.85 &   0.70 &   0.07 &&  0.85 & 0.70 & 0.07 \\
$\Delta(Obs,G)$ 		&    N &          5 &&   20 &    0.75 &   0.40 &   0.15 &&   20 &    0.75 &   0.40 &   0.15 &&  0.75 & 0.40 & 0.15 \\

\nominalmirroring 		&    N &         10 &&   40 &    0.72 &   0.55 &   0.11 &&   20 &    0.80 &   0.65 &   0.09 &&  0.65 & 0.45 & 0.14 \\
$\Delta(Obs,G)$ 		&    N &         10 &&   40 &    0.55 &   0.30 &   0.17 &&   20 &    0.60 &   0.40 &   0.15 &&  0.50 & 0.20 & 0.20 \\

\nominalmirroring 		&    N &         30 &&  120 &    0.76 &   0.61 &   0.10 &&   20 &    1.00 &   1.00 &   0.00 &&  0.50 & 0.25 & 0.19 \\
$\Delta(Obs,G)$ 		&    N &         30 &&  120 &    0.59 &   0.28 &   0.18 &&   20 &    0.75 &   0.40 &   0.15 &&  0.50 & 0.15 & 0.21 \\

\nominalmirroring 		&    N &         50 &&  200 &    0.81 &   0.65 &   0.09 &&   20 &    1.00 &   1.00 &   0.00 &&  0.50 & 0.25 & 0.19 \\
$\Delta(Obs,G)$ 		&    N &         50 &&  200 &    0.54 &   0.27 &   0.18 &&   20 &    0.70 &   0.55 &   0.11 &&  0.40 & 0.15 & 0.21 \\
\hline
\bottomrule
\end{tabular}
\caption{Experimental results of our recognition approaches over both \textit{actual} and \textit{nominal} models for the non--linear navigation domains (2D--NAV and 3D--NAV). $M$ represents the model type (A represents \textit{actual models}, and N represents \textit{nominal models}), \% $Obs$ is observation level, and $N$ is the total number of observed states. Note that the average number of goal hypothesis $|\mathcal{G}|$ in these datasets is 4, and the planning horizon $H$ for these datasets is $20$.}
\label{tab:recognition_results-non--linear}
\end{table}

In comparison to the results for the linear LQR--based domain models (Table~\ref{tab:recognition_results}), we can see that the results for the non--linear domains are better for both our recognition approaches over all settings, as shown in Table~\ref{tab:recognition_results-non--linear}. 
It is also possible to see that, again, \nominalmirroring~dominates~$\Delta(Obs,G)$ in all settings, but not as much as it is for the linear LQR--based domains. 
For non--linear domains, \nominalmirroring~is outperforms $\Delta(Obs,G)$, when using the \textit{Top--2} measure, in $8.8$\% of the judgment points considered. 
Moreover, $\Delta(Obs,G)$ outperforms~\nominalmirroring in $3.5$\% of the cases, and both recognition approaches are in agreement and correct in $30.1$\% of cases. 
An interesting aspect regarding the results for non--linear domains is about performance compared to the linear LQR--based domain models. 
We note that the main difference between the recognition datasets for linear and non--linear domains is the length of the planning horizon $H$, which is $100$ for the linear LQR--based domains, and $20$ for the non--linear navigation domains. 
This may lead us to conclude that the length of the horizon affects the planning process when extracting trajectories, by accumulating and propagating error along trajectories, especially for nominal models.

With respect to recognition time, the average time per goal
recognition problem for \nominalmirroring~over the linear LQR--based datasets is~$\approx$~1,100 seconds, whereas for $\Delta(Obs,G)$ is~$\approx$~1,600 seconds. 
As for the recognition time of our approaches when dealing with non--linear domains, the average time per problem is~$\approx$~1,900 seconds for \nominalmirroring, and~$\approx$~2,500 seconds for~$\Delta(Obs,G)$. 
Note that the (linear and non--linear) domains and problems we used in our experiments are non-trivial real-world domains, and all of them have continuous state--space and actions, in which the planning process usually takes substantial time for extracting optimal trajectories.

%#########################################################
\section{Chapter Remarks}\label{section:GR_NominalModels_Remarks}

Model--based goal and plan recognition is a real-world, non--trivial and challenging application of causal reasoning, and, in this chapter, we adapt past approaches to model--based goal recognition as different implementations of Halpern's \emph{but-for} test of sufficient causality. 
We also show that learning techniques can be used to generate predictions which are good enough to enable the generation of meaningful counterfactuals and by extension ``true'' causal reasoning~\cite{pearl:09:causality}.

% \frm{What are the recent statements to the contrary? Is it Pearl's? It's not clear here}

% \frm[inline]{I'm not sure I like the narrative here, it gives the impression that you have a lot more experiments to carry out, can you rethink this a bit?}
% We look forward to further investigating three questions that this chapter leaves open. 
% FRM - Again, we don't need to do anything
Future work on goal recognition over \textit{nominal models} should investigate three key questions. 
First, to what extent the proposed recognition approaches can handle increasing variance for the random variable $w$. 
Second, we need to determine whether it is possible to modify the loss function used for training the nominal models in a way that takes into account the accumulated error along trajectories, rather than just the errors in predicting the next state. 
Last, as discussed in Section~\ref{subsection:GoalRecognition_Results}, cost--based goal 
recognition is very sensitive to planners converging to unhelpful local minima, which seems to be an inherent characteristic of stochastic optimization algorithms. 
These recognition approaches could also be evaluated using planners that rely on Differential Dynamic Programming (DDP)~\cite{mitrovic:10:adaptive,yamaguchi:16:icra}, which \textit{may} converge faster to better local minima of the cost function.

In this chapter, we developed novel approaches to goal recognition that expand the applicability of model--based goal and plan recognition by replacing carefully engineered models for carefully curated datasets. 
The approaches developed in this chapter are also examples of how to exploit latent synergies between \textit{Planning}, \emph{Optimal Control}, \textit{Optimization}, and \textit{Machine Learning}, as we integrate algorithms, techniques, and concepts to address in novel ways a high-level, transversal problem relevant to many fields in \textit{Artificial Intelligence}.

%% file: cap5_Related_Work.tex
%!TEX root = ppgcc-thesis.tex
%----------------------------------------------------------------------------------
% Chapter: Related Work
%
% + ?
%
%----------------------------------------------------------------------------------
\chapter{Related Work}\label{chapter:RelatedWork}

In this chapter, we survey and review the most significant work on \textit{Goal and Plan Recognition} and \textit{Planning} that are closely related to the contributions presented in this thesis.
In Section~\ref{section:GoalPlanRecognitionAsPlanning}, we present model--based approaches to goal and plan recognition that directly rely on \textit{Planning} techniques. After, in Section~\ref{section:PlanningImperfectDomainModels}, we describe the approaches in the literature that deal with incomplete domain information for recognizing goals and plans. 
Finally, in Section~\ref{section:PlanRecognitionIncomplePlanLibraries}, we describe existing work on planning over \textit{imperfect domain models}, namely, planning approaches that deal with either incomplete discrete domain models, or approximate hybrid (continuous and discrete) domain models.

% \frm[inline]{A key thing you should do in the related work is to \textbf{compare and contrast} previous work with your current work.}

% \frm[inline]{The order of the sections is odd here. Why not do 5.2, 5.3, 5.1? This makes a lot more sense to me!}
% Agreed.

%$$$$$$$$$$$$$$$$$$$$$$$$$$$$$$$$$$$$$$$$$$$$$$$$$$$$$$$$$$$$$$$$$$$$$$$$$$$$$$$$$$
\section{Goal and Plan Recognition as Planning}\label{section:GoalPlanRecognitionAsPlanning}

%\frm[inline]{Avoid filler stuff like ``According to the literature''. Avoid passive voice.}

% \frm[inline]{Why is this after GR with incomplete models? Chronologically this comes earlier, doesn't it?}

% \rfp[inline]{Please, have a look at the next two paragraphs.}

Over the past years, the task of goal and plan recognition as planning has received much attention in the \textit{Automated Planning} community, resulting in several remarkable contributions in recent years.
One of the first planning--based approach for recognizing goals and plans was developed by Hong~\cite{HongGoalRecognition_2001}. 
Hong develops an approach that extends the concept of planning graph~\cite{BlumFastPlanning_95}, proposing a similar structure that represents every possible path (\exemp, state transitions that connect facts and actions) from an initial state to a goal state, and calling this structure a goal graph. 
As actions are observed during a plan execution, a goal graph is constructed, in which facts that represent recognized goals are linked to a goal level. 
This work is one the first work to address the task of goal recognition without using plan--libraries, showing that it is possible to be fast and accurate for recognizing goals without explicitly defining the plans that achieve the goals.
% \frm[inline]{This paragraph just trails off, you need to conclude the thought, I guess. What's important about this approach?}

Later, Ram{\'{\i}}rez and Geffner~\cite{RamirezG_IJCAI2009} introduce the problem of \textit{Plan Recognition as Planning}, by using planning domain models to describe the agents' behavior and planning techniques to perform the recognition task.
For recognizing goals and plans, Ram{\'{\i}}rez and Geffner use modified (optimal and sub--optimal) planning algorithms to determine the distance to every goal in a set of candidate goals given a sequence of observations. 
In this work, Ram{\'{\i}}rez and Geffner develop two approaches, specifically, they consider optimal or sub--optimal plans, in which goals that have become impossible being removed from the set of candidate goals. 
Once candidate goals have been eliminated, they are never reconsidered during the recognition process. 
Ram{\'{\i}}rez and Geffner~\cite{RamirezG_IJCAI2009} also work with an assumption of partial observability in that only a sub-sequence of the plan is available as evidence to the  recognizer. 
In their subsequent work, Ram{\'{\i}}rez and Geffner~\cite{RamirezG_AAAI2010} develop a robust probabilistic framework for goal and plan recognition by using off-the-shelf planners, providing a posterior probability distribution over goals, given an observation sequence as evidence.
% \frm[inline]{I think you'd need to explain what is the key difference between these approaches (which we know, but not the reader, especially givefn the feedback we get from our papers). Which is that the first approach rules goals in and out, whereas the latter approach provides a probability distribution (i.e. the $P(G | O)$)}.

In~\cite{PattisonGoalRecognition_2010}, Pattison and Long propose AUTOGRAPH (AUTOmatic Goal Recognition with A Planning Heuristic), a probabilistic heuristic-based goal recognition approach over planning domains. 
Much like the work of Ram{\'{\i}}rez and Geffner~\cite{RamirezG_IJCAI2009}, AUTOGRAPH uses heuristic estimation and domain analysis to determine which goal(s) a plan execution of an observed agent is pursuing. 
However, unlike most work on goal and plan recognition as planning, the set of candidate goals is not given as part of the recognition problem, the intended goal is inferred based on a planning domain model, an initial state, and a sequence of observations. 

As one of the very first to address the task of \textit{Plan Recognition as Planning} with multiple agents, in~\cite{MAPR_STRIPS_Zhuo_2012}, Zhuo~\etal~develop a recognition approach in which the team behavior model is defined as a planning domain definition (\idest, every agent behavior is based on the planning domain definition), and the task of plan recognition analyzes a partial team observation trace for recognizing team plans. 
To perform the recognition process, Zhuo~\etal~first translate the multi-agent plan recognition as a satisfiability problem, and then solve the recognition problem by using a weighted MAX-SAT solver. 
Zhuo~\etal~also show a comparison between their previous plan--library based approach~\cite{ZhuoL_MAPR_11} and the planning domain definition based approach.

Dealing explicitly with ambiguity in goal and plan recognition is a very complex task. In~\cite{GoalRecognitionDesign_Keren2014}, Keren~\etal~develop an alternate view of the goal recognition problem, and rather than developing new goal recognition algorithms, they develop novel techniques that modify the domain model in order to facilitate the goal recognition process. Specifically, the approaches of Keren~\etal~\cite{GoalRecognitionDesign_Keren2014,GoalRecognitionDesign_Keren2015,GoalRecognitionDesign_Keren2016}~aim to attempt to reduce the number of non-unique plans for each candidate goal in a set of candidate goals, and thus, simplifying the process of goal recognition by redesigning the planning domain model.

E.-Martín~\etal~\cite{NASA_GoalRecognition_IJCAI2015} propose a planning-based goal recognition approach that propagates cost and interaction information in a plan graph, and uses this information to estimate goal probabilities over the set of candidate goals and the observation sequence. We note that the approach of E.-Martín~\etal~\cite{NASA_GoalRecognition_IJCAI2015} is the first one in the literature that obviates calling a planner to perform the recognition task, resulting in a very fast approach to goal recognition.

Sohrabi~\etal~\cite{Sohrabi_IJCAI2016} extend the probabilistic framework of Ram{\'{\i}}rez and Geffner~\cite{RamirezG_AAAI2010}, and developed a novel probabilistic recognition approach that deals explicitly with unreliable and spurious observations (\idest, noisy or missing observations), and recognizes both goals and plans. 
The probabilistic approaches of Sohrabi~\etal~\cite{Sohrabi_IJCAI2016} use multiple high-quality plans (using a planner that generates multiple plans with high quality) to produce a probabilistic distribution over the goals. 
In this paper, the authors show that, for some domains, the use of multiple high-quality plans along with this novel probabilistic framework yields better results than using only one plan~\cite{RamirezG_AAAI2010}.

Vered, Kaminka, and Biham~\etal~\cite{Mor_ACS_16}~introduce the concept of \textit{Mirroring} to develop an online goal recognition approach for continuous domain models. 
Based on this work, Vered~\etal~\cite{MorEtAl_AAMAS18} develop an online goal recognition approach that combines the concept of \textit{Mirroring} and landmarks, showing this combination can improve not only the recognition time, but also the accuracy for recognizing goals in the online fashion. 
In~\cite{Kaminka_18_AAAI}, Kaminka~\etal~propose a new probabilistic framework for plan recognition approach over both continuous and discrete domains, in which the core of this framework is the concept of \textit{Mirroring}.

Masters and Sardi{\~{n}}a~\cite{Masters_IJCAI2017,MastersS_JAIR_19} propose a fast and accurate goal recognition approach that works strictly in the context of path-planning, providing a novel probabilistic framework for goal recognition in path planning, which is basically a revised and improved version of the probabilistic framework of Ram{\'{\i}}rez and Geffner~\cite{RamirezG_AAAI2010}. This novel probabilistic framework for path-planning shows that it is possible to compute the probability distribution over the goals much simpler and faster than the one proposed by Ram{\'{\i}}rez and Geffner~\cite{RamirezG_AAAI2010}, considering only a single observation, namely the current state.
In their most recent work, Masters and Sardi{\~{n}}a~\cite{MastersS_AAMAS_19} improve their previous probabilistic approach to deal with both rational and irrational agent behavior during the recognition process.
% \frm{Maybe talk about the insight they have in not needing any observation but the last one? And the Euclidean assumption.}

% \frm{Since this is your (non-blind) thesis, it is fine to talk about this work as ``our previous work''}
In previous work~\cite{PereiraMeneguzzi_ECAI2016,PereiraNirMeneguzzi_AAAI2017}, we develop landmark--based approaches for goal recognition as planning. 
Such approaches are recognition heuristics that strictly rely on the concept of landmarks.
Their first heuristic approach performs the recognition task by computing the ratio between the number of achieved landmarks and the total number of landmarks for a given candidate goal, and then, the candidate goal (s) with the highest heuristic value is (are) considered as the most likely intended one.
The second heuristic approach uses the concept of \textit{landmark uniqueness value}, representing the information value of the landmark for a particular candidate goal when compared to landmarks for all candidate goals. 
Thus, the heuristic estimation provided by this second heuristic is the ratio between the sum of the uniqueness value of the achieved landmarks and the sum of the uniqueness value of all landmarks of a candidate goal. 
Most recently, in their extended work, in~\cite{PereiraOM_AIJ_2020}, they show through several experiments that these landmark--based heuristics are the fastest ones in the literature. 

Freedman~\etal~\cite{freedman2018towards} proposed an approach to perform probabilistic plan recognition along the lines of the work of Ram{\'{\i}}rez and Geffner~\cite{RamirezG_AAAI2010}, in which, instead of calling a full-fledged planner for each candidate goal, it takes advantage of a multiple-goal heuristic search algorithm~\cite{davidov2006multiple} to search for all goals simultaneously, avoiding repeatedly expanding the same nodes in the search tree. 
This approach has not been implemented and evaluated yet, the authors have provided only theoretical concepts regarding this approach.

Unlike the heuristic approaches we developed in Chapter~\ref{chapter:GR_IncompleteDomains}, most recent planning-based recognition approaches~\cite{RamirezG_IJCAI2009,RamirezG_AAAI2010,Sohrabi_IJCAI2016,Mor_ACS_16} use a planner to recognize goals and plans from observations, calling a planner at least $2 \times \mathcal{G}$ times during the recognition process. 
Conversely, E.-Martín~\etal~\cite{NASA_GoalRecognition_IJCAI2015} and Pereira~\etal~\cite{PereiraNirMeneguzzi_AAAI2017} are similar to our heuristics approaches because these approaches avoid the use of automated planners during the goal and plan recognition process, and use only planning information extracted from planning instances, \idest, planning-graphs and landmarks, respectively.  
Keren~\etal~\cite{GoalRecognitionDesign_Keren2014} developed an approach that assumes planning domain models are not fixed, and it changes (re-designs) the domain definition to facilitate the task of goal recognition in planning domain models. 
However, these approaches differ from ours because they only deal with complete (even if modified) domain models, and most of them transform/compile the goal/plan recognition problem into a planning problem to be solved by a planner.
Such a transformation or compilation process may not necessarily work with incomplete STRIPS domain models, given the very large number of potential models. 
We note that the approach of E.-Martín~\etal~\cite{NASA_GoalRecognition_IJCAI2015} could work in incomplete domain models with some adaptations (\exemp, by ignoring all \emph{possible} preconditions and effects), though this approach would likely be less accurate than our approaches (as shown in~\cite{PereiraOM_AIJ_2020} for complete and correct domain models) because it does not deal intentionally with possible preconditions and effects, while our recognition approaches do.

%$$$$$$$$$$$$$$$$$$$$$$$$$$$$$$$$$$$$$$$$$$$$$$$$$$$$$$$$$$$$$$$$$$$$$$$$$$$$$$$$$$
\newpage
\section{Planning over Imperfect Domain Models}\label{section:PlanningImperfectDomainModels}

% \frm[inline]{We need a stronger motivation than ``people developed approaches like that for planning, so we are using them''. Don't lose the focus from the fact that complete and accurate domain models are usually only possible in toy domains or planning benchmarks. The real world is messy, so we need models that can represent that mess.}
% \frm[inline]{Again, motivate this.}

There has been comparatively little research in the \textit{Automated Planning} literature to deal explicitly with inaccurate and imperfect domains models. 
One of the first planning approaches to address incomplete information in discrete domain models is the work of Garland and Lesh~\cite{GarlandLesh_AAAI2002}. 
In this work, Garland and Lesh~\cite{GarlandLesh_AAAI2002} develop a planning approach for incomplete domain models, allowing the use of annotations (possible preconditions and effects) to specify incomplete actions in the domain description. 
Their planning approach analyzes a set of extracted plans that identifies critical faults (facts that may cause plan failure) in these action sequences, returning the plan with the best quality (\idest, the plan with the minimal number of critical faults). 
Weber and Brycen~\cite{WeberBryce_ICAPS_2011} use the same annotations for incomplete domains, and develop a planner called \textsc{DeFault}, which aims to search for plans by minimizing their risks to fail for achieving goals.
\textsc{DeFault} uses a heuristic approach based on the Fast-Forward (FF) heuristic~\cite{FFHoffmann_2001}, breaking ties using a novel heuristic that counts failure models, called \textsc{Prime Implicant} heuristic.
Most recently, Nguyen~\etal~\cite{Nguyen_AIJ_2017} develop two approaches for planning in incomplete domain models, the \textsc{PISA} and $\mathcal{C}$\textsc{PISA} planners~\cite{PlanningIncomplete_NguyenK_2014}. 
\textsc{PISA} is a planner that uses a stochastic local search to synthesize robust plans in incomplete planning domains. 
$\mathcal{C}$\textsc{PISA} extends the techniques from \textsc{PISA} incorporating Bayesian learning to enhance the planning process. \textsc{PISA} and $\mathcal{C}$\textsc{PISA} outperform \textsc{DeFault} for planning in most incomplete domain models. 
The incomplete domain formalism we use for the task of goal recognition over incomplete domains in (as presented in Chapter~\ref{chapter:GR_IncompleteDomains}) is based on the work presented above~\cite{GarlandLesh_AAAI2002,WeberBryce_ICAPS_2011,PlanningIncomplete_NguyenK_2014,Nguyen_AIJ_2017}. 

\sigla{MILP}{Mixed-Integer Linear Program}
% \sigla{BNN}{Binarized Neural Network}
\sigla{BLP}{Binary Linear Programming}
Recent \textit{Deep Learning} techniques have shown to be very effective to learn linear and non-linear transition functions from data. 
Say~\etal~\cite{SayWZS:ijcai17} use state-of-the-art deep learning techniques to approximate the transition function of hybrid (mixed discrete and continuous) domain models, based on datasets that contain plan traces represented as state transitions. 
For planning, Say~\etal~developed a Mixed-Integer Linear Program (MILP) based planner that works in two states: (1) it encodes the learned transition function and a hybrid domain model into a MILP Program; and (2) given this MILP encoding, 
Say~\etal~use an off-the-shelf MILP solver to find plans for a given planning horizon. 
Subsequently, Wu, Say, and Sanner~\cite{WuSS:nips17} use the same approach to approximate the transition function of hybrid domain models, but unlike the work of~Say~\etal~\cite{SayWZS:ijcai17}, they use pure learning techniques to develop their planning approach. 
Namely, they develop a planning approach based on Tensorflow~\cite{Tensorflow_2015} and a gradient descent optimization (RMSProp\footnote{Developed by Geoff Hinton (in Lecture 6 of his Coursera Class), RMSprop is an unpublished optimization algorithm designed for Neural Networks.}). 
Over extensive experimental results, they show that the resulting learning approach is very competitive in comparison to the MILP based planner on several linear and non-linear hybrid planning domains. 
Recently, in~\cite{SayS:ijcai18}, Say and Sanner develop two alternative approaches to the ones presented in~\cite{SayWZS:ijcai17,WuSS:nips17}. 
Unlike the previous approaches, the recent approaches of Say and Sanner~\cite{SayS:ijcai18}~use modern \textit{Machine Learning} approaches to learn the transition function of hybrid domain models, whereas for planning, they develop two approaches: the first one compiles the learned transition function and a hybrid domain model into a Boolean Satisfiability problem, and use SAT solver to find plans for a given horizon, while the second one compiles the problem into a Binary Linear Programming (BLP) formulation, and then use a BLP solver for the planning process.

As we presented in Chapter~\ref{chapter:GR_NominalModels}, the core of our goal recognition approaches over \textit{nominal models} is inspired by the work of Say~\etal~\cite{SayWZS:ijcai17} and Wu, Say, and Sanner~\cite{WuSS:nips17}. 
To learn and approximate transition functions from datasets, we use the learning approach of Say~\etal~\cite{SayWZS:ijcai17}, and for planning over nominal models, we make use of the Tensorflow planner developed by Wu, Say, and Sanner in~\cite{WuSS:nips17}.

%$$$$$$$$$$$$$$$$$$$$$$$$$$$$$$$$$$$$$$$$$$$$$$$$$$$$$$$$$$$$$$$$$$$$$$$$$$$$$$$$$$
\section{Plan Recognition with Incomplete Domain Information}\label{section:PlanRecognitionIncomplePlanLibraries}

To the best of our knowledge, the earlier work on goal and plan recognition that deal explicitly with \textit{incomplete domain models} are that of Lee and McCartney~\cite{lee1998partial} and Kerkez and Cox~\cite{kerkez2002case}, and most recently, the work of Zhuo~\cite{Zhuo_EtAl_TIST_2019}. 
The main characteristic that these approaches~\cite{lee1998partial,kerkez2002case,Zhuo_EtAl_TIST_2019} have in common with our approaches to goal recognition over incomplete domains (Chapter~\ref{chapter:GR_IncompleteDomains}), is the use of incomplete information in the domain model description.
% \frm{I just read Zhuo's work, and he does not include the unknowns, he just assumes that there are missing parts on the preconditions and effects and relaxes }
% \frm{Zhuo's stuff is multiagent}

Lee and McCartney~\cite{lee1998partial} developed a plan recognition approach that uses stochastic models (Hidden Markov Models) to model possible ambiguities in the agent behavior, and learning techniques to learn actions and properties of the model based on an incomplete agent behavior model from partial observation, which are stored as a history of interactions in a dataset. 
To describe the agent behavior model, the authors use a graph structure similar to a plan--library, but with incomplete information, and in this incomplete behavior model, such incomplete information represents the set of unknown properties and actions of the agent behavior model.
% \frm{I'm not sure I follow this last sentence. If it's unknown, then how do you list them?}
% I just enumerated the type of incomplete information they used, because it can be different and vary according to the model.

Unlike the approach of Lee and McCartney~\cite{lee1998partial} that uses learning techniques to fill the incomplete part of the domain model, the plan recognition approach of Kerkez and Cox~\cite{kerkez2002case} takes as input an incomplete plan--library and deals with incomplete domain information by using an automated planner. 
More specifically, the approach of Kerkez and Cox uses a planner to fill and complete an incomplete plan--library from the observations, and then recognizes the observed agent's goal using a mapping technique, matching the resulting plan--library with the observations. 
% \frm{What is an isomorphism mapping technique?}
% Good question ;) I changed to just mapping technique :)

Zhuo~\cite{Zhuo_EtAl_TIST_2019} develops a multi-agent plan recognition approach that uses as domain knowledge the combination of an incomplete action description model and a set of incomplete team plans (\idest, a set of possible team plans that the agents can perform to achieve their goals).
For recognizing multi-agent team plans, the approach works in two stages: (1) the approach transforms the multi-agent recognition problem as a satisfiability problem, encoding the problem by using soft and hard constraints; and (2) based on the encoded constraints, the approach then solves the recognition problem by using a weighted MAX-SAT solver.

% \frm{I think it's useful to emphasize that this is neither a plan library nor a domain description, but rather a set of flat plans with missing actions (Because the entire approach uses matrix matching to encode the SAT problem)}
% Correct, I fixed it, thanks.

We argue that the plan recognition approaches~\cite{lee1998partial,kerkez2002case,Zhuo_EtAl_TIST_2019} described above are quite different from the heuristic recognition approaches over incomplete domains we developed and presented in Chapter~\ref{chapter:GR_IncompleteDomains} in several key aspects, and we single out these differences as follows. 
Firstly, these recognition approaches in~\cite{lee1998partial,kerkez2002case,Zhuo_EtAl_TIST_2019} use incomplete plan--libraries or incomplete team plans to represent the agent behavior model, whereas our heuristic approaches only use incomplete planning domain models.
Plan--libraries are usually encoded manually and laborious to model, requiring a description of a set of plans in order to know how to achieve the set of possible goals. Even using incomplete information, the task of modeling plan--libraries can be quite complex, requiring much design effort and domain knowledge. 
Moreover, plans that are not defined in the plan--libraries cannot be recognized during the recognition process.
We note the task of modeling incomplete planning domain models requires much less effort, requiring only a description of the set of predicates (properties) and actions of the environment. 
Secondly, another key difference is that we use no learning approach to learn or fill in the incomplete part of the domain model, and deal explicitly with the incomplete domain information. 
Finally, we use no planner or any other kind of solver for recognizing goals over incomplete domain information and only use the information provided by the landmark extraction process. 

% \rfp[inline]{Please, have a look at the next paragraph. What do you think?!}
% \frm[inline]{Rewrote the last sentence, to make it correct}

We have not provided any comparison against these approaches because the formalisms they use for incomplete domains are incompatible with the formalism we use in our work. 
However, a comparison against the work of Zhuo~\cite{Zhuo_EtAl_TIST_2019} would be possible if we modify his work to cope only with a single agent, generate a set of incomplete plans for a single agent for achieving the possible candidate goals, and adapt our work to his formalism for incomplete domain models. 
% % FRM The below is not correct
% We argue that, even with such adaptations and modifications, the comparison would not be fair, since Zhuo's approach combines the use of the set of incomplete plans and the incomplete domain to perform the recognition process, while our work only uses the incomplete domain model.
We argue that, even with such adaptations and modifications, the comparison would not be fair, since Zhuo's approach assumes plans both in the plan--library and in the observations have a fixed size and contain information about exactly what are times of the missing observations, which our work does not assume. 

% \frm[inline]{The key thing you are missing in this section is to argue why you did not compare any of these approaches with yours. There are reasons for that (incompatible formalisms, lack of source code and incomplete papers, etc, use them)}

% \frm{They don't need to be completely different, just different in key aspects. So, the problem with Zhuo (which I think might have an older approach for single agent GR), is that he restricts the plan preference information to flat plans}

%$$$$$$$$$$$$$$$$$$$$$$$$$$$$$$$$$$$$$$$$$$$$$$$$$$$$$$$$$$$$$$$$$$$$$$$$$$$$$$$$$$

\section{Chapter Remarks}

% \frm[inline]{I'm not sure you need this subsection if you are only going to say the things below. Put the sentences in the subsections above and be done with it.}
% I think that I should keep it in order to be consistent with the other chapters, dont you think?!

We conclude this chapter by noting that, we have not surveyed all approaches to goal and plan recognition in the literature, and instead, we chose to focus on presenting the most relevant approaches to \textit{Plan Recognition as Planning}. 
The literature of goal and plan recognition is vast, and contains several types of significant approaches that are not directly related to \textit{Automated Planning}, for instance, goal and plan recognition approaches that rely on plan--libraries~\cite{AvrahamiZilberbrandK_IJCAI2005,PR_Mirsky_2016,MIRSKY_2018_AIJ} and context-free grammars~\cite{Geib_PPR_AIJ2009}. 
Thus, we state that the main focus of this chapter is presenting existing work that is based on planning techniques and incomplete domain information.

%% file: cap6_Conclusions.tex
%!TEX root = ppgcc-thesis.tex
%----------------------------------------------------------------------------------
% Chapter: Conclusions
%
% + ?
%
%----------------------------------------------------------------------------------

\chapter{Conclusions}\label{chapter:Conclusions}

In this thesis, we introduced new formalizations for goal recognition problems that allow for \textit{imperfections} over two distinct types of domains models, \idest, \textit{incomplete discrete domain models} that have \textit{possible}, rather than \textit{known}, preconditions and effects in action descriptions, and \textit{approximate continuous domain models}, where the transition function is \textit{approximate} and not well-defined. 
We developed novel goal recognition approaches that can cope with these two types of imperfect domain models, and we have empirically shown that such approaches are accurate when dealing with imperfect domains in several recognition settings. 
In Section~\ref{chapter:conclusions:contrib}, we summarize and discuss the main contributions of this thesis, and after that, in Sections~\ref{chapter:conclusions:issues}~and~\ref{chapter:conclusions:futurework}, we discuss, respectively, the open issues and limitations of our approaches, as well as the avenues that we can for improving our recognition approaches. 

%#########################################################
\section{Contributions}\label{chapter:conclusions:contrib}

We now outline the two main contributions of this thesis, as follows.

\begin{enumerate}
	\item Our first contribution, presented in Chapter~\ref{chapter:GR_IncompleteDomains}, is related to the task of goal recognition over \textit{incomplete discrete domain models}, and resulting in the following the specific contributions.
		\begin{itemize}
			\item \emph{A new problem formalization for goal recognition over incomplete domain models} (Section~\ref{subsection:GR_IncompleteDomains_Formalism}), combining the standard formalization of goal recognition of Ramírez and Geffner~\cite{RamirezG_IJCAI2009,RamirezG_AAAI2010} with the formalization of incomplete domain models introduced by Nguyen~\etal~\cite{PlanningIncomplete_NguyenK_2014,Nguyen_AIJ_2017}. 
			This new problem formalization allows the use of incomplete domain models for recognizing goals, relaxing the need for complete and correct discrete domain models;
			\item \emph{A novel landmark extraction algorithm that deals with incomplete domain models} (Section~\ref{section:ExtractingLandmarksInIncompleteDomains}), adapted from~\cite{Hoffmann2004_OrderedLandmarks}, enabling us to use new notions of landmarks from incomplete planning instance, and therefore, obtaining more information for recognizing goals over incomplete domains;
			\item \emph{New notions of landmarks for incomplete domain models} (Section~\ref{section:ExtractingLandmarksInIncompleteDomains}, formally defined in Definitions~\ref{def:DefiniteLandmark}, \ref{def:PossibleLandmark}, and~\ref{def:OverlookedLandmark}), namely, \textit{definite}, \textit{possible}, and \textit{overlooked landmarks}, that we use to develop enhanced recognition heuristics~\cite{PereiraNirMeneguzzi_AAAI2017}; and
			\item \emph{Enhanced landmark--based heuristics for goal recognition over incomplete domain models} (Sections~\ref{subsec:goal_completion_heuristic}~and~\ref{subsec:uniqueness_heuristic}), developed based on the new notions of landmarks over incomplete domain models. 
			Experiments over thousands of goal recognition problems (in fifteen incomplete planning domain models) show that our enhanced recognition approaches are fast and accurate when dealing with incomplete domains at all variations of observability and percentage of domain incompleteness.
		\end{itemize}
	
	\item The second contribution of this thesis is about the problem of recognizing goals over \textit{approximate continuous domain models}, presented in Chapter~\ref{chapter:GR_NominalModels}. We now enumerate our specific contributions towards solving this problem.
		\begin{itemize}
			\item \emph{A new problem formalization for goal recognition over nominal models} (Section~\ref{section:GR_NominalModels_Formalism}), that extends the formalization of Ramírez and Geffner~\cite{RamirezG_IJCAI2009,RamirezG_AAAI2010} by reasoning about the agent behavior using FHOC problems, in which the transition function is not well-defined but approximate. We note that this new formalization relaxes the need for using known and well-defined transition functions, allowing the representation of transition functions as ``black boxes'', which can be learned from data using learning techniques~\cite{SayWZS:ijcai17};
			\item \emph{A goal recognition approach based on the concept of Mirroring} (Section~\ref{section:GR_NominalModels_Mirroring}), that adopts the probabilistic framework of Kaminka~\etal~\cite{Kaminka_18_AAAI} and improves the efficiency of the original \textit{Mirroring} approach~\cite{Mor_ACS_16}. 
			In Section~\ref{section:GR_NominalModels_ExperimentsEvaluation}, we show that this approach is accurate when dealing with \textit{linear} and \textit{non-linear} domain models at very low levels of observability;
			\item \emph{A goal recognition approach based on cost-differences} (Section~\ref{section:GR_NominalModels_Counterfactual}), that follows the well-known probabilistic approach of Ramírez and Geffner~in~\cite{RamirezG_AAAI2010}. 
			Our recognition approach uses the concept of artificial potential fields, like~\cite{RamirezG_AAAI2010}, but computes the cost-difference based on the values of modified cost functions. Experiments and evaluation showed that the latter approach is not as accurate as the former, possibly because the planner we use converges to unhelpful local minima when extracting the trajectories.
		\end{itemize}
\end{enumerate}

%#########################################################
% \newpage
\section{Open Issues and Limitations}\label{chapter:conclusions:issues}

The contributions of this thesis are based on over two distinct types of domains models and techniques, and consequently, the issues and limitations of such contributions are not related.
In this section, we discuss the issues and limitations of our approaches to goal recognition over \textit{imperfect domain models}, and identify possible approaches to overcome such issues and limitations in future work.

As our ablation study in Section~\ref{section:GR_AblationStudy} shows, the main limitation of our enhanced heuristic approaches to goal recognition over \textit{incomplete domain models} is dealing with low observability. 
Our approaches are sensitive to the amount of landmark information, and as a result, they are not accurate at low levels of observability when increasing the percentage of domain incompleteness, \idest, for 10\% and 30\% of observability when the percentage of domain incompleteness is more than 60\%. 
he set of \textit{possible preconditions} may provide a new source of information during the landmark extraction, since our landmark extraction algorithm only explores the set of \textit{possible add effects} by ignoring the possible preconditions to build an ORPG. 
By doing so, one can potentially obtain more landmarks, and as a result, have more information to use in our heuristic approaches.

The main limitation of our approaches to goal recognition over \textit{nominal models} is that they rely on good trajectories to be accurate. 
Therefore, such recognition approaches are sensitive to the trajectories provided by the planner. 
We see in Section~\ref{section:GR_NominalModels_ExperimentsEvaluation} that, the results of one of our recognition approaches (Section~\ref{section:GR_NominalModels_Counterfactual}) have been directly affected by the gradient-based optimization planner~\cite{bueno:19:aaai,WuSS:nips17} we used in our experiments. 
We carefully analyzed the extracted trajectories for this approach, and we identified that for most problems this planner is converging to local minima. 
A possible solution to this would be using a different planner that relies on other optimization techniques, for instance, a planner that exploits Dynamic Programming (DDP)~\cite{mitrovic:10:adaptive,yamaguchi:16:icra}.

%#########################################################
\newpage
\section{Future Work}\label{chapter:conclusions:futurework}

Besides addressing the issues and limitations we pointed out above, as future work, we aim to explore multiple avenues to extend the goal recognition approaches we presented in this thesis. 
With respect to our work on goal recognition over \textit{incomplete discrete domain models}, we intend to explore two potential ideas to extend this work.
First, we intend to use a propagated RPG to reason about \textit{impossible incomplete domain models}, much like in~\cite{WeberBryce_ICAPS_2011}, in which the authors use such information to build a planning heuristic for planning over incomplete domain models.
Second, we aim to explore recent work that could be used as part of a complete methodology to develop domains includes an approach to acquire and infer information from domains with incomplete information based plan traces.
In this work, Zhuo~\etal~\cite{RefiningSTRIPSIncomplete_Zhuo_2013} developed an approach to refine incomplete domain models based on plan traces.

% As future work, we aim to compare our heuristic approaches to goal recognition over incomplete domains against a modified version of the work of E.-Martín~\etal~\cite{NASA_GoalRecognition_IJCAI2015}.\frm{I would not mention this here in the thesis. If there is a comparison to be made and not any substantial future work, then this opens you up to criticism, and to people in the committee telling you to do it.}

In order to improve our work on goal recognition over \textit{nominal models}, we intend to investigate modern probabilistic frameworks for goal recognition from the literature that cope with Euclidean space, such as the one introduced by Masters and Sardi{\~{n}}a~in~\cite{Masters_IJCAI2017,MastersS_JAIR_19}. 
An interesting extension for this work would be addressing another imperfect aspect in this type of model, such as \textit{approximate cost functions}. 
This involves learning the \textit{cost function} from data, much like we did to learn the transition function (Section~\ref{section:DNN_NominalModels}), but using a different learning technique, such as Linear Regression~\cite{Zhang04solvinglarge}.
We have conducted some preliminary tests in this regard, by learning the cost function from data, and using it along with \textit{nominal models}. 
Specifically, we modified the implementation of $\textsc{Tf-Plan}$ in~\cite{WuSS:nips17} to extract trajectories over models in which both the transition and cost functions are approximate. 
Our preliminary tests showed that for most problems the planner is overestimating the states in the trajectories. 
To overcome this issue, we intend to use the learned approximate functions along with another planner, and then evaluate our approaches over this extended setting, in which both the transition and cost functions are approximate. 

% \frm[inline]{Try to end with some kind of broader statement of how people can use your work, or something to that effect.}
% \rfp[inline]{I have already done that in the Introduction. Do you think that I should do that again here?!}
% \frm[inline]{I was thinking something more visionary, of the style of ``make the world a better place'' in the style of Silicon Valley, but not so egocentric. If you can't think of something, it's fine if you leave this out.}
%----------------------------------------------------------------------------------

%% file: apA_GR_Heuristics.tex
%!TEX root = ppgcc-thesis.tex
%----------------------------------------------------------------------------------
% Chapter (Appendix): Goal Recognition Heuristics
%
% ?
%
%----------------------------------------------------------------------------------
\chapter{Landmark-Based Heuristics for Goal Recognition}\label{appendixA:goalrecognition_heuristics}

In this Appendix, we develop goal recognition heuristics that rely on planning landmarks over \textit{complete and correct domains models}. We next show how we build such heuristics, presenting examples and theoretical properties. Note that we build the goal recognition heuristics presented in Chapter~\ref{chapter:GR_IncompleteDomains} using the concepts developed in this Appendix.

% @@@@@@@@@@@@@@@@@@@@@@@@@@@@@@@@@@@@@@@@@@@@@@@@@@@@@@@@@@@@@@@@@@@@@@@@@@@@@@@@@@@@@@@@@@
\section*{Computing Achieved Landmarks in Observations}\label{subsec:computingAchievedLandmarks}

An essential part of our heuristic approaches to goal recognition is the ability to track and compute the evidence of achieved fact landmarks in the observations. 
To do so, we compute the evidence of achieved fact landmarks in preconditions and effects of observed actions during a plan execution using the \textsc{ComputeAchievedLandmarks} function shown in Algorithm~\ref{alg:ComputeAchievedLandmarks}. 
This algorithm takes as input an initial state $\mathcal{I}$, a set of candidate goals $\mathcal{G}$, a sequence of observed actions $Obs$, and a map $\mathcal{L}_{\mathcal{G}}$ containing candidate goals and their extracted fact landmarks (provided by the \textsc{ExtractLandmarks}\footnote{This landmark extraction algorithm takes as input a planning domain definition $\Xi = \langle \Sigma, \mathcal{A}\rangle$, an initial state $\mathcal{I}$, and a set of candidate goals $\mathcal{G}$ or a single goal $G$. In case the input is a set of candidate goals $\mathcal{G}$, this function outputs a map $\mathcal{L}_{\mathcal{G}}$ that associates candidate goals to their respective ordered fact landmarks (\idest, a set of landmarks with an order relation).} function which computes fact landmarks given a planning domain). 
Note that Algorithm~\ref{alg:ComputeAchievedLandmarks} can be easily modified to allow it to deal with observations as states, so instead of analyzing preconditions and effects of actions, we compare the observations directly to computed landmarks.

Algorithm~\ref{alg:ComputeAchievedLandmarks} iterates over the set of candidate goals $\mathcal{G}$ (Line~\ref{alg:line:IterateCandidateGoals}) selecting the fact landmarks $\mathcal{L}_{G}$ of each goal $G$ in $\mathcal{L}_{\mathcal{G}}$ in Line~\ref{alg:line:GetLandmarks} and computes the fact landmarks that are in the initial state in Line~\ref{alg:line:CheckAchievedLandmarksInitialState}.
With this information, the algorithm iterates over the observed actions $Obs$ to compute the achieved fact landmarks of $G$ in Lines~\ref{alg:line:IterateObservations} to~\ref{alg:line:AddAchievedLandmarksGoal}. 
For each observed action $o$ in $Obs$, the algorithm computes all fact landmarks of $G$ that are either in the preconditions or effects of $o$ in Line~\ref{alg:line:CheckAchievedLandmarks}. 
As we deal with partial observations in a plan execution some executed actions may be missing from the observation sequence, thus whenever we identify a fact landmark, we also infer that its predecessors must have been achieved in Line~\ref{alg:line:CheckAchievedLandmarksPredecessors}. 
For example, consider that the set of fact landmarks to achieve a goal from a state is represented by the following ordered facts: \pred{(at A)} $\prec$ \pred{(at B)} $\prec$ \pred{(at C)} $\prec$ \pred{(at D)}, and we observe just one action during a plan execution, and this observed action contains the fact landmark \pred{(at C)} as an effect. 
From this observed action, we can infer that the predecessors of \pred{(at C)} must have been achieved before this observation (\idest, \pred{(at A)} and \pred{(at B)}). Therefore, we also include them as achieved landmarks. 
At the end of each iteration over an observed action $o$, the algorithm stores the set of achieved landmarks of $G$ in $\mathcal{AL}_{G}$ in Line \ref{alg:line:AddAchievedLandmarksGoal}. 
Finally, after computing the evidence of achieved landmarks in the observations for a candidate goal $G$, the algorithm stores the set of achieved landmarks $\mathcal{AL}_{G}$ of $G$ in $\Lambda_{\mathcal{G}}$ (Line~\ref{alg:line:AddAchievedLandmarksIntoMap}) and returns a map $\Lambda_{\mathcal{G}}$ containing all candidate goals and their respective achieved fact landmarks (Line~\ref{alg:line:ReturnAchievedLandmarks}). Example~\ref{exemp:computingAchievedLandmarks} illustrates the execution of Algorithm~\ref{alg:ComputeAchievedLandmarks} to compute achieved landmarks from the observations of our running example.

\floatname{algorithm}{Algorithm}
\begin{algorithm}[h!]
    \caption{Compute Achieved Landmarks in Observations.}
    \textbf{Input:} $\mathcal{I}$ \textit{initial state}, $\mathcal{G}$ \textit{set of candidate goals}, $Obs$ \textit{observations}, and $\mathcal{L}_{\mathcal{G}}$ \textit{goals and their extracted landmarks}.
    \\\textbf{Output:} \textit{A map of goals to their achieved landmarks.}
	\label{alg:ComputeAchievedLandmarks}
    \begin{algorithmic}[1]
        \Function{ComputeAchievedLandmarks}{$\mathcal{I}, \mathcal{G}, Obs, \mathcal{L}_{\mathcal{G}}$}
        \State $\Lambda_{\mathcal{G}} \gets \langle \rangle$ 
        \For{\textbf{each} goal $G$ in $\mathcal{G}$}\label{alg:line:IterateCandidateGoals} \Comment{\textit{Map goals $\mathcal{G}$ to their respective achieved landmarks}.}
			\State $\mathcal{L}_{G} \gets$ fact landmarks of $G$ s.t $\langle G, \mathcal{L}_{G}\rangle$ in $\mathcal{L}_{\mathcal{G}}$\label{alg:line:GetLandmarks} 
			\State $\mathcal{L_{\mathcal{I}}} \gets$ all fact landmarks $L \in \mathcal{I}$\label{alg:line:CheckAchievedLandmarksInitialState}
			\State $\mathcal{L} \gets \emptyset$
			\For{\textbf{each} observed action $o$ in $Obs$}\label{alg:line:IterateObservations}
				\State $\mathcal{L} \gets$ $\{ L \in \mathcal{L}_{G} | L \in \mathit{pre}(o) \cup \mathit{eff}(o)^{+} \land L \notin \mathcal{L} \}$\label{alg:line:CheckAchievedLandmarks}
				\State $\mathcal{L}_{\prec} \gets$ predecessors $L_{\prec}$ of all $L \in \mathcal{L}$, s.t $L_{\prec} \notin  \mathcal{L}$ \label{alg:line:CheckAchievedLandmarksPredecessors} 
				\State $\mathcal{AL}_{G} \gets \mathcal{AL}_{G} \cup \lbrace \mathcal{L_{\mathcal{I}}} \cup \mathcal{L} \cup \mathcal{L}_{\prec}\rbrace$\label{alg:line:AddAchievedLandmarksGoal}
			\EndFor
			\State $\Lambda_{\mathcal{G}}(G) \gets \mathcal{AL}_{G}$ \Comment{\textit{Achieved landmarks of $G$.}}\label{alg:line:AddAchievedLandmarksIntoMap}
		\EndFor
		\State \textbf{return} $\Lambda_{\mathcal{G}}$\label{alg:line:ReturnAchievedLandmarks}
        \EndFunction
    \end{algorithmic}
\end{algorithm} 

\begin{example}\label{exemp:computingAchievedLandmarks}
Consider the \textsc{Blocks-World} example from {\normalfont Figure~\ref{fig:GoalRecognition-BlocksWorld}}, and the following observed actions: {\normalfont \pred{(unstack E A)}} and {\normalfont \pred{(stack E D)}}.  From these observed actions, the candidate goal {\normalfont \pred{RED}}, and the set of fact landmarks of this candidate goal {\normalfont (Figure~\ref{fig:RED-AchievedLandmarks})}, our algorithm computes that the following fact landmarks have been achieved: 

{
\normalfont
\begin{itemize}
	\item $\mathcal{AL}_{\pred{RED}}=\lbrace$\pred{[(clear R)]}, \pred{[(on E D)]}, \\\pred{[(clear R) (ontable R) (handempty)]}, \\ \pred{[(on E A) (clear E) (handempty)]}, \\\pred{[(clear D) (holding E)]}, \\\pred{[(on D B) (clear D) (handempty)]}$\rbrace$
\end{itemize}
}

In the preconditions of {\normalfont \pred{(unstack E A)}} the algorithm computes {\normalfont \pred{[(on E A) (clear E) (handempty)]}}. 
Subsequently, in the preconditions and effects of {\normalfont \pred{(stack E D)}} the algorithm computes {\normalfont \pred{[(clear D) (holding E)]}} and {\normalfont \pred{[(on E D)]}}, while it computes the other achieved landmarks for the word {\normalfont \pred{RED}} from the initial state.
{\normalfont Figure~\ref{fig:RED-AchievedLandmarks}} shows the set of achieved landmarks for the word {\normalfont \pred{RED}} in gray.
\end{example}

\begin{figure}[th!]
  \centering
  \includegraphics[width=0.8\linewidth]{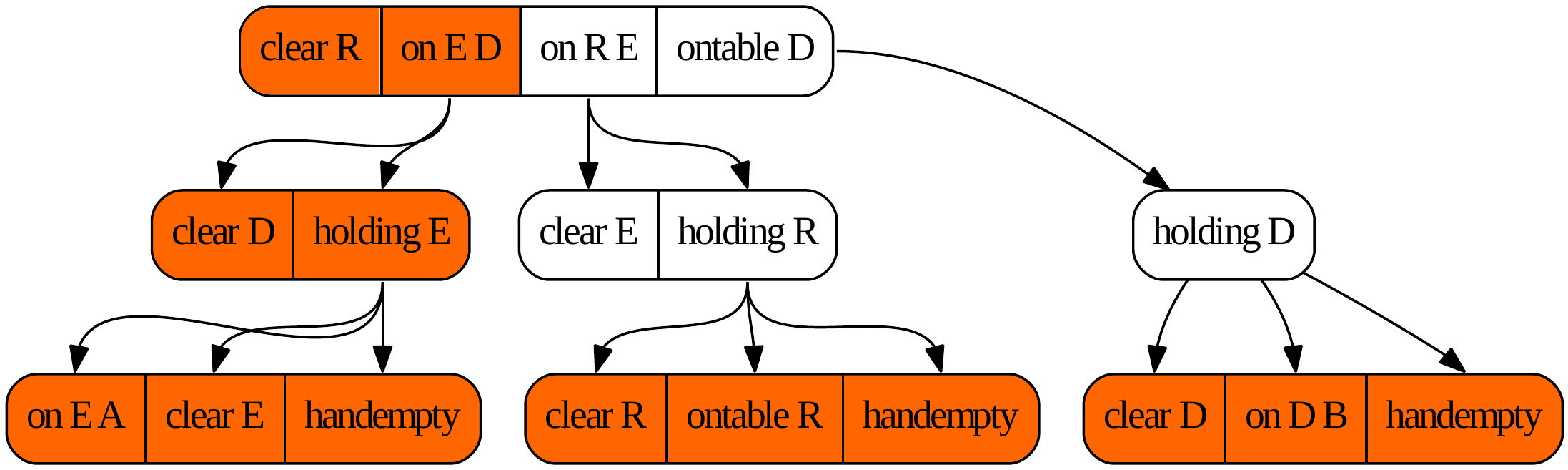}
  \caption{Ordered fact landmarks extracted for the word \pred{RED}. Fact landmarks that must be true together are represented by connected boxes. Connected boxes in grey represent achieved fact landmarks. Edges represent prerequisites between landmarks.}
  \label{fig:RED-AchievedLandmarks}
\end{figure}

The complexity of computing achieved landmarks in observations (Algorithm~\ref{alg:ComputeAchievedLandmarks}) with the process of extracting landmarks ($EL$) is: $O(EL + |\mathcal{G}|\cdot|Obs|\cdot|\mathcal{L}_{\mathcal{G}}|)$, where $\mathcal{G}$ is the set of candidate goals, $Obs$ is the observation sequence, and $\mathcal{L}_{\mathcal{G}}$ is the extracted landmarks for $\mathcal{G}$. 
The complexity of our approach is dominated by the complexity of $EL$, and thus, given a suitable implementation of $EL$, our approach has polynomial complexity.

% @@@@@@@@@@@@@@@@@@@@@@@@@@@@@@@@@@@@@@@@@@@@@@@@@@@@@@@@@@@@@@@@@@@@@@@@@@@@@@@@@@@@@@@@@@
\section*{Landmark-Based Goal Completion Heuristic}\label{subsec:goalCompletionHeuristic}
We now describe a goal recognition heuristic that estimates the percentage of completion of a goal based on the number of landmarks that have been detected, and are required to achieve that goal~\cite{PereiraNirMeneguzzi_AAAI2017}. 
This estimate represents the percentage of sub-goals in a goal that have been accomplished based on the evidence of achieved fact landmarks in the observations. 
We note that a candidate goal is composed of sub-goals comprised of the atomic facts that are part of a conjunction of facts in the goal definition.

Our recognition heuristic estimates the percentage of completion towards a goal by using the set of achieved landmarks computed by Algorithm~\ref{alg:ComputeAchievedLandmarks} (\textsc{ComputeAchievedLandmarks}).
Namely, this heuristic operates by aggregating the percentage of completion of each sub-goal into an overall percentage of completion for all facts of a goal. 
We denote this heuristic as $\mathit{h_{gc}}$, and it is formally defined by Equation~\ref{eq:heuristic}, where $\mathcal{AL}_{g}$ is the number of achieved landmarks from observations of every sub-goal $g$ of a goal $G$ in $\mathcal{AL}_{G}$, and $\mathcal{L}_{g}$ represents the number of necessary landmarks to achieve every sub-goal $g$ of $G$ in $\mathcal{L}_{G}$.

\begin{equation}
\label{eq:heuristic}
h_{gc}(G, \mathcal{AL}_{G}, \mathcal{L}_{G}) = \left(\frac{\sum_{g \in G} \frac{|\mathcal{AL}_{g} \in \mathcal{AL}_{G} |}{|\mathcal{L}_{g} \in \mathcal{L}_{G}|}}{ |G| }\right)
\end{equation}

Thus, heuristic $\mathit{h_{gc}}$ estimates the completion of a goal $G$ by calculating the ratio between the sum of the percentage of completion for every sub-goal $g \in G$, \idest, $\sum_{g \in G} \frac{|\mathcal{AL}_{g} \in \mathcal{AL}_{G} |}{|\mathcal{L}_{g} \in \mathcal{L}_{G}|}$, and the size $|G|$ of the set of sub-goals, that is, the number of sub-goals in $G$. 

Algorithm~\ref{alg:RecognizeHeuristicCompletion} describes how to recognize goals using the $\mathit{h_{gc}}$ heuristic and takes as input a goal recognition problem $T_{GR}$, as well as a threshold value $\theta$. 
The $\theta$ threshold gives us flexibility to avoid eliminating candidate goals whose percentage of goal completion are close to the highest completion value.
In Line~\ref{alg:Algo2:ExtractLandmarks}, the algorithm uses the \textsc{ExtractLandmarks} function to extract fact landmarks for all candidate goals. 
By taking as input the initial state $\mathcal{I}$, the observations $O$, and the extracted landmarks $\mathcal{L}_{\mathcal{G}}$, in Line~\ref{alg:Algo2:ComputeAchievedLandmarks}, our algorithm first computes the set of achieved landmarks $\Lambda_{\mathcal{G}}$ for every candidate goal using Algorithm~\ref{alg:ComputeAchievedLandmarks}. 
Finally, the algorithm uses the heuristic $\mathit{h_{gc}}$ to estimate goal completion for every candidate $G$ in $\mathcal{G}$, and as output (Line~\ref{alg:Algo2:ReturnHeuristicCompletion}), the algorithm returns those candidate goals with the highest estimated value within the threshold $\theta$. Example~\ref{exemp:goalCompletionHeuristic} shows how heuristic $\mathit{h_{gc}}$ estimates the completion of a candidate goal. 

\floatname{algorithm}{Algorithm}
\begin{algorithm}[h!]
    \caption{Recognize goals using the Goal Completion Heuristic $\mathit{h_{gc}}$.} 
    \textbf{Input:} $\Xi$ \textit{planning domain definition}, $\mathcal{I}$ \textit{initial state}, $\mathcal{G}$ \textit{set of candidate goals}, $Obs$ \textit{observations}, and $\theta$ \emph{threshold}.
    \\\textbf{Output:} \textit{Recognized goal(s).}
    \label{alg:RecognizeHeuristicCompletion}
    \begin{algorithmic}[1]
        \Function{Recognize}{$\Xi, \mathcal{I}, \mathcal{G}, Obs, \theta$}
        		\State $\mathcal{L}_{\mathcal{G}} \gets$ \textsc{ExtractLandmarks}($\Xi, \mathcal{I}, \mathcal{G}$)\label{alg:Algo2:ExtractLandmarks}
        		\State $\Lambda_{\mathcal{G}} \gets$ \textsc{ComputeAchievedLandmarks}($\mathcal{I}, \mathcal{G}, Obs, \mathcal{L}_{\mathcal{G}}$)\label{alg:Algo2:ComputeAchievedLandmarks}
        		\State $\mathit{maxh} \gets \displaystyle \max_{G' \in \mathcal{G}} h_{gc}(G',\Lambda_{\mathcal{G}}(G'), \mathcal{L}_{\mathcal{G}}(G'))$
        		\State \textbf{return} {all $G$ s.t $G \in \mathcal{G}$ and \newline {\phantom{return }} $\mathit{h_{gc}}(G, \Lambda_{\mathcal{G}}(G), \mathcal{L}_{\mathcal{G}}(G)) \geq (\mathit{maxh} - \theta)$}\label{alg:Algo2:ReturnHeuristicCompletion}
        \EndFunction
    \end{algorithmic}
\end{algorithm}

\begin{example}\label{exemp:goalCompletionHeuristic}
As an example of how heuristic $\mathit{h_{gc}}$ estimates goal completion of a candidate goal, recall the \textsc{Blocks-World} example from {\normalfont Figure~\ref{fig:GoalRecognition-BlocksWorld}}. Consider that among these candidate goals {\normalfont (\pred{RED}, \pred{BED},} and {\normalfont \pred{SAD})} the correct hidden goal is {\normalfont \pred{RED}}, and we observe the following partial sequence of actions: {\normalfont \pred{(unstack E A)}} and {\normalfont \pred{(stack E D)}}. Thus, based on the achieved landmarks $\mathcal{AL}_{\texttt{RED}}$ computed using {\normalfont Algorithm~\ref{alg:ComputeAchievedLandmarks}} {\normalfont (Figure~\ref{fig:RED-AchievedLandmarks})}, our heuristic $\mathit{h_{gc}}$ estimates that the percentage of completion for the goal {\normalfont \texttt{RED}} is 0.66: {\normalfont \pred{(clear R)} = $\frac{1}{1}$ $+$ \pred{(on E D)} = $\frac{3}{3}$ $+$ \pred{(on R E)} = $\frac{1}{3}$ $+$ \pred{(ontable D)} = $\frac{1}{3}$}, and hence, $\frac{2.66}{4}$ = 0.66. For the words {\normalfont \pred{BED}} and {\normalfont \pred{SAD}} our heuristic $\mathit{h_{gc}}$ estimates respectively, 0.54 and 0.58.	
\end{example}

Besides extracting landmarks for every candidate goal ($EL$), our landmark-based goal completion approach iterates over the set of candidate goals $\mathcal{G}$, the observations sequence $Obs$, and the extracted landmarks $\mathcal{L}_{\mathcal{G}}$. 
The heuristic computation of $\mathit{h_{gc}}$ ($HC$) is linear on the number of fact landmarks. Thus, the complexity of this approach is: $O(EL + |\mathcal{G}|\cdot|Obs|\cdot|\mathcal{L}_{\mathcal{G}}| + HC)$. 
Finally, the goal ranking based on $\mathit{h_{gc}}$ always ensures (under full observability) that the correct goal ranks highest (\idest, it is sound), with possible ties, as stated in Theorem~\ref{thm:hgc_soundness}. 

\begin{theorem}[\textbf{Soundness of the $h_{gc}$ Goal Recognition Heuristic}]\label{thm:hgc_soundness}
Let $T_{GR} = \langle\Xi,\mathcal{I} ,\mathcal{G}, Obs\rangle$ be a goal recognition problem with candidate goals $\mathcal{G}$ such that $\forall{G_1,G_2}\in\mathcal{G}, G_1 \neq G_2 \rightarrow G_1 \not\subset G_2$, a complete and noiseless observation sequence $Obs = \langle o_1, o_2, ..., o_n\rangle$. 
If $G^{*} \in \mathcal{G}$ is the correct hidden goal, then, for any landmark extraction algorithm that generates fact landmarks $\mathcal{L}_{\mathcal{G}}$ and computed landmarks $\Lambda_{\mathcal{G}}$, the estimated value of $\mathit{h_{gc}}$ will always be highest for the correct hidden goal $G^{*}$, \idest, $\forall G \in \mathcal{G}$ it is the case that $\mathit{h_{gc}}(G^{*}, \Lambda_{\mathcal{G}}(G^{*}), \mathcal{L}_{\mathcal{G}}(G^{*})) \geq \mathit{h_{gc}}(G, \Lambda_{\mathcal{G}}(G), \mathcal{L}_{\mathcal{G}}(G)) $.
\end{theorem}

\begin{proof}\label{pf:hgc_soundness}
The proof is straightforward from the definition of fact landmarks ensuring they are necessary conditions to achieve a goal $G$ and that all facts $g \in G$ are necessary. 
Let us first assume that any pair of goals $G_1,G_2 \in \mathcal{G}$ are different, \idest, $G_1 \cap G_2 \neq \emptyset$, and that no action $a$ in the domain $\Xi$ achieves facts that are in any pair of goals simultaneously. 
Since any landmark extraction algorithm includes all facts $g \in G_1$ as landmarks for a goal $G_1$, then, for every other goal $G_2$, there exists at least one fact $g$ such that $g \in G_1 \land g \not\in G_2$ that sets it apart from  $G_2$. 
Under these circumstances, an observation sequence $Obs$ for the correct goal $G^{*}$ will have achieved a set of landmarks $\Lambda_{\mathcal{G}}(G^{*})$ that is exactly the same as the complete computed set of landmarks $\mathcal{L}_{\mathcal{G}}(G^{*})$ for $G^{*}$. 
Hence $\mathit{h_{gc}}(G^{*}, \Lambda_{\mathcal{G}}(G^{*}), \mathcal{L}_{\mathcal{G}}(G^{*})) = 1$, and $\mathit{h_{gc}}(G, \Lambda_{\mathcal{G}}(G), \mathcal{L}_{\mathcal{G}}(G)) < 1$ for any other goal $G \in \mathcal{G}$, since the numerator of the $\mathit{h_{gc}}$ computation will be missing fact $g$ for $G$ as $g$ is not a landmark of $G$. 
If we drop the assumption about the actions not achieving facts simultaneously in any pair of goals or that goals are identical, it is possible that $\mathit{h_{gc}}(G^{*}, \Lambda_{\mathcal{G}}(G^{*}), \mathcal{L}_{\mathcal{G}}(G^{*})) = \mathit{h_{gc}}(G, \Lambda_{\mathcal{G}}(G), \mathcal{L}_{\mathcal{G}}(G)) = 1$, which still ensures that the under $\mathit{h_{gc}}$, $G^{*}$ always ranks at the top, possibly tied with other goals.
\end{proof}

Thus, our goal completion heuristic is sound under full observability in the sense that it can never rank the wrong goal higher than the correct goal when we observe the landmarks. 
We note that there is one specific case when our landmark approach can provide wrong rankings, but which we explicitly exclude from the theorem, which is when the set of candidate goals contains two goals such that one is a sub-goal of the other (\idest, $G_1, G_2 \in\mathcal{G}$ and $G_1 \subseteq G_2$). 
In this case, any kind of ``distance'' to goal metric will report $G_1$ as being more likely than $G_2$ until the observations take the observed agent to $G_2$ (and these two goals will be tied in the heuristic). 
We close this section by commenting on the effect of landmark orderings on the accuracy of the heuristic. 
Specifically, although we do use the landmark order to infer the achievement of necessary prior landmarks that were not observed under missing observations, our heuristic  does not consider the actual ordering of the landmarks. 
We infer prior landmarks to obtain more landmarks when we deal with partial observability.
Nevertheless, we have experimented with different scoring mechanisms to account for landmarks having --- or not having --- been observed in the expected order, and these showed almost no advantage over the current heuristic. 
Consequently, although there are various different algorithms that generate better landmark orderings~\cite{Hoffmann2004_OrderedLandmarks}, the way in which we use the landmarks does not seem to be affected by more or less accurate landmark orderings. 

There are two additional properties provable for our $h_{gc}$ heuristic, first, given how our heuristic accounts for landmarks, the value outputted by the heuristic is strictly increasing as observations increase in length. 

\begin{proposition}[\bf Monotonicity of $h_{gc}$] The value of $h_{gc}$ is monotonically (non-strictly) increasing in the observation sequence.
\end{proposition}
\begin{proof}
By definition, $\mathcal{AL}_G$ is monotonically increasing, while all other values in $h_{gc}$ remain constant. 
Therefore, from Equation \ref{eq:heuristic}, it is clear that $h_{gc}$ must increase.
\end{proof}

Further, a corollary of Theorem~\ref{thm:hgc_soundness} is that, under full observation, only the correct goal can reach a heuristic value of $1$. 
This also illustrates why we restrict the  theorem to settings where candidate goals are not subgoals of each other. 
Consider a goal to be at position $d$, and another to be at position $g$, with landmarks $a,b,c,d,e,f,g$. Since $d$ itself is a landmark of $g$, $d$ is implicitly a subgoal of $g$. 
If we observe all landmarks in an observation, then $h_{gc}(d)=\frac{4}{4}=1$, and $h_{gc}(g)=\frac{7}{7}=1$, which leads to Corollary~\ref{cor:thereCanBeOnly1}. 

\begin{corollary}\label{cor:thereCanBeOnly1} If the goal being recognized has no subgoals being recognized under full observability, then $h_{gc}=1$ iff the goal the heuristic is recognizing has been achieved.
\end{corollary}

\begin{proof}
$h_{gc}=1$ when $\frac{\sum_g \frac{|\mathcal{AL}_G|}{\mathcal{L}}}{|G|}=1$, which can only occur when $|\mathcal{AL}_g|=|\mathcal{L}_g|$ for all $g \in G$. This clearly occurs when the goal being recognized is achieved. However, if the heuristic is also recognizing a subgoal, then this condition can be satisfied for the subgoal, hence the exception in the proposition.
\end{proof}

% @@@@@@@@@@@@@@@@@@@@@@@@@@@@@@@@@@@@@@@@@@@@@@@@@@@@@@@@@@@@@@@@@@@@@@@@@@@@@@@@@@@@@@@@@@
\section*{Landmark-Based Uniqueness Heuristic}
\label{subsec:uniquenessHeuristic}

We now turn our attention to another heuristic which uses a measure of the uniqueness of landmarks. 
Many goal recognition problems contain multiple candidate goals that share common fact landmarks, generating ambiguity for our previous approaches. 
Clearly, landmarks that are common to multiple candidate goals are less useful for recognizing a goal than landmarks that exist for only a single goal. 
As a consequence, computing how unique (and thus informative) each landmark is can help disambiguate similar goals for a set of candidate goals. 
To develop this heuristic based on this intuition, we introduce the concept of \textit{landmark uniqueness}, which is the inverse frequency of a landmark among the landmarks found in a set of candidate goals, and lies in the range (0,1]. For example, consider a landmark $L$ that occurs only for a single goal within a set of candidate goals; since such a landmark is clearly unique, its uniqueness value is maximal (i.e., 1). Equation~\ref{eq:LandmarksUniqueness} formalizes this intuition, describing how the \textit{landmark uniqueness value} is computed for a landmark $L$ and a set of landmarks for goals $\mathcal{L}_{\mathcal{G}}$.

Using the \textit{landmark uniqueness value}, we estimate which candidate goal is the intended one by summing the uniqueness values of the landmarks achieved in the observations. 
Unlike our previous heuristic, which estimates progress towards goal completion by analyzing sub-goals and their achieved landmarks, the landmark-based uniqueness heuristic estimates the goal completion of a candidate goal $G$ by calculating the ratio between the sum of the uniqueness value of the achieved landmarks of $G$ and the sum of the uniqueness value of all landmarks of $G$. 
This algorithm effectively weighs the completion value by the informational value of a landmark so that unique landmarks have the highest weight. 
To estimate goal completion using the landmark uniqueness value, we calculate the uniqueness value for every extracted landmark in the set of landmarks of the candidate goals using Equation~\ref{eq:LandmarksUniqueness}.
This computes the landmark uniqueness value of every landmark $L$ of $\mathcal{L}_{\mathcal{G}}$ and store it into $\Upsilon_{uv}$. This heuristic is denoted as $\mathit{h_{uniq}}$ and formally defined in Equation~\ref{eq:HeuristicLandmarksUniqueness}. 

\begin{equation}
\label{eq:LandmarksUniqueness}
L_{\mathit{Uniq}}(L, \mathcal{L}_{\mathcal{G}}) = \left(\frac{1}{\displaystyle\sum_{\mathcal{L} \in \mathcal{L_G}} |\{L |L \in \mathcal{L}\}|}\right)
\end{equation}
\begin{equation}
\label{eq:HeuristicLandmarksUniqueness}
h_{\mathit{uniq}}(G, \mathcal{AL}_{G}, \mathcal{L}_{G}, \Upsilon_{uv}) = \left(
\frac
{\displaystyle\sum_{\mathcal{A}_{L} \in \mathcal{AL}_{G}}\Upsilon_{uv}(\mathcal{A}_{L})}
{\displaystyle\sum_{L \in \mathcal{L}_{G}}\Upsilon_{uv}(L)}\right)
\end{equation}

Algorithm~\ref{alg:RecognizeHeuristicUniqueness} formalizes a goal recognition function that uses the $\mathit{h_{uniq}}$ heuristic. 
This algorithm takes as input the same parameters as the previous approach: a goal recognition problem and a threshold $\theta$. 
Like Algorithm~\ref{alg:ComputeAchievedLandmarks}, this algorithm extracts the set of landmarks for all candidate goals from the initial state $\mathcal{I}$, stores them in $\mathcal{L}_{\mathcal{G}}$ (Line~\ref{alg:Algo3:ExtractLandmarks}), and computes the set of achieved landmarks based on the observations, storing these in $\Lambda_{\mathcal{G}}$. 
Unlike Algorithm~\ref{alg:RecognizeHeuristicCompletion}, in Line~\ref{alg:Algo3:ComputeLandmarkUniquenessValue} this algorithm computes the landmark uniqueness value for every landmark $L$ in $\mathcal{L}_{\mathcal{G}}$ and stores it into $\Upsilon_{uv}$. 
Finally, using these computed structures, the algorithm recognizes which candidate goal is being pursued from observations using the heuristic $\mathit{h_{uniq}}$, returning those candidate goals with the highest estimated value within the $\theta$ threshold. Example~\ref{exemp:uniquenessHeuristic} shows how heuristic $\mathit{h_{uniq}}$ uses the concept of landmark uniqueness value to goal recognition.

\floatname{algorithm}{Algorithm}
\begin{algorithm}[h!]
    \caption{Recognize goals using Uniqueness Heuristic $\mathit{h_{uniq}}$.} 
    \textbf{Input:} $\Xi$ \textit{planning domain definition}, $\mathcal{I}$ \textit{initial state}, $\mathcal{G}$ \textit{set of candidate goals}, $Obs$ \textit{observations}, and $\theta$ \emph{threshold}.
    \\\textbf{Output:} \textit{Recognized goal(s).}
    \label{alg:RecognizeHeuristicUniqueness}
    \begin{algorithmic}[1]
        \Function{Recognize}{$\Xi, \mathcal{I}, \mathcal{G}, Obs, \theta$}
        		\State $\mathcal{L}_{\mathcal{G}} \gets$ \textsc{ExtractLandmarks}($\Xi, \mathcal{I}, \mathcal{G}$)\label{alg:Algo3:ExtractLandmarks}
				\State $\Lambda_{\mathcal{G}} \gets$ \textsc{ComputeAchievedLandmarks}($\mathcal{I},\mathcal{G},Obs,\mathcal{L}_{\mathcal{G}}$)\label{alg:Algo3:ComputeAchievedLAndmarks}
        		\State $\Upsilon_{uv} \gets \langle\rangle$ \Comment{\textit{Map of landmarks to their uniqueness value}.}
        		\For{\textbf{each} fact landmark $L$ in $\mathcal{L}_{\mathcal{G}}$}
        			\State $\Upsilon_{uv}(L) \gets L_{Uniq}(L,\mathcal{L}_{\mathcal{G}})$ \label{alg:Algo3:ComputeLandmarkUniquenessValue}
        		\EndFor
			\State $\mathit{maxh} \gets \displaystyle \max_{G' \in \mathcal{G}} h_{uniq}(G',\Lambda_{\mathcal{G}}(G'), \mathcal{L}_{\mathcal{G}}(G'), \Upsilon_{uv})$
			\State \textbf{return} {all $G$ s.t $G \in \mathcal{G}$ and \newline {\phantom{return }} $h_{uniq}(G, \Lambda_{\mathcal{G}}(G), \mathcal{L}_{\mathcal{G}}(G),\Upsilon_{uv}) \geq (\mathit{maxh} - \theta)$}
        \EndFunction
    \end{algorithmic}
\end{algorithm}

\begin{example}\label{exemp:uniquenessHeuristic}
Recall the \textsc{Blocks-World} example from {\normalfont Figure~\ref{fig:GoalRecognition-BlocksWorld}} and consider the following observed actions: {\normalfont \pred{(unstack E A)}} and {\normalfont \pred{(stack E D)}}. 
{\normalfont Listing~\ref{lst:FactLandmarksUniquenessValue}} shows the set of extracted fact landmarks for the candidate goals in the \textsc{Blocks-World} example and their respective uniqueness value. 
Based on the set of achieved landmarks (shown in bold in {\normalfont Listing~\ref{lst:FactLandmarksUniquenessValue}}), our heuristic $\mathit{h_{uniq}}$ estimates the following percentage for each candidate goal: $\mathit{h_{uniq}}${\normalfont (\pred{RED})} = $\frac{3.66}{6.33}$ = 0.58; $\mathit{h_{uniq}}${\normalfont (\pred{BED})} = $\frac{2.66}{6.33}$ = 0.42; and $\mathit{h_{uniq}}${\normalfont (\pred{SAD})} = $\frac{3.66}{8.33}$ = 0.44. 
In this case, {\normalfont Algorithm~\ref{alg:RecognizeHeuristicUniqueness}} correctly estimates {\normalfont \pred{RED}} to be the intended goal since it has the highest heuristic value.	
\end{example}

\begin{lstlisting}[float=!h,caption={Extracted fact landmarks for the \textsc{Blocks-World} example in Figure~\ref{fig:GoalRecognition-BlocksWorld} and their respective uniqueness value.},label={lst:FactLandmarksUniquenessValue},basicstyle=\ttfamily\footnotesize]
- (and (clear B) (on B E) (on E D) (ontable D)) = 6.33
  (*\bfseries[(on E D)] = 0.5*), (*\bfseries[(clear D) (holding E)] = 0.5*),
  (*\bfseries[(on E A) (clear E) (handempty)] = 0.33*), [(ontable D)] = 0.33,
  (*\bfseries[(on D B) (clear D) (handempty)] = 0.33*), [(holding D)] = 0.33,
  (*\bfseries[(clear B) (ontable B) (handempty)] = 1.0*), [(on B E)] = 1.0,
  [(clear B)] = 1.0, [(clear E) (holding B)] = 1.0

- (and (clear S) (on S A) (on A D) (ontable D)) = 8.33
  (*\bfseries[(clear S)] = 1.0*), [(on A D)] = 1.0, [(on S A)] = 1.0,
  (*\bfseries[(clear A) (ontable A) (handempty)] = 1.0*), [(ontable D)] = 0.33,
  (*\bfseries[(clear S) (ontable S) (handempty)] = 1.0*), [(holding D)] = 0.33,
  (*\bfseries[(on E A) (clear E) (handempty)] = 0.33*),
  (*\bfseries[(on D B) (clear D) (handempty)] = 0.33*),
  [(clear A) (holding S)] = 1.0, [(clear D) (holding A)] = 1.0

- (and (clear R) (on R E) (on E D) (ontable D)) = 6.33
  (*\bfseries[(clear R)] = 1.0*), (*\bfseries[(clear R) (ontable R) (handempty)] = 1.0*),
  (*\bfseries[(clear D) (holding E)] = 0.5*), (*\bfseries[(on E D)] = 0.5*),
  (*\bfseries[(on E A) (clear E) (handempty)] = 0.33*), [(ontable D)] = 0.33,
  (*\bfseries[(on D B) (clear D) (handempty)] = 0.33*), [(holding D)] = 0.33,
  [(on R E)] = 1.0, [(clear E) (holding R)] = 1.0
\end{lstlisting}

Similar to our landmark-based goal completion approach, this approach iterates over the set of candidate goals $\mathcal{G}$, the observations sequence $Obs$, and the extracted landmarks $\mathcal{L}_{\mathcal{G}}$. 
However, in this approach we weight each landmark by how common this landmark is across all goal hypotheses. 
We call this weight the uniqueness value ($CLUniq$) and its computation is linear on the number of landmarks. 
The heuristic computation of $\mathit{h_{uniq}}$ ($HC$) is also linear on the number of fact landmarks. 
Thus, the complexity of this approach is: $O(EL + |\mathcal{G}|\cdot|Obs|\cdot|\mathcal{L}_{\mathcal{G}}| + CLUniq + HC)$.
Finally, since this is just a weighted version of the $\mathit{h_{gc}}$ heuristic, it follows trivially from Theorem~\ref{thm:hgc_soundness} that, for full observations, $\mathit{h_{uniq}}$ always ranks the correct goal $G^{*}$ highest.

\begin{corollary}[\textbf{Correctness of $\mathit{h_{uniq}}$ Goal Recognition Heuristic}]\label{cor:huniq_correct}
Let $T_{GR} = \langle\Xi,\mathcal{I} ,\mathcal{G}, Obs\rangle$ be a goal recognition problem with candidate goals $G \in \mathcal{G}$, a complete and noiseless observation sequence $Obs = \langle o_1, o_2, ..., o_n\rangle$. 
If $G^{*} \in \mathcal{G}$ is the correct goal, then, for any landmark extraction algorithm that generates fact landmarks $\mathcal{L}_{\mathcal{G}}$ and computed landmarks $\Lambda_{\mathcal{G}}$, the estimated value of $\mathit{h_{uniq}}$ will always be highest for the correct goal $G^{*}$, more specifically, $\forall G \in \mathcal{G}$ it is the case that $\mathit{h_{uniq}}(G^{*}, \Lambda_{\mathcal{G}}(G^{*}), \mathcal{L}_{\mathcal{G}}(G^{*})) \geq \mathit{h_{uniq}}(G, \Lambda_{\mathcal{G}}(G), \mathcal{L}_{\mathcal{G}}(G))$.
\end{corollary}

%% file: apB_IncompleteDomains_Results.tex
%!TEX root = ppgcc-thesis.tex
%----------------------------------------------------------------------------------
% Chapter (Appendix): Goal Recognition over Incomplete Domain Models - Detailed Results
%
% ?
%
%----------------------------------------------------------------------------------
\chapter{Goal Recognition over Incomplete Domain Models - Detailed Results}\label{appendixB:goalrecognition_incompletedomains}

We now present a detailed experimental evaluation by showing a comparison between our enhanced heuristics (the ones with the best results, \idest, $\mathit{h_{\widetilde{GC}}}$ and $\mathit{h_{\widetilde{UNIQ}}}$ using all types of landmarks, \textit{define}, \textit{possible} and \textit{overlooked} landmarks) and the baseline approaches $\mathit{h_{gc}}$ and $\mathit{h_{uniq}}$ (\ref{appendixA:goalrecognition_heuristics}) per domain over our recognition datasets with incomplete domain models.
We evaluated these approaches over thousands of goal recognition problems using the fifth-ten incomplete domains models, varying the domain incompleteness between 20\% and 80\%. 
The results we present here have been used to build the tables and figures in Section~\ref{section:GR_IncompleteDomains_ExperimentsEvaluation}, more specifically, Table~\ref{tab:AblationStudy}, Figures~\ref{fig:gc-definite_possible_overlooked}--\ref{fig:uniq-possible_overlooked} (\textit{Correlation}), and Figure~\ref{fig:F1-score_comparison} (\textit{F1-score}).

Tables \ref{tab:ExperimentalResults1} and \ref{tab:ExperimentalResults2} show the results using the following metrics: recognition time in seconds (\textit{Time}); \textit{Accuracy} (\textit{Acc} \%)\footnote{This metric is analogous to the  \textit{Quality} metric (also denoted as Q), used for most planning-based goal recognition approaches~\cite{RamirezG_IJCAI2009,RamirezG_AAAI2010,NASA_GoalRecognition_IJCAI2015,Sohrabi_IJCAI2016}.}, representing the fraction of time steps in which the correct goal was among the goals found to be most likely, \idest, how good our approaches are for recognizing the correct goal $G$ in $\mathcal{G}$ over time; and \textit{Spread in} $\mathcal{G}$ (\textit{S}) represents the average number of returned goals. Each row in the tables express averages for the number of candidate goals $\mathcal{G}$; the percentage of the plan that is actually observed \% $Obs$; the average number of observations (actions) per problem $Obs$; and for each approach, the time in seconds to recognize the goal given the observations (\textit{Time}); \textit{Acc \%} with which the approaches correctly infer the goal; and $S$ represents the average number of returned goals. Below the name of each domain contains the number of goal recognition problems for all percentage of domain incompleteness.

Tables~\ref{tab:AblationStudy_Observability}~and~\ref{tab:AblationStudy_Observability_70_100} show a set of inner tables reporting detailed results for all evaluated approaches and domains, varying not only the percentage of domain incompleteness (20\%, 40\%, 60\%, and 80\%) but also the percentage of observability (10\%, 30\%, 50\%, 70\%, and 100\%) of the observation sequence, showing the averages for all types of landmarks, \textit{F1-score}, and \textit{Correlation}.

\afterpage{
\begin{landscape}
\begin{table*}[ht!]
\centering
\setlength\tabcolsep{2.5pt}
\fontsize{5}{6}\selectfont
\fontfamily{cmr}\selectfont

\begin{tabular}{cccccccccccccccccccccccccccc}
\hline
\multicolumn{4}{c}{{\textit{Incompleteness of} $\mathcal{\widetilde{D}}$ (\%)}} & \multicolumn{6}{c}{20\%}                                       & \multicolumn{6}{c}{40\%}                                       & \multicolumn{6}{c}{60\%}                                       & \multicolumn{6}{c}{80\%}                                  \\ \hline
\multicolumn{4}{c}{}               

& \multicolumn{3}{c}{Baseline ($\mathit{h_{gc}}$)}       
& \multicolumn{3}{c}{$h_{\widetilde{GC}}$ (\textit{D+P+O})} 

& \multicolumn{3}{c}{Baseline ($\mathit{h_{gc}}$)}       
& \multicolumn{3}{c}{$h_{\widetilde{GC}}$ (\textit{D+P+O})} 

& \multicolumn{3}{c}{Baseline ($\mathit{h_{gc}}$)}
& \multicolumn{3}{c}{$h_{\widetilde{GC}}$ (\textit{D+P+O})} 

& \multicolumn{3}{c}{Baseline ($\mathit{h_{gc}}$)}  
& \multicolumn{3}{c}{$h_{\widetilde{GC}}$ (\textit{D+P+O})}  \\ \hline
\#        & $\mathcal{G}$    & \textit{Obs} (\%)    & $Obs$    & \textit{Time} & \textit{Acc} \% & \textit{S} & \textit{Time}  & \textit{Acc} \%  & \textit{S} & \textit{Time} & \textit{Acc} \% & \textit{S} & \textit{Time}  & \textit{Acc} \%  & \textit{S} & \textit{Time} & \textit{Acc} \% & \textit{S} & \textit{Time}    & \textit{Acc} \%   & \textit{S}   & \textit{Time} & \textit{Acc} \% & \textit{S} & \textit{Time}    & \textit{Acc} \%   & \textit{S}   \\
\hline
\end{tabular}
\caption{Experimental results of our enhanced goal completion heuristic (using \textit{definite D}, \textit{possible P}, and \textit{overlooked O} landmarks) against the baseline (goal completion) for recognizing goals in incomplete STRIPS domain models.}
\label{tab:ExperimentalResults1}
\end{table*}
\end{landscape}
}

%---------------------------------------------------------------------------------------------------------------

\afterpage{
\begin{landscape}
\begin{table*}[ht!]
\centering
\setlength\tabcolsep{2.5pt}
\fontsize{5}{6}\selectfont
\fontfamily{cmr}\selectfont

\begin{tabular}{cccccccccccccccccccccccccccc}
\hline
\multicolumn{4}{c}{{\textit{Incompleteness of} $\mathcal{\widetilde{D}}$ (\%)}} & \multicolumn{6}{c}{20\%}                                       & \multicolumn{6}{c}{40\%}                                       & \multicolumn{6}{c}{60\%}                                       & \multicolumn{6}{c}{80\%}                                  \\ \hline
\multicolumn{4}{c}{}               

& \multicolumn{3}{c}{Baseline ($\mathit{h_{uniq}}$)}       
& \multicolumn{3}{c}{$h_{\widetilde{UNIQ}}$ (\textit{D+P+O})} 

& \multicolumn{3}{c}{Baseline ($\mathit{h_{uniq}}$)}       
& \multicolumn{3}{c}{$h_{\widetilde{UNIQ}}$ (\textit{D+P+O})} 

& \multicolumn{3}{c}{Baseline ($\mathit{h_{uniq}}$)}       
& \multicolumn{3}{c}{$h_{\widetilde{UNIQ}}$ (\textit{D+P+O})} 

& \multicolumn{3}{c}{Baseline ($\mathit{h_{uniq}}$)}       
& \multicolumn{3}{c}{$h_{\widetilde{UNIQ}}$ (\textit{D+P+O})} 
\\ \hline

\#        & $\mathcal{G}$    & \textit{Obs} (\%)    & $Obs$    & \textit{Time} & \textit{Acc} \% & \textit{S} & \textit{Time}  & \textit{Acc} \%  & \textit{S} & \textit{Time} & \textit{Acc} \% & \textit{S} & \textit{Time}  & \textit{Acc} \%  & \textit{S} & \textit{Time} & \textit{Acc} \% & \textit{S} & \textit{Time}    & \textit{Acc} \%   & \textit{S}   & \textit{Time} & \textit{Acc} \% & \textit{S} & \textit{Time}    & \textit{Acc} \%   & \textit{S}   \\
\hline
\end{tabular}
\caption{Experimental results of our enhanced uniqueness heuristic (using \textit{definite D}, \textit{possible P}, and \textit{overlooked O} landmarks) against the baseline (uniqueness) for recognizing goals in incomplete STRIPS domain models.}
\label{tab:ExperimentalResults2}
\end{table*}
\end{landscape}
}

\afterpage{
\begin{landscape}
\begin{table*}[tb]
\centering
\setlength\tabcolsep{0.7pt}
\fontsize{7}{7.5}\selectfont
\fontfamily{cmr}\selectfont

\begin{tabular}{ccccccclcccccclcccccclcccccc}
\toprule
		 \multicolumn{28}{c}{10\% of Observability} \\
\toprule
		 \cline{2-7} \cline{9-14} \cline{16-21} \cline{23-28}
         
         & \multicolumn{6}{c}{ Domain Incompleteness 20\%} &  
		 & \multicolumn{6}{c}{ Domain Incompleteness 40\%} &  
		 & \multicolumn{6}{c}{ Domain Incompleteness 60\%} &  
		 & \multicolumn{6}{c}{ Domain Incompleteness 80\%} 
		 \\ \cline{2-7} \cline{9-14} \cline{16-21} \cline{23-28}
         & $|D|$    & $|P|$    & $|O|$	&  $|S|$  & $|F_{1}|$    & $CD/CP/CO$    &  & $|D|$    & $|P|$    & $|O|$    &  $|S|$  &  $|F_{1}|$    & $CD/CP/CO$  &  & $|D|$    & $|P|$    & $|O|$    &  $|S|$  & $|F_{1}|$    & $CD/CP/CO$  &  & $|D|$    & $|P|$    & $|O|$    &  $|S|$  &  $|F_{1}|$    & $CD/CP/CO$  \\ \cline{1-7} \cline{9-14} \cline{16-21} \cline{23-28} 
Baseline ($\mathit{h_{gc}}$) 	& 11.5 & 0.0 & 0.0 & 1.77 & 0.33 & -0.04/0/0 &
							& 4.9 & 0.0 & 0.0 & 2.41 & 0.28 & -0.02/0/0 &
							& 3.5 & 0.0 & 0.0 & 2.45 & 0.25 & 0.34/0/0 &
							& 2.9 & 0.0 & 0.0 & 2.79 & 0.26 & 0.36/0/0 \\

Baseline ($\mathit{h_{uniq}}$) 	& 11.5 & 0.0 & 0.0 & 1.36 & 0.33 & -0.11/0/0 &
							& 4.9 & 0.0 & 0.0 & 1.67 & 0.29 & -0.02/0/0 &
							& 3.5 & 0.0 & 0.0 & 1.42 & 0.28 & 0.08/0/0 &
							& 2.9 & 0.0 & 0.0 & 1.90 & 0.28 & 0.09/0/0 \\

$\widetilde{GC}$ (\textit{D+P+O})	& 11.6 & 1.4 & 1.3 & 1.94 & 0.46 & -0.11/-0.08/-0.61 &
									& 9.6 & 2.1 & 1.2 & 2.24 & 0.44 & 0.16/-0.13/-0.47 &
									& 7.1 & 2.2 & 1.4 & 2.86 & 0.40 & 0.31/0.16/-0.26 &
									& 6.8 & 1.1 & 1.3 & 3.27 & 0.39 & 0.21/-0.40/-0.24 \\

$\widetilde{GC}$ (\textit{D+O}) & 11.6 & 0.0 & 1.7 & 2.03 & 0.45 & -0.13/0/-0.55 &
								& 9.6 & 0.0 & 1.7 & 2.51 & 0.44 & 0.15/0/-0.42 &
								& 7.1 & 0.0 & 1.9 & 3.23 & 0.40 & 0.28/0/-0.16 &
								& 6.8 & 0.0 & 1.6 & 3.58 & 0.38 & 0.20/0/-0.23 \\
\hline
\bottomrule
\end{tabular}

\begin{tabular}{ccccccclcccccclcccccclcccccc}
\toprule
		 \multicolumn{28}{c}{30\% of Observability} \\
\toprule		 

		 \cline{2-7} \cline{9-14} \cline{16-21} \cline{23-28}
         & \multicolumn{6}{c}{ Domain Incompleteness 20\%} &  
		 & \multicolumn{6}{c}{ Domain Incompleteness 40\%} &  
		 & \multicolumn{6}{c}{ Domain Incompleteness 60\%} &  
		 & \multicolumn{6}{c}{ Domain Incompleteness 80\%} 
		 \\ \cline{2-7} \cline{9-14} \cline{16-21} \cline{23-28}
         & $|D|$    & $|P|$    & $|O|$	&  $|S|$  & $|F_{1}|$    & $CD/CP/CO$    &  & $|D|$    & $|P|$    & $|O|$    &  $|S|$  &  $|F_{1}|$    & $CD/CP/CO$  &  & $|D|$    & $|P|$    & $|O|$    &  $|S|$  & $|F_{1}|$    & $CD/CP/CO$  &  & $|D|$    & $|P|$    & $|O|$    &  $|S|$  &  $|F_{1}|$    & $CD/CP/CO$  \\ \cline{1-7} \cline{9-14} \cline{16-21} \cline{23-28} 
Baseline ($\mathit{h_{gc}}$) & 11.5 & 0.0 & 0.0 & 1.46 & 0.33 & 0.14/0/0 &
						 & 4.8 & 0.0 & 0.0 & 1.64 & 0.28 & 0.29/0/0 &
						 & 3.5 & 0.0 & 0.0 & 1.90 & 0.25 & 0.55/0/0 &
						 & 2.9 & 0.0 & 0.0 & 2.16 & 0.26 & 0.59/0/0 \\

Baseline ($\mathit{h_{uniq}}$) & 11.5 & 0.0 & 0.0 & 1.33 & 0.33 & 0.15/0/0 &
						   & 4.8 & 0.0 & 0.0 & 1.59 & 0.29 & 0.20/0/0 &
						   & 3.5 & 0.0 & 0.0 & 1.72 & 0.28 & 0.21/0/0 &
						   & 2.9 & 0.0 & 0.0 & 2.07 & 0.28 & 0.30/0/0 \\

$\widetilde{GC}$ (\textit{D+P+O}) & 11.5 & 1.4 & 4.4 & 1.39 & 0.66 & 0.09/0.05/-0.64 &
								  & 9.6 & 2.1 & 4.0 & 1.55 & 0.63 & 0.48/0.16/-0.45 &
								  & 7.1 & 2.2 & 4.4 & 1.88 & 0.57 & 0.70/0.54/-0.04 &
								  & 6.8 & 1.1 & 4.2 & 2.15 & 0.55 & 0.53/-0.08/0.10 \\

$\widetilde{GC}$ (\textit{D+O}) & 11.6 & 0.0 & 5.5 & 1.41 & 0.65 & 0.09/0/-0.56 &
								& 9.6 & 0.0 & 5.4 & 1.64 & 0.63 & 0.50/0/-0.35 &
								& 7.1 & 0.0 & 6.0 & 2.01 & 0.57 & 0.69/0/0.16 &
								& 6.8 & 0.0 & 5.1 & 2.29 & 0.55 & 0.54/0/0.15 \\
\hline
\bottomrule
\end{tabular}

\begin{tabular}{ccccccclcccccclcccccclcccccc}
\toprule
		 \multicolumn{28}{c}{50\% of Observability} \\
\toprule		 

		 \cline{2-7} \cline{9-14} \cline{16-21} \cline{23-28}
         & \multicolumn{6}{c}{ Domain Incompleteness 20\%} &  
		 & \multicolumn{6}{c}{ Domain Incompleteness 40\%} &  
		 & \multicolumn{6}{c}{ Domain Incompleteness 60\%} &  
		 & \multicolumn{6}{c}{ Domain Incompleteness 80\%} 
		 \\ \cline{2-7} \cline{9-14} \cline{16-21} \cline{23-28}
         & $|D|$    & $|P|$    & $|O|$	&  $|S|$  & $|F_{1}|$    & $CD/CP/CO$    &  & $|D|$    & $|P|$    & $|O|$    &  $|S|$  &  $|F_{1}|$    & $CD/CP/CO$  &  & $|D|$    & $|P|$    & $|O|$    &  $|S|$  & $|F_{1}|$    & $CD/CP/CO$  &  & $|D|$    & $|P|$    & $|O|$    &  $|S|$  &  $|F_{1}|$    & $CD/CP/CO$  \\ \cline{1-7} \cline{9-14} \cline{16-21} \cline{23-28} 
Baseline ($\mathit{h_{gc}}$) & 11.5 & 0.0 & 0.0 & 1.36 & 0.44 & 0.38/0/0 &
						 & 4.8 & 0.0 & 0.0 & 1.50 & 0.36 & 0.43/0/0 &
						 & 3.5 & 0.0 & 0.0 & 1.74 & 0.31 & 0.63/0/0 &
						 & 2.9 & 0.0 & 0.0 & 1.95 & 0.34 & 0.71/0/0 \\

Baseline ($\mathit{h_{uniq}}$) & 11.5 & 0.0 & 0.0 & 1.27 & 0.45 & 0.32/0/0 &
						   & 4.8 & 0.0 & 0.0 & 1.48 & 0.34 & 0.41/0/0 &
						   & 3.5 & 0.0 & 0.0 & 1.56 & 0.32 & 0.31/0/0 &
						   & 2.9 & 0.0 & 0.0 & 1.74 & 0.32 & 0.46/0/0 \\

$\widetilde{GC}$ (\textit{D+P+O}) & 11.6 & 1.4 & 8.8 & 1.24 & 0.76 & 0.29/0.17/-0.45 &
								  & 9.6 & 2.1 & 7.9 & 1.31 & 0.76 & 0.63/0.37/-0.20 &
								  & 7.1 & 2.2 & 8.5 & 1.53 & 0.70 & 0.85/0.55/0.24 &
								  & 6.8 & 1.1 & 8.2 & 1.76 & 0.67 & 0.65/0.05/0.35 \\

$\widetilde{GC}$ (\textit{D+O}) & 11.6 & 0.0 & 10.9 & 1.26 & 0.76 & 0.34/0.0/-0.35 &
								& 9.6 & 0.0 & 10.6 & 1.33 & 0.76 & 0.66/0.0/-0.04 &
								& 7.1 & 0.0 & 11.4 & 1.59 & 0.69 & 0.84/0.0/0.46 &
								& 6.8 & 0.0 & 9.8 & 1.84 & 0.67 & 0.65/0.0/0.39 \\
\hline
\bottomrule
\end{tabular}

\caption{Performance results for ablation study, comparing the baseline approaches against our heuristics using some possible combinations of landmark types, separated by observability (10\%, 30\%, and 50\%).}
\label{tab:AblationStudy_Observability}
\end{table*}
\end{landscape}
}

\afterpage{
\begin{landscape}
\begin{table*}[tb]
\centering
\setlength\tabcolsep{0.7pt}
\fontsize{7}{7.5}\selectfont
\fontfamily{cmr}\selectfont

\begin{tabular}{ccccccclcccccclcccccclcccccc}
\toprule
		 \multicolumn{28}{c}{70\% of Observability} \\
\toprule		 

		 \cline{2-7} \cline{9-14} \cline{16-21} \cline{23-28}
         & \multicolumn{6}{c}{Domain Incompleteness 20\%} &  
		 & \multicolumn{6}{c}{Domain Incompleteness 40\%} &  
		 & \multicolumn{6}{c}{Domain Incompleteness 60\%} &  
		 & \multicolumn{6}{c}{Domain Incompleteness 80\%} 
		 \\ \cline{2-7} \cline{9-14} \cline{16-21} \cline{23-28}
         & $|D|$    & $|P|$    & $|O|$	&  $|S|$  & $|F_{1}|$    & $CD/CP/CO$    &  & $|D|$    & $|P|$    & $|O|$    &  $|S|$  &  $|F_{1}|$    & $CD/CP/CO$  &  & $|D|$    & $|P|$    & $|O|$    &  $|S|$  & $|F_{1}|$    & $CD/CP/CO$  &  & $|D|$    & $|P|$    & $|O|$    &  $|S|$  &  $|F_{1}|$    & $CD/CP/CO$  \\ \cline{1-7} \cline{9-14} \cline{16-21} \cline{23-28} 
Baseline ($\mathit{h_{gc}}$) & 11.5 & 0.0 & 0.0 & 1.35 & 0.50 & 0.49/0/0 &
						 & 4.8 & 0.0 & 0.0 & 1.46 & 0.41 & 0.54/0/0 &
						 & 3.5 & 0.0 & 0.0 & 1.70 & 0.33 & 0.67/0/0 &
						 & 2.9 & 0.0 & 0.0 & 1.90 & 0.37 & 0.71/0/0 \\

Baseline ($\mathit{h_{uniq}}$) & 11.6 & 0.0 & 0.0 & 1.23 & 0.49 & 0.45/0/0 &
						   & 4.8 & 0.0 & 0.0 & 1.37 & 0.36 & 0.49/0/0 &
						   & 3.5 & 0.0 & 0.0 & 1.45 & 0.33 & 0.42/0/0 &
						   & 2.9 & 0.0 & 0.0 & 1.67 & 0.34 & 0.50/0/0 \\

$\widetilde{GC}$ (\textit{D+P+O}) & 11.5 & 1.4 & 14.4 & 1.09 & 0.88 & 0.38/0.27/-0.14 &
								  & 9.6 & 2.1 & 12.8 & 1.14 & 0.87 & 0.62/0.31/0.12 &
								  & 7.1 & 2.2 & 13.6 & 1.30 & 0.83 & 0.72/0.32/0.27 &
								  & 6.8 & 1.1 & 13.2 & 1.43 & 0.80 & 0.65/0.11/0.48 \\

$\widetilde{GC}$ (\textit{D+O}) & 11.6 & 0.0 & 17.5 & 1.11 & 0.87 & 0.38/0/-0.04 &
								& 9.6 & 0.0 & 17.0 & 1.15 & 0.87 & 0.62/0/0.20 &
								& 7.1 & 0.0 & 18.0 & 1.33 & 0.82 & 0.78/0/0.51 &
								& 6.8 & 0.0 & 15.6 & 1.48 & 0.80 & 0.66/0/0.48 \\
\hline
\bottomrule
\end{tabular}

\begin{tabular}{ccccccclcccccclcccccclcccccc}
\toprule
		 \multicolumn{28}{c}{100\% of Observability} \\
\toprule		 

		 \cline{2-7} \cline{9-14} \cline{16-21} \cline{23-28}
         & \multicolumn{6}{c}{ Domain Incompleteness 20\%} &  
		 & \multicolumn{6}{c}{ Domain Incompleteness 40\%} &  
		 & \multicolumn{6}{c}{ Domain Incompleteness 60\%} &  
		 & \multicolumn{6}{c}{ Domain Incompleteness 80\%} 
		 \\ \cline{2-7} \cline{9-14} \cline{16-21} \cline{23-28}
         & $|D|$    & $|P|$    & $|O|$	&  $|S|$  & $|F_{1}|$    & $CD/CP/CO$    &  & $|D|$    & $|P|$    & $|O|$    &  $|S|$  &  $|F_{1}|$    & $CD/CP/CO$  &  & $|D|$    & $|P|$    & $|O|$    &  $|S|$  & $|F_{1}|$    & $CD/CP/CO$  &  & $|D|$    & $|P|$    & $|O|$    &  $|S|$  &  $|F_{1}|$    & $CD/CP/CO$  \\ \cline{1-7} \cline{9-14} \cline{16-21} \cline{23-28} 
Baseline ($\mathit{h_{gc}}$) & 11.5 & 0.0 & 0.0 & 1.42 & 0.53 & 0.58/0/0 &
						 & 4.8 & 0.0 & 0.0 & 1.57 & 0.42 & 0.62/0/0 &
						 & 3.5 & 0.0 & 0.0 & 1.77 & 0.36 & 0.66/0/0 &
						 & 2.9 & 0.0 & 0.0 & 1.89 & 0.38 & 0.72/0/0 \\

Baseline ($\mathit{h_{uniq}}$) & 11.5 & 0.0 & 0.0 & 1.39 & 0.53 & 0.59/0/0 &
						   & 4.8 & 0.0 & 0.0 & 1.54 & 0.38 & 0.54/0/0 &
						   & 3.5 & 0.0 & 0.0 & 1.63 & 0.35 & 0.48/0/0 &
						   & 2.9 & 0.0 & 0.0 & 1.83 & 0.35 & 0.52/0/0 \\

$\widetilde{GC}$ (\textit{D+P+O}) & 11.6 & 1.4 & 21.2 & 1.04 & 0.96 & 0.35/0.23/0.32 &
								  & 9.6 & 2.1 & 19.0 & 1.07 & 0.94 & 0.46/0.31/0.34 &
								  & 7.1 & 2.2 & 19.6 & 1.21 & 0.91 & 0.58/0.23/0.41 &
								  & 6.8 & 1.1 & 19.1 & 1.27 & 0.88 & 0.57/0.25/0.60 \\

$\widetilde{GC}$ (\textit{D+O}) & 11.6 & 0.0 & 25.5 & 1.05 & 0.95 & 0.35/0/0.30 &
								& 9.6 & 0.0 & 24.7 & 1.08 & 0.93 & 0.46/0/0.36 &
								& 7.1 & 0.0 & 25.7 & 1.23 & 0.90 & 0.61/0/0.53 &
								& 6.8 & 0.0 & 22.5 & 1.29 & 0.88 & 0.57/0/0.60 \\
\hline
\bottomrule
\end{tabular}

\caption{Performance results for ablation study, comparing the baseline approaches against our heuristics using some possible combinations of landmark types, separated by observability (70\% and 100\%).}
\label{tab:AblationStudy_Observability_70_100}
\end{table*}
\end{landscape}
}

%% file: apC_NominalModels_Results.tex
%!TEX root = ppgcc-thesis.tex
%----------------------------------------------------------------------------------
% Chapter (Appendix): Goal Recognition over Nominal Models - Detailed Results
%
% ?
%
%----------------------------------------------------------------------------------
\chapter{Goal Recognition over Nominal Models - Detailed Results}\label{appendixC:goalrecognition_nominalmodels}

In this Appendix, we present detailed experimental results of our approaches for goal recognition over \textit{nominal models}. Specifically, we report on a set of experiments that aims to evaluate the performance of our recognition approaches~\nominalmirroring~and~$\Delta(Obs,G)$, presented in Chapter~\ref{chapter:GR_NominalModels}, over both \textit{actual} and \textit{nominal models}, when dealing with \textit{linear} and \textit{non--linear} domain models. 

%----------------------------------------------------------------------------------
\section*{Linear Domain Models}\label{app:section:linear}

Tables~\ref{tab:recognition_results_lqr_1d}, \ref{tab:recognition_results_lqr_2d}, and~\ref{tab:recognition_results_lqr_2d_mu} show the experimental results for three (linear) LQR--based domain models (formally defined in Section~\ref{subsection:Domains_NM}): 1D--LQR--Navigation, 2D--LQR--Navigation with a single vehicle, and 2D--LQR--Navigation with multiple vehicles.
We measure the recognition performance of our approaches over these domains in three settings: \textit{online} (the recognition process is performed incrementally after observing a state), \textit{offline} (considering observed states at once), \textit{1st observation}. For evaluation, we use the same metrics we used in Section~\ref{subsection:GoalRecognition_Results}, namely, True Positive Rate (\textit{TPR}), False Positive Rate (\textit{FPR}), \textit{Top-2}. In these tables, the column $M$ represents the model type (A represents \textit{actual models}, and N represents \textit{nominal models}), \% $Obs$ is observation level, and $N$ is the total number of observed states. Note that the average number of goal hypothesis $|\mathcal{G}|$ in the datasets for these three domains is 5. We use two levels of observability in these sets of experiments: 5\% and 10\%, more specifically, for each of these levels of observability we observe 5 and 10 states per problem. 

Though the results are very similar when comparing our approaches for all settings in all these three linear domains, we note that, from the results in these tables, the best results we had are for \textit{offline} goal recognition when using the~\nominalmirroring~approach. 
The 1D--LQR--Navigation domain is the linear domain in which our recognition approaches have achieved better results, and we note that it might due to the complexity of this domain, in which we have only one dimension and one vehicle.
After analyzing the results of $\Delta(Obs,G)$, we note that this approach has achieved relatively poor results for all linear domains due to the fact the planner we used most likely get trapped in local minimum for most problems in these datasets. As we mentioned in Section~\ref{subsection:GoalRecognition_Results}, we dug deeper into the extracted rewards for $J^{+}$ and $J^{-}$, and to overcome this issue, we aim to use different solvers to see how they will behave when dealing with the modified cost functions of $\Delta(Obs,G)$.

\begin{table}[h!]
\setlength\tabcolsep{1.7pt}
\fontfamily{cmr}\selectfont
\fontsize{12}{13}\selectfont
\centering
\begin{tabular}{llllllllllllllllll}
\toprule
		 \multicolumn{17}{c}{1D--LQR--Navigation (Linear Domain)} \\
\toprule
\hline
& &	& \phantom{a} & &
\multicolumn{3}{c}{\sc Online} & \phantom{a} & &
\multicolumn{3}{c}{\sc Offline} & &
\multicolumn{3}{c}{\sc 1st Observation} \\
\cmidrule{6-8} \cmidrule{11-13} \cmidrule{15-17}
\textit{Approach} & \textit{M} &  \textit{Obs} (\%) &&  \textit{N} &  \textit{Top-2} &  \textit{TPR} &  \textit{FPR} &&  \textit{N} &  \textit{Top-2} &  \textit{TPR} &  \textit{FPR} && \textit{Top-2} &  \textit{TPR} &  \textit{FPR} \\
\midrule
\nominalmirroring 		&     A &          5 &&  150 &    0.87 &   0.73 &   0.06 &&   30 &    1.00 &   0.97 &   0.01 &&  0.77 & 0.40 & 0.13 \\
$\Delta(Obs,G)$ 		&     A &          5 &&  150 &    0.56 &   0.29 &   0.16 &&   30 &    0.63 &   0.47 &   0.12 &&  0.40 & 0.20 & 0.18 \\

\nominalmirroring 		&     A &         10 &&  300 &    0.90 &   0.76 &   0.05 &&   30 &    1.00 &   0.97 &   0.01 &&  0.63 & 0.33 & 0.15 \\
$\Delta(Obs,G)$ 		&     A &         10 &&  300 &    0.45 &   0.26 &   0.17 &&   30 &    0.47 &   0.27 &   0.16 &&  0.37 & 0.27 & 0.16 \\

\midrule

\nominalmirroring 		&    N &          5 &&  150 &    0.72 &   0.54 &   0.11 &&   30 &    0.93 &   0.80 &   0.05 &&  0.53 & 0.27 & 0.17 \\
$\Delta(Obs,G)$ 		&    N &          5 &&  150 &    0.57 &   0.29 &   0.16 &&   30 &    0.57 &   0.33 &   0.15 &&  0.77 & 0.33 & 0.15 \\

\nominalmirroring 		&    N &         10 &&  300 &    0.77 &   0.51 &   0.11 &&   30 &    0.93 &   0.83 &   0.05 &&  0.37 & 0.20 & 0.17 \\
$\Delta(Obs,G)$ 		&    N &         10 &&  300 &    0.53 &   0.32 &   0.15 &&   30 &    0.43 &   0.27 &   0.16 &&  0.63 & 0.40 & 0.13 \\
\hline
\bottomrule
\end{tabular}
\caption{Experimental results of our recognition approaches \nominalmirroring~and~$\Delta(Obs,G)$, over both \textit{actual} and \textit{nominal models} for the 1D--LQR--Navigation linear domain.}
\label{tab:recognition_results_lqr_1d}
\end{table}

\begin{table}[h!]
\setlength\tabcolsep{1.7pt}
\fontfamily{cmr}\selectfont
\fontsize{12}{13}\selectfont
\centering
\begin{tabular}{llllllllllllllllll}
\toprule
		 \multicolumn{17}{c}{2D--LQR--Navigation Single Vehicle (Linear Domain)} \\
\toprule
\hline
& &	& \phantom{a} & &
\multicolumn{3}{c}{\sc Online} & \phantom{a} & &
\multicolumn{3}{c}{\sc Offline} & &
\multicolumn{3}{c}{\sc 1st Observation} \\
\cmidrule{6-8} \cmidrule{11-13} \cmidrule{15-17}
\textit{Approach} & \textit{M} &  \textit{Obs} (\%) &&  \textit{N} &  \textit{Top-2} &  \textit{TPR} &  \textit{FPR} &&  \textit{N} &  \textit{Top-2} &  \textit{TPR} &  \textit{FPR} && \textit{Top-2} &  \textit{TPR} &  \textit{FPR} \\
\midrule
\nominalmirroring 		&     A &          5 &&  150 &    0.89 &   0.86 &   0.03 &&   30 &    1.00 &   1.00 &   0.00 &&  0.60 & 0.50 & 0.10 \\
$\Delta(Obs,G)$ 		&     A &          5 &&  150 &    0.47 &   0.25 &   0.15 &&   30 &    0.37 &   0.23 &   0.16 &&  0.53 & 0.30 & 0.14 \\

\nominalmirroring 		&     A &         10 &&  300 &    0.92 &   0.86 &   0.03 &&   30 &    1.00 &   1.00 &   0.00 &&  0.63 & 0.40 & 0.12 \\
$\Delta(Obs,G)$ 		&     A &         10 &&  300 &    0.42 &   0.27 &   0.15 &&   30 &    0.30 &   0.20 &   0.16 &&  0.57 & 0.33 & 0.13 \\

\midrule

\nominalmirroring 		&    N &          5 &&  150 &    0.64 &   0.41 &   0.12 &&   30 &    0.73 &   0.63 &   0.07 &&  0.50 & 0.27 & 0.15 \\
$\Delta(Obs,G)$ 		&    N &          5 &&  150 &    0.43 &   0.19 &   0.17 &&   30 &    0.37 &   0.22 &   0.16 &&  0.50 & 0.13 & 0.18 \\

\nominalmirroring 		&    N &         10 &&  300 &    0.62 &   0.43 &   0.11 &&   30 &    0.73 &   0.57 &   0.09 &&  0.43 & 0.23 & 0.15 \\
$\Delta(Obs,G)$ 		&    N &         10 &&  300 &    0.45 &   0.24 &   0.16 &&   30 &    0.43 &   0.17 &   0.17 &&  0.47 & 0.25 & 0.16 \\
\hline
\bottomrule
\end{tabular}
\caption{Experimental results of our recognition approaches \nominalmirroring~and~$\Delta(Obs,G)$, over both \textit{actual} and \textit{nominal models} for the 2D--LQR--Navigation single vehicle linear domain.}
\label{tab:recognition_results_lqr_2d}
\end{table}

\begin{table}[h!]
\setlength\tabcolsep{1.7pt}
\fontfamily{cmr}\selectfont
\fontsize{12}{13}\selectfont
\centering
\begin{tabular}{llllllllllllllllll}
\toprule
		 \multicolumn{17}{c}{2D--LQR--Navigation with Multiple Vehicles (Linear Domain)} \\
\toprule
\hline
& &	& \phantom{a} & &
\multicolumn{3}{c}{\sc Online} & \phantom{a} & &
\multicolumn{3}{c}{\sc Offline} & &
\multicolumn{3}{c}{\sc 1st Observation} \\
\cmidrule{6-8} \cmidrule{11-13} \cmidrule{15-17}
\textit{Approach} & \textit{M} &  \textit{Obs} (\%) &&  \textit{N} &  \textit{Top-2} &  \textit{TPR} &  \textit{FPR} &&  \textit{N} &  \textit{Top-2} &  \textit{TPR} &  \textit{FPR} && \textit{Top-2} &  \textit{TPR} &  \textit{FPR} \\
\midrule
\nominalmirroring 		&     A &          5 &&  150 &    0.84 &   0.71 &   0.06 &&   30 &    0.90 &   0.83 &   0.03 &&  0.63 & 0.43 & 0.12 \\
$\Delta(Obs,G)$ 		&     A &          5 &&  150 &    0.45 &   0.18 &   0.18 &&   30 &    0.47 &   0.20 &   0.18 &&  0.53 & 0.27 & 0.16 \\

\nominalmirroring 		&     A &         10 &&  300 &    0.88 &   0.73 &   0.06 &&   30 &    0.93 &   0.90 &   0.02 &&  0.67 & 0.23 & 0.17 \\
$\Delta(Obs,G)$ 		&     A &         10 &&  300 &    0.49 &   0.26 &   0.16 &&   30 &    0.57 &   0.37 &   0.14 &&  0.47 & 0.23 & 0.17 \\

\midrule

\nominalmirroring 		&    N &          5 &&  150 &    0.61 &   0.43 &   0.12 &&   30 &    0.83 &   0.57 &   0.09 &&  0.30 & 0.20 & 0.18 \\
$\Delta(Obs,G)$ 		&    N &          5 &&  150 &    0.34 &   0.16 &   0.19 &&   30 &    0.37 &   0.23 &   0.17 &&  0.27 & 0.08 & 0.21 \\

\nominalmirroring 		&    N &         10 &&  300 &    0.74 &   0.54 &   0.10 &&   30 &    0.93 &   0.77 &   0.05 &&  0.43 & 0.33 & 0.15 \\
$\Delta(Obs,G)$ 		&    N &         10 &&  300 &    0.37 &   0.21 &   0.18 &&   30 &    0.40 &   0.18 &   0.18 &&  0.47 & 0.30 & 0.17 \\
\hline
\bottomrule
\end{tabular}
\caption{Experimental results of our recognition approaches \nominalmirroring~and~$\Delta(Obs,G)$, over both \textit{actual} and \textit{nominal models} for the 2D--LQR--Navigation linear domain with multiple vehicles.}
\label{tab:recognition_results_lqr_2d_mu}
\end{table}

%----------------------------------------------------------------------------------
\clearpage
\section*{Non-Linear Domain Models}\label{app:section:nonlinear}

Tables~\ref{tab:recognition_results_2d_nav}~and~\ref{tab:recognition_results_3d_nav} show the experimental results for two (non--linear) navigation domains models (formally defined in Section~\ref{subsection:Domains_NM}): 2D--NAV and 3D--NAV. For these sets of experiments, we used four levels of observability: 5\%, 10\%, 30\%, and 50\%. Since the planning horizon $H$ for these two domains is 20, for 5\% of observability we have one observed state, whereas for 10\% we have two states, and 6 and 10 states for the other two levels of observability, respectively, 30\% and 50\%. To evaluate the performance of our recognition approaches over non--linear domains, we used the same metrics over the same settings, as we mentioned in the previous section. Note that, our recognition approaches performed better over non--linear domains for all evaluated settings, in comparison to the results for linear domains, especially for offline goal recognition. 

\begin{table}[h!]
\setlength\tabcolsep{1.7pt}
\fontfamily{cmr}\selectfont
\fontsize{12}{13}\selectfont
\centering
\begin{tabular}{llllllllllllllllll}
\toprule
		 \multicolumn{17}{c}{2D--NAV (Non-Linear Domain)} \\
\toprule
\hline
& &	& \phantom{a} & &
\multicolumn{3}{c}{\sc Online} & \phantom{a} & &
\multicolumn{3}{c}{\sc Offline} & &
\multicolumn{3}{c}{\sc 1st Observation} \\
\cmidrule{6-8} \cmidrule{11-13} \cmidrule{15-17}
\textit{Approach} & \textit{M} &  \textit{Obs} (\%) &&  \textit{N} &  \textit{Top-2} &  \textit{TPR} &  \textit{FPR} &&  \textit{N} &  \textit{Top-2} &  \textit{TPR} &  \textit{FPR} && \textit{Top-2} &  \textit{TPR} &  \textit{FPR} \\
\midrule

\nominalmirroring 		&    A &          5 &&   10 &    1.00 &   1.00 &   0.00 &   10 &&    1.00 &   1.00 &   0.00 &&  1.00 & 1.00 & 0.00 \\
$\Delta(Obs,G)$ 		&    A &          5 &&   10 &    0.70 &   0.60 &   0.10 &   10 &&    0.70 &   0.60 &   0.10 &&  0.70 & 0.60 & 0.10 \\

\nominalmirroring 		&    A &         10 &&   20 &    1.00 &   0.75 &   0.06 &   10 &&    1.00 &   0.80 &   0.05 &&  1.00 & 0.70 & 0.07 \\
$\Delta(Obs,G)$ 		&    A &         10 &&   20 &    0.50 &   0.30 &   0.17 &   10 &&    0.50 &   0.20 &   0.20 &&  0.50 & 0.40 & 0.15 \\

\nominalmirroring 		&    A &         30 &&   60 &    0.70 &   0.47 &   0.13 &   10 &&    0.80 &   0.50 &   0.12 &&  0.50 & 0.30 & 0.17 \\
$\Delta(Obs,G)$ 		&    A &         30 &&   60 &    0.58 &   0.35 &   0.16 &   10 &&    0.60 &   0.40 &   0.15 &&  0.60 & 0.40 & 0.15 \\

\nominalmirroring 		&    A &         50 &&  100 &    0.96 &   0.86 &   0.04 &   10 &&    1.00 &   1.00 &   0.00 &&  0.80 & 0.40 & 0.15 \\
$\Delta(Obs,G)$ 		&    A &         50 &&  100 &    0.60 &   0.35 &   0.16 &   10 &&    0.70 &   0.50 &   0.12 &&  0.50 & 0.20 & 0.20 \\

\midrule

\nominalmirroring 		&     N &          5 &&   10 &    0.90 &   0.80 &   0.05 &   10 &&    0.90 &   0.80 &   0.05 &&  0.90 & 0.80 & 0.05 \\
$\Delta(Obs,G)$ 		&     N &          5 &&   10 &    0.60 &   0.20 &   0.20 &   10 &&    0.60 &   0.20 &   0.20 &&  0.60 & 0.20 & 0.20 \\

\nominalmirroring 		&     N &         10 &&   20 &    0.80 &   0.60 &   0.10 &   10 &&    0.80 &   0.70 &   0.07 &&  0.80 & 0.50 & 0.12 \\
$\Delta(Obs,G)$ 		&     N &         10 &&   20 &    0.45 &   0.25 &   0.19 &   10 &&    0.50 &   0.30 &   0.17 &&  0.40 & 0.20 & 0.20 \\

\nominalmirroring 		&     N &         30 &&   60 &    0.77 &   0.62 &   0.10 &   10 &&    1.00 &   1.00 &   0.00 &&  0.50 & 0.20 & 0.20 \\
$\Delta(Obs,G)$ 		&     N &         30 &&   60 &    0.48 &   0.10 &   0.23 &   10 &&    0.50 &   0.10 &   0.23 &&  0.50 & 0.10 & 0.23 \\

\nominalmirroring 		&     N &         50 &&  100 &    0.87 &   0.68 &   0.08 &   10 &&    1.00 &   1.00 &   0.00 &&  0.50 & 0.20 & 0.20 \\
$\Delta(Obs,G)$ 		&     N &         50 &&  100 &    0.42 &   0.14 &   0.21 &   10 &&    0.50 &   0.30 &   0.17 &&  0.40 & 0.20 & 0.20 \\

\hline
\bottomrule
\end{tabular}
\caption{Experimental results of our recognition approaches over both \textit{actual} and \textit{nominal models} for the 2D--NAV non-linear domain.}
\label{tab:recognition_results_2d_nav}
\end{table}

\begin{table}[h!]
\setlength\tabcolsep{1.7pt}
\fontfamily{cmr}\selectfont
\fontsize{12}{13}\selectfont
\centering
\begin{tabular}{llllllllllllllllll}
\toprule
		 \multicolumn{17}{c}{3D--NAV (Non-Linear Domain)} \\
\toprule
\hline
& &	& \phantom{a} & &
\multicolumn{3}{c}{\sc Online} & \phantom{a} & &
\multicolumn{3}{c}{\sc Offline} & &
\multicolumn{3}{c}{\sc 1st Observation} \\
\cmidrule{6-8} \cmidrule{11-13} \cmidrule{15-17}
\textit{Approach} & \textit{M} &  \textit{Obs} (\%) &&  \textit{N} &  \textit{Top-2} &  \textit{TPR} &  \textit{FPR} &&  \textit{N} &  \textit{Top-2} &  \textit{TPR} &  \textit{FPR} && \textit{Top-2} &  \textit{TPR} &  \textit{FPR} \\
\midrule

\nominalmirroring 		&    A &          5 &&   10 &    0.70 &   0.70 &   0.07 &&   10 &    0.70 &   0.70 &   0.07 &&  0.70 & 0.70 & 0.07 \\
$\Delta(Obs,G)$ 		&    A &          5 &&   10 &    0.90 &   0.60 &   0.10 &&   10 &    0.90 &   0.60 &   0.10 &&  0.90 & 0.60 & 0.10 \\

\nominalmirroring 		&    A &         10 &&   20 &    0.80 &   0.80 &   0.05 &&   10 &    1.00 &   1.00 &   0.00 &&  0.60 & 0.60 & 0.10 \\
$\Delta(Obs,G)$ 		&    A &         10 &&   20 &    0.65 &   0.35 &   0.16 &&   10 &    0.70 &   0.50 &   0.12 &&  0.60 & 0.20 & 0.20 \\

\nominalmirroring 		&    A &         30 &&   60 &    0.85 &   0.63 &   0.09 &&   10 &    1.00 &   1.00 &   0.00 &&  0.70 & 0.40 & 0.15 \\
$\Delta(Obs,G)$ 		&    A &         30 &&   60 &    0.72 &   0.38 &   0.15 &&   10 &    0.90 &   0.60 &   0.10 &&  0.50 & 0.10 & 0.23 \\

\nominalmirroring 		&    A &         50 &&  100 &    0.79 &   0.60 &   0.10 &&   10 &    1.00 &   1.00 &   0.00 &&  0.70 & 0.30 & 0.17 \\
$\Delta(Obs,G)$ 		&    A &         50 &&  100 &    0.65 &   0.39 &   0.15 &&   10 &    0.90 &   0.80 &   0.05 &&  0.40 & 0.10 & 0.23 \\

\midrule

\nominalmirroring 		&     N &          5 &&   10 &    0.80 &   0.60 &   0.10 &&   10 &    0.80 &   0.60 &   0.10 &&  0.80 & 0.60 & 0.10 \\
$\Delta(Obs,G)$ 		&     N &          5 &&   10 &    0.90 &   0.70 &   0.07 &&   10 &    0.90 &   0.70 &   0.07 &&  0.90 & 0.70 & 0.07 \\

\nominalmirroring 		&     N &         10 &&   20 &    0.65 &   0.50 &   0.12 &&   10 &    0.80 &   0.60 &   0.10 &&  0.50 & 0.40 & 0.15 \\
$\Delta(Obs,G)$ 		&     N &         10 &&   20 &    0.60 &   0.45 &   0.14 &&   10 &    0.60 &   0.40 &   0.15 &&  0.60 & 0.50 & 0.12 \\

\nominalmirroring 		&     N &         30 &&   60 &    0.75 &   0.60 &   0.10 &&   10 &    1.00 &   1.00 &   0.00 &&  0.50 & 0.30 & 0.17 \\
$\Delta(Obs,G)$ 		&     N &         30 &&   60 &    0.52 &   0.30 &   0.17 &&   10 &    0.50 &   0.30 &   0.17 &&  0.50 & 0.30 & 0.17 \\

\nominalmirroring 		&     N &         50 &&  100 &    0.76 &   0.61 &   0.10 &&   10 &    1.00 &   1.00 &   0.00 &&  0.50 & 0.30 & 0.17 \\
$\Delta(Obs,G)$ 		&     N &         50 &&  100 &    0.54 &   0.29 &   0.18 &&   10 &    1.00 &   0.30 &   0.17 &&  0.50 & 0.30 & 0.17 \\

\hline
\bottomrule
\end{tabular}
\caption{Experimental results of our recognition approaches \nominalmirroring~and~$\Delta(Obs,G)$, over both \textit{actual} and \textit{nominal models} for the 3D--NAV non-linear domain.}
\label{tab:recognition_results_3d_nav}
\end{table}

%% file: ppgcc-thesis.bbl
\begin{thebibliography}{100}
\pretolerance=2500
\hyphenpenalty=10000
\tolerance=200
\emergencystretch=15em
\newcommand{\enquote}[1]{``#1''}
\providecommand{\urlprefix}{\bblcaptured}
\providecommand{\selectlanguage}[1]{\relax}
\newcommand{\Capitalize}[1]{\uppercase{#1}}
\newcommand{\capitalize}[1]{\expandafter\Capitalize#1}
\providecommand{\bibAnnoteFile}[1]{%
  \IfFileExists{#1}{\begin{quotation}\noindent\textsc{Key:} #1\\
  \textsc{Annotation:}\ \input{#1}\end{quotation}}{}}
\providecommand{\bibAnnote}[2]{%
  \begin{quotation}\noindent\textsc{Key:} #1\\
  \textsc{Annotation:}\ #2\end{quotation}}
\providecommand{\eprint}[2][]{#2}

\bibitem{Tensorflow_2015}
Abadi, M.; Agarwal, A.; Barham, P.; Brevdo, E.; Chen, Z.; Citro, C.; Corrado,
  G.~S.; Davis, A.; Dean, J.; Devin, M.; Ghemawat, S.; Goodfellow, I.~J.; Harp,
  A.; Irving, G.; Isard, M.; Jia, Y.; J{\'{o}}zefowicz, R.; Kaiser, L.; Kudlur,
  M.; Levenberg, J.; Man{\'{e}}, D.; Monga, R.; Moore, S.; Murray, D.~G.; Olah,
  C.; Schuster, M.; Shlens, J.; Steiner, B.; Sutskever, I.; Talwar, K.; Tucker,
  P.~A.; Vanhoucke, V.; Vasudevan, V.; Vi{\'{e}}gas, F.~B.; Vinyals, O.;
  Warden, P.; Wattenberg, M.; Wicke, M.; Yu, Y.; Zheng, X. \enquote{Tensorflow:
  Large-scale machine learning on heterogeneous distributed systems},
  \textit{Computing Research Repository (CoRR)}, \bblvol{} abs/1603.04467,
  Mar~2016, \bblpp{} 1--19.
\bibAnnoteFile{Tensorflow_2015}

\bibitem{Amado_Demo_ICAPS_2019}
Amado, L.; Pereira, R.~F.; Aires, J.~P.; Magnaguagno, M.; Granada, R.; Licks,
  G.~P.; Meneguzzi, F. \enquote{Latrec: Recognizing goals in latent space}.
  \capitalize{\bblin{}}: Proceedings of the System Demonstrations and Exhibits
  at the International Conference on Automated Planning and Scheduling (ICAPS),
  2019, \bblpp{} 1--2.
\bibAnnoteFile{Amado_Demo_ICAPS_2019}

\bibitem{Amado2018}
Amado, L.; Pereira, R.~F.; Aires, J.~P.; Magnaguagno, M.; Granada, R.;
  Meneguzzi, F. \enquote{Goal recognition in latent space}.
  \capitalize{\bblin{}}: Proceedings of the International Joint Conference on
  Neural Networks (IJCNN), 2018, \bblpp{} 1--8.
\bibAnnoteFile{Amado2018}

\bibitem{Armentano_AIJ_2007}
Armentano, M.~G.; Amandi, A. \enquote{Plan recognition for interface agents},
  \textit{Artificial Intelligence Review}, \bblvol{}~28--2, Aug~2007, \bblpp{}
  131--162.
\bibAnnoteFile{Armentano_AIJ_2007}

\bibitem{AAAI2018_MasataroAsai_PlanningLatSpace}
Asai, M.; Fukunaga, A. \enquote{{Classical Planning in Deep Latent Space:
  Bridging the Subsymbolic-Symbolic Boundary}}. \capitalize{\bblin{}}:
  Proceedings of the Conference of the Association for the Advancement of
  Artificial Intelligence (AAAI), 2018, \bblpp{} 6094--6101.
\bibAnnoteFile{AAAI2018_MasataroAsai_PlanningLatSpace}

\bibitem{AvrahamiZilberbrandK_IJCAI2005}
Avrahami{-}Zilberbrand, D.; Kaminka, G.~A. \enquote{{Fast and Complete Symbolic
  Plan Recognition}}. \capitalize{\bblin{}}: Proceedings of the International
  Joint Conference on Artificial Intelligence (IJCAI), 2005, \bblp{} 653–658.
\bibAnnoteFile{AvrahamiZilberbrandK_IJCAI2005}

\bibitem{DoritGalAAAI07}
Avrahami-Zilberbrand, D.; Kaminka, G.~A. \enquote{Incorporating observer biases
  in keyhole plan recognition (efficiently!)}. \capitalize{\bblin{}}:
  Proceedings of the Conference of the Association for the Advancement of
  Artificial Intelligence (AAAI), 2007, \bblp{} 944–949.
\bibAnnoteFile{DoritGalAAAI07}

\bibitem{baker:09:cognition}
Baker, C.~L.; Joshua B.~Tenenbaum, J.~B.; Saxe, R. \enquote{Action
  understanding as inverse planning}, \textit{Cognition}, \bblvol{} 113--3,
  Jul~2009, \bblpp{} 329--349.
\bibAnnoteFile{baker:09:cognition}

\bibitem{bemporad:2002:lqr}
Bemporad, A.; Morari, M.; Dua, V.; N.~Pistikopoulos, E. \enquote{The explicit
  linear quadratic regulator for constrained systems}, \textit{Automatica},
  \bblvol{}~38, Jan~2002, \bblpp{} 3--20.
\bibAnnoteFile{bemporad:2002:lqr}

\bibitem{Bertsekas_DP_17}
Bertsekas, D.~P. \enquote{Dynamic Programming and Optimal Control}. Athena
  Scientific, 2017, 4th \bbledn{}, 520p.
\bibAnnoteFile{Bertsekas_DP_17}

\bibitem{LQR_BETETO2018422}
Beteto, M.~A.; Assunção, E.; Teixeira, M.~C.; Silva, E.~R.; Buzachero, L.~F.;
  Caun, R.~P. \enquote{{New Design of Robust LQR-State Derivative Controllers
  via LMIs}}, \textit{International Federation of Automatic Control},
  \bblvol{}~51--25, Nov~2018, \bblpp{} 422--427.
\bibAnnoteFile{LQR_BETETO2018422}

\bibitem{BlumFastPlanning_95}
Blum, A.~L.; Furst, M.~L. \enquote{{Fast Planning Through Planning Graph
  Analysis}}, \textit{Journal of Artificial Intelligence Research},
  \bblvol{}~90, Feb~1997, \bblpp{} 281--300.
\bibAnnoteFile{BlumFastPlanning_95}

\bibitem{borrelli:17:predictive}
Borrelli, F.; Bemporad, A.; Morari, M. \enquote{Predictive control for linear
  and hybrid systems}. Cambridge University Press, 2017, 1st \bbledn{}, 440p.
\bibAnnoteFile{borrelli:17:predictive}

\bibitem{ReachabilityBryceK_2007}
Bryce, D.; Kambhampati, S. \enquote{{A Tutorial on Planning Graph Based
  Reachability Heuristics}}, \textit{{AI} Magazine}, \bblvol{}~28--1, Mar~2007,
  \bblpp{} 47--83.
\bibAnnoteFile{ReachabilityBryceK_2007}

\bibitem{bueno:19:aaai}
Bueno, T.~P.; Barros, L.; Maua, D.~D.; Sanner, S. \enquote{Deep reactive
  policies for planning in stochastic nonlinear domains}.
  \capitalize{\bblin{}}: Proceedings of the Conference of the Association for
  the Advancement of Artificial Intelligence (AAAI), 2019, \bblpp{} 7530--7537.
\bibAnnoteFile{bueno:19:aaai}

\bibitem{calafiore:14:optimization}
Calafiore, G.~C.; El-Ghaoui, L. \enquote{Optimization Models}. Cambridge
  University Press, 2014, 1st \bbledn{}, 650p.
\bibAnnoteFile{calafiore:14:optimization}

\bibitem{davidov2006multiple}
Davidov, D.; Markovitch, S. \enquote{Multiple-goal heuristic search},
  \textit{Journal of Artificial Intelligence Research}, \bblvol{}~26, Aug~2006,
  \bblpp{} 417--451.
\bibAnnoteFile{davidov2006multiple}

\bibitem{dennett:1983}
Dennett, D. \enquote{Intentional systems in cognitive ethology: The
  "panglossian paradigm defended"}, \textit{Behavioral and Brain Sciences},
  \bblvol{}~6, Sep~1983, \bblpp{} 343--390.
\bibAnnoteFile{dennett:1983}

\bibitem{NASA_GoalRecognition_IJCAI2015}
E.{-}Mart{\'{\i}}n, Y.; R.{-}Moreno, M.~D.; Smith, D.~E. \enquote{{A Fast Goal
  Recognition Technique Based on Interaction Estimates}}.
  \capitalize{\bblin{}}: Proceedings of the International Joint Conference on
  Artificial Intelligence (IJCAI), 2015, \bblp{} 761–768.
\bibAnnoteFile{NASA_GoalRecognition_IJCAI2015}

\bibitem{FernandezGonzalez18_Scotty}
Fern{\'{a}}ndez{-}Gonz{\'{a}}lez, E.; Williams, B.~C.; Karpas, E.
  \enquote{Scottyactivity: Mixed discrete-continuous planning with convex
  optimization}, \textit{Journal of Artificial Intelligence Research},
  \bblvol{}~62, Jul~2018, \bblpp{} 579--664.
\bibAnnoteFile{FernandezGonzalez18_Scotty}

\bibitem{STRIPSFikes1971}
Fikes, R.~E.; Nilsson, N.~J. \enquote{{STRIPS}: A new approach to the
  application of theorem proving to problem solving}, \textit{Artificial
  Intelligence}, \bblvol{}~2--3, Sep~1971, \bblpp{} 189--208.
\bibAnnoteFile{STRIPSFikes1971}

\bibitem{freedman2018towards}
Freedman, R.~G.; Fung, Y.~R.; Ganchin, R.; Zilberstein, S. \enquote{Towards
  quicker probabilistic recognition with multiple goal heuristic search}.
  \capitalize{\bblin{}}: The Workshop on Plan, Activity, and Intent Recognition
  (PAIR) at the Conference of the Association for the Advancement of Artificial
  Intelligence (AAAI), 2018, \bblpp{} 1--8.
\bibAnnoteFile{freedman2018towards}

\bibitem{GarlandLesh_AAAI2002}
Garland, A.; Lesh, N. \enquote{Plan evaluation with incomplete action
  descriptions}. \capitalize{\bblin{}}: Proceedings of the Conference of the
  Association for the Advancement of Artificial Intelligence (AAAI), 2002,
  \bblpp{} 461--467.
\bibAnnoteFile{GarlandLesh_AAAI2002}

\bibitem{Bonet_Planning_13}
Geffner, H.; Bonet, B. \enquote{A Concise Introduction to Models and Methods
  for Automated Planning}. Morgan \& Claypool, 2013, 1st \bbledn{}, 141p.
\bibAnnoteFile{Bonet_Planning_13}

\bibitem{ProblemsWithElderCare_AAAI2002}
Geib, C.~W. \enquote{{Problems with Intent Recognition for Elder Care}}.
  \capitalize{\bblin{}}: Proceedings of the Conference of the Association for
  the Advancement of Artificial Intelligence (AAAI), 2002, \bblpp{} 13--17.
\bibAnnoteFile{ProblemsWithElderCare_AAAI2002}

\bibitem{GeibPlanRecognitionIntrusionDect_DARPA2001}
Geib, C.~W.; Goldman, R.~P. \enquote{{Plan Recognition in Intrusion Detection
  Systems}}. \capitalize{\bblin{}}: Proceedings of the DARPA Information
  Survivability Conference and Exposition (DISCEX), 2001, \bblpp{} 46--55.
\bibAnnoteFile{GeibPlanRecognitionIntrusionDect_DARPA2001}

\bibitem{Geib_PPR_AIJ2009}
Geib, C.~W.; Goldman, R.~P. \enquote{{A Probabilistic Plan Recognition
  Algorithm Based on Plan Tree Grammars}}, \textit{Artificial Intelligence},
  \bblvol{} 173--11, Jul~2009, \bblpp{} 1101--1132.
\bibAnnoteFile{Geib_PPR_AIJ2009}

\bibitem{AutomatedPlanning_Book2011}
Ghallab, M.; Nau, D.~S.; Traverso, P. \enquote{{Automated Planning - Theory and
  Practice.}} Elsevier, 2004, 1st \bbledn{}, 635p.
\bibAnnoteFile{AutomatedPlanning_Book2011}

\bibitem{AutomatedPlanning_Book2016}
Ghallab, M.; Nau, D.~S.; Traverso, P. \enquote{{Automated Planning and
  Acting}}. Elsevier, 2016, 1st \bbledn{}, 368p.
\bibAnnoteFile{AutomatedPlanning_Book2016}

\bibitem{goodfellow:16:dl}
Goodfellow, I.; Bengio, Y.; Courville, A. \enquote{Deep Learning}. MIT Press,
  2016, 1st \bbledn{}, 775p.
\bibAnnoteFile{goodfellow:16:dl}

\bibitem{Granada2017}
Granada, R.; Pereira, R.~F.; Monteiro, J.; Barros, R.; Ruiz, D.; Meneguzzi, F.
  \enquote{{Hybrid Activity and Plan Recognition for Video Streams}}.
  \capitalize{\bblin{}}: The Workshop on Plan, Activity, and Intent Recognition
  (PAIR) at the Conference of the Association for the Advancement of Artificial
  Intelligence (AAAI), 2017, \bblpp{} 1--8.
\bibAnnoteFile{Granada2017}

\bibitem{LQR_2013}
Hajiyev, C. \enquote{{LQR Controller with Kalman Estimator Applied to UAV
  Longitudinal Dynamics}}, \textit{Positioning}, \bblvol{}~04, Jan~2013,
  \bblpp{} 36--41.
\bibAnnoteFile{LQR_2013}

\bibitem{halpern:16:causality}
Halpern, J.~Y. \enquote{Actual Causality}. The MIT Press, 2016, 1st \bbledn{},
  229p.
\bibAnnoteFile{halpern:16:causality}

\bibitem{FFHoffmann_2001}
Hoffmann, J.; Nebel, B. \enquote{{The FF Planning System: Fast Plan Generation
  Through Heuristic Search}}, \textit{Journal of Artificial Intelligence
  Research}, \bblvol{}~14, May~2001, \bblpp{} 253--302.
\bibAnnoteFile{FFHoffmann_2001}

\bibitem{Hoffmann2004_OrderedLandmarks}
Hoffmann, J.; Porteous, J.; Sebastia, L. \enquote{{Ordered Landmarks in
  Planning}}, \textit{Journal of Artificial Intelligence Research},
  \bblvol{}~22--1, Nov~2004, \bblpp{} 215--278.
\bibAnnoteFile{Hoffmann2004_OrderedLandmarks}

\bibitem{HongGoalRecognition_2001}
Hong, J. \enquote{{Goal recognition through goal graph analysis}},
  \textit{Journal of Artificial Intelligence Research}, \bblvol{}~15, Jul~2001,
  \bblpp{} 1--30.
\bibAnnoteFile{HongGoalRecognition_2001}

\bibitem{LQR_1998}
{Jae Weon Choi}; {Young Bong See}; {Wan Suk Yoo}; {Man Hyung Lee}.
  \enquote{{LQR approach using Eigenstructure assignment with an active
  suspension control application}}. \capitalize{\bblin{}}: Proceedings of the
  IEEE International Conference on Control Applications, 1998, \bblpp{}
  1235--1239.
\bibAnnoteFile{LQR_1998}

\bibitem{Kambhampati_AAAI07}
Kambhampati, S. \enquote{Model-lite planning for the web age masses: The
  challenges of planning with incomplete and evolving domain models}.
  \capitalize{\bblin{}}: Proceedings of the Conference of the Association for
  the Advancement of Artificial Intelligence (AAAI), 2007, \bblp{} 1601–1604.
\bibAnnoteFile{Kambhampati_AAAI07}

\bibitem{Kaminka_18_AAAI}
Kaminka, G.~A.; Vered, M.; Agmon, N. \enquote{Plan recognition in continuous
  domains}. \capitalize{\bblin{}}: Proceedings of the Conference of the
  Association for the Advancement of Artificial Intelligence (AAAI), 2018,
  \bblpp{} 6202--6210.
\bibAnnoteFile{Kaminka_18_AAAI}

\bibitem{GoalRecognitionDesign_Keren2014}
Keren, S.; Gal, A.; Karpas, E. \enquote{{Goal Recognition Design}}.
  \capitalize{\bblin{}}: Proceedings of the International Conference on
  Automated Planning and Scheduling (ICAPS), 2014, \bblpp{} 1--8.
\bibAnnoteFile{GoalRecognitionDesign_Keren2014}

\bibitem{GoalRecognitionDesign_Keren2015}
Keren, S.; Gal, A.; Karpas, E. \enquote{{Goal Recognition Design for
  Non-Optimal Agents}}. \capitalize{\bblin{}}: Proceedings of the Conference of
  the Association for the Advancement of Artificial Intelligence (AAAI), 2015,
  \bblpp{} 3298--3304.
\bibAnnoteFile{GoalRecognitionDesign_Keren2015}

\bibitem{GoalRecognitionDesign_Keren2016}
Keren, S.; Gal, A.; Karpas, E. \enquote{{Goal Recognition Design with
  Non-Observable Actions}}. \capitalize{\bblin{}}: Proceedings of the
  Conference of the Association for the Advancement of Artificial Intelligence
  (AAAI), 2016, \bblpp{} 3152--3158.
\bibAnnoteFile{GoalRecognitionDesign_Keren2016}

\bibitem{kerkez2002case}
Kerkez, B.; Cox, M.~T. \enquote{{Case-Based Plan Recognition with Incomplete
  Plan Libraries}}. \capitalize{\bblin{}}: Proceedings of the Conference of the
  Association for the Advancement of Artificial Intelligence (AAAI) Fall
  Symposium on Intent Inference, 2002, \bblpp{} 52--54.
\bibAnnoteFile{kerkez2002case}

\bibitem{HM_Landmarks_2010}
Keyder, E.; Richter, S.; Helmert, M. \enquote{Sound and complete landmarks for
  and/or graphs}. \capitalize{\bblin{}}: Proceedings of the European Conference
  on Artificial Intelligence (ECAI), 2010, \bblp{} 335–340.
\bibAnnoteFile{HM_Landmarks_2010}

\bibitem{Kong_Tomi_NominalM_2013}
Kong, K.; Tomizuka, M. \enquote{Nominal model manipulation for enhancement of
  stability robustness for disturbance observer-based control systems},
  \textit{International Journal of Control, Automation and Systems},
  \bblvol{}~11, Jan~2013, \bblp{} 12–20.
\bibAnnoteFile{Kong_Tomi_NominalM_2013}

\bibitem{lee1998partial}
Lee, J.-J.; McCartney, R. \enquote{{Partial Plan Recognition with Incomplete
  Information}}. \capitalize{\bblin{}}: Proceedings of International Conference
  on Multi Agent Systems, 1998, \bblpp{} 445--446.
\bibAnnoteFile{lee1998partial}

\bibitem{ljung1998system}
Ljung, L. \enquote{System identification}. \capitalize{\capitalize{\bblin{}}:
  \textit{Signal Analysis and Prediction}}, Springer, 1998, \bblpp{} 163--173.
\bibAnnoteFile{ljung1998system}

\bibitem{Masters_IJCAI2017}
Masters, P.; Sardi{\~{n}}a, S. \enquote{{Cost-Based Goal Recognition for
  Path-Planning}}. \capitalize{\bblin{}}: Proceedings of the International
  Conference on Autonomous Agents and Multiagent Systems (AAMAS), 2017,
  \bblpp{} 750--758.
\bibAnnoteFile{Masters_IJCAI2017}

\bibitem{MastersS_JAIR_19}
Masters, P.; Sardi{\~{n}}a, S. \enquote{Cost-based goal recognition in
  navigational domains}, \textit{Journal of Artificial Intelligence Research},
  \bblvol{}~64, Feb~2019, \bblpp{} 197--242.
\bibAnnoteFile{MastersS_JAIR_19}

\bibitem{MastersS_AAMAS_19}
Masters, P.; Sardi{\~{n}}a, S. \enquote{Goal recognition for rational and
  irrational agents}. \capitalize{\bblin{}}: Proceedings of the 18th
  International Conference on Autonomous Agents and MultiAgent Systems (AAMAS),
  2019, \bblpp{} 440--448.
\bibAnnoteFile{MastersS_AAMAS_19}

\bibitem{PDDLMcdermott1998}
McDermott, D.; Ghallab, M.; Howe, A.; Knoblock, C.; Ram, A.; Veloso, M.; Weld,
  D.; Wilkins, D. \enquote{{PDDL} $-$ {The Planning Domain Definition
  Language}}. \capitalize{\bblin{}}: Proceedings of the International
  Conference on Artificial Intelligence Planning Systems (AIPS), 1998, \bblpp{}
  1--8.
\bibAnnoteFile{PDDLMcdermott1998}

\bibitem{PR_EXP_Mirsky2017}
Mirsky, R.; Gal, Y.~K.; Shieber, S.~M. \enquote{{CRADLE: An Online Plan
  Recognition Algorithm for Exploratory Domains}}, \textit{ACM Transactions on
  Intelligent Systems and Technology}, \bblvol{}~8--3, Apr~2017, \bblpp{}
  45:1--45:22.
\bibAnnoteFile{PR_EXP_Mirsky2017}

\bibitem{Mirsky_UISP17}
Mirsky, R.; Gal, Y.~K.; Tolpin, D. \enquote{Session analysis using plan
  recognition}. \capitalize{\bblin{}}: The Workshop on User Interfaces and
  Scheduling and Planning at the International Conference on Automated Planning
  and Scheduling (ICAPS), 2017, \bblpp{} 1--7.
\bibAnnoteFile{Mirsky_UISP17}

\bibitem{MIRSKY_2018_AIJ}
Mirsky, R.; Stern, R.; Gal, K.; Kalech, M. \enquote{Sequential plan
  recognition: An iterative approach to disambiguating between hypotheses},
  \textit{Artificial Intelligence}, \bblvol{} 260, Jul~2018, \bblpp{} 51--73.
\bibAnnoteFile{MIRSKY_2018_AIJ}

\bibitem{PR_Mirsky_2016}
Mirsky, R.; Stern, R.; Gal, Y.~K.; Kalech, M. \enquote{{Sequential Plan
  Recognition}}. \capitalize{\bblin{}}: Proceedings of the International Joint
  Conference on Artificial Intelligence (IJCAI), 2016, \bblpp{} 401--407.
\bibAnnoteFile{PR_Mirsky_2016}

\bibitem{mitrovic:10:adaptive}
Mitrovic, D.; Klanke, S.; Vijayakumar, S. \enquote{Adaptive optimal feedback
  control with learned internal dynamics models}, \textit{From Motor Learning
  to Interaction Learning in Robots}, \bblvol{} 264, Jan~2010, \bblpp{} 65--84.
\bibAnnoteFile{mitrovic:10:adaptive}

\bibitem{montufar:14:nips}
Montufar, G.~F.; Pascanu, R.; Cho, K.; Bengio, Y. \enquote{On the number of
  linear regions of deep neural networks}. \capitalize{\bblin{}}: Proceedings
  of the Annual Conference on Neural Information Processing Systems (NIPS),
  2014, \bblpp{} 1--9.
\bibAnnoteFile{montufar:14:nips}

\bibitem{Sirdi_2007_CSC}
M'Sirdi, N.; Rabhi, A.; Naamane, A. \enquote{{A Nominal Model for Vehicle
  Dynamics and Estimation of Input Forces and Tire Friction}}.
  \capitalize{\bblin{}}: International Conference on Control Systems and
  Computer Science (CSC), 2007, \bblpp{} 1--7.
\bibAnnoteFile{Sirdi_2007_CSC}

\bibitem{nair:10:icml}
Nair, V.; Hinton, G.~E. \enquote{Rectified linear units improve restricted
  boltzmann machines}. \capitalize{\bblin{}}: Proceedings of the International
  Conference on Machine Learning (ICML), 2010, \bblp{} 807–814.
\bibAnnoteFile{nair:10:icml}

\bibitem{Nguyen_AIJ_2017}
Nguyen, T.; Sreedharan, S.; Kambhampati, S. \enquote{{Robust Planning with
  Incomplete Domain Models}}, \textit{Artificial Intelligence}, \bblvol{} 245,
  Apr~2017, \bblpp{} 134 -- 161.
\bibAnnoteFile{Nguyen_AIJ_2017}

\bibitem{PlanningIncomplete_NguyenK_2014}
Nguyen, T.~A.; Kambhampati, S. \enquote{{A Heuristic Approach to Planning with
  Incomplete {STRIPS} Action Models}}. \capitalize{\bblin{}}: Proceedings of
  the International Conference on Automated Planning and Scheduling (ICAPS),
  2014, \bblp{} 190–198.
\bibAnnoteFile{PlanningIncomplete_NguyenK_2014}

\bibitem{DDNs_Functions_2000}
{Patra}, J.~C.; {Pal}, R.~N.; {Chatterji}, B.~N.; {Panda}, G.
  \enquote{Identification of nonlinear dynamic systems using functional link
  artificial neural networks}, \textit{IEEE Transactions on Systems, Man, and
  Cybernetics}, \bblvol{}~29--2, Apr~1999, \bblpp{} 254--262.
\bibAnnoteFile{DDNs_Functions_2000}

\bibitem{PattisonGoalRecognition_2010}
Pattison, D.; Long, D. \enquote{{Domain Independent Goal Recognition.}}
  \capitalize{\bblin{}}: Proceedings of the Starting AI Researcher Symposium
  (STAIRS), 2010, \bblpp{} 1--10.
\bibAnnoteFile{PattisonGoalRecognition_2010}

\bibitem{pearl:09:causality}
Pearl, J. \enquote{Causality: Models, Reasoning and Inference}. Cambridge
  University Press, 2009, 1st \bbledn{}, 464p.
\bibAnnoteFile{pearl:09:causality}

\bibitem{PereiraMeneguzzi_ECAI2016}
Pereira, R.~F.; Meneguzzi, F. \enquote{{Landmark-Based Plan Recognition}}.
  \capitalize{\bblin{}}: Proceedings of the European Conference on Artificial
  Intelligence (ECAI), 2016, \bblpp{} 1706--1707.
\bibAnnoteFile{PereiraMeneguzzi_ECAI2016}

\bibitem{Pereira_Meneguzzi_PRDatasets_2017}
Pereira, R.~F.; Meneguzzi, F. \enquote{{Goal and Plan Recognition Datasets
  using Classical Planning Domains}}. (Accessed July 2019), 2017.
\bibAnnoteFile{Pereira_Meneguzzi_PRDatasets_2017}

\bibitem{AAAI2018_PereiraMeneguzzi}
Pereira, R.~F.; Meneguzzi, F. \enquote{{Goal Recognition in Incomplete Domain
  Models}}. \capitalize{\bblin{}}: Proceedings of Association for the
  Advancement of Artificial Intelligence (AAAI), 2018, \bblpp{} 8127--8128.
\bibAnnoteFile{AAAI2018_PereiraMeneguzzi}

\bibitem{PAIR18_PereiraMeneguzzi}
Pereira, R.~F.; Meneguzzi, F. \enquote{{Goal Recognition in Incomplete STRIPS
  Domain Models}}. \capitalize{\bblin{}}: The Workshop on Plan, Activity, and
  Intent Recognition (PAIR) at the Conference of the Association for the
  Advancement of Artificial Intelligence (AAAI), 2018, \bblpp{} 1--8.
\bibAnnoteFile{PAIR18_PereiraMeneguzzi}

\bibitem{PereiraOrenMeneguzzi_AAMAS2017}
Pereira, R.~F.; Oren, N.; Meneguzzi, F. \enquote{{Detecting Commitment
  Abandonment by Monitoring Sub-Optimal Steps During Plan Execution}}.
  \capitalize{\bblin{}}: Proceedings of the Conference on Autonomous Agents and
  MultiAgent Systems (AAMAS), 2017, \bblpp{} 1685--1687.
\bibAnnoteFile{PereiraOrenMeneguzzi_AAMAS2017}

\bibitem{PereiraNirMeneguzzi_AAAI2017}
Pereira, R.~F.; Oren, N.; Meneguzzi, F. \enquote{{Landmark-Based Heuristics for
  Goal Recognition}}. \capitalize{\bblin{}}: Proceedings of the Conference of
  the Association for the Advancement of Artificial Intelligence (AAAI), 2017,
  \bblpp{} 3622--3628.
\bibAnnoteFile{PereiraNirMeneguzzi_AAAI2017}

\bibitem{PAIR17_PereiraOrenMeneguzzi}
Pereira, R.~F.; Oren, N.; Meneguzzi, F. \enquote{{Monitoring Plan Optimality
  using Landmarks and Domain-Independent Heuristics}}. \capitalize{\bblin{}}:
  The Workshop on Plan, Activity, and Intent Recognition (PAIR) at the
  Conference of the Association for the Advancement of Artificial Intelligence
  (AAAI), 2017, \bblpp{} 1--8.
\bibAnnoteFile{PAIR17_PereiraOrenMeneguzzi}

\bibitem{PereiraOM_AIJ_2020}
Pereira, R.~F.; Oren, N.; Meneguzzi, F. \enquote{Landmark-based approaches for
  goal recognition as planning}, \textit{Artificial Intelligence}, \bblvol{}
  279, Feb~2020, \bblpp{} 1--32.
\bibAnnoteFile{PereiraOM_AIJ_2020}

\bibitem{PereiraOM_TIST_2020}
Pereira, R.~F.; Oren, N.; Meneguzzi, F. \enquote{Using sub-optimal plan
  detection to identify commitment abandonment in discrete environments},
  \textit{ACM Transactions on Intelligent Systems and Technology},
  \bblvol{}~11, Feb~2020, \bblpp{} 1--26.
\bibAnnoteFile{PereiraOM_TIST_2020}

\bibitem{PereiraPM_ICAPS_19}
Pereira, R.~F.; Pereira, A.~G.; Meneguzzi, F. \enquote{Landmark-enhanced
  heuristics for goal recognition in incomplete domain models}.
  \capitalize{\bblin{}}: Proceedings of the International Conference on
  Automated Planning and Scheduling (ICAPS), 2019, \bblpp{} 329--337.
\bibAnnoteFile{PereiraPM_ICAPS_19}

\bibitem{PereiraVMR_IJCAI19}
Pereira, R.~F.; Vered, M.; Meneguzzi, F.; Ram{\'{\i}}rez, M. \enquote{Online
  probabilistic goal recognition over nominal models}. \capitalize{\bblin{}}:
  Proceedings of the International Joint Conference on Artificial Intelligence
  (IJCAI), 2019, \bblpp{} 5547--5553.
\bibAnnoteFile{PereiraVMR_IJCAI19}

\bibitem{WellmanTraffic_2013}
Pynadath, D.~V.; Wellman, M.~P. \enquote{{Accounting for Context in Plan
  Recognition, with Application to Traffic Monitoring}}, \textit{Computing
  Research Repository (CoRR)}, \bblvol{} abs/1302.4980, Aug~2013, \bblpp{}
  472--481.
\bibAnnoteFile{WellmanTraffic_2013}

\bibitem{RamirezG_IJCAI2009}
Ram{\'{\i}}rez, M.; Geffner, H. \enquote{{Plan Recognition as Planning}}.
  \capitalize{\bblin{}}: Proceedings of the International Joint Conference on
  Artificial Intelligence (IJCAI), 2009, \bblpp{} 1778--1783.
\bibAnnoteFile{RamirezG_IJCAI2009}

\bibitem{RamirezG_AAAI2010}
Ram{\'{\i}}rez, M.; Geffner, H. \enquote{{Probabilistic Plan Recognition Using
  Off-the-Shelf Classical Planners}}. \capitalize{\bblin{}}: Proceedings of the
  Conference of the Association for the Advancement of Artificial Intelligence
  (AAAI), 2010, \bblp{} 1121–1126.
\bibAnnoteFile{RamirezG_AAAI2010}

\bibitem{rasmussen:02:gaussian}
Rasmussen, C.~E.; Williams, C. K.~I. \enquote{Gaussian Processes for Machine
  Learning}. MIT Press, 2006, 1st \bbledn{}, 245p.
\bibAnnoteFile{rasmussen:02:gaussian}

\bibitem{LandmarksRichter_2008}
Richter, S.; Helmert, M.; Westphal, M. \enquote{{Landmarks Revisited}}.
  \capitalize{\bblin{}}: Proceedings of the Conference of the Association for
  the Advancement of Artificial Intelligence (AAAI), 2008, \bblpp{} 975--982.
\bibAnnoteFile{LandmarksRichter_2008}

\bibitem{RichterLPG_2010}
Richter, S.; Westphal, M. \enquote{{The LAMA Planner: Guiding Cost-based
  Anytime Planning with Landmarks}}, \textit{Journal of Artificial Intelligence
  Research}, \bblvol{}~39--1, Jan~2010, \bblpp{} 127--177.
\bibAnnoteFile{RichterLPG_2010}

\bibitem{AIModernApproachRussell_2009}
Russell, S.; Norvig, P. \enquote{{Artificial intelligence: A Modern Approach}}.
  Prentice Hall, 2010, 3 \bbledn{}, 1132p.
\bibAnnoteFile{AIModernApproachRussell_2009}

\bibitem{sanner:11:rddl}
Sanner, S. \enquote{{Relational Dynamic Influence Diagram Language (RDDL):
  Language Description}}, \bbltechrep{}, Australian National University, 2011,
  25p.
\bibAnnoteFile{sanner:11:rddl}

\bibitem{SayS:ijcai18}
Say, B.; Sanner, S. \enquote{Planning in factored state and action spaces with
  learned binarized neural network transition models}. \capitalize{\bblin{}}:
  Proceedings of the International Joint Conference on Artificial Intelligence
  (IJCAI), 2018, \bblpp{} 4815--4821.
\bibAnnoteFile{SayS:ijcai18}

\bibitem{SayWZS:ijcai17}
Say, B.; Wu, G.; Zhou, Y.~Q.; Sanner, S. \enquote{Nonlinear hybrid planning
  with deep net learned transition models and mixed-integer linear
  programming}. \capitalize{\bblin{}}: Proceedings of the International Joint
  Conference on Artificial Intelligence (IJCAI), 2017, \bblpp{} 750--756.
\bibAnnoteFile{SayWZS:ijcai17}

\bibitem{SchmidtSG_78}
Schmidt, C.~F.; Sridharan, N.~S.; Goodson, J.~L. \enquote{{The Plan Recognition
  Problem: An Intersection of Psychology and Artificial Intelligence}},
  \textit{Journal of Artificial Intelligence Research}, \bblvol{}~11--1-2,
  May~1978, \bblpp{} 45--83.
\bibAnnoteFile{SchmidtSG_78}

\bibitem{Sohrabi_IJCAI2016}
Sohrabi, S.; Riabov, A.~V.; Udrea, O. \enquote{{Plan Recognition as Planning
  Revisited}}. \capitalize{\bblin{}}: Proceedings of the International Joint
  Conference on Artificial Intelligence (IJCAI), 2016, \bblpp{} 3258--3264.
\bibAnnoteFile{Sohrabi_IJCAI2016}

\bibitem{ActivityIntentPlanRecogition_Book2014}
Sukthankar, G.; Goldman, R.~P.; Geib, C.; Pynadath, D.~V.; Bui, H.~H.
  \enquote{{Plan, Activity, and Intent Recognition: Theory and Practice}}.
  Elsevier, 2014, 1st \bbledn{}, 424p.
\bibAnnoteFile{ActivityIntentPlanRecogition_Book2014}

\bibitem{Sutton_2018_RLI}
Sutton, R.~S.; Barto, A.~G. \enquote{Reinforcement Learning: An Introduction}.
  USA: A Bradford Book, 2018, 1st \bbledn{}, 552p.
\bibAnnoteFile{Sutton_2018_RLI}

\bibitem{tassa:12:iros}
Tassa, Y.; Erez, T.; Todorov, E. \enquote{Synthesis and stabilization of
  complex behaviours through online trajectory optimization}.
  \capitalize{\bblin{}}: Proceedings of the International Conference on
  Intelligent Robots and Systems (IROS), 2012, \bblpp{} 4906--4913.
\bibAnnoteFile{tassa:12:iros}

\bibitem{UzanDSG_PR_2015}
Uzan, O.; Dekel, R.; Seri, O.; Gal, Y.~K. \enquote{{Plan Recognition for
  Exploratory Learning Environments Using Interleaved Temporal Search}},
  \textit{AI Magazine}, \bblvol{}~36--2, Aug~2015, \bblpp{} 10--21.
\bibAnnoteFile{UzanDSG_PR_2015}

\bibitem{Mor_ACS_16}
Vered, M.; Kaminka, G.~A.; Biham, S. \enquote{Online goal recognition through
  mirroring: Humans and agents}. \capitalize{\bblin{}}: Proceedings of the
  Annual Conference on Advances in Cognitive Systems (ACS), 2016, \bblpp{}
  1--12.
\bibAnnoteFile{Mor_ACS_16}

\bibitem{MorEtAl_AAMAS18}
Vered, M.; Pereira, R.~F.; Magnaguagno, M.; Meneguzzi, F.; Kaminka, G.~A.
  \enquote{{Towards Online Goal Recognition Combining Goal Mirroring and
  Landmarks}}. \capitalize{\bblin{}}: Proceedings of the International
  Conference on Autonomous Agents and Multiagent Systems (AAMAS), 2018,
  \bblpp{} 2112--2114.
\bibAnnoteFile{MorEtAl_AAMAS18}

\bibitem{Vicent_ActionLandmarks_2005}
Vidal, V.; Geffner, H. \enquote{{Solving Simple Planning Problems with More
  Inference and No Search}}. \capitalize{\bblin{}}: Proceedings of the
  Conference on Principles and Practice of Constraint Programming (CP), 2005,
  \bblpp{} 682--696.
\bibAnnoteFile{Vicent_ActionLandmarks_2005}

\bibitem{warren:89:icra}
Warren, C.~W. \enquote{Global path planning using artificial potential fields}.
  \capitalize{\bblin{}}: Proceedings of the International Conference on
  Robotics and Automation (ICRA), 1989, \bblpp{} 316--321.
\bibAnnoteFile{warren:89:icra}

\bibitem{WeberBryce_ICAPS_2011}
Weber, C.; Bryce, D. \enquote{{Planning and Acting in Incomplete Domains}}.
  \capitalize{\bblin{}}: Proceedings of the International Conference on
  Automated Planning and Scheduling (ICAPS), 2011, \bblpp{} 1--8.
\bibAnnoteFile{WeberBryce_ICAPS_2011}

\bibitem{WuSS:nips17}
Wu, G.; Say, B.; Sanner, S. \enquote{Scalable planning with tensorflow for
  hybrid nonlinear domains}. \capitalize{\bblin{}}: Proceedings of the Annual
  Conference on Neural Information Processing Systems (NIPS), 2017, \bblpp{}
  6273--6283.
\bibAnnoteFile{WuSS:nips17}

\bibitem{yamaguchi:16:icra}
Yamaguchi, A.; Atkeson, C.~G. \enquote{Neural networks and differential dynamic
  programming for reinforcement learning problems}. \capitalize{\bblin{}}:
  Proceedings of the International Conference on Robotics and Automation
  (ICRA), 2016, \bblpp{} 5434--5441.
\bibAnnoteFile{yamaguchi:16:icra}

\bibitem{Zhang04solvinglarge}
Zhang, T. \enquote{Solving large scale linear prediction problems using
  stochastic gradient descent algorithms}. \capitalize{\bblin{}}: Proceedings
  of the International Conference on Machine Learning (ICML), 2004, \bblpp{}
  919--926.
\bibAnnoteFile{Zhang04solvinglarge}

\bibitem{Landmarks_Zhugivan_2003}
Zhu, L.; Givan, R. \enquote{Landmark extraction via planning graph
  propagation}. \capitalize{\bblin{}}: Proceedings of the Doctoral Consortium
  at the International Conference on Automated Planning and Scheduling (ICAPS),
  2003, \bblpp{} 1--7.
\bibAnnoteFile{Landmarks_Zhugivan_2003}

\bibitem{Zhuo_EtAl_TIST_2019}
Zhuo, H.~H. \enquote{{Recognizing Multi-Agent Plans When Action Models and Team
  Plans Are Both Incomplete}}, \textit{ACM Transactions on Intelligent Systems
  and Technology}, \bblvol{}~10--3, May~2019, \bblpp{} 1--24.
\bibAnnoteFile{Zhuo_EtAl_TIST_2019}

\bibitem{ZhuoL_MAPR_11}
Zhuo, H.~H.; Li, L. \enquote{Multi-agent plan recognition with partial team
  traces and plan libraries}. \capitalize{\bblin{}}: Proceedings of the
  International Joint Conference on Artificial Intelligence (IJCAI), 2011,
  \bblpp{} 484--489.
\bibAnnoteFile{ZhuoL_MAPR_11}

\bibitem{RefiningSTRIPSIncomplete_Zhuo_2013}
Zhuo, H.~H.; Nguyen, T.~A.; Kambhampati, S. \enquote{Refining incomplete
  planning domain models through plan traces}. \capitalize{\bblin{}}:
  Proceedings of the International Joint Conference on Artificial Intelligence
  (IJCAI), 2013, \bblpp{} 2451--2458.
\bibAnnoteFile{RefiningSTRIPSIncomplete_Zhuo_2013}

\bibitem{MAPR_STRIPS_Zhuo_2012}
Zhuo, H.~H.; Yang, Q.; Kambhampati, S. \enquote{{Action-Model Based Multi-agent
  Plan Recognition}}. \capitalize{\bblin{}}: Proceedings of the Annual
  Conference on Neural Information Processing Systems (NIPS), \bblpp{}
  377--385.
\bibAnnoteFile{MAPR_STRIPS_Zhuo_2012}

\end{thebibliography}
